\documentclass[12pt]{report}	
\usepackage{utdiss2}  		

\usepackage[letterpaper, margin=1in]{geometry}
\usepackage[numbers, sort&compress]{natbib}
\usepackage[table]{xcolor}
\usepackage{hyperref}
\usepackage{url}
\usepackage{booktabs}
\usepackage{amsmath}
\usepackage{amsthm}
\usepackage{amsfonts}
\usepackage{bm}
\usepackage{nicefrac}
\usepackage{microtype}
\usepackage{multicol}
\usepackage{multirow}
\usepackage{adjustbox}
\usepackage{graphicx}
\usepackage{tikz}
\usepackage{makecell}
\usepackage{wrapfig}
\usepackage{tabularx}
\usepackage{mathrsfs}
\usepackage{mathtools}
\usepackage{stmaryrd}
\usepackage{xspace}
\usepackage[ruled,vlined]{algorithm2e}
\usepackage{epigraph}

\usepackage{footmisc}
\newcommand{\codeurl}{\url{https://github.com/cognizant-ai-labs/aquasurf}}
\newcommand{\dataurl}{\url{https://github.com/cognizant-ai-labs/act-bench}}

\newcommand{\technique}{AQuaSurF\xspace}
\newcommand{\techniqueexpanded}{Activation Quality with a Surrogate Function\xspace}

\author{Garrett Joseph Bingham}

\address{binghamgarrett@gmail.com}

\title{Optimizing Neural Networks through Activation Function Discovery and Automatic Weight Initialization}

\supervisor
	{Risto Miikkulainen}

\committeemembers
	[Quoc Le]
	[Qiang Liu]
	{Yuke Zhu}

\previousdegrees{B.S.}
\graduationmonth{May}      
\graduationyear{2023}   
\oneandonehalfspacequote

\longtocentry

\theoremstyle{plain} 

\newtheorem{theorem}{Theorem}
\newtheorem{lemma}{Lemma}

\newcommand{\inn}{\mathrm{in}}
\newcommand{\out}{\mathrm{out}}
\newcommand{\E}{\mathrm{E}}
\newcommand{\Var}{\mathrm{Var}}

\newcommand*\diff{\mathop{}\!\mathrm{d}}

\newcommand{\vect}{\textrm{vec}}

\newcommand{\latexe}{{\LaTeX\kern.125em2%
                      \lower.5ex\hbox{$\varepsilon$}}}

\chardef\bslash=`\\	

\makeatletter		
\def\square{\RIfM@\bgroup\else$\bgroup\aftergroup$\fi
  \vcenter{\hrule\hbox{\vrule\@height.6em\kern.6em\vrule}%
                                              \hrule}\egroup}
\makeatother		

\makeindex    

\begin{document}

\copyrightpage          

%
%
%
\commcertpage           

\titlepage              

%
\begin{dedication}
\index{Dedication@\emph{Dedication}}%
Dedicated to anyone on a mental health journey;
\end{dedication}

\begin{acknowledgments}		
\index{Acknowledgments@\emph{Acknowledgments}}%

This dissertation wouldn't have been possible without incredible support from my family, friends, mentors, and coworkers.

To Mom, for understanding when no one else does.
To Dad, for time spent in the mountains.
To Hunter, for listening.
To Kate, for being my friend.
To Cole, for making me laugh.
To Luke, for being yourself.
To Blake, for reminding me how fun it is to explore.
To Drew, for giving the best hugs.

To all my wonderful friends from Ogden, Cambridge, New Haven, Bilbao, Wilmington, Budapest, New York, Seattle, Austin, and Oakland: thank you for making me smile.

Thank you to all of my professors, internship advisors, and especially the teachers from my childhood.  I love learning because of you.

The LEAF team at Cognizant directly supported my research over the last few years.  I am especially grateful to Elliot and Xin for our brainstorming sessions, to Mohak, Sid, and Dan for their help implementing my ideas, and to Olivier and Talia for making me feel at home away from home.

Finally, to Risto: thank you for giving me all the space to explore and all the support to succeed.

\end{acknowledgments}

%
\utabstract
\index{Abstract}%
\indent
Automated machine learning (AutoML) methods improve upon existing models by optimizing various aspects of their design.  While present methods focus on hyperparameters and neural network topologies, other aspects of neural network design can be optimized as well.  To further the state of the art in AutoML, this dissertation introduces techniques for discovering more powerful activation functions and establishing more robust weight initialization for neural networks.  These contributions improve performance, but also provide new perspectives on neural network optimization.  First, the dissertation demonstrates that discovering solutions specialized to specific architectures and tasks gives better performance than reusing general approaches.  Second, it shows that jointly optimizing different components of neural networks is synergistic, and results in better performance than optimizing individual components alone.  Third, it demonstrates that learned representations are easier to optimize than hard-coded ones, creating further opportunities for AutoML. The dissertation thus makes concrete progress towards fully automatic machine learning in the future.

\tableofcontents   

\listoftables      
\listoffigures     

%
%

\chapter{Introduction}

\epigraph{We want AI agents that can discover like we can, not which contain what we have discovered.}{Rich Sutton, \textit{The Bitter Lesson}}

Recursive self-improvement is hypothesized to be one means to artificial general intelligence (AGI) \cite{bostrom2014superintelligence, yudkowsky2007levels}.  The idea is straightforward: if an AI is sophisticated enough, it could design an improved version of itself.  The improved AI, which is smarter than the original, could then design an even more capable version of itself.  This process could continue, potentially indefinitely, resulting in an arbitrarily capable agent.  

Currently, AI systems do not exhibit this kind of recursive self-improvement.  In the field of automated machine learning (AutoML), however, a single step of self-improvement is possible.  Research in AutoML could therefore conceivably lead to recursive self-improvement.  This dissertation improves upon the state of the art in AutoML, and thus provides a stepping stone towards eventual AGI.

\section{Motivation}

Early machine learning approaches relied on human-engineered features in order to learn representations of data.  As computing power increased, these approaches gave way to more general methods that automatically extract relevant features from the data \cite{sutton2019bitter}.  This trend has persisted across subfields of machine learning.

For example, computer vision algorithms initially relied on detecting edges, corners, and other human-inspired image features \cite{lowe1999object}.  By leveraging additional compute and large amounts of data, convolutional neural networks learn such features automatically, and result in networks that surpass humans in image classification \cite{krizhevsky2017imagenet, he2015delving, deng2009imagenet}.

Similarly, in the field of natural language processing, manually crafted features like bag-of-words or TF-IDF (term frequency-inverse document frequency) were initially useful for rudimentary text understanding \cite{salton1975vector}. Later, techniques like word2vec captured more nuanced word semantics by learning representations automatically from a billion word corpus \cite{mikolov2013efficient}.  Today, large language models have pushed this trend even further, using enormous amounts of compute and data in order to model long-range dependencies in text and handle complex tasks such as question answering and machine translation \cite{vaswani2017attention, devlin2018bert, brown2020language}.

To date, methods leveraging large amounts of compute for general feature learning have been more successful than specialized approaches relying on human knowledge across a broad range of tasks.  Interestingly, this pattern extends beyond feature learning to algorithm design itself: AI can design better AI than humans can \cite{stanley2002evolving, real2019regularized, zoph2016neural}.

However, even state-of-the-art AutoML algorithms still make use of human-inspired designs in some areas \cite{elsken2019neural, wistuba2019survey}.  For example, many of these approaches will automate the design of a neural network topology, but will reuse human designed activation functions or weight initialization strategies.  This situation is suboptimal, but also points to an opportunity for improvement.  Indeed, suboptimal human designs could function as a bottleneck that prevents single step self-improvement from progressing to recursive self-improvement.  This bottleneck can be avoided by automating the entire machine learning pipeline.  

With this motivation, this dissertation introduces automated approaches to activation function discovery and weight initialization, two areas where suboptimal human designs are frequently used.  Experiments show the contributions already improve upon the state of the art.  As more compute becomes available, the approaches can be scaled up, providing even better solutions for machine learning problems in the future.  This dissertation thus provides concrete steps towards fully automatic machine learning.

\section{Challenges}

Automating activation function discovery and weight initialization introduces a number of challenges, discussed below.

Evaluating new activation functions is computationally expensive.  In order to do it, a neural network must be trained from scratch.  The training process is already costly, and repeating this process for many different activation functions can become prohibitively so.  In order to make this research possible, Chapters \ref{chap:gecco} and \ref{chap:pangaea} utilize distributed high performance computing.  The evaluation of different activation functions is parallelized across multiple machines, and the results are later aggregated.  However, Chapter \ref{chap:aquasurf} makes significant contribution to this area: It introduces a surrogate that makes the search for new activation functions orders of magnitude more efficient. The surrogate makes it possible to search for new activation functions with more standard hardware, or alternatively search in larger spaces with parallel hardware.

In order for a weight initialization system to be general, it must be able to stabilize the signal propagation for an arbitrary neural network.  This is difficult, due to the wide variety of neural network designs.  In Chapter \ref{chap:autoinit}, a method is developed to address this obstacle by analyzing neural network signal propagation at the level of individual layers.  Although many types of neural networks exist, modern architectures typically use similar kinds of layers.  The proposed method models signal propagation based on these layer types, and initializes the weights accordingly.  The method also provides fallback mechanisms in case unknown layer types are encountered: The signal propagation can be derived manually, or using Monte Carlo sampling.

Perhaps the biggest challenge of AutoML is overcoming previous human design biases, especially when it comes to interactions between different neural network components.  For example, ReLU is an extremely popular activation function, and dropout was designed with ReLU in mind \cite{nair2010rectified, srivastava2014dropout}.  However, when using the SELU activation function, dropout has to be modified because this activation function does not saturate to zero \cite{selu}.  Indeed, neural networks are famously brittle systems, and changing one component often ruins performance if the changes are not made carefully.  This dissertation addresses this challenge in two ways.  First, for activation functions, the dissertation discovers functions that are specialized to individual tasks.  Thus, even if a model has a set of hyperparameters that were tuned with a different activation function in mind, a better function can still be discovered by treating those hyperparameters as fixed and optimizing the activation function against them.  Second, the automated weight initialization system is designed to be as general as possible.  Thus, even if a neural network has suboptimal manual designs, the initialization algorithm still takes them into account and improves performance.  Indeed, experiments in Chapter \ref{chap:autoinit} show that the initialization approach is more robust to suboptimal hyperparameters than previous methods are.  Thus, although the eventual goal of AutoML is to fully automate the entire machine learning pipeline, the contributions in this dissertation are able to progress towards this goal while still being compatible with suboptimal human designs that are often used in practice.

\section{Approach}

This dissertation introduces automated methods for activation function discovery and weight initialization.  Four separate systems were created in order to thoroughly understand these areas.  Each system provides unique insights; the systems and discoveries are briefly summarized next.

In the first system, called CAFE, novel activation functions are discovered with a variety of approaches, namely exhaustive search in a smaller search space and random search and evolution in a larger space.  Evolution discovers activation functions that achieve high accuracy and outperform baseline functions like ReLU and Swish, demonstrating its creativity and ability to efficiently explore large spaces.  CAFE reveals two key findings.  The first is that optimizing the design of the activation function is by itself a means to improving the performance of neural networks.  This discovery therefore elevates the importance of activation functions in neural network design by showing that their design must be considered in order to maximize performance.  The second conclusion is that evolutionary search is creative, discovering designs unlikely to be created by humans.  It therefore provides a promising search mechanism for activation functions and other aspects of neural network design in the future.

The second system, PANGAEA, combines evolution and gradient descent into one optimization process.  PANGAEA extends CAFE by making it more flexible in several ways.  First, instead of using fixed tree structures, activation functions are represented as arbitrary computation graphs, and more powerful mutation operators are introduced to explore the larger search space efficiently.  Second, while the first system used fixed activation functions, PANGAEA utilizes parametric functions.  This construction allows the activation functions to change shape during the different stages of training and at different locations within a network.  Third, PANGAEA discovers specialized activation functions that are customized to specific tasks, leading to even better performance.  In conclusion, PANGAEA shows that better performance can be found by removing human design biases and by giving the optimization process more flexibility and freedom to be creative.  

In order to automate weight initialization for different neural networks, a third system, AutoInit, was developed.  While many initialization strategies have been proposed in the past, they usually apply only to neural networks with specific activation functions, topologies, or layer types.  This situation makes it difficult to evaluate new architectures or activation functions, because it is hard to initialize them properly with existing techniques. AutoInit addresses this issue by calculating analytic mean- and variance-preserving weight initialization for neural networks automatically.  It provides an appropriate default initialization automatically, resulting in better and more reliable performance.  Thus, in addition to better performance from optimizing hyperparameters, architectures, and activation functions, AutoInit shows that the weight initialization can also be optimized.  It therefore provides further evidence of the power of fully automatic machine learning.  In addition to improving performance on individual tasks, AutoInit accelerates neural architecture search and activation function discovery.  This result is particularly illuminating: It shows that there is a synergistic effect when multiple aspects of neural network design are optimized in tandem.  This result should further motivate progress in AutoML in the future.

The fourth system, AQuaSurF, makes activation function AutoML practical and scalable by introducing a surrogate approach. It also demonstrates that a surrogate can be learned from the data, leading to fundamental insights into what activation functions are made of. Note that early computer vision systems looked for human-designed features like edges and corners, but eventually the features were automatically learned with convolutional neural networks.  Similarly, researchers often design new activation functions based on intuitive characteristics like smoothness, groundedness, monotonicity, and limit behavior \cite{apicella2021survey, nwankpa2018activation}, but it should also be possible to learn automatically what features an activation function must possess in order to be successful.  To this end, AQuaSurF learns better representations of activation functions in a data-driven way.  Convolutional, residual, and vision transformer based architectures are trained from scratch with 2{,}913 different activation functions, resulting in three activation function benchmark datasets: \texttt{Act-Bench-CNN}, \texttt{Act-Bench-ResNet}, and \texttt{Act-Bench-ViT}.  Exploratory data analysis with these benchmark datasets reveals two activation function properties that are highly indicative of performance: (1) the spectrum of the Fisher information matrix associated with the model’s predictive distribution at initialization, and (2) the activation function’s output distribution.  Based on these features, a metric space is created where a low-dimensional representation of the activation functions can be learned. This space is then used as a surrogate in the search for good activation functions.  It turns out that the space is so powerful that out-of-the-box regression algorithms discover good activation functions in only tens of evaluations, improving performance on datasets as large as ImageNet.  Thus, AQuaSurF provides a unique perspective by showing that the underlying representation of the activation function is important.  By moving away from human-inspired encodings to automatically learned ones, AI is better able to improve itself.  Learning new representations for other aspects of neural network design may similarly be an important step towards fully automated machine learning in the future.

Thus, the four systems automate the design of activation functions and weight initialization.  Each system provides unique insights that will be useful in extending this work towards fully automated machine learning.  The first system, CAFE, provides a new perspective on neural network optimization, showing that performance can be improved by optimizing the design of the activation function.  The second system, PANGAEA, shows that introducing flexibility across multiple dimensions in the search process is both possible and beneficial, resulting in more creative and powerful solutions.  The third system, AutoInit, shows that proper weight initialization helps to fairly evaluate novel architectures and activation functions, thus accelerating research in these areas.  AutoInit also shows that optimizing multiple aspects of neural network design in tandem produces better results than just focusing on one aspect alone.  The fourth system, AQuaSurF, learns new representations for activation functions in a data-driven way; the resulting surrogate approach is orders of magnitude more efficient than previous work and may serve as a foundation for practical optimization of other aspects of neural network design in the future.

\section{Guide to the Reader}

The remainder of this dissertation is organized as follows:

Chapter \ref{chap:background} details previous research that inspired the contributions in this dissertation.  Chapter \ref{chap:gecco} introduces CAFE, demonstrating that designing better activation functions is a new way to optimize neural networks.  Chapter \ref{chap:pangaea} presents PANGAEA, which combines evolutionary search and gradient descent into one optimization process for optimizing parametric activation functions.  Chapter \ref{chap:autoinit} automates weight initialization with AutoInit, and shows that automating the design of both activation functions and weight initialization simultaneously leads to better results.  Chapter \ref{chap:aquasurf} introduces the activation function benchmark datasets, and learns a new representation for activation functions.  The new representation leads to the AQuaSurF surrogate-based method for activation function optimization that is orders of magnitude more efficient than existing work.  Chapter \ref{chap:discussion} discusses the contributions of this dissertation, and includes ideas for possible future research.  The main conclusions are reviewed in Chapter \ref{chap:conclusion}.

\chapter{Background}
\label{chap:background}

This chapter reviews related research that motivates and provides a foundation for the work in this dissertation.  First, automated machine learning is discussed, since it provides the main inspiration for this work.  Next, research on activation functions is discussed, with connections to Chapters \ref{chap:gecco} and \ref{chap:pangaea}.  Then, weight initialization research is reviewed as inspiration to the contributions in Chapter \ref{chap:autoinit}.  Finally, the Fisher information matrix is introduced, which informs the surrogate approach in Chapter \ref{chap:aquasurf}. 

\section{Automated Machine Learning}

Building a machine learning system requires making many design decisions: hyperparameters, neural architectures, data augmentation, and other components need to be configured \cite{hutter2019automated}.  Instead of ad hoc decisions by human researchers, automated machine learning (AutoML) serves to make one or more of these decisions in an automated, principled manner.  AutoML simultaneously makes machine learning accessible, since human expertise is not required in every scenario.

AutoML is a broad research area.  This section reviews the components of AutoML that are most relevant to this dissertation.

\subsection{Hyperparameter Optimization}

Neural networks have many hyperparameters that need to be chosen before learning can proceed, and choosing suitable values can play a large role in how successful the learning process is.  There exist many approaches to hyperparameter optimization, including Bayesian optimization, bilevel programming, evolution strategies, random search, and others \cite{klein2017fast, franceschi2018bilevel, maclaurin2015gradient, loshchilov2016cma, feurer2015initializing, yang2020hyperparameter, bergstra2011algorithms, bergstra2012random, feurer2019hyperparameter, jaderberg2017population}.

Hyperparameters can often be represented as a low-dimensional, real-valued vector, making it straightforward to experiment with existing optimization methods.  In contrast, other aspects of neural network design, such as their topology, activation function, or weight initialization strategy, are more complicated objects that require a more creative optimization approach.

\subsection{Neural Architecture Search}

In neural architecture search \citep[NAS;][]{wistuba2019survey, elsken2019neural, zoph2016neural, wang2019scalable, cai2017efficient, cai2018path, chen2018searching, chen2019progressive, gong2019autogan, zhong2018practical, zoph2018learning, liu2017hierarchical, luo2018neural}, the goal is to design a neural network architecture automatically.  NAS approaches typically focus on optimizing the type and location of the layers and the connections between them.  A popular NAS approach is neuroevolution, where neural networks are optimized with evolutionary algorithms \cite{so2019evolved, suganuma2017genetic, wistuba2018deep, suganuma2018exploiting, stanley2009hypercube, risi2010evolving, martinez2020lights, real2017large, stanley2002evolving, xie2017genetic, real2019regularized, yao1997new, angeline1994evolutionary, gomez2008accelerated}.  Reinforcement learning \cite{baker2016designing, zoph2016neural, gao2019graphnas}, Monte Carlo tree search \cite{negrinho2017deeparchitect}, gradient descent \cite{liu2018darts}, and random search \cite{li2020random} are also used.  NAS approaches often use standard choices for other components like the activation function, weight initialization, loss function, and so on.  These components have received less attention in AutoML, but can similarly be optimized, as demonstrated in this dissertation.

\subsection{Other Aspects of Neural Network Design}

Just like the topology of a neural network can be optimized, so too can other aspects of neural network design.  For instance, Gonzalez and Miikkulainen \cite{gonzalez2020improved, gonzalez2020evolving} used a genetic algorithm \cite{koza1992genetic, whitley1994genetic} to construct novel loss functions, and then optimized the coefficients of the loss functions with a covariance-matrix adaptation evolutionary strategy.  They discovered a loss function that results in faster training and higher accuracy compared to the standard cross-entropy loss.  \citet{liu2020evolving} evolved normalization-activation layers.  They searched for a computation graph that replaced both batch normalization and the activation function in neural networks.  The design of components like the learning rate schedule \cite{carvalho2020autolr, defazio2023learning}, data augmentation strategy \cite{cubuk2020randaugment, lim2019fast, cubuk2018autoaugment}, optimization algorithm \cite{chen2023symbolic, alber2018backprop, bello2017neural, cui2018evolutionary}, and other objects can similarly be automated \cite{houthooft2018evolved, real2020automl}.

\subsection{Zero-Cost Proxies}
A common drawback to AutoML is its computational cost.  Candidate designs must be evaluated in order to understand their performance, and this evaluation is often expensive.  In order to partially alleviate this issue, recent work has developed proxy measures that approximate the final performance of neural networks \cite{white2021powerful, shen2021proxybo, mellor2020neural}.  Instead of training the network, these proxies use only cheap surrogate calculations.  These proxies inspired the surrogate approach in Chapter \ref{chap:aquasurf}, where the efficiency of activation function search was dramatically improved.

\subsection{Meta-Learning}

The field of meta-learning, also called learning to learn, is another related subfield of AutoML \cite{finn2017meta, finn2017model, hospedales2020meta, antoniou2018train, denevi2018learning, hochreiter2001learning, li2017meta, mishra2018simple}.  Rather than optimize the design of the model for a single task, meta-learning approaches typically incorporate knowledge and experience from multiple tasks in order to guide the learning algorithm on new, unseen tasks.  

\section{Activation Functions}

Activation functions are a crucial component of neural networks.  They are what allow deep networks to learn complex, nonlinear relationships from the training data.  As such, various types of activation functions have been used in neural networks over the years, including manually designed functions and automatically discovered ones \cite{nwankpa2018activation, apicella2021survey, godfrey2017parameterized, godfrey2019evaluation, karlik2011performance, sitzmann2020implicit, vijayaprabakaran2020towards, jadon2019improving, urban2018neural}.

\subsection{Manually Designed Functions}

Early on, sigmoid and tanh were often used as activation functions \cite{chen1990back, cybenko1989approximation}.  These functions have limited range, and thus were helpful in restricting the magnitude of signals propagating through neural network layers.  However, their asymptotic behavior often caused optimization difficulty due to vanishing gradients.  ReLU, being unbounded as $x \rightarrow \infty$, addressed this limitation, and to this day is arguably the most widely used activation function \cite{nair2010rectified}.  Later, Leaky ReLU was introduced to address the dying neuron problem: with ReLU, neurons often become stuck and always output zero \cite{maas2013rectifier}.  The ELU activation function contains a negative saturation regime which helps to control the forward propagated variance, and the SELU activation function contributed subtle refinements to ELU for even better stable signal propagation \cite{elu, selu}.  

\subsection{Automatically Discovered Functions}

The activation functions above are some of the most widely known and illustrative examples, but many more exist \cite{nwankpa2018activation, apicella2021survey}.  In general, researchers often design activation functions to have specific properties in order to increase performance.  This practice works in certain cases, but eventually new and better activation functions are needed as more difficult tasks arise.  This situation points to two opportunities.  First, automating the design of activation functions is a promising means to discovering better functions and avoiding human design biases.  Second, using one activation function for all tasks is likely suboptimal, while leveraging different activation functions specialized to specific tasks can yield better results.  Fortunately, the two opportunities are complimentary: while it may be time-consuming for researchers to manually design specialized activation functions for new tasks, the specialization can be done easily through automated search.

Indeed, prior work on automatic activation function discovery has shown that it is a promising research area.  The approaches have been based on reinforcement learning (RL), evolutionary computation, or gradient descent, and are summarized next.

\paragraph{Reinforcement Learning}
\citet{DBLP:conf/iclr/RamachandranZL18} used RL to design novel activation functions.  They discovered multiple functions, but analyzed just one in depth: $\textrm{Swish}(x) = x \cdot \sigma(x)$.  Of the top eight functions discovered, only Swish and $\max\{x, \sigma(x)\}$ consistently outperformed ReLU across multiple tasks, suggesting that improvements are possible but often task specific.  This hypothesis is confirmed in Chapter \ref{chap:pangaea}, where specialized activation functions are discovered for different tasks, yielding higher performance.

\paragraph{Evolutionary Computation}
Marchisio et al. \cite{marchisio2018methodology} and Hagg et al. \cite{hagg2017evolving} used evolutionary computation to select activation functions from predefined lists, but did not discover novel functions.  \citet{basirat2018quest} used a genetic algorithm to discover novel task-specific piecewise activation functions, but the functions did not outperform ELiSH and HardELiSH, two hand-designed activation functions \citep{basirat2018quest}.  This dissertation scales up these approaches in a number of meaningful ways in order to advance the state of the art, including using larger search spaces and more powerful exploration methods.

\paragraph{Gradient Descent}
Learnable activation functions (LAFs) encode functions with general forms such as polynomial, rational, or piecewise linear, and utilize gradient descent to discover optimal parameterizations during training \citep{apl-agostinelli2014learning, pade-molina2019pad, goyal1906learning, tavakoli2020splash}.  The general forms allow most LAFs to approximate arbitrary continuous functions.  However, just because a LAF can represent an activation function does not guarantee that the optimal function will be discovered by gradient descent.  In Chapter \ref{chap:pangaea}, evolutionary computation and gradient descent are combined into a single optimization process that discovers activation functions that outperform LAFs.

\section{Weight Initialization}

In order to ensure that a new activation function does not cause vanishing or exploding signals, it is important to initialize the weights of the network appropriately.  This section reviews previous research in neural network weight initialization, which has focused on stabilizing signals by accounting for specific components of neural networks such as the activation function, topology, layer types, and training data distribution. However, these approaches fail to generalize to networks that do not meet certain design restrictions, limiting their effectiveness.  Chapter \ref{chap:autoinit} introduces AutoInit, an adaptive weight initialization algorithm designed to address these limitations.

\subsection{Activation-Function-Dependent Initialization}
\label{sec:autoinit:weight_init_for_afns}
As is common in the literature, \texttt{fan\_in} and \texttt{fan\_out} refer to the number of connections feeding into and out of a node, respectively.  \citet{lecun2012efficient} recommend sampling weights from a distribution with mean zero and standard deviation $\sqrt{\texttt{fan\_in}}$.  This initialization encourages propagated signals to have variance approximately one if used with an activation function symmetric about the origin, like $1.7159\tanh\left(\frac{2}{3}x\right)$ or $\tanh(x) + \alpha x$ for some small choice of $\alpha$. The standard sigmoid $f(x) = 1/(1+e^{-x})$ induces a mean shift and should not be used in this setting.

\citet{glorot2010understanding} proposed one initialization strategy to ensure unit variance in the forward-propagated signals and another to ensure unit variance for the backward-propagated gradients.  As a compromise between the two strategies, they initialized weights by sampling from $\mathcal{U}\left(-\frac{\sqrt{6}}{\sqrt{\texttt{fan\_in} + \texttt{fan\_out}}}, \frac{\sqrt{6}}{\sqrt{\texttt{fan\_in} + \texttt{fan\_out}}}\right)$.  They also avoided sigmoid, and instead chose symmetric functions with unit derivatives at 0, such as tanh or Softsign$(x) = x/(1+|x|)$.

\citet{he2015delving} introduced the PReLU activation function and a variance-preserving weight initialization to be used with it that samples weights from $\mathcal{N}(0, \sqrt{2/\texttt{fan\_in}})$.  Similarly, \citet{selu} introduced SELU, an activation function with self-normalizing properties.  These properties are only realized when SELU is used with the initialization scheme by \citet{lecun2012efficient}.

The above weight initialization strategies attempt to solve the same fundamental problem: How can weights be scaled so that repeated applications of the activation function do not result in vanishing or exploding signals?  While these approaches solve this problem in a few special cases, the issue is more general.  Manually deriving the correct scaling is intractable for complicated activation functions.  One approach for an arbitrary function $f$ is to sample Gaussian inputs $x$ and adjust the weights according to the empirical variance $\textrm{Var}(f(x))$ \cite{brock2021characterizing}.  Chapter \ref{chap:autoinit} proposes an alternative and potentially more accurate approach: integration by adaptive quadrature \cite{piessens2012quadpack, 2020SciPy-NMeth, numpy-harris2020array}.  The result is a weight initialization strategy that is compatible with any integrable activation function.  Indeed, previous activation-function-dependent initializations are special cases of the AutoInit algorithm.

\subsection{Topology-Dependent Initialization}
The activation-function-dependent initializations discussed above were designed for neural networks composed of convolutional or dense layers.  After the introduction of residual networks \citep[ResNets;][]{he2016deep, he2016identity}, new weight initialization schemes had to be developed to account for the effect of shortcut connections and various types of residual branches.

\citet{taki2017deep} analyzed signal propagation in plain and batch-normalized ResNets.  They developed a new weight initialization to stabilize training, but did not consider modifications like using deeper residual blocks or reordering components like the activation function or batch normalization layers.  In contrast, AutoInit is topology-agnostic: It adapts to any of these changes automatically.

\citet{zhang2019fixup} introduced Fixup, an initialization method that rescales residual branches to stabilize training.  Fixup replaces batch normalization in standard and wide residual networks \cite{ioffe2015batch, he2016deep, he2016identity, zagoruyko2016wide} and replaces layer normalization \cite{ba2016layer} in transformer models \cite{vaswani2017attention}.  The disadvantages of this scheme are that it only applies to residual architectures, needs proper regularization to get optimal performance, and requires additional learnable scalars that slightly increase model size.

\citet{arpit2019initialize} proposed a new initialization scheme for weight-normalized networks \cite{salimans2016weight} that relies on carefully scaling weights, residual blocks, and stages in the network.  Like related approaches, this technique improves performance in specific cases, but imposes design constraints, like requiring ReLU activation functions and a specific Conv $\rightarrow$ ReLU $\rightarrow$ Conv block structure.

Just as tanh-inspired weight initialization does not stabilize training of ReLU networks, initialization schemes designed for non-residual networks fail with ResNets \cite{hanin2018start, bachlechner2020rezero, brock2021characterizing}.
This observation suggests that future classes of neural networks will again require developing new weight initializations.  Additionally, practitioners with models that do not fit neatly within the restricted settings of existing weight initialization research are left to derive their own initialization or use a suboptimal one.  For example, many initialization schemes assume that the activation function is ReLU \cite{he2015delving, taki2017deep, arpit2019initialize, zhang2019fixup, de2020batch}.  Indeed, ReLU is currently the most popular activation function \cite{nwankpa2018activation, apicella2021survey}, but it is not the best choice in every case \cite{hanin2019deep}.  ReLU prevents dynamical isometry \cite{saxe2013exact, pennington2017resurrecting}, weakens adversarial training \cite{xie2020smooth}, and results in poorer accuracy compared to other activation functions in certain tasks \cite{bingham2022discovering}.  A general weight initialization strategy that does not impose architectural constraints and achieves good performance in diverse settings is needed. AutoInit is designed to meet this challenge.

\subsection{Layer-Dependent Initialization}
\citet{hendrycks2016adjusting} noted that dropout layers \cite{srivastava2014dropout} also affect the variance of forward-propagated signals in a network.  To stabilize training properly, it is necessary to take dropout layers and the specific dropout rate into account in weight initialization.  In fact, pooling, normalization, recurrent, padding, concatenation, and other layer types affect the signal variance in a similar way, but current initialization schemes do not take this effect into account.  AutoInit is designed to adapt to each of these layer types dynamically, and can be extended to include new layer types as they are introduced in the future.

\subsection{Data-Dependent Initialization}
\citet{mishkin2015all} fed data samples through a network and normalized the output of each layer to have unit variance.  \citet{krahenbuhl2015data} adopted a similar approach, but opted to normalize along the channel dimension instead of across an entire layer.  Data-dependent weight initializations are most similar in spirit to AutoInit; they rely on empirical variance estimates derived from the data in order to be model-agnostic.  However, data-dependent weight initializations introduce a computational overhead \cite{mishkin2015all}, and are not applicable in settings where data is not available or its distribution may shift over time, such as online learning or reinforcement learning.  The quality of the initialization is also dependent on the number of the data samples chosen, and suffers when the network is very deep \cite{zhang2019fixup}.  AutoInit instead uses an analytic approach for greater efficiency and higher accuracy.

\subsection{Summary}
Previous techniques solved the initialization problem for networks with specific activation functions, topologies, and layer types. In contrast, AutoInit does not impose design constraints, depend on data samples, or incur a parameter overhead \cite{dauphin2019metainit, zhu2021gradinit} and is therefore a good starting point especially in new settings.

\section{Fisher Information Matrix}

Proper initialization makes evaluating new activation functions more reliable; however, this evaluation step is still computationally expensive.  By providing an inexpensive prediction of final performance, a surrogate model could accelerate the search for new functions.  The Fisher information matrix (FIM) makes this surrogate approach possible.

Consider a neural network $f$ parameterized by weights $\bm{\theta}$ and given inputs $\mathbf{x}$ drawn from a training distribution $Q_{\mathbf{x}}$.  This neural network defines the conditional distribution $R_{\mathbf{y} | f(\mathbf{x} ; \bm{\theta})}$.  The FIM associated with this model is
\begin{equation}
    \mathbf{F} = \mathop{\E}_{\mathclap{\substack{
        \mathbf{x} \sim Q_{\mathbf{x}} \\
        \mathbf{y} \sim R_{\mathbf{y} | f(\mathbf{x} ; \bm{\theta})}
    }}}
    \left[
        \nabla_{\bm{\theta}} \mathcal{L}(\mathbf{y}, f(\mathbf{x}; \bm{\theta})) 
        \nabla_{\bm{\theta}} \mathcal{L}(\mathbf{y}, f(\mathbf{x}; \bm{\theta}))^\top
    \right],
\end{equation}
where the loss function $\mathcal{L}(\mathbf{y}, \mathbf{z})$ represents the negative log-likelihood associated with $R_{\mathbf{y} | f(\mathbf{x} ; \bm{\theta})}$.

The FIM and its eigenvalues $\lambda(\mathbf{F})$ are important quantities in machine learning with many uses.  For example, in optimal experiment design \cite{emery1998optimal}, the cost of experimentation is minimized by optimizing a chosen criterion.  Different criteria with different statistical guarantees are used, but they are typically functions of the eigenvalues of the FIM, such as the maximum or minimum eigenvalue, or the trace of the FIM (sum of the eigenvalues) or determinant of the FIM (product of the eigenvalues).  

Past work has also used the eigenvalues of the FIM to analyze the learning dynamics of neural networks in order to infer optimal values of the batch size or learning rate \cite{liao2018approximate, hayase2021spectrum, karakida2021pathological, furusho2019effects, furusho2020theoretical}.  Because the FIM is related to the optimization landscape at initialization, it provides insights on the learning dynamics of SGD \cite{jastrzebski2021catastrophic} and the dynamics of signal propagation at different layers of neural networks \cite{huang2020layer}.  This information can be used to inform network design by identifying specific layers that are poorly conditioned and therefore difficult to optimize.  The FIM is also used in second-order optimization algorithms for neural networks \cite{grosse2016kronecker, martens2015optimizing, martens2018kronecker}.  These approaches are more expensive but often require fewer iterations than first-order methods.

Chapter \ref{chap:aquasurf} builds on this work by using the FIM in a surrogate approach to accelerate activation function search.  However, instead of choosing one optimality criterion or only considering one summary statistic, all of the eigenvalues of the FIM are kept and an optimal distribution is learned experimentally.  These eigenvalues are then used as a feature vector to predict the performance of different activation functions efficiently.

\section{Conclusion}

Work in AutoML has shown that it is possible to optimize the design of complex structures.  Previous work on activation function design and weight initialization strategies shows that they are critical components for neural network performance, and that improving upon them can lead to better results.  These results suggest that automating activation function design and weight initialization is a promising avenue towards improving the current state of the art as well as an important step progressing towards fully automatic machine learning.

\chapter{CAFE: Evolutionary Optimization of Deep Learning Activation Functions}
\label{chap:gecco}

This chapter demonstrates that novel activation functions can outperform baseline functions by statistically significant margins.  Because this chapter serves as a foundation for other work in this dissertation, part of the motivation for this work is exploratory.  To this end, activation functions are discovered in both small and large search spaces using a variety of search algorithms, including exhaustive search, random search, and evolution with two kinds of fitness functions.  The experiments show that evolutionary search is particularly effective and creative \cite{lehman2020surprising}.  This work, called CAFE (Creative Activation Function Evolution), was done in collaboration with graduate student William Macke \cite{bingham2020gecco}.

\section{Evolving Activation Functions}
\label{sec:gecco:evolving_functions}
This section presents the approach to evolving activation functions, introducing the search space, mutation and crossover implementations, and the overall evolutionary algorithm.

\subsection{Search Space}
\label{sec:gecco:searchspace}
Each activation function is represented as a tree consisting of unary and binary operators. Functions are grouped in layers such that two unary operators always feed into one binary operator.  The following operators, modified slightly from the search space of \citet{DBLP:conf/iclr/RamachandranZL18}, are used:
\begin{itemize}
    \item \textbf{Unary:} 0, 1, $x$, $-x$, $|x|$, $x^2$, $x^3$, $\sqrt{x}$, $e^x$, $e^{-x^2}$, \allowbreak $\log(1+e^x)$, $\log(|x + \epsilon|)$, $\sin(x)$, $\textrm{sinh}(x)$, $\textrm{arcsinh}(x)$, $\cos(x)$, $\textrm{cosh}(x)$, \allowbreak $\textrm{tanh}(x)$, $\textrm{arctanh}(x)$, $\max\{x, 0\}$, $\min\{x, 0\}$, $\sigma(x)$, $\textrm{erf}(x)$,  $\textrm{sinc}(x)$;
    \item \textbf{Binary:} $x_1 + x_2$, $x_1 - x_2$, $x_1 \cdot x_2$, $x_1 / (x_2 + \epsilon)$, $\max\{x_1, x_2\}$, $\min\{x_1, x_2\}$.
\end{itemize}
Following Ramachandran et al., a ``core unit'' is an activation function that can be represented as \texttt{core\_unit = binary(unary1(x), unary2(x))}.  Let $F$ be the set of balanced core unit trees.  $S$ is then defined as a family of search spaces
\begin{equation}
S_{d\in \mathbb{N}} = \{f\in F \mid \textrm{depth}(f) = d\}.
\end{equation}
For example, $S_1$ corresponds to the set of functions that can be represented by one core unit, $S_2$ represents functions of the form: \texttt{core\_unit1(core\_unit2(x), core\_unit3(x))}, and so on. Examples of functions in $S_2$ are illustrated in Figures \ref{fig:gecco:mutation} and \ref{fig:gecco:crossover}.

\begin{figure}
    \centering
    \includegraphics[width=0.5\linewidth]{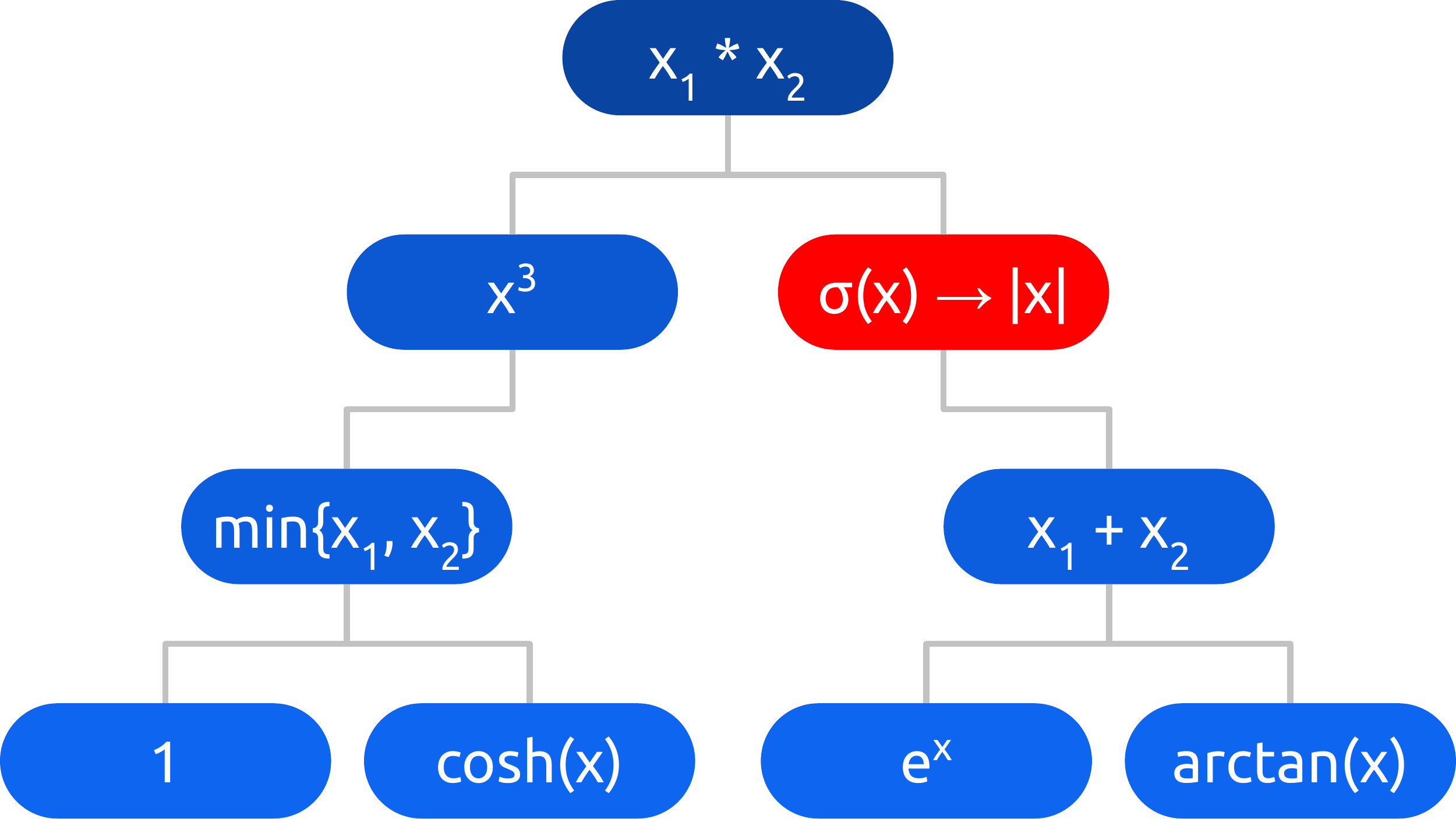}
    \caption{An example of activation function mutation.  The tree represents the activation function {\normalfont $(\min\{1, \textrm{cosh}(x)\})^3 * \sigma(e^x + \textrm{arctan}(x))$}.  One node in the tree is selected uniformly at random and replaced with another operator in the search space, also uniformly at random.  The resulting activation function is {\normalfont $(\min\{1, \textrm{cosh}(x)\})^3 * |e^x + \textrm{arctan}(x)|$}.  By introducing variability, mutation ensures evolution explores the search space sufficiently.  It prevents high-performing activation functions from overly skewing the early generations of the search process.}
    \label{fig:gecco:mutation}
\end{figure}

\subsection{Mutation}
In mutation, one node in an activation function tree is selected uniformly at random.  The operator at that node is replaced with another random operator in the search space.  Unary operators are always replaced with unary operators, and binary operators with binary operators.  An example of mutation is shown in Figure~\ref{fig:gecco:mutation}.  Theoretically, mutation alone is sufficient for constructing any activation function.  However, preliminary experiments showed that crossover can increase the rate at which good activation functions are found.

\begin{figure}
    \centering
    \includegraphics[width=0.75\linewidth]{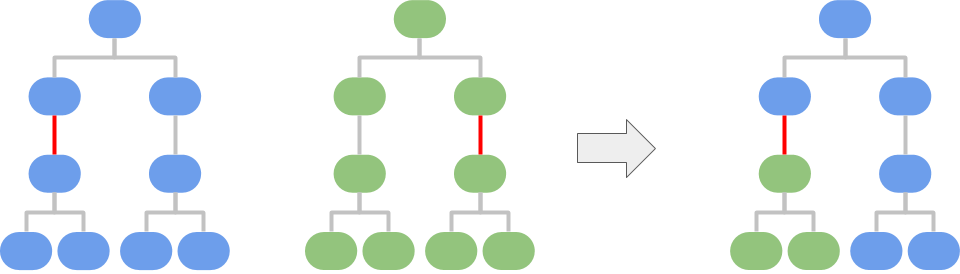}
    \caption{In crossover, two parent activation functions exchange randomly selected subtrees of equivalent depth, producing one new child activation function.  Crossover enables the best activation functions to pass on their characteristics to the rest of the population.  This mechanism is what enables evolution to discover better activation functions more quickly than random search.}
    \label{fig:gecco:crossover}
\end{figure}

\subsection{Crossover}
In crossover, two parent activation functions exchange randomly selected subtrees, producing one new child activation function.  The subtrees are constrained to be of the same depth, ensuring the child activation function is a member of the same search space as its parents.  Crossover is depicted in Figure~\ref{fig:gecco:crossover}.

\subsection{Evolution}
Starting with a population of $N$ activation functions, a neural network is trained with each function on a given training dataset.  Each function is assigned a fitness $p_i$ equal to the softmax of an evaluation metric $L_i$.  This metric could be either accuracy or negative loss obtained on the validation dataset.  More specifically, 
\begin{equation}
    p_i = \frac{e^{L_i}}{\sum\limits_{j=1..N}e^{L_j}}.
\end{equation}
The softmax operation converts the fitness values to a probability distribution, allowing functions to be randomly sampled.  From the $N$ activation functions, $2(N-m)$ are selected with replacement for reproduction with probability proportional to their fitness.  Crossover followed by mutation is applied to the selected activation functions to obtain a new population of size $N-m$.  In order to increase exploration, $m$ randomly generated functions are added to a population that will again be of size $N$.  This process is repeated for several generations, and the activation functions with the best performance over the history of the search are returned as a result.
\section{Experiments}
\label{sec:gecco:experiments}

This section presents experiments with the WRN-28-10 and WRN-40-4 architectures on the CIFAR-10 and CIFAR-100 datasets \cite{zagoruyko2016wide, krizhevsky2009learning}.  In the experiments, the default ReLU activation functions are replaced with candidate functions, and all other settings are kept to default values.  Implementation details are in Appendix \ref{ap:details:gecco}.

\subsection{Search Strategies}
Three different techniques are used to explore the space of activation functions: exhaustive search, random search, and evolution. Exhaustive search evaluates every function in $S_1$, while random search and evolution explore $S_2$. It is noteworthy that evolution is able to discover high-performing activation functions in $S_2$, where the search space contains over 41 billion possible function strings.  In each experiment, the top three activation functions by validation accuracy from the entire search are kept.  These functions are then reevaluated on the test set, and the median accuracy of five independent runs is reported.

\paragraph{Exhaustive Search}

Ramachandran et al. search for activation functions using reinforcement learning and argue that simple activation functions consistently outperform more complicated ones \cite{DBLP:conf/iclr/RamachandranZL18}.  Although evolution is capable of discovering high-performing, complex activation functions in an enormous search space, exhaustive search can be effective in smaller search spaces.  $S_1$, for example, contains 3,456 possible function strings.
\paragraph{Random Search}

An illustrative baseline comparison with evolution is random search.  Instead of evolving a population of 50 activation functions for 10 generations, 500 random activation functions from $S_2$ are grouped into 10 ``generations'' of 50 functions each.

\paragraph{Evolution}
As shown in Figure \ref{fig:gecco:gecco_evolution}, evolution discovers better activation functions more quickly than random search in $S_2$, a search space where exhaustive search is infeasible.  During evolution, candidate activation functions are assigned a fitness value based on either accuracy or loss on the validation set.  Accuracy-based fitness favors exploration over exploitation: activation functions with poor validation accuracy still have a reasonable probability of surviving to the next generation.  A hypothetical activation function that achieves 90\% validation accuracy is only 2.2 times more likely to be chosen for the next generation than a function with only 10\% validation accuracy since $e^{0.9} / e^{0.1} \approx 2.2$.

\begin{figure}[ht]
    \centering
    \includegraphics[width=0.5\linewidth]{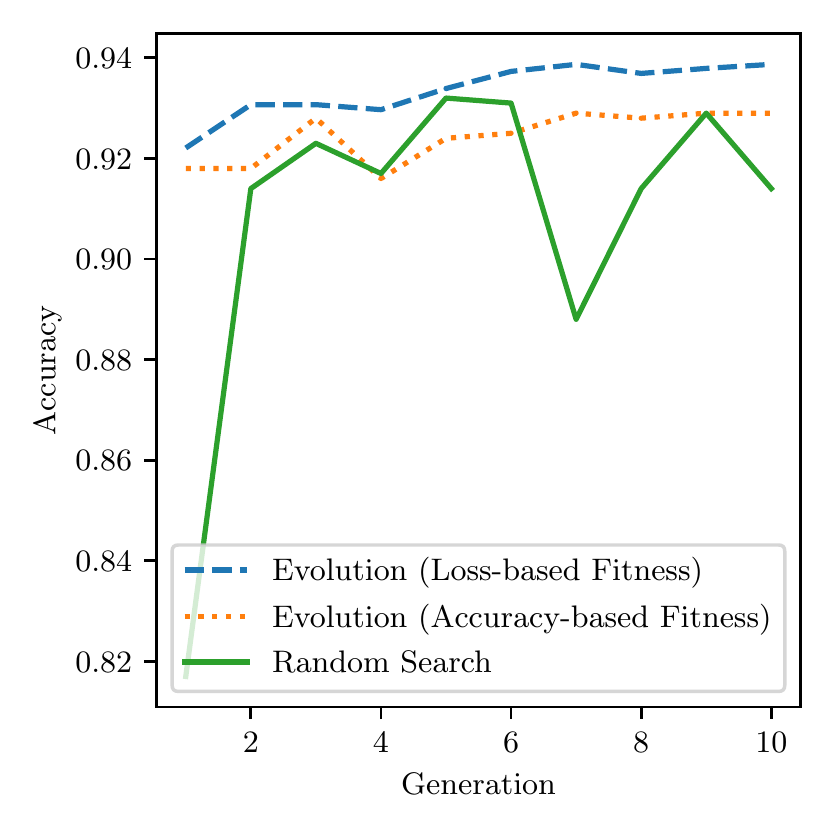}
    \caption{Top validation accuracy per generation for three search strategies in the $S_2$ search space.  There are 50 activation functions in each generation of a search.  All activation functions are trained with WRN-28-10 on CIFAR-10 for 50 epochs, and the highest validation accuracy obtained among all activation functions in a given generation is reported. Evolution with loss-based fitness finds better activation functions more quickly than evolution with accuracy-based fitness or random search. The first generation of random search is poor due to chance; since each generation is independent, the generations of random search could be arbitrarily reordered.}
    \label{fig:gecco:gecco_evolution}
\end{figure}

Loss-based fitness sharply penalizes poor activation functions.  It finds high-performing activation functions more quickly, and gives them greater influence over future generations.  An activation function with 0.01 validation loss is 21,807 times more likely to be selected for the following generation than a function with a validation loss of 10.  $e^{-0.01} / e^{-10} \approx 21{,}807$.

Both experiments begin with an initial population of $N=50$ random activation functions, and run through 10 generations of evolution.  Each new generation of 50 activation functions is comprised of the top five functions from the previous generation, $m=10$ random functions, and 35 functions created by applying crossover and mutation to existing functions in the population. 

\subsection{Activation Function Specialization}

An important question is the extent to which activation functions are specialized for the architecture and dataset for which they were evolved, or perform well across different architectures and datasets.  To address this question, activation functions discovered for WRN-28-10 on CIFAR-10 are transferred to WRN-40-4 on CIFAR-100.  These activation functions are compared with the best from a small search (1.9K activation functions from $S_1$) with WRN-40-4 on CIFAR-100. 

\section{Results}
\label{sec:gecco:results}
This section presents the experimental results, which demonstrate that evolved activation functions can outperform baseline functions like ReLU and Swish.

\definecolor{color1}{rgb}{0.12156862745098039, 0.4666666666666667, 0.7058823529411765}
\definecolor{color2}{rgb}{1.0, 0.4980392156862745, 0.054901960784313725}
\definecolor{color3}{rgb}{0.17254901960784313, 0.6274509803921569, 0.17254901960784313}
\begin{table}
\centering
\caption{The top three activation functions discovered by each search strategy, along with the baseline activation functions ReLU and Swish.  The functions achieved the highest validation set accuracy after 50 epochs of training with WRN-28-10 on CIFAR-10.  The final test accuracies listed on the right are the median of five runs after training WRN-28-10 from scratch for 200 epochs with each activation function on CIFAR-10 and CIFAR-100.  Function plots have domain $x \in (-5, 5)$ but have different ranges.  Functions marked with an asterisk (*) occasionally did not train to completion due to asymptotes at $x = -\epsilon$.  Exhaustive search finds multiple activation functions in a simple search space ($S_1$) that consistently outperform ReLU and Swish.  Evolution with loss-based fitness was the only technique that was able to discover such activation functions in a more complicated search space ($S_2$), suggesting that it is the most promising technique for scaling up in the future.\\}
\adjustbox{max width=\linewidth}{%
\begin{tabular}{clcc}
\toprule
 & & \multicolumn{2}{c}{\textbf{Accuracy}} \\
 \textbf{Function Plots} & & \multicolumn{1}{c}{CIFAR-10} & \multicolumn{1}{c}{CIFAR-100} \\ 
 \midrule
 
 \multirow{4}{*}{\includegraphics[width=10ex]{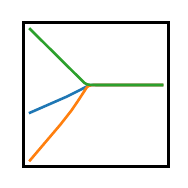}} &
 \multicolumn{1}{c}{\textbf{Evolution with Loss-based Fitness ($S_2$)}} & & \\
 & {\LARGE \textcolor{color1}{\textbullet}} 
 $(e^{(\min\{\textrm{erf}(x), 0\}) - (\max\{x, 0\})}) * (\min\{(\arctan((x)^3)) * (\max\{|x|, 0\}), 0\})$ 
 & \textbf{94.1} & 73.9 \\
 & {\LARGE \textcolor{color2}{\textbullet}} 
 $(e^{\max\{\min\{\textrm{erf}(x), 0\}, \max\{x, 0\}\}}) * (\min\{(\arctan((x)^3)) * (\max\{|x|, 0\}), 0\})$         
 & 10.0 & 01.0 \\
 & {\LARGE \textcolor{color3}{\textbullet}} 
 $(-((\arctan((x)^3)) * (\cos(1)))) * (-((\arctan(\min\{x, 0\})) * (\max\{|x|, 0\})))$              
 & 93.9 & 74.1 \\ 
 \midrule
 
 \multirow{4}{*}{\includegraphics[width=10ex]{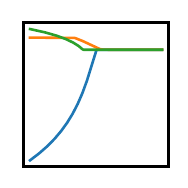}} & \multicolumn{1}{c}{\textbf{Evolution with Accuracy-based Fitness ($S_2$)}} & & \\
 & {\LARGE \textcolor{color1}{\textbullet}} 
 $\min\{e^{-(\min\{(\sinh(x))^2, (0)^2\})^2}, \min\{\min\{\textrm{erf}(\log(1 + e^x)), \textrm{arcsinh}(x)\}, 0\}\}$ 
 & 10.0 & 01.0 \\
 & {\LARGE \textcolor{color2}{\textbullet}} 
 $\min\{\cos(\max\{(\min\{x, 0\})^3, \log(1 + e^1)\}), e^{-((|\max\{x, 0\}|) + (e^{\sigma(x)}))^2}\}$   
 & 93.5 & 72.5 \\
 & {\LARGE \textcolor{color3}{\textbullet}} 
 $\max\{\max\{(\log(|(\min\{x, 0\}) + \epsilon|)) * (\sigma(\textrm{erf}(x))), 0\}, 0\}$ \hfill (*) 
 & 92.5 & 72.3 \\ 
 \midrule
 
 \multirow{4}{*}{\includegraphics[width=10ex]{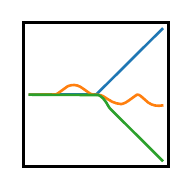}} & \multicolumn{1}{c}{\textbf{Random Search ($S_2$)}} & & \\
 & {\LARGE \textcolor{color1}{\textbullet}} 
 {\tiny $(\min\{\max\{(\log(|(x) + \epsilon|))^3, -(\log(|(x) + \epsilon|))\}, 0\}) + (e^{(\min\{\tanh(x), 0\}) + (\log(|(\max\{x, 0\}) + \epsilon|))})$ \hfill (*) }
 & 93.8 & 73.9 \\
 & {\LARGE \textcolor{color2}{\textbullet}} 
 {\footnotesize $(\arctan(\min\{\sinh(\sin(x)), \arctan(\max\{x, 0\})\})) * (\tanh((-(\sin(x))) * (\textrm{arcsinh}(x))))$ }          
 & 93.3 & 72.1 \\ 
 & {\LARGE \textcolor{color3}{\textbullet}} 
 $(\max\{\frac{\min\{(x)^3, 0\}}{\max\{\sin(x), 0\} + \epsilon}, 0\}) - (\min\{(x)^2, \max\{x, 0\}\})$  
 & 93.9 & 73.2 \\ 
 \midrule
 
 \multirow{4}{*}{\includegraphics[width=10ex]{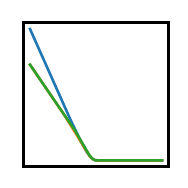}} & \multicolumn{1}{c}{\textbf{Exhaustive Search ($S_1$)}} & & \\
 & {\LARGE \textcolor{color1}{\textbullet}} 
 $(\arctan(x)) * (\min\{x, 0\})$ 
 & 94.0 & \textbf{74.5} \\
 & {\LARGE \textcolor{color2}{\textbullet}} $(\tanh(x)) * (\min\{x, 0\})$
 & \textbf{94.1} & 74.3 \\
 & {\LARGE \textcolor{color3}{\textbullet}} 
 $(\min\{x, 0\}) * (\textrm{erf}(x))$
 & 94.0 & 74.2 \\
 \midrule
 
 \multirow{4}{*}{\includegraphics[width=10ex]{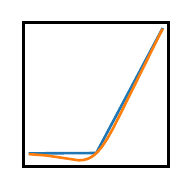}} & \multicolumn{1}{c}{\textbf{Baseline Activation Functions}} & & \\
 & {\LARGE \textcolor{color1}{\textbullet}} ReLU$(x)$                                                   & 94.0 & 73.3 \\
 & {\LARGE \textcolor{color2}{\textbullet}} Swish$(x)$                                                  & 93.8 & 71.1\\ \\
 \bottomrule
\end{tabular}%
}
\label{tab:gecco:results}
\end{table}

\subsection{Improving Performance}

Table \ref{tab:gecco:results} lists the activation functions that achieved the highest validation set accuracies after 50 epochs of training with WRN-28-10 on CIFAR-10.  The top three activation functions for each search strategy are included.  To emulate their true performance, a WRN-28-10 with each activation function was trained for 200 epochs five times on both CIFAR-10 and CIFAR-100 and evaluated on the test set.  The median accuracy of these tests is reported in Table~\ref{tab:gecco:results}.  Although no search was performed on CIFAR-100 with WRN-28-10, the functions that perform well on CIFAR-10 successfully generalize to CIFAR-100.

The best three activation functions discovered through exhaustive search in $S_1$ outperform ReLU and Swish. This finding shows how important it is to have an effective search method. There are good functions even in $S_1$. It is likely that there are even better functions in $S_2$, but with billions of possible functions, a more sophisticated search method is necessary.

The activation functions discovered by random search have unintuitive shapes (Table~\ref{tab:gecco:results}).  Although they fail to outperform the baseline activation functions, it is impressive that they still consistently reach a reasonable accuracy.  One of the functions (marked with (*) in Table \ref{tab:gecco:results}) discovered by random search occasionally failed to train to completion due to an asymptote at $x = -\epsilon$.

Evolution with accuracy-based fitness is less effective because it does not penalize poor activation functions severely enough.  One of the functions failed to learn anything better than random guessing.  It was likely too sensitive to random initialization or was unable to learn with the slightly different learning rate schedule of a full 200-epoch training.  Another function (marked with (*) in Table \ref{tab:gecco:results}) often did not train to completion due to an asymptote at $x = -\epsilon$.  The one function that consistently trained well still failed to outperform ReLU.

Evolution with loss-based fitness is able to find good functions in $S_2$. One of the three activation functions discovered by evolution outperformed both ReLU and Swish on CIFAR-10, and two of the three discovered outperformed ReLU and Swish on CIFAR-100.  Figure \ref{fig:gecco:best_per_generation} shows the top activation function after each generation of loss-based evolution.  This approach discovered both novel and unintuitive functions that perform reasonably well, as well as simple, smooth, and monotonic functions that outperform ReLU and Swish. It is therefore the most promising search method in large spaces of activation functions.  

\begin{figure}
    \centering
    \includegraphics[width=0.5\linewidth]{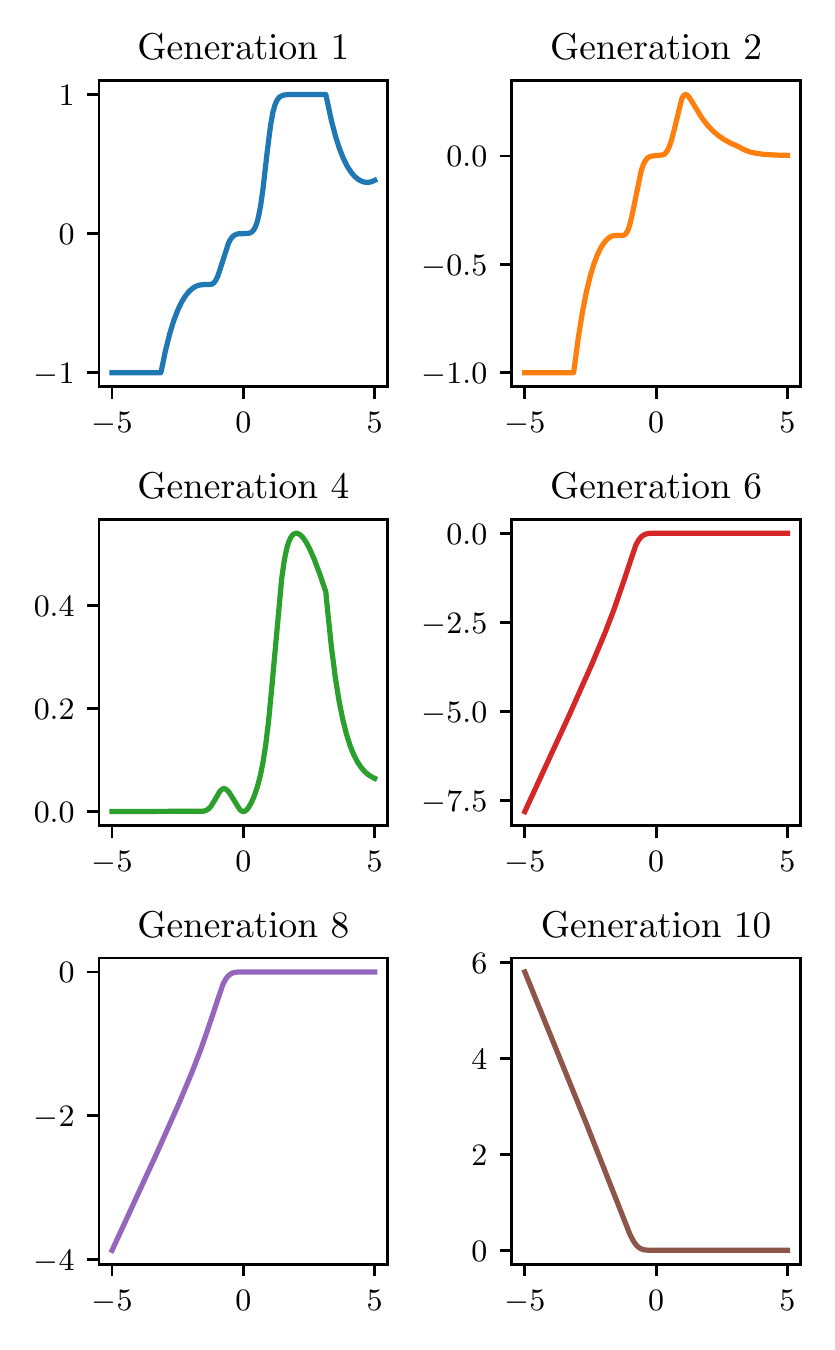}
    \caption{Best activation function by WRN-28-10 validation accuracy after 50 epochs of training on CIFAR-10 using evolution with loss-based fitness.  Top validation accuracy improves from 92.2 in Generation 1 to 93.9 in Generation 10. Evolution is able to discover effective activation functions that are not likely to be discovered by hand.  In particular, the top activation function discovered by evolution is smooth everywhere, unlike ReLU, which is not smooth at $x=0$.  This difference is likely the reason for its superior performance.}
    \label{fig:gecco:best_per_generation}
\end{figure}

The performance gains on CIFAR-10 are consistent but small, and the improvement on CIFAR-100 is larger.  It is possible that more difficult datasets provide more room for improvement that a novel activation function can exploit.

To evaluate the significance of these results, WRN-28-10 was trained on CIFAR-100 for 200 epochs 50 times with ReLU, 50 times with the best function found by exhaustive search in $S_1$, $(\arctan(x)) * (\min\{x, 0\})$, 25 times with Swish, and 15 times with the best function found by evolution in $S_2$, $(-((\arctan((x)^3)) * (\cos(1)))) * (-((\arctan(\min\{x, 0\})) * (\max\{|x|, 0\})))$.  Table \ref{tab:gecco:statistical_significance} shows 95\% confidence intervals and the results of independent $t$-tests comparing the mean accuracies achieved with each activation function.  The results show that replacing a baseline activation function with an evolved one results in a statistically significant increase in accuracy.

\begin{table}[t]
    \centering

    \caption{Confidence intervals (95\%) and independent $t$-tests comparing mean accuracies after training WRN-28-10 on CIFAR-100 for 200 epochs with the best function from $S_1$, $(\arctan(x)) * (\min\{x, 0\})$, the best function from $S_2$, $(-((\arctan((x)^3)) * (\cos(1)))) * (-((\arctan(\min\{x, 0\})) * (\max\{|x|, 0\})))$, and baseline functions ReLU and Swish.  The discovered functions perform significantly better than both baselines.  WRN-28-10 was not trained 50 times with Best from $S_2$ and with Swish due to time constraints.  In repeated trials, Swish occasionally caused the network to stall during training, explaining its low mean accuracy.\\}
    
    \begin{tabular}{ccc} \toprule
         Activation Function & Mean Accuracy (95\% C.I.) & Repeats \\ \midrule
         Best from $S_1$ & 74.2 ($\pm 0.1$) & 50 \\
         Best from $S_2$ & 74.0 ($\pm 0.2$) & 15 \\
         ReLU            & 73.2 ($\pm 0.2$) & 50 \\
         Swish           & 49.6 ($\pm 11.6$) & 25 \\ \bottomrule \\
    \end{tabular}
    
    \begin{tabular}{cc} \toprule
    Activation Functions & $t$-statistic; $p$-value \\ \midrule
    Best from $S_1$ vs. ReLU  & 9.73; $4.64 \times 10^{-16}$\\
    Best from $S_1$ vs. Swish & 5.91; $9.91 \times 10^{-8}$\\
    Best from $S_2$ vs. ReLU  & 4.51; $2.91 \times 10^{-5}$\\
    Best from $S_2$ vs. Swish & 3.17; $2.98 \times 10^{-3}$\\ \bottomrule 
    \end{tabular}
    
    \label{tab:gecco:statistical_significance}
\end{table}

Among the top activation functions discovered, many are smooth and monotonic.  Hand-engineered activation functions frequently share these properties \cite{nwankpa2018activation}.  Two notable exceptions were found by random search and evolution with accuracy-based fitness.  Although these functions do not outperform ReLU, the fact that WRN-28-10 was able to achieve such high accuracy with these arbitrary functions raises questions as to what makes an activation function effective. Ramachandran et al.~\cite{DBLP:conf/iclr/RamachandranZL18} asserted that simpler activation functions consistently outperformed more complicated ones. However, the high accuracy achieved with activation functions discovered by evolution in $S_2$ demonstrates that complicated activation functions can compete with simpler ones. Such flexibility may be particularly useful in specialization to different architectures and datasets. It is plausible that there exist many unintuitive activation functions which can outperform the more general ones in specialized settings.  Evolution is well-positioned to discover them.

\subsection{Specialized Activation Functions}
Since different functions are seen to emerge in different experiments, an important question is: How general or specialized are they to a particular architecture and dataset?  To answer this question, the top activation function discovered for WRN-28-10 on CIFAR-10, $\tanh(x) * \min\{x, 0\}$, was trained with WRN-40-4 on CIFAR-100 for 200 epochs.  This result was then compared with performance achieved by $\sigma(x) * \textrm{erf}(x)$, an activation function discovered specifically for WRN-40-4 on CIFAR-100.  Table \ref{tab:gecco:specialized} summarizes the result: the activation function discovered for the first task does transfer to the second task, but even higher performance is achieved when a specialized function is discovered specifically for the second task.

\begin{table}
    \centering
    \caption{Test set accuracy of WRN-40-4 with various activation functions after 200 epochs of training on CIFAR-100.  Results reported are median of five runs.  The top activation function discovered for WRN-28-10 on CIFAR-10, $\tanh(x) * \min\{x, 0\}$, successfully transfers to this new task, outperforming both baselines.  However, a search designed specifically for WRN-40-4 on CIFAR-100 discovers a novel activation function, {\normalfont $\sigma(x) * \textrm{erf}(x)$} that results in even higher performance.  This result demonstrates that the main power of activation function metalearning is to be able to specialize the function to the architecture and dataset.\\}
    \begin{tabular}{cc} \toprule
        Activation Function & Accuracy \\ \midrule
        $\sigma(x) * \textrm{erf}(x)$ & \textbf{72.6} \\
        $\tanh(x) * \min\{x, 0\}$ & 72.1 \\
        $\textrm{Swish}(x)$ & 71.1 \\
        $\textrm{ReLU}(x)$ & 71.0 \\ \bottomrule 
    \end{tabular}
    \label{tab:gecco:specialized}
\end{table}

The specialized activation function, $\sigma(x) * \textrm{erf}(x)$, is shown in Figure \ref{fig:gecco:summary}c.  It is similar to $\sigma(x)$ in that it tends towards 0 as $x \rightarrow -\infty$, and approaches 1 as $x \rightarrow \infty$.  It differs from $\sigma(x)$ in that it has a non-monotonic bump for small, negative values of $x$.  A 50-epoch training of WRN-40-4 on CIFAR-100 with activation $\sigma(x)$ achieved validation accuracy of just 63.2.  The superior performance of $\sigma(x) * \textrm{erf}(x)$ suggests that the negative bump was important, as the shapes of the two activation functions are otherwise similar. This result demonstrates how evolution can discover specializations that make a significant difference.


\section{Discussion}
\label{sec:gecco:discussion}

Among the top activation functions discovered, many are smooth and monotonic.  Hand-engineered activation functions frequently share these properties \cite{nwankpa2018activation}.  Two notable exceptions were found by random search and evolution with accuracy-based fitness.  Although these functions do not outperform ReLU, the fact that WRN-28-10 was able to achieve such high accuracy with these arbitrary functions raises questions as to what makes an activation function effective. Ramachandran et al.~\cite{DBLP:conf/iclr/RamachandranZL18} asserted that simpler activation functions consistently outperformed more complicated ones. However, the high accuracy achieved with activation functions discovered by evolution in $S_2$ demonstrates that complicated activation functions can compete with simpler ones. Such flexibility may be particularly useful in specialization to different architectures and datasets. It is plausible that there exist many unintuitive activation functions which can outperform the more general ones in specialized settings.  Evolution is well-positioned to discover them.

Activation functions discovered by evolution perform best on the architectures and datasets for which they were evolved.  Figure \ref{fig:gecco:scatter} demonstrates this principle.  More generally, it shows the performance of several activation functions when trained with WRN-28-10 for 50 epochs on CIFAR-10 and when trained with WRN-40-4 for 50 epochs on CIFAR-100.  Activation functions that perform well for one task often perform well on another task, but not always.  Therefore, if possible, one should evolve them specifically for each architecture and dataset.  However, as the results in Section \ref{sec:gecco:results} show, it is feasible to evolve using smaller architectures and datasets and then transfer to scaled up architectures and more difficult datasets within the same domain.

\begin{figure}[t]
    \centering
    \includegraphics[width=0.5\linewidth]{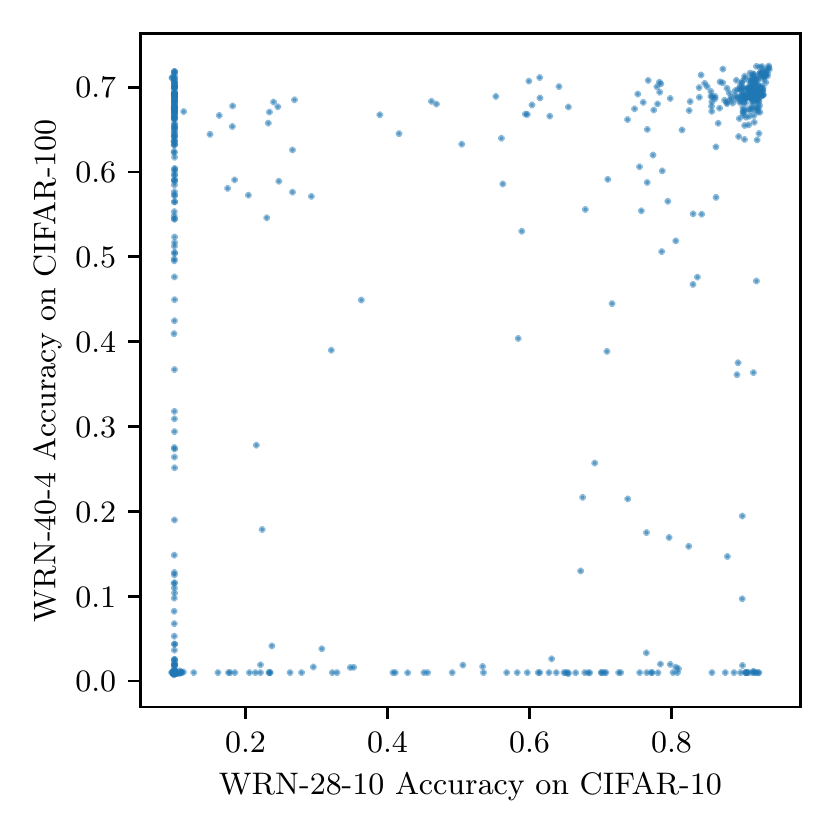}
    \caption{Activation function accuracy across tasks.  Each data point represents validation accuracy when training with a given activation function from $S_1$ with WRN-40-4 on CIFAR-100 for 50 epochs, and with WRN-28-10 on CIFAR-10 for 50 epochs.  Some activation functions perform well when paired with a different architecture and dataset.  Other functions are specialized to a given architecture and dataset, and do not transfer.  The results suggest that reasonable performance can be expected from general evolved activation functions, but that the best performance comes from evolving activation functions for a specific task.}
    \label{fig:gecco:scatter}
\end{figure}

\begin{figure}
    \centering
    \includegraphics[width=\linewidth]{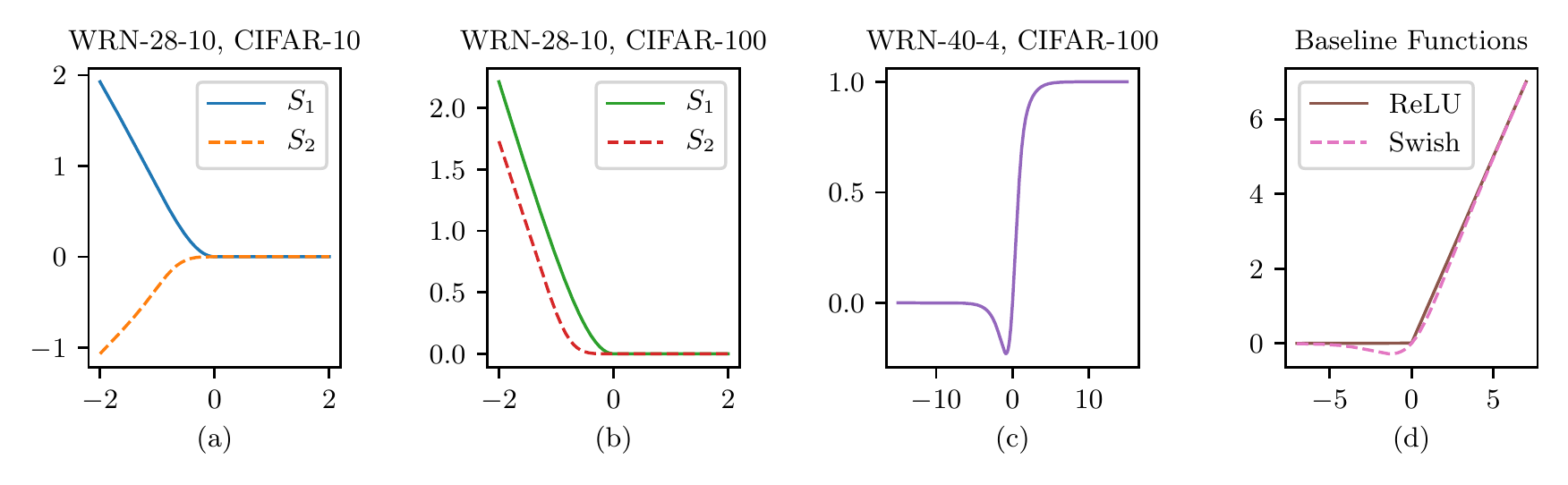}
    \caption{A summary of the best activation functions found in the first set of experiments.  (a): Exhaustive search in $S_1$ discovers $(\tanh(x)) * (\min\{x, 0\})$, and evolution in $S_2$ discovers {\normalfont $(e^{(\min\{\textrm{erf}(x), 0\}) - (\max\{x, 0\})}) * (\min\{(\arctan((x)^3)) * (\max\{|x|, 0\}), 0\})$}.  Both functions achieve a median test accuracy of 94.1 with WRN-28-10 on CIFAR-10, outperforming that of ReLU (94.0) and Swish (93.8).  (b): The functions $(\arctan(x)) * (\min\{x, 0\})$ in $S_1$ and $(-((\arctan((x)^3)) * (\cos(1)))) * (-((\arctan(\min\{x, 0\})) * (\max\{|x|, 0\})))$ in $S_2$ outperform ReLU and Swish by statistically significant margins with WRN-28-10 on CIFAR-100, demonstrating the power of novel activation functions.  (c): Functions evolved for WRN-28-10 on CIFAR-10 perform well with WRN-40-4 on CIFAR-100, but a new function discovered specifically for WRN-40-4 on CIFAR-100, {\normalfont $\sigma(x) * \textrm{erf}(x)$}, achieves even higher performance.  This result shows that the biggest advantage of activation function search is the ability to discover functions that are specialized to the architecture and dataset.  (d): The baseline activation functions, ReLU and Swish, are included for visual comparison.}
    \label{fig:gecco:summary}
\end{figure}

In the future, it may be possible to push such generalization further, by evaluating functions across multiple architectures and datasets.  In this manner, evolution may be able to combine the requirements of multiple tasks, and discover functions that perform well in general.  However, the main power in activation function search is to discover functions that are specialized to each architecture and dataset.  In that setting most significant improvements are possible.

\section{Conclusion}
\label{sec:gecco:conclusion}

Multiple strategies for discovering novel, high-performing activation functions were presented and evaluated: namely exhaustive search in a small search space ($S_1$) and random search and evolution in a larger search space ($S_2$).  Evolution with loss-based fitness finds activation functions that achieve high accuracy and outperform standard functions such as ReLU and novel functions such as Swish, demonstrating the power of search in large spaces. The best activation functions successfully transfer from CIFAR-10 to CIFAR-100 and from WRN-28-10 to WRN-40-4. However, the main power of activation function search is in finding specialized functions for each architecture and dataset, leading to significant improvement.  These results (summarized in Figure \ref{fig:gecco:summary}) provide a foundation for the next chapter, where evolution and gradient descent explore a flexible search space and discover activation functions that adapt to different network locations and stages of training.

\chapter{PANGAEA: Discovering Parametric Activation Functions}
\label{chap:pangaea}
\usetikzlibrary{shapes.geometric, arrows, positioning, calc}
\tikzstyle{input} = [rectangle, rounded corners, text centered, draw=black, fill=red!30]
\tikzstyle{unary} = [rectangle, rounded corners, text centered, draw=black, fill=yellow!30]
\tikzstyle{binary} = [rectangle, rounded corners, text centered, draw=black, fill=blue!30]
\tikzstyle{arrow} = [thick, ->,>=stealth]
\tikzstyle{output} = [rectangle, rounded corners, text centered, draw=black, fill=green!30]
\tikzstyle{invisible} = [rectangle, text centered]
\tikzstyle{maybeparam} = [draw, circle, dashed, fill=cyan!30]

This chapter describes an extended approach for discovering activation functions called PANGAEA (Parametric ActivatioN functions Generated Automatically by an Evolutionary Algorithm) \cite{bingham2022discovering}.  Chapter \ref{chap:gecco} explored activation function discovery with multiple search algorithms and search spaces, and identified evolution in large spaces as a creative and powerful approach.  In order to understand the full potential of such a system, this chapter relaxes design constraints and scales up in multiple ways.  The activation function representation is expanded to include arbitrary computation graphs, which allows for representing more complex functional forms.  Additional mutation operators are then introduced to allow for more efficient exploration of this search space.  The approach utilizes a synergy of two different optimization processes: evolutionary population-based search discovers the general form of the activation function, and gradient descent introduces further flexibility by fine-tuning the shape of the function across locations in a network and over time as training progresses.  PANGAEA discovers general activation functions that improve performance overall over previously-proposed functions.  It also produces specialized functions for different architectures, such as Wide ResNet, ResNet, and Preactivation ResNet, that perform even better than the general functions, demonstrating its ability to customize activation functions to architectures.

\section{The PANGAEA Method}
\label{sec:pangaea:pangaea_method}

Activation functions in PANGAEA are represented as computation graphs, which allow for comprehensive search, efficient implementation, and effective parameterization. Regularized evolution with reranking is used as the search method to encourage exploration and to reduce noise.

\subsection{Representing and Modifying Activation Functions}

Activation functions are represented as computation graphs in which each node is a unary or a binary operator (Table \ref{tab:pangaea:searchspace}).  All of these operators have TensorFlow \cite{abadi2016tensorflow} implementations, which allows for taking advantage of under-the-hood optimizations.  Safe operator implementations are chosen when possible (e.g.\ the binary operator $x_1 / x_2$ is implemented as \texttt{tf.math.divide\_no\_nan}, which returns $0$ if $x_2 = 0$).  Operators that are periodic (e.g.\ $\sin(x)$) and operators that contain repeated asymptotes are not included; in preliminary experiments they often caused training instability.  All of the operators have domain $\mathbb{R}$, making it possible to compose them arbitrarily.  The operators in Table \ref{tab:pangaea:searchspace} were chosen to create a large and expressive search space that contains activation functions unlikely to be discovered by hand.  Indeed, all piecewise real analytic functions can be represented with a PANGAEA computation graph (Theorem \ref{thm:pangaea:piecewise_real_analytic} of Section \ref{sec:pangaea:proofs}).

\begin{table*}
    \centering
    \caption{The operator search space consists of basic unary and binary functions as well as existing activation functions (Section \ref{sec:pangaea:baseline}).  $\sigma(x) = (1+e^{-x})^{-1}$. The unary operators bessel\_i0e and bessel\_i1e are the exponentially scaled modified Bessel functions of order 0 and 1, respectively.\newline}

    \begin{adjustbox}{max width=\textwidth}
    \begin{tabular}{lllllllll} \toprule 
        \multicolumn{7}{c}{\textbf{Unary}} & \multicolumn{2}{c}{\textbf{Binary}} \\ \midrule
        $0$           & $|x|$    & $\textrm{erf}(x)$  & $\textrm{tanh}(x)$ & $\textrm{arcsinh}(x)$ & $\textrm{ReLU}(x)$  & $\textrm{Softplus}(x)$    & $x_1 + x_2$ & $x_1^{x_2}$ \\
        $1$           & $x^{-1}$ & $\textrm{erfc}(x)$ & $e^x-1$ & $\textrm{arctan}(x)$ & $\textrm{ELU}(x)$   & $\textrm{Softsign}(x)$    & $x_1 - x_2$ & $\max\{x_1, x_2\}$ \\
        $x$           & $x^2$    & $\textrm{sinh}(x)$ & $\sigma(x)$ & $\textrm{bessel\_i0e}(x)$ & $\textrm{SELU}(x)$  & $\textrm{HardSigmoid}(x)$ & $x_1 \cdot x_2$ & $\min\{x_1, x_2\}$\\
        $-x$          & $e^x$    & $\textrm{cosh}(x)$ & $\log(\sigma(x))$ & $\textrm{bessel\_i1e}(x)$ & $\textrm{Swish}(x)$ &                           & $x_1 / x_2$ & \\
        \bottomrule
    \end{tabular}
    \end{adjustbox}
    \label{tab:pangaea:searchspace}
\end{table*}

PANGAEA begins with an initial population of $P$ random activation functions.  Each function is either of the form $f(x) = \texttt{unary1}(\texttt{unary2}(x))$ or $f(x) = \texttt{binary}(\texttt{unary1}(x), \allowbreak \texttt{unary2}(x))$, as shown in Figure~\ref{fig:pangaea:initialization}.  Both forms are equally likely, and the unary and binary operators are also selected uniformly at random.  The computation graphs in Figure~\ref{fig:pangaea:initialization} represent the simplest non-trivial computation graphs with and without a binary operator.  This design choice is inspired by previous work in neuroevolution, which demonstrated the power of starting from simple structures and gradually complexifying them \cite{stanley2002evolving}.

\begin{figure}
    \centering
    \begin{tikzpicture}[node distance=3em]
    
        \node (output2) [output] {$f(x)$};
        \node (unary4) [unary, below of=output2] {Unary};
        \node (unary3) [unary, below of=unary4] {Unary};
        \node (input3) [input, below of=unary3] {$x$};

        \draw [arrow] (input3) -- (unary3);
        \draw [arrow] (unary3) -- (unary4);
        \draw [arrow] (unary4) -- (output2);

        \node (output) [output, right of=output2, xshift=4em] {$f(x)$};
        \node (binary) [binary, below of=output] {Binary};
        \node (unary1) [unary, below of=binary, xshift=-2em] {Unary};
        \node (unary2) [unary, below of=binary, xshift=2em] {Unary};
        \node (input1) [input, below of=unary1] {$x$};
        \node (input2) [input, below of=unary2] {$x$};
        
        \draw [arrow] (input1) -- (unary1);
        \draw [arrow] (input2) -- (unary2);
        \draw [arrow] (unary1) -- (binary);
        \draw [arrow] (unary2) -- (binary);
        \draw [arrow] (binary) -- (output);

    \end{tikzpicture}
    \caption{Random activation function initialization. The initial population consists of random samples of two kinds of computation graphs, randomly initialized with the operators in Table~\ref{tab:pangaea:searchspace}. In this manner, the search starts with simple graphs and gradually expands to more complex forms.
    }
    
    \label{fig:pangaea:initialization}
\end{figure}
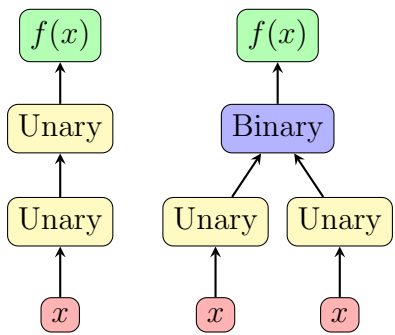

During the search, all ReLU activation functions in a given neural network are replaced with a candidate activation function.  No other changes to the network or training setup are made.  The network is trained on the dataset, and the activation function is assigned a fitness score equal to the network's accuracy on the validation set.

Given a parent activation function, a child activation function is created by applying one of four possible mutations (Figure~\ref{fig:pangaea:mutation}).  The possible mutations include elementary graph modifications like inserting, removing, or changing a node.  These mutations are useful for local exploration.  A special ``regenerate'' mutation is also introduced to accelerate exploration.  Other possible evolutionary operators like crossover are not used in this chapter.  All mutations are equally likely with two special cases.  If a remove mutation is selected for an activation function with just one node, a change mutation is applied instead.  Additionally, if an activation function with greater than seven nodes is selected, the mutation is a remove mutation, in order to reduce bloat.

\paragraph{Insert} In an insert mutation, one operator in the search space is selected uniformly at random.  This operator is placed on a random edge of a parent activation function graph.  In Figure~\ref{fig:pangaea:mutation}\emph{b}, the unary operator $\textrm{Swish}(x)$ is inserted at the edge connecting the output of $\tanh(x)$ to the input of $x_1 + x_2$.  After mutating, the parent activation function $(\tanh(x) + |\textrm{erf}(x)|)^2$ produces the child activation function $(\textrm{Swish}(\tanh(x)) + |\textrm{erf}(x)|)^2$.  If a binary operator is randomly chosen for the insertion, the incoming input value is assigned to the variable $x_1$.  If the operator is addition or subtraction, the input to $x_2$ is set to $0$.  If the operator is multiplication, division, or exponentiation, the input to $x_2$ is set to $1$.  Finally, if the operator is the maximum or minimum operator, the input to $x_2$ is a copy of the input to $x_1$.  When a binary operator is inserted into a computation graph, the activation function computed remains unchanged.  However, the structure of the computation graph is modified and can be further altered by future mutations.

\paragraph{Remove} In a remove mutation, one node is selected uniformly at random and deleted.  The node's input is rewired to its output.  If the removed node is binary, one of the two inputs is chosen at random and is deleted.  The other input is kept.  In Figure~\ref{fig:pangaea:mutation}\emph{c}, the addition operator is removed from the parent activation function.  The two inputs to addition, $\tanh(x)$ and $|\textrm{erf}(x)|$, cannot both be kept.  By chance, $\tanh(x)$ is discarded, resulting in the child activation function $|\textrm{erf}(x)|^2$.  

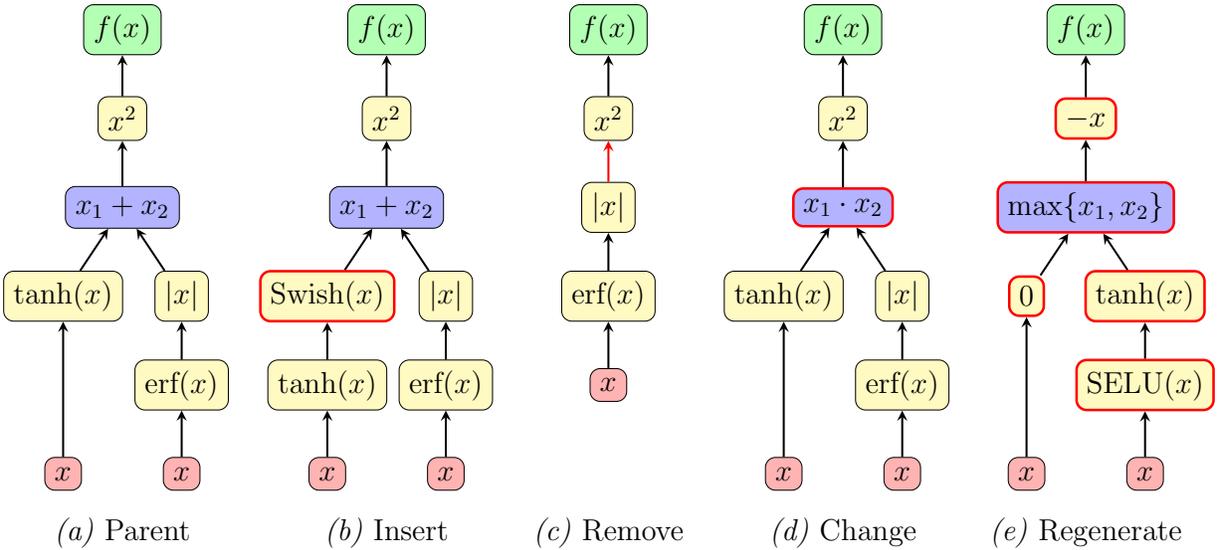
\begin{figure}
    \centering
    \begin{adjustbox}{max width=\linewidth}
    \begin{tabular}{ccccc}
    \begin{tikzpicture}[node distance=3em]
    
        \node [output] (out1) {$f(x)$};
        \node [unary, below of=out1] (u1) {$x^2$};
        \node [binary, below of=u1] (b1) {$x_1 + x_2$};
        \node [unary, below of=b1, xshift=-2em] (u2) {$\textrm{tanh}(x)$};
        \node [unary, below of=b1, xshift=2em] (u3) {$|x|$};
        \node [unary, below of=u3] (u4) {$\textrm{erf}(x)$};
        \node [input, below of=u4] (i1) {$x$};
        \node [input, left of=i1, xshift=-1em] (i2) {$x$};
        
        \node [invisible, below of=out1, yshift=-14em] (d) {\normalsize \emph{(a)} Parent};
        
        \draw [arrow] (i1) -- (u4);
        \draw [arrow] (i2) -- (u2);
        \draw [arrow] (u4) -- (u3);
        \draw [arrow] (u3) -- (b1);
        \draw [arrow] (u2) -- (b1);
        \draw [arrow] (b1) -- (u1);
        \draw [arrow] (u1) -- (out1);
    
    \end{tikzpicture}
    
    &
    
    \begin{tikzpicture}[node distance=3em]
    
        \node [output] (out1) {$f(x)$};
        \node [unary, below of=out1] (u1) {$x^2$};
        \node [binary, below of=u1] (b1) {$x_1 + x_2$};
        \node [unary, below of=b1, xshift=-2em, draw=red, line width=1] (swish) {$\textrm{Swish}(x)$};
        \node [unary, below of=swish] (u2) {$\textrm{tanh}(x)$};
        \node [unary, below of=b1, xshift=2em] (u3) {$|x|$};
        \node [unary, below of=u3] (u4) {$\textrm{erf}(x)$};
        \node [input, below of=u4] (i1) {$x$};
        \node [input, left of=i1, xshift=-1em] (i2) {$x$};
        
        \node [invisible, below of=out1, yshift=-14em] (d) {\normalsize \emph{(b)} Insert};
        
        \draw [arrow] (i1) -- (u4);
        \draw [arrow] (i2) -- (u2);
        \draw [arrow] (u4) -- (u3);
        \draw [arrow] (u3) -- (b1);
        \draw [arrow] (u2) -- (swish);
        \draw [arrow] (swish) -- (b1);
        \draw [arrow] (b1) -- (u1);
        \draw [arrow] (u1) -- (out1);
    
    \end{tikzpicture}
    
    &
    
    \begin{tikzpicture}[node distance=3em]
    
        \node [output] (out1) {$f(x)$};
        \node [unary, below of=out1] (u1) {$x^2$};
        \node [unary, below of=u1] (u3) {$|x|$};
        \node [unary, below of=u3] (u4) {$\textrm{erf}(x)$};
        \node [input, below of=u4] (i1) {$x$};
        
        \node [invisible, below of=out1, yshift=-14em] (d) {\normalsize \emph{(c)} Remove};
        
        \draw [arrow] (i1) -- (u4);
        \draw [arrow] (u4) -- (u3);
        \draw [arrow, draw=red] (u3) -- (u1);
        \draw [arrow] (u1) -- (out1);
    
    \end{tikzpicture}
    
    &
    
    \begin{tikzpicture}[node distance=3em]
    
        \node [output] (out1) {$f(x)$};
        \node [unary, below of=out1] (u1) {$x^2$};
        \node [binary, below of=u1, draw=red, line width=1] (b1) {$x_1 \cdot x_2$};
        \node [unary, below of=b1, xshift=-2em] (u2) {$\textrm{tanh}(x)$};
        \node [unary, below of=b1, xshift=2em] (u3) {$|x|$};
        \node [unary, below of=u3] (u4) {$\textrm{erf}(x)$};
        \node [input, below of=u4] (i1) {$x$};
        \node [input, left of=i1, xshift=-1em] (i2) {$x$};
        
        \node [invisible, below of=out1, yshift=-14em] (d) {\normalsize \emph{(d)} Change};
        
        \draw [arrow] (i1) -- (u4);
        \draw [arrow] (i2) -- (u2);
        \draw [arrow] (u4) -- (u3);
        \draw [arrow] (u3) -- (b1);
        \draw [arrow] (u2) -- (b1);
        \draw [arrow] (b1) -- (u1);
        \draw [arrow] (u1) -- (out1);
    
    \end{tikzpicture}
    
    &

    \begin{tikzpicture}[node distance=3em]
    
        \node [output] (out1) {$f(x)$};
        \node [unary, below of=out1, draw=red, line width=1] (u1) {$-x$};
        \node [binary, below of=u1, draw=red, line width=1] (b1) {$\max\{x_1, x_2\}$};
        \node [unary, below of=b1, xshift=-2em, draw=red, line width=1] (u2) {$0$};
        \node [unary, below of=b1, xshift=2em, draw=red, line width=1] (u3) {$\textrm{tanh}(x)$};
        \node [unary, below of=u3, draw=red, line width=1] (u4) {$\textrm{SELU}(x)$};
        \node [input, below of=u4] (i1) {$x$};
        \node [input, left of=i1, xshift=-1em] (i2) {$x$};
        
        \node [invisible, below of=out1, yshift=-14em] (d) {\normalsize \emph{(e)} Regenerate};
        
        \draw [arrow] (i1) -- (u4);
        \draw [arrow] (i2) -- (u2);
        \draw [arrow] (u4) -- (u3);
        \draw [arrow] (u3) -- (b1);
        \draw [arrow] (u2) -- (b1);
        \draw [arrow] (b1) -- (u1);
        \draw [arrow] (u1) -- (out1);
    
    \end{tikzpicture}

    \end{tabular}
    \end{adjustbox}
    \caption{Evolutionary operations on activation functions. In an `Insert' mutation, a new operator is inserted in one of the edges of the computation graph, like the Swish$(x)$ in \emph{(b)}. In a `Remove' mutation, a node in the computation graph is deleted, like the addition in \emph{(c)}. In a `Change' mutation, an operator at a node is replaced with another, like addition with multiplication in \emph{(d)}. These first three mutations are useful in refining the function locally. In contrast, in a `Regenerate' mutation \emph{(e)}, every operator in the graph is replaced by a random operator, thus increasing exploration.}
    \label{fig:pangaea:mutation}
\end{figure}

\paragraph{Change} To perform a change mutation, one node in the computation graph is selected at random and replaced with another operator from the search space, also uniformly at random.  Unary operators are always replaced with unary operators, and binary operators with binary operators.  Figure~\ref{fig:pangaea:mutation}\emph{d} shows how changing addition to multiplication produces the activation function $(\tanh(x) \cdot |\textrm{erf}(x)|)^2$.

\paragraph{Regenerate} In a regenerate mutation, every operator in the computation graph is replaced with another operator from the search space.  As with change mutations, unary operators are replaced with unary operators, and binary operators with binary operators.  Although every node in the graph is changed, the overall structure of the computation graph remains the same.  Regenerate mutations are useful for increasing exploration, and are similar in principle to burst mutation and delta coding \citep{gomezgecco03,whitleyicga91}.  Figure~\ref{fig:pangaea:mutation}\emph{e} shows the child activation function $-\max\{0, \tanh(\textrm{SELU}(x))\}$, which is quite different from the parent function in Figure~\ref{fig:pangaea:mutation}\emph{a}.

\paragraph{Parameterization of Activation Functions}

After mutation (or random initialization), activation functions are
parameterized (Figure~\ref{fig:pangaea:parameterization}).  A value
$k \in \{0, 1, 2, 3\}$ is chosen uniformly at random, and $k$ edges of
the activation function graph are randomly selected.
Multiplicative per-channel parameters are inserted at these edges and
initialized to one.  Whereas evolution is well suited for discovering the
general form of the activation function in a discrete, structured
search space, parameterization makes it possible to fine-tune the
function using gradient descent.
The function parameters are updated at	every epoch during backpropagation, resulting in
different activation functions in different stages of training.  As the parameters are per-channel, the process
creates different activation functions at different locations in the
neural network.  Thus, parameterization gives neural networks additional flexibility to customize activation functions.

\begin{figure}[ht]
    \centering
    \begin{multicols}{2}
    \null \vfill 
     \begin{tikzpicture}[node distance=3em]
    
        \node (output) [output] {$f(x)$};
        \node (unary) [unary, below of=output] {$\sigma(x)$};
        \node (binary) [binary, below of=unary] {$x_1 - x_2$};
        \node (unary1) [unary, below of=binary, xshift=-2em] {$|x|$};
        \node (unary2) [unary, below of=binary, xshift=2em] {$\textrm{arctan}(x)$};
        \node (input1) [input, below of=unary1] {$x$};
        \node (input2) [input, below of=unary2] {$x$};

        \draw [arrow] (input1) -- (unary1);
        \draw [arrow, draw=red] (input2) -- (unary2);
        \draw [arrow, draw=red] (unary1) -- (binary);
        \draw [arrow] (unary2) -- (binary);
        \draw [arrow] (binary) -- (unary);
        \draw [arrow, draw=red] (unary) -- (output);
        
    \end{tikzpicture}
    
    \vfill \null
    \begin{tikzpicture}[node distance=3em]
    
        \node (output) [output] {$f(x)$};
        \node (p1) [maybeparam, below of=output] {$\alpha$};
        \node (unary) [unary, below of=p1] {$\sigma(x)$};
        \node (binary) [binary, below of=unary] {$x_1 - x_2$};
        \node (p5) [maybeparam, below of=binary, xshift=-2em] {$\beta$};
        \node (unary1) [unary, below of=p5] {$|x|$};
        \node (unary2) [unary, below of=binary, xshift=2em] {$\textrm{arctan}(x)$};
        \node (p4) [maybeparam, below of=unary2] {$\gamma$};
        \node (input1) [input, below of=unary1] {$x$};
        \node (input2) [input, below of=p4] {$x$};

        \draw [arrow] (input1) -- (unary1);
        \draw [arrow] (input2) -- (p4);
        \draw [arrow] (p4) -- (unary2); 
        \draw [arrow] (unary1) -- (p5);
        \draw [arrow] (p5) -- (binary);
        \draw [arrow] (unary2) -- (binary);
        \draw [arrow] (binary) -- (unary);
        \draw [arrow] (unary) -- (p1);
        \draw [arrow] (p1) -- (output);
        
    \end{tikzpicture}
    \end{multicols}
    \caption{Parameterization of activation functions.  In this example, parameters are added to $k=3$ random edges, yielding the parametric activation function $\alpha \sigma(\beta |x| - \textrm{arctan}(\gamma x))$.}
    \label{fig:pangaea:parameterization}
\end{figure}
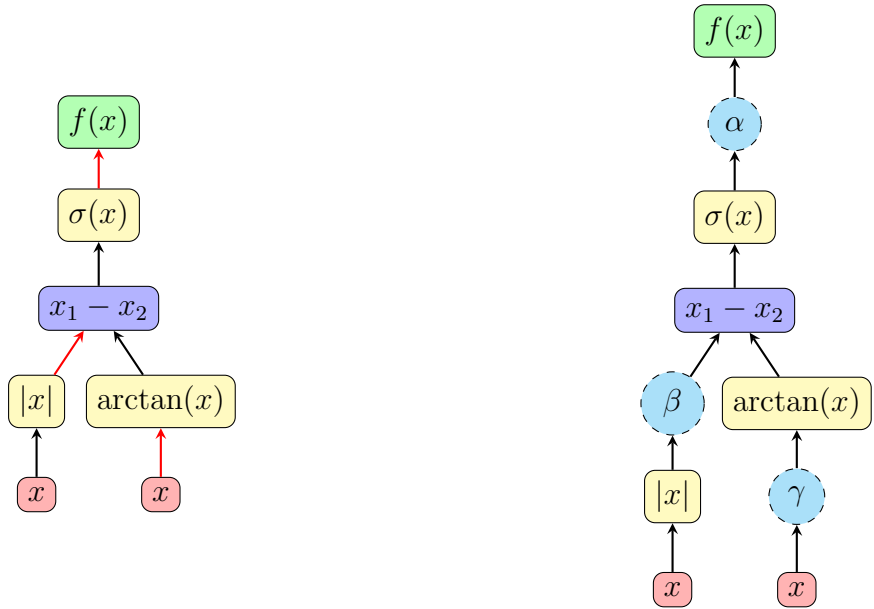

\subsection{Discovering Activation Functions with Evolution}
\label{sec:pangaea:evolution}

Activation functions are discovered by regularized evolution \citep{real2019regularized}.  Initially, $P$ random activation functions are created, parameterized, and assigned fitness scores.
To generate a new activation function, $S$ functions are sampled with
replacement from the current population.  The function with the
highest validation accuracy serves as the parent, and is mutated to create a child activation
function.  This function is parameterized and assigned a fitness
score.  The new activation function is then added to the population, and the
oldest function in the population is removed, ensuring the
population is always of size $P$.  This process continues until $C$ functions have been
evaluated in total, and the top functions over the history
of the search are returned as a result.  

Any activation function that
achieves a fitness score less than a threshold $V$ is discarded.
These functions are not added to the population, but they
do count towards the total number of $C$ activation functions
evaluated for each architecture.  This quality control mechanism
allows evolution to focus only on the most promising candidates.

To save computational resources during evolution, each activation function is evaluated by training a neural network for 100 epochs using a compressed learning rate schedule. After evolution is complete, the top 10 activation functions from the entire search are reranked.  Each function receives an adjusted fitness score equal to the average validation accuracy from two independent 200-epoch training runs using the original learning rate schedule.  The top three activation functions after reranking proceed to the final testing experiments.

During evolution, it is possible that some activation functions achieve unusually high validation accuracy by chance.  The 100-epoch compressed learning rate schedule may also have a minor effect on which activation functions are optimal compared to a full 200-epoch schedule.  Reranking thus serves two purposes.  Full training reduces bias from the compressed schedule, and averaging two such runs lessens the impact of activation functions that achieved high accuracy by chance.

\section{Datasets and Architectures}
\label{sec:pangaea:datasets_architectures}

The experiments in this chapter focus primarily on the CIFAR-100 image classification dataset \citep{krizhevsky2009learning}. This dataset is a more difficult version of the popular CIFAR-10 dataset, with 100 object categories instead of 10. Fifty images from each class were randomly selected from the training set to create a balanced validation set, resulting in a training/validation/test split of 45K/5K/10K images.

To demonstrate that PANGAEA can discover effective activation functions in various settings, it is evaluated with three different neural networks.  The models were implemented in TensorFlow \citep{abadi2016tensorflow}, mirroring the original authors' training setup as closely as possible (see Appendix \ref{ap:details:pangaea} for training details and Appendix \ref{ap:details:pangaea_custom} for code that shows how to train with custom activation functions).

{\bf Wide Residual Network}
\citep[WRN-10-4;][]{zagoruyko2016wide} has a depth of 10 and widening factor of four.  Wide residual networks provide an interesting comparison because they are shallower and wider than many other popular architectures, while still achieving good results.  WRN-10-4 was chosen because its CIFAR-100 accuracy is competitive, yet it trains relatively quickly.

{\bf Residual Network}
\citep[ResNet-v1-56;][]{he2016deep}, with a depth of 56, provides an important contrast to WRN-10-4.  It is significantly deeper and has a slightly different training setup, which may have an effect on the performance of different activation functions.  

{\bf Preactivation Residual Network}
\citep[ResNet-v2-56;][]{he2016identity} has identical depth to ResNet-v1-56, but is a fundamentally different architecture.  Activation functions are not part of the skip connections, as is the case in ResNet-v1-56.  Since information does not have to pass through an activation function, this structure makes it easier to train very deep architectures.  PANGAEA should exploit this structure and discover different activation functions for ResNet-v2-56 and ResNet-v1-56.

\section{Main Results}
\label{sec:pangaea:results}

This section demonstrates that PANGAEA is able to discover good activation functions for various architectures.  General activation functions that perform well on multiple architectures are found, but the best performance comes from activation functions that are specialized to a given neural network.

\subsection{Overview}

Separate evolution experiments were run to discover novel
activation functions for each of the three architectures. Evolutionary
parameters $P=64$, $S=16$, $C=1{,}000$, and $V=20\%$ were used since
they were found to work well in preliminary experiments.

Figure~\ref{fig:pangaea:evolution} visualizes progress in these experiments.
For all three architectures, PANGAEA quickly discovered
activation functions that outperform ReLU.  It continued to make
further progress, gradually discovering better activation functions, and did not
plateau during the time allotted for the experiment.  Each run
took approximately 1,000 GPU hours on GeForce GTX 1080 and 1080 Ti GPUs
(see Appendices \ref{ap:details:pangaea} and \ref{ap:infrastructure} for implementation and compute details).

\begin{figure}
    \centering
    \includegraphics[width=\linewidth]{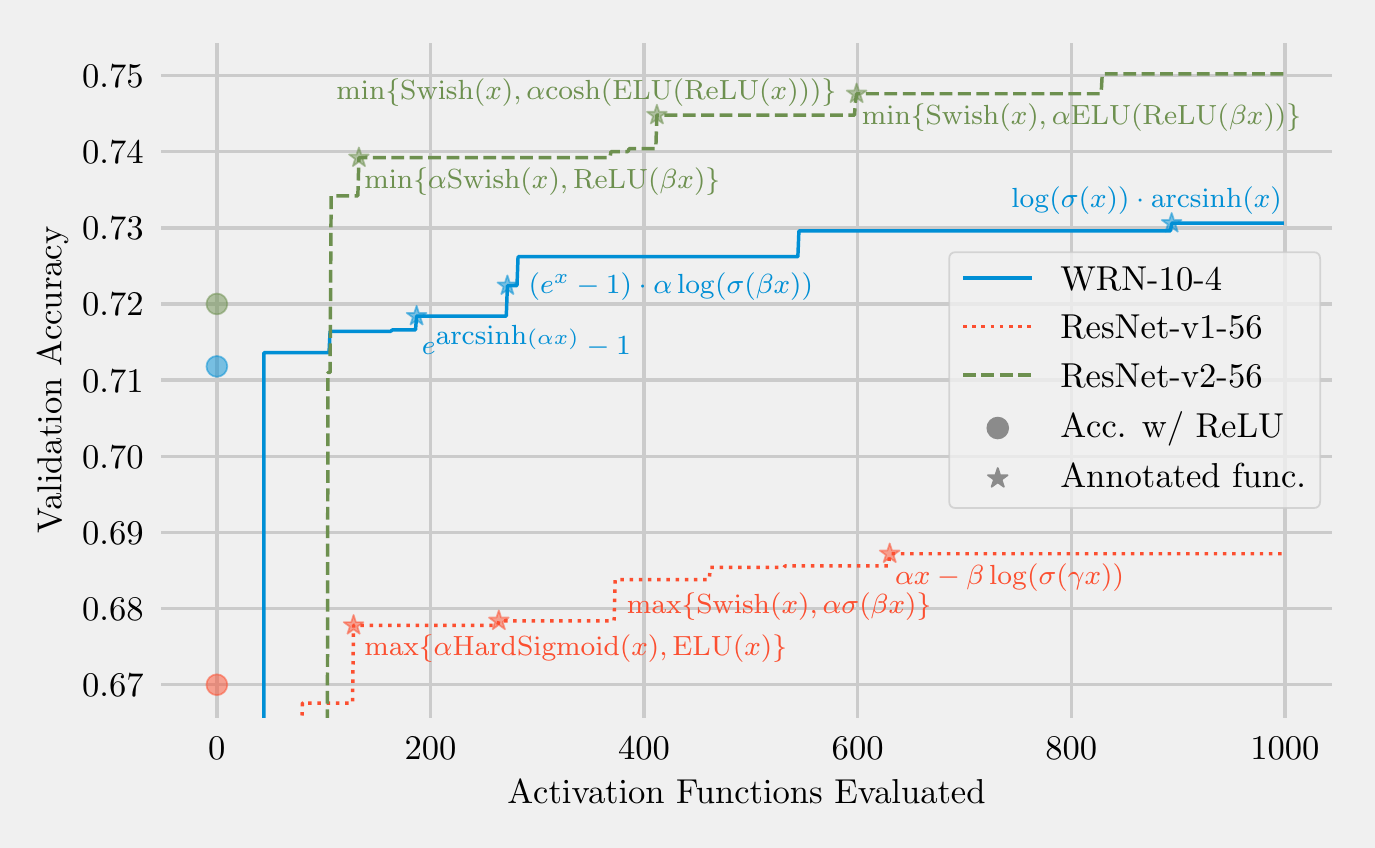}
    \caption{Progress of PANGAEA with three different neural networks.  The plots show the best accuracy achieved among all activation functions evaluated so far.  The stars on the plot specify the time when notable activation functions were discovered during evolution; the expression for each such function is written next to the star. Evolution quickly discovered activation functions that outperform      
      ReLU (accuracy with ReLU shown at $x=0$), 
      and continued to improve throughout the experiment.   Note that this figure shows validation accuracy, while Table~\ref{tab:pangaea:results} lists test set accuracy.}
    \label{fig:pangaea:evolution}
\end{figure}

Table \ref{tab:pangaea:results} shows the final test accuracy for the top specialized activation functions discovered by PANGAEA in each run. For comparison, the accuracy of the top general functions discovered in this process are also shown, as well as the accuracy of several baseline activation functions (see Section \ref{sec:pangaea:baseline} for baseline activation function details and Section \ref{sec:pangaea:pausplash} for additional results with learnable baseline functions). In sum, PANGAEA discovered the best activation function for each of the three architectures.

\begin{table*}
    \centering
    \renewcommand{\arraystretch}{.6666667} 
    \caption{CIFAR-100 test set accuracy aggregated over ten runs, shown as mean $\pm$ sample standard deviation.  Asterisks indicate a statistically significant improvement in mean accuracy over ReLU, with * if $p \leq 0.05$, ** if $p \leq 0.01$, and *** if $p \leq 0.001$;  $p$-values are from one-tailed Welch's $t$-tests.  The top accuracy for each architecture is in bold.  Baseline activation function details and references are given in Section \ref{sec:pangaea:baseline}.\newline
    }
    \centering
    \begin{adjustbox}{max width=\textwidth}
    \small
    \begin{tabular}{llll} \toprule
    & \textbf{WRN-10-4} & \textbf{ResNet-v1-56} & \textbf{ResNet-v2-56} \\ \midrule

    \textbf{Specialized for WRN-10-4}\\
    $\log(\sigma(\alpha x)) \cdot \beta \textrm{arcsinh}(x)$ & 
    \cellcolor{black!5} $\bm{73.20} {\scriptscriptstyle \pm 0.37~***}$ &
    $18.63 {\scriptscriptstyle \pm 21.04}$ &
    $45.88 {\scriptscriptstyle \pm 30.70}$ \\
    
    $\log(\sigma(\alpha x)) \cdot \textrm{arcsinh}(x)$ & 
    \cellcolor{black!5} $73.16 {\scriptscriptstyle \pm 0.41~***}$ &
    $19.34 {\scriptscriptstyle \pm 20.14}$ &
    $64.30 {\scriptscriptstyle \pm 21.32}$ \\
    
    $-\textrm{Swish}(\textrm{Swish}(\alpha x))$ & 
    \cellcolor{black!5} $72.49 {\scriptscriptstyle \pm 0.55~***}$ &
    $58.86 {\scriptscriptstyle \pm 2.88}$ &
    $74.71 {\scriptscriptstyle \pm 0.20~*}$ \\ 
    \midrule 

    \textbf{Specialized for ResNet-v1-56}\\
    $\alpha x - \beta \log(\sigma(\gamma x))$ & 
    $70.28 {\scriptscriptstyle \pm 0.37}$ &
    \cellcolor{black!5} $\bm{71.01} {\scriptscriptstyle \pm 0.64~***}$ &
    $74.35 {\scriptscriptstyle \pm 0.45}$ \\
    
    $\alpha x - \log(\sigma(\beta x))$ & 
    $70.47 {\scriptscriptstyle \pm 0.53}$ &
    \cellcolor{black!5} $70.30 {\scriptscriptstyle \pm 0.58~*}$ &
    $74.70 {\scriptscriptstyle \pm 0.23~*}$ \\
    
    $\max\{\textrm{Swish}(x), 0\}$ & 
    $72.10 {\scriptscriptstyle \pm 0.33~**}$ &
    \cellcolor{black!5} $69.43 {\scriptscriptstyle \pm 0.69}$ &
    $74.97 {\scriptscriptstyle \pm 0.25~**}$ \\
    \midrule 
    
    \textbf{Specialized for ResNet-v2-56}\\
    $\textrm{Softplus}(\textrm{ELU}(x))$ & 
    $71.36 {\scriptscriptstyle \pm 0.34}$ &
    $69.96 {\scriptscriptstyle \pm 0.39}$ &
    \cellcolor{black!5} $\bm{75.61} {\scriptscriptstyle \pm 0.42~***}$ \\
    
    $\min\{\log(\sigma(x)), \alpha \log(\sigma(\beta x))\}$ & 
    $72.04 {\scriptscriptstyle \pm 0.34~**}$ &
    $69.56 {\scriptscriptstyle \pm 0.48}$ &
    \cellcolor{black!5} $75.19 {\scriptscriptstyle \pm 0.39~***}$ \\
    
    $\textrm{SELU}(\textrm{Swish}(x))$ & 
    $01.00 {\scriptscriptstyle \pm 0.00}$ &
    $01.00 {\scriptscriptstyle \pm 0.00}$ &
    \cellcolor{black!5} $75.02 {\scriptscriptstyle \pm 0.35~**}$ \\
    \midrule 
    
    \textbf{General Activation Functions}\\
    $\max\{\textrm{Swish}(x), \alpha \log (\sigma (\textrm{ReLU}(x)))\}$ & 
     $72.54 {\scriptscriptstyle \pm 0.26~***}$ &
     $69.91 {\scriptscriptstyle \pm 0.37}$ &
     $75.20 {\scriptscriptstyle \pm 0.41~***}$ \\
    
    $\min\{\textrm{Swish}(x), \alpha \textrm{ELU}(\textrm{ReLU}(\beta x))\}$ & 
     $72.39 {\scriptscriptstyle \pm 0.29~***}$ &
     $69.82 {\scriptscriptstyle \pm 0.40}$ &
     $75.27 {\scriptscriptstyle \pm 0.38~***}$ \\
    
    $\log(\sigma(x))$ & 
     $72.33 {\scriptscriptstyle \pm 0.32~***}$ &
     $69.58 {\scriptscriptstyle \pm 0.35}$ &
     $75.53 {\scriptscriptstyle \pm 0.37~***}$ \\
    \midrule
    
    \textbf{Fixed Baseline Functions}\\
    $\textrm{ReLU}$ &
    $71.46 {\scriptscriptstyle \pm 0.50}$ &
    $69.64 {\scriptscriptstyle \pm 0.65}$ &
    $74.39 {\scriptscriptstyle \pm 0.44}$ \\

    $\textrm{ELiSH}$ & 
    $01.00 {\scriptscriptstyle \pm 0.00}$ &
    $01.00 {\scriptscriptstyle \pm 0.00}$ &
    $75.20 {\scriptscriptstyle \pm 0.31~***}$ \\
    
    $\textrm{ELU}$ & 
    $72.30 {\scriptscriptstyle \pm 0.32~***}$ &
    $69.67 {\scriptscriptstyle \pm 0.46}$ &
    $74.95 {\scriptscriptstyle \pm 0.30~**}$ \\
    
    $\textrm{GELU}$ & 
    $71.95 {\scriptscriptstyle \pm 0.35~*}$ &
    $70.19 {\scriptscriptstyle \pm 0.40~*}$ &
    $74.86 {\scriptscriptstyle \pm 0.33~**}$ \\
    
    $\textrm{HardSigmoid}$ & 
    $54.99 {\scriptscriptstyle \pm 1.00}$ &
    $32.55 {\scriptscriptstyle \pm 4.06}$ &
    $64.90 {\scriptscriptstyle \pm 0.69}$ \\
    
    $\textrm{Leaky ReLU}$ & 
    $71.73 {\scriptscriptstyle \pm 0.33}$ &
    $69.78 {\scriptscriptstyle \pm 0.33}$ &
    $74.73 {\scriptscriptstyle \pm 0.35~*}$ \\
    
    $\textrm{Mish}$ & 
    $71.95 {\scriptscriptstyle \pm 0.41~*}$ &
    $69.88 {\scriptscriptstyle \pm 0.54}$ &
    $75.32 {\scriptscriptstyle \pm 0.29~***}$ \\
    
    $\textrm{SELU}$ & 
    $70.53 {\scriptscriptstyle \pm 0.42}$ &
    $68.52 {\scriptscriptstyle \pm 0.29}$ &
    $73.79 {\scriptscriptstyle \pm 0.36}$ \\
    
    $\textrm{sigmoid}$ & 
    $56.10 {\scriptscriptstyle \pm 0.98}$ &
    $36.47 {\scriptscriptstyle \pm 3.32}$ &
    $66.45 {\scriptscriptstyle \pm 0.92}$ \\
    
    $\textrm{Softplus}$ & 
    $72.27 {\scriptscriptstyle \pm 0.26~***}$ &
    $69.71 {\scriptscriptstyle \pm 0.36}$ &
    $75.46 {\scriptscriptstyle \pm 0.52~***}$ \\
    
    $\textrm{Softsign}$ & 
    $56.30 {\scriptscriptstyle \pm 2.16}$ &
    $58.38 {\scriptscriptstyle \pm 0.96}$ &
    $69.33 {\scriptscriptstyle \pm 0.39}$ \\
    
    $\textrm{Swish}$ & 
    $72.26 {\scriptscriptstyle \pm 0.28~***}$ &
    $69.68 {\scriptscriptstyle \pm 0.38}$ &
    $75.08 {\scriptscriptstyle \pm 0.36~***}$ \\
    
    $\textrm{tanh}$ & 
    $56.52 {\scriptscriptstyle \pm 1.53}$ &
    $63.88 {\scriptscriptstyle \pm 0.38}$ &
    $70.44 {\scriptscriptstyle \pm 0.40}$ \\
    \midrule
    
    \textbf{Parametric Baseline Functions}\\
    
    $\textrm{PReLU}$ &
    $72.23 {\scriptscriptstyle \pm 0.37~***}$ &
    $69.77 {\scriptscriptstyle \pm 0.40}$ &
    $75.10 {\scriptscriptstyle \pm 0.53~**}$ \\

    $\textrm{PSwish} = x \cdot \sigma(\beta x)$ & 
    $72.40 {\scriptscriptstyle \pm 0.31~***}$ &
    $70.16 {\scriptscriptstyle \pm 0.46~*}$ &
    $75.39 {\scriptscriptstyle \pm 0.28~***}$ \\
    \midrule
    
    \textbf{Learnable Baseline Functions}\\
    
    $\textrm{APL}$ &
    $72.88 {\scriptscriptstyle \pm 0.32~***}$ &
    $70.81 {\scriptscriptstyle \pm 0.20~***}$ &
    $75.02 {\scriptscriptstyle \pm 0.28~***}$ \\

    $\textrm{PAU}$ &
    $41.46 {\scriptscriptstyle \pm 22.66}$ &
    $01.00 {\scriptscriptstyle \pm 0.00}$ &
    $02.38 {\scriptscriptstyle \pm 4.36}$ \\
    
    $\textrm{SPLASH}$ &
    $72.16 {\scriptscriptstyle \pm 0.81~*}$ &
    $01.00 {\scriptscriptstyle \pm 0.00}$ &
    $73.45 {\scriptscriptstyle \pm 0.43}$ \\
    
    \bottomrule
    \end{tabular}
    \end{adjustbox}
    \label{tab:pangaea:results}
\end{table*}

\subsection{Specialized Activation Functions}
For all three architectures, there are baseline activation functions that outperform ReLU by statistically significant margins.  This result already demonstrates that activation functions should be chosen carefully, and that the common practice of using ReLU by default is suboptimal.  Furthermore, the best baseline activation function is different for different architectures, suggesting that specializing activation functions to the architecture is a good approach.

Because PANGAEA uses validation accuracy from a single neural network to assign fitness scores to activation functions, there is selective pressure to discover functions that exploit the structure of the network.  The functions thus become specialized to the architecture. They increase the performance of that architecture; however, they may not be as effective with other architectures.  Specialized activation function accuracies are highlighted with the gray background in Table \ref{tab:pangaea:results}.  To verify that the functions are customized to a specific architecture, the functions were cross-evaluated with other architectures.

PANGAEA discovered two specialized activation functions for WRN-10-4 that outperformed all baseline functions by a statistically significant margin ($p \leq 0.05$).  The top specialized function on ResNet-v1-56 also significantly outperformed all baseline functions, except APL (for which $p = 0.19$).  The top specialized activation function on ResNet-v2-56 similarly significantly outperformed all except Softplus ($p = 0.25$) and PSwish ($p = 0.09$). These results strongly demonstrate the power of customizing activation functions to architectures.  Indeed, specializing activation functions is a new dimension of activation function search not considered by previous work \cite{DBLP:conf/iclr/RamachandranZL18, basirat2018quest}.

\subsection{General Activation Functions}
Although the best performance comes from specialization, it is also useful to discover activation functions that achieve high accuracy across multiple architectures. For instance, they could be used initially on a new architecture before spending compute on specialization. A powerful albeit computationally demanding approach would be to evolve general functions directly, by evaluating candidates on multiple architectures during evolution. However, it turns out that each specialized evolution run already generates a variety of functions, many of which are general.

To evaluate whether the PANGAEA runs discovered general functions as well, the top 10 functions from each run were combined into a pool of 30 candidate functions.  Each candidate was assigned three fitness scores equal to the average validation accuracy from two independent training runs on each of the three architectures.  Candidate functions that were Pareto-dominated, were functionally equivalent to one of the baseline activation functions, or had already been selected as a specialized activation function were discarded, leaving three Pareto-optimal general activation functions.

These functions indeed turned out to be effective as general activation functions.  All three achieved good accuracy with ResNet-v1-56 and significantly outperformed ReLU with WRN-10-4 and ResNet-v2-56.  However, specialized activation functions, i.e.\ those specifically evolved for each architecture, still give the biggest improvements. 

\subsection{Shapes of Discovered Functions}

Many of the top discovered activation functions are compositions of
multiple unary operators.  These functions do not exist in the core
unit search space of \citet{DBLP:conf/iclr/RamachandranZL18}, which
requires binary operators.  They also do not exist in the $S_1$ or
$S_2$ search spaces in CAFE, which are too shallow.  
The design of the search space is therefore as
important as the search algorithm itself.  Previous search spaces that
rely on repeated fixed building blocks only have limited
representational power. In contrast, PANGAEA utilizes a flexible
search space that can represent activation functions in an arbitrary
computation graph (see Section \ref{sec:pangaea:searchspace} for an analysis on the size of the PANGAEA search space). 

Furthermore, while the learnable baseline functions can in principle approximate the functions discovered by PANGAEA, they do not consistently match its performance.  PANGAEA utilizes both evolutionary search and gradient descent to discover activation functions, and apparently this combination of optimization processes is more powerful than gradient descent alone.

\begin{figure}
    \centering
    \includegraphics[clip, trim=1.5em 0em 4.5em 0em, width=\linewidth]{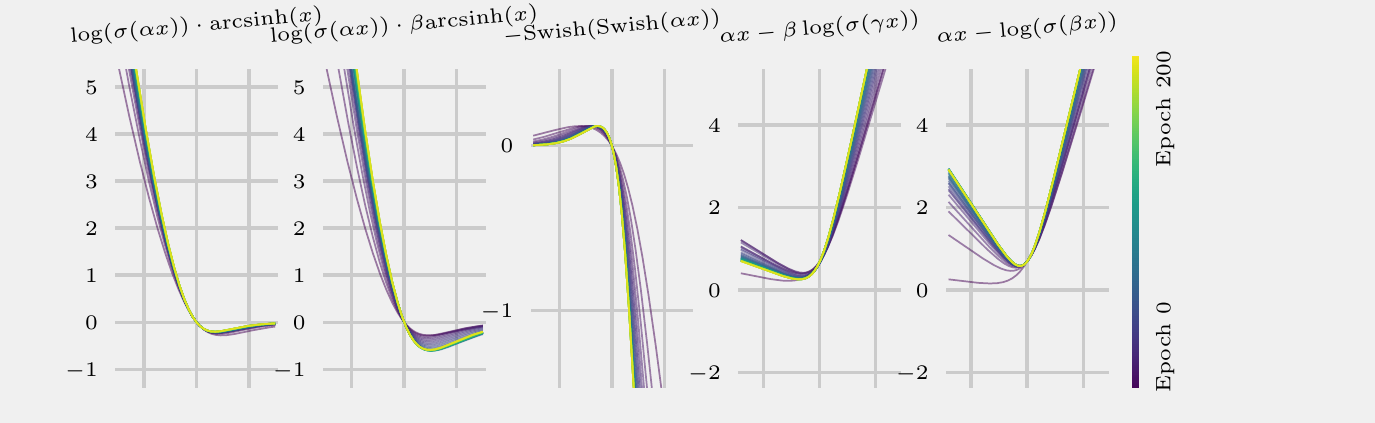}\\
    \vspace{-0.5em}
    \includegraphics[clip, trim=1.5em 0.5em 4.5em 1em, width=\linewidth]{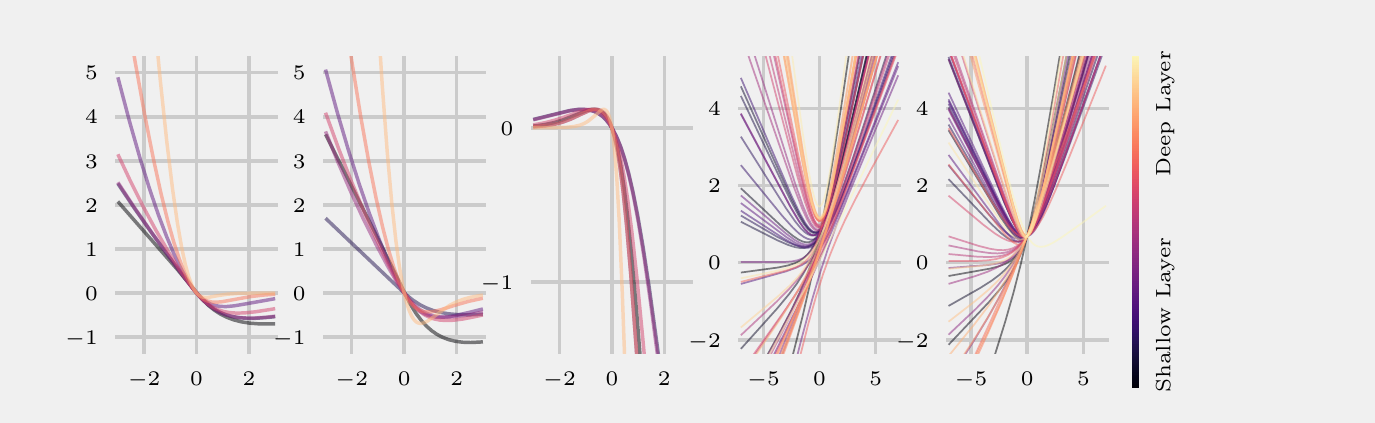}
    \vspace{-1.5em}
    \caption{Adaptation of parametric activation functions over time and        
      space. \textbf{Top:} The parameters change during training,           
      resulting in different activation functions in the early and late         
      stages. The plots were created by averaging the values of $\alpha$,       
      $\beta$, and $\gamma$ across the entire network at different training epochs. \textbf{Bottom:} The           
      parameters are updated separately in each channel, inducing               
      different activation functions at different locations of a neural         
      network. The plots were created by averaging $\alpha$, $\beta$, and       
      $\gamma$ at each layer of the network after the completion of             
      training.}
    \label{fig:pangaea:varying}
\end{figure}

Figure~\ref{fig:pangaea:varying} shows examples of parametric activation
functions discovered by PANGAEA.  As training progresses, gradient
descent makes small adjustments to the function parameters $\alpha$,
$\beta$, and $\gamma$, resulting in activation functions that change
over time.  This result suggests that
it is advantageous to have one activation function in the early stages
of training when the network learns rapidly, and a different
activation function in the later stages of training when the network
is focused on fine-tuning. The parameters $\alpha$, $\beta$, and
$\gamma$ are also learned separately for the different channels,
resulting in activation functions that vary with location in a neural
network. Functions in deep layers
(near the output) are more nonlinear than those in shallow
layers (closer to the input), possibly contrasting the need to
form regularized embeddings with the need to form categorizations. In
this manner, PANGAEA customizes	the activation functions to both time
and space for each architecture.

\section{Diving Deeper: Experiments on Ablations, Variations, and Training Dynamics}
\label{sec:pangaea:ablations}

PANGAEA is a method with many moving parts.  The main results from Section \ref{sec:pangaea:results} already showed the power of PANGAEA, and the experiments from this section illuminate how each component of PANGAEA contributes to its success.  Evolutionary search and gradient descent working in tandem provided a better strategy than either optimization algorithm alone (Section \ref{sec:pangaea:baselinesearchstrategies}).  The top activation functions benefitted from their learnable parameters (Section \ref{sec:pangaea:effectofparameterization}), but baseline functions did not (Section \ref{sec:pangaea:parametric_baseline}), showing how PANGAEA discovered functional forms well-suited to parameterization.  PANGAEA is robust: the activation functions it discovers transfer to larger networks (Section \ref{sec:pangaea:scalingup}) and PANGAEA achieves impressive results with a new dataset and architecture (Section \ref{sec:pangaea:allcnnc}).  The activation functions discovered by PANGAEA improve accuracy by easing optimization and implicitly regularizing the network (Section \ref{sec:pangaea:dynamics}).

\subsection{Additional Baseline Search Strategies}
\begin{table}[ht]
    \centering
    \caption{WRN-10-4 accuracy with different activation functions on CIFAR-100, shown as mean $\pm$ sample standard deviation across ten runs.  PANGAEA discovers better activation functions than random search and nonparametric evolution.\newline}
    \begin{adjustbox}{max width=\linewidth}
    \begin{tabular}{ll} \toprule
        \textbf{PANGAEA} \\
        $\log(\sigma(\alpha x)) \cdot \beta \textrm{arcsinh}(x)$ & 
        $\bm{73.20} {\scriptscriptstyle \pm 0.37}$ \\
        $\log(\sigma(\alpha x)) \cdot \textrm{arcsinh}(x)$ & 
        $73.16 {\scriptscriptstyle \pm 0.41}$ \\
        $-\textrm{Swish}(\textrm{Swish}(\alpha x))$ & 
        $72.49 {\scriptscriptstyle \pm 0.55}$ \\ 
        \midrule
        
        \textbf{Random Search}\\
        $\alpha \textrm{Swish}(x)$ & $72.85 {\scriptscriptstyle \pm 0.25}$ \\
        $\textrm{Softplus}(x) \cdot \arctan(\alpha x)$ & $72.81 {\scriptscriptstyle \pm 0.35}$ \\
        $\textrm{ReLU}(\alpha \textrm{arcsinh}(\beta \sigma(x))) \cdot \textrm{SELU}(\gamma x)$ & $72.69 {\scriptscriptstyle \pm 0.21}$ \\
        \midrule 
        
        \textbf{Nonparametric Evolution}\\
        $\cosh(1) \cdot \textrm{Swish}(x)$ & 
        $72.78 {\scriptscriptstyle \pm 0.24}$ \\    
        $(e^1-1) \cdot \textrm{Swish}(x)$ & 
        $72.52 {\scriptscriptstyle \pm 0.34}$ \\
        $\textrm{ReLU}(\textrm{Swish}(x))$ &
        $72.04 {\scriptscriptstyle \pm 0.54}$ \\
        \midrule
        
        ReLU & $71.46 {\scriptscriptstyle \pm 0.50}$ \\
        Swish & $72.26 {\scriptscriptstyle \pm 0.28}$ \\

    \bottomrule
    \end{tabular}
    \end{adjustbox}
    \label{tab:pangaea:search_strategy}
\end{table}
\label{sec:pangaea:baselinesearchstrategies}

As additional baseline comparisons, two alternative search strategies were used to discover activation functions for WRN-10-4.  First, a random search baseline was established by applying random mutations without regard to fitness values.  This approach corresponds to setting evolutionary parameters $P=1$, $S=1$, and  $V=0\%$.  Second, to understand the effects of function parameterization, a nonparametric evolution baseline was run.  This setting is identical to PANGAEA, except functions are not parameterized (Figure \ref{fig:pangaea:parameterization}).  Otherwise, both baselines follow the same setup as PANGAEA, including evaluating $C=1{,}000$ candidate functions and reranking the most promising ones (Section \ref{sec:pangaea:evolution}).

Table \ref{tab:pangaea:search_strategy} shows the results of this experiment.  Random search is able to discover good functions that outperform ReLU, but the functions are not as powerful as those discovered by PANGAEA.  This result demonstrates the importance of fitness selection in evolutionary search.  The functions discovered by nonparametric evolution similarly outperform ReLU but underperform PANGAEA.  Interestingly, without parameterization, evolution is not as creative: two of the three functions discovered are merely Swish multiplied by a constant.  Random search and nonparametric evolution both discovered good functions that improved accuracy, but PANGAEA achieves the best performance by combining the advantages of fitness selection and function parameterization.

\subsection{Effect of Parameterization}
\begin{table}
    \centering

    \caption{CIFAR-100 test set accuracy aggregated over ten runs, shown as mean $\pm$ sample standard deviation.  The parametric evolved functions tend to outperform their nonparametric counterparts, demonstrating the value of parameterization.\newline}

    \begin{adjustbox}{max width=\linewidth}
    \begin{tabular}{ll} \toprule
         
    \multicolumn{2}{l}{\textbf{WRN-10-4}}\\
    $\log(\sigma(\alpha x)) \cdot \beta \textrm{arcsinh}(x)$ & 
    $\bm{73.20} {\scriptscriptstyle \pm 0.37}$ \\
    $\log(\sigma(\alpha x)) \cdot \textrm{arcsinh}(x)$ & 
    $73.16 {\scriptscriptstyle \pm 0.41}$ \\
    $\log(\sigma(x)) \cdot \textrm{arcsinh}(x)$ & 
    $72.51 {\scriptscriptstyle \pm 0.30}$ \\
    $-\textrm{Swish}(\textrm{Swish}(\alpha x))$ & 
    $\bm{72.49} {\scriptscriptstyle \pm 0.55}$ \\
    $-\textrm{Swish}(\textrm{Swish}(x))$ & 
    $71.97 {\scriptscriptstyle \pm 0.22}$ \\
    \midrule
    
    \multicolumn{2}{l}{\textbf{ResNet-v1-56}} \\
    $\alpha x - \beta \log(\sigma(\gamma x))$ & 
    $\bm{71.01} {\scriptscriptstyle \pm 0.64}$ \\
    $\alpha x - \log(\sigma(\beta x))$ & 
    $70.30 {\scriptscriptstyle \pm 0.58}$ \\
    $x - \log(\sigma(x))$ & 
    $69.29 {\scriptscriptstyle \pm 0.45}$ \\
    \midrule
    
    \multicolumn{2}{l}{\textbf{ResNet-v2-56}} \\
    $\min\{\log(\sigma(x)), \alpha \log(\sigma(\beta x))\}$ & 
    $75.19 {\scriptscriptstyle \pm 0.39}$ \\
    $\log(\sigma(x))$ & 
    $\bm{75.53} {\scriptscriptstyle \pm 0.37}$ \\
    \bottomrule
    \end{tabular}
    \end{adjustbox}

    \label{tab:pangaea:disable_parameters}
\end{table}
\label{sec:pangaea:effectofparameterization}

To understand the effect that parameterizing activation functions has on their performance, the specialized functions (Table \ref{tab:pangaea:results}) were trained without them.  As Table \ref{tab:pangaea:disable_parameters} shows, when parameters are removed, performance drops.  The function $\log(\sigma(x))$ is the only exception to this rule, but its high performance is not surprising, since it was previously discovered as a general activation function (Table \ref{tab:pangaea:results}).  These results confirm that the learnable parameters contributed to the success of PANGAEA.

\subsection{Adding Parameters to Fixed Baseline Activation Functions}
\label{sec:pangaea:parametric_baseline}

As demonstrated in Tables \ref{tab:pangaea:results} and \ref{tab:pangaea:disable_parameters}, learnable parameters are an important component of PANGAEA.  An interesting question is whether accuracy can be increased simply by augmenting existing activation functions with learnable parameters.  Table \ref{tab:pangaea:parametric_baseline} shows that this is not the case: trivially adding parameters to fixed activation functions does not reliably improve performance.  This experiment implies that certain functional forms are better suited to taking advantage of parameterization than others.  By utilizing evolutionary search, PANGAEA is able to discover these functional forms automatically.

\begin{table}
    \centering
    \caption{CIFAR-100 test set accuracy aggregated over ten runs, shown as mean $\pm$ sample standard deviation.  Trivially parameterizing existing fixed activation functions does not substantially improve performance.  PANGAEA, on the other hand, utilizes evolutionary search to discover functional forms that are well suited to taking advantage of the parameters, leading to better performance.\newline
    }
    \centering
    \begin{adjustbox}{max width=\textwidth}
    \begin{tabular}{llll} \toprule
    & \textbf{WRN-10-4} & \textbf{ResNet-v1-56} & \textbf{ResNet-v2-56} \\ \midrule

    \textbf{Best Specialized Functions}\\
    $\log(\sigma(\alpha x)) \cdot \beta \textrm{arcsinh}(x)$ & 
    $\bm{73.20} {\scriptscriptstyle \pm 0.37}$ \\

    $\alpha x - \beta \log(\sigma(\gamma x))$ & &
    $\bm{71.01} {\scriptscriptstyle \pm 0.64}$ & \\
    
    $\textrm{Softplus}(\textrm{ELU}(x))$ & & &
    $\bm{75.61} {\scriptscriptstyle \pm 0.42}$ \\
    
    \midrule 
    
    \textbf{Parameterized Functions}\\
    $\alpha \textrm{ReLU}(\beta x)$ & 
    $71.96 {\scriptscriptstyle \pm 0.31}$ &
    $68.93 {\scriptscriptstyle \pm 0.22}$ &
    $73.52 {\scriptscriptstyle \pm 0.37}$ \\
    
    $\alpha \textrm{ELiSH}(\beta x)$ & 
    $01.00 {\scriptscriptstyle \pm 0.00}$&
    $01.00 {\scriptscriptstyle \pm 0.00}$&
    $73.94 {\scriptscriptstyle \pm 0.33}$\\
    
    $\alpha \textrm{ELU}(\beta x)$ & 
    $71.98 {\scriptscriptstyle \pm 0.24}$ &
    $69.06 {\scriptscriptstyle \pm 0.37}$ &
    $73.97 {\scriptscriptstyle \pm 0.45}$ \\
    
    $\alpha \textrm{GELU}(\beta x)$ & 
    $71.96 {\scriptscriptstyle \pm 0.34}$ &
    $69.39 {\scriptscriptstyle \pm 0.35}$ & 
    $73.83 {\scriptscriptstyle \pm 0.24}$ \\
    
    $\alpha \textrm{HardSigmoid}(\beta x)$ & 
    $66.70 {\scriptscriptstyle \pm 0.64}$ &
    $34.33 {\scriptscriptstyle \pm 6.53}$ &
    $65.10 {\scriptscriptstyle \pm 0.40}$ \\ 
    
    $\alpha \textrm{Leaky ReLU}(\beta x)$ & 
    $71.74 {\scriptscriptstyle \pm 0.39}$ &
    $69.11 {\scriptscriptstyle \pm 0.47}$ &
    $73.44 {\scriptscriptstyle \pm 0.29}$ \\ 
    
    $\alpha \textrm{Mish}(\beta x)$ & 
    $72.11 {\scriptscriptstyle \pm 0.31}$ &
    $69.51 {\scriptscriptstyle \pm 0.67}$ &
    $73.72 {\scriptscriptstyle \pm 0.32}$ \\
    
    $\alpha \textrm{SELU}(\beta x)$ & 
    $71.07 {\scriptscriptstyle \pm 0.33}$ &
    $68.05 {\scriptscriptstyle \pm 0.39}$ &
    $73.37 {\scriptscriptstyle \pm 0.38}$ \\
    
    $\alpha \textrm{sigmoid}(\beta x)$ & 
    $66.98 {\scriptscriptstyle \pm 0.66}$ &      
    $44.40 {\scriptscriptstyle \pm 2.62}$ &
    $66.98 {\scriptscriptstyle \pm 0.85}$ \\ 
    
    $\alpha \textrm{Softplus}(\beta x)$ & 
    $71.73 {\scriptscriptstyle \pm 0.31}$ &
    $68.84 {\scriptscriptstyle \pm 0.30}$ &
    $73.95 {\scriptscriptstyle \pm 0.37}$ \\ 
    
    $\alpha \textrm{Softsign}(\beta x)$ & 
    $62.12 {\scriptscriptstyle \pm 0.83}$ &
    $09.18  {\scriptscriptstyle \pm 13.75}$ &
    $68.87 {\scriptscriptstyle \pm 0.38}$ \\ 
    
    $\alpha \textrm{Swish}(\beta x)$ & 
    $72.26 {\scriptscriptstyle \pm 0.29}$ &
    $69.25 {\scriptscriptstyle \pm 0.28}$ &
    $73.93 {\scriptscriptstyle \pm 0.22}$ \\ 
    
    $\alpha \textrm{tanh}(\beta x)$ & 
    $63.55 {\scriptscriptstyle \pm 0.56}$ &
    $02.92 {\scriptscriptstyle \pm 6.07}$ &
    $69.55 {\scriptscriptstyle \pm 0.62}$ \\ 
    
    \bottomrule
    \end{tabular}
    \end{adjustbox}
    \label{tab:pangaea:parametric_baseline}
\end{table}

\subsection{Scaling Up to Larger Networks} 
\label{sec:pangaea:scalingup}

PANGAEA discovered specialized activation functions for WRN-10-4, ResNet-v1-56, and ResNet-v2-56.  Table \ref{tab:pangaea:scale_up} shows the performance of these activation functions when paired with the larger WRN-16-8, ResNet-v1-110, and ResNet-v2-110 architectures.  Due to time constraints, ReLU is the only baseline activation function in these experiments. 

Two of the three functions discovered for WRN-10-4 outperform ReLU with WRN-16-8, and all three functions discovered for ResNet-v2-56 outperform ReLU with ResNet-v2-110.  Interestingly, ReLU achieves the highest accuracy for ResNet-v1-110, where activation functions are part of the skip connections, but not for ResNet-v2-110, where they are not. Thus, it is easier to achieve high performance with specialized activation functions on very deep architectures when they are not confounded by skip connections.  Notably, ResNet-v2-110 with $\textrm{Softplus}(\textrm{ELU}(x))$ performs comparably to the much larger ResNet-v2-1001 with ReLU (77.14 vs.\ 77.29, as reported by \citet{he2016identity}). 

Evolving novel activation functions can be computationally expensive.  The results in Table \ref{tab:pangaea:scale_up} suggest that it is possible to reduce this cost by evolving activation functions for smaller architectures, and then using the discovered functions with larger architectures.

\begin{table}[ht]
    \centering
    \caption{Specialized activation functions discovered for WRN-10-4, ResNet-v1-56, and ResNet-v2-56 are evaluated on larger versions of those architectures: WRN-16-8, ResNet-v1-110, and ResNet-v2-110, respectively.  CIFAR-100 test accuracy is reported as mean $\pm$ sample standard deviation across three runs.  Specialized activation functions successfully transfer to WRN-16-8 and ResNet-v2-110, outperforming ReLU.
    \newline}
    \begin{adjustbox}{max width=\linewidth}
    \begin{tabular}{ll} 
        \toprule
         \textbf{WRN-16-8}  \\ 
         $\log(\sigma(\alpha x)) \cdot \beta \textrm{arcsinh}(x)$ & $\bm{78.36} {\scriptscriptstyle \pm 0.17}$ \\
         $\log(\sigma(\alpha x)) \cdot \textrm{arcsinh}(x)$ & $78.34 {\scriptscriptstyle \pm 0.20}$ \\
         $-\textrm{Swish}(\textrm{Swish}(\alpha x))$ & $78.00 {\scriptscriptstyle \pm 0.35}$ \\ 
         ReLU & $78.15 {\scriptscriptstyle \pm 0.03}$\\
        \midrule
         \textbf{ResNet-v1-110}  \\ 
         $\alpha x - \beta \log(\sigma(\gamma x))$ & $70.85 {\scriptscriptstyle \pm 0.50}$ \\
         $\alpha x - \log(\sigma(\beta x))$ & $70.34 {\scriptscriptstyle \pm 0.60}$ \\
         $\max\{\textrm{Swish}(x), 0\}$ & $70.36 {\scriptscriptstyle \pm 0.56}$ \\ 
         ReLU & $\bm{71.23} {\scriptscriptstyle \pm 0.25}$ \\
        \midrule 
         \textbf{ResNet-v2-110}  \\ 
         $\textrm{Softplus}(\textrm{ELU}(x))$ & $\bm{77.14} {\scriptscriptstyle \pm 0.38}$ \\
         $\min\{\log(\sigma(x)), \alpha \log(\sigma(\beta x))\}$ & $76.93 {\scriptscriptstyle \pm 0.19}$ \\
         $\textrm{SELU}(\textrm{Swish}(x))$ & $76.96 {\scriptscriptstyle \pm 0.14}$ \\ 
         ReLU & $76.34 {\scriptscriptstyle \pm 0.11}$ \\
        \bottomrule
    \end{tabular}
    \end{adjustbox}
    \label{tab:pangaea:scale_up}
\end{table}

\subsection{A New Task: All-CNN-C on CIFAR-10} 
\begin{table}
    \centering
    \caption{All-CNN-C accuracy with different activation functions on CIFAR-10, shown as mean $\pm$ sample standard deviation across ten runs.  PANGAEA improves performance significantly also with this different architecture and task.\newline}
    \begin{adjustbox}{max width=\linewidth}
    \begin{tabular}{ll}
    \toprule
    $\alpha\textrm{ReLU}(\beta |\textrm{ReLU}(\gamma x)|)$  & $\bm{92.77} {\scriptscriptstyle \pm 0.13}$\\
    $\alpha \textrm{Swish}(x) \cdot \cosh(\beta)$ & $92.66 {\scriptscriptstyle \pm 0.08}$\\
    $\alpha \textrm{Swish}(\beta x)$ & $76.15 {\scriptscriptstyle \pm 34.86}$\\
    ReLU & $88.47 {\scriptscriptstyle \pm 0.14}$\\
    \bottomrule
    \end{tabular}
    \end{adjustbox}
    \label{tab:pangaea:allcnnc}
\end{table}
\label{sec:pangaea:allcnnc}
To verify that PANGAEA is effective with different datasets and types of architectures, activation functions were evolved for the All-CNN-C \citep{springenberg2015striving} architecture on the CIFAR-10 dataset.  All-CNN-C is quite distinct from the architectures considered above: It contains only convolutional layers, activation functions, and a global average pooling layer, but it does not have residual connections.

As shown in Table \ref{tab:pangaea:allcnnc}, PANGAEA improves significantly over ReLU in this setting as well.  The accuracy improvement from 88.47\% to 92.77\% corresponds to an impressive 37.29\% reduction in the error rate.  This experiment provides further evidence that PANGAEA can improve performance for different architectures and tasks.

\subsection{Training Dynamics of Evolved Activation Functions}
\label{sec:pangaea:dynamics}

PANGAEA discovers activation functions that improve accuracy over baseline functions.  An interesting question is: What mechanisms do these evolved functions use in order to achieve better performance?  By examining training and validation curves qualitatively for different activation functions, it appears that some functions ease optimization, while others improve performance through implicit regularization.

For instance, Figure \ref{fig:pangaea:allcnnc_training_curves} shows training and validation curves for the All-CNN-C architecture and four discovered activation functions, plus ReLU. With All-CNN-C, the learning rate starts at 0.01 and decreases by a factor of 0.1 after epochs 200, 250, and 300, with training ending at epoch 350.  With some discovered activation functions, the training and validation curves are consistently higher than the curves from ReLU across all epochs of training, indicating easier optimization (Figure \ref{fig:pangaea:allcnnc_training_curves}a).  With other activation functions, the training and validation curves actually remain lower than those from ReLU until the final stage in the learning rate schedule, suggesting implicit regularization (Figure \ref{fig:pangaea:allcnnc_training_curves}b). In such cases, the network is learning difficult patterns in the early stages of training; in contrast, the ReLU model memorizes simpler patterns, which leads to early gains but difficulty generalizing later on \citep{li2019towards}.

\input{chapters/pangaea/figures/allcnnc_training_curves.tex}

Even more complex behavior can be observed in some cases.  For example, some discovered functions have training and validation curves that start out higher than those from ReLU, plateau to a lower value, but then again surpass those from ReLU at later stages in the learning rate schedule (Figure \ref{fig:pangaea:allcnnc_training_curves}c).  Others have curves that start out lower than those from ReLU, but surpass it within a few dozen epochs (Figure \ref{fig:pangaea:allcnnc_training_curves}d). Such diverse behavior suggests that these mechanisms can be combined in complex ways. Thus, the plots in Figure \ref{fig:pangaea:allcnnc_training_curves} suggest that in the future it may be feasible to search for activation functions with specific properties depending on the task at hand.  For example, a larger network may benefit from an activation function that implicitly regularizes it, while a smaller network may be better suited to an activation function that eases optimization.

\section{Evaluating Robustness: Experiments on Reliability, Flexibility, and Efficiency}
\label{sec:pangaea:robustness}

This section demonstrates robustness of PANGAEA with three experiments.  In the first experiment, two independent PANGAEA processes were run from scratch.  The processes discovered similarly powerful activation functions, demonstrating that PANGAEA reliably discovers good activation functions each time it is run.  In the second experiment, variations of PANGAEA with per-layer and per-neuron learnable parameters were run, to complement the original PANGAEA with per-channel parameters.  The results show that PANGAEA is effective with all three variations.  Third, statistics from activation functions across the per-layer, two per-channel, and the per-neuron variations were aggregated to demonstrate computational efficiency of PANGAEA.  Every PANGAEA run discovers an activation function that outperforms ReLU early on in the search process, before much compute is spent; they also eventually discover activation functions that perform much better and train almost as quickly as ReLU.

\subsection{Reliability of PANGAEA}

PANGAEA is inherently a stochastic process.  Therefore, an important question is whether PANGAEA can discover good activation functions reliably every time it is run.  To answer this question, PANGAEA was run from scratch on ResNet-v1-56 independently two times.  These runs utilized per-channel parameters, and were identical to the original PANGAEA run, except they were allowed to evaluate up to $C=2{,}000$ activation functions instead of $C=1{,}000$ from the original run.  They were run on ResNet-v1-56 since it proved to be the most difficult architecture to optimize (Table \ref{tab:pangaea:results}).  

\begin{table}
    \centering
    \caption{Performance of the best activation functions from multiple PANGAEA runs with ResNet-v1-56.  CIFAR-100 test accuracy is shown as the mean $\pm$ sample standard deviation across three runs.  The three independent per-channel runs produce activation functions of similar performance, demonstrating the reliability of PANGAEA.  PANGAEA also discovers good activation functions with per-layer or per-neuron parameters, showing its flexibility. The very best per-layer and per-neuron functions are difficult to find, suggesting that their distribution is long-tailed.\\}
    \begin{tabular}{ll}
        \toprule
        \textbf{Per-layer PANGAEA}\\
        $\max\{\textrm{Swish}(x) , x\}$ & $71.00 {\scriptscriptstyle \pm 0.28}$\\
        $\max\{x, \alpha \cdot \log(\sigma(\textrm{SELU}(x)))\}$ & $70.76 {\scriptscriptstyle \pm 0.29}$\\
        $\max\{\alpha \cdot \max\{\beta \cdot \textrm{ReLU}(\textrm{arcsinh}(x)), x\}, \max\{\gamma \cdot \textrm{ReLU}(\textrm{arcsinh}(x)), x\}\}$ & $70.63 {\scriptscriptstyle \pm 0.35}$\\
        \midrule
        
        \textbf{Original Per-channel PANGAEA Run}\\
        $\alpha x - \beta \log(\sigma(\gamma x))$ & 
        $71.01 {\scriptscriptstyle \pm 0.64}$ \\
        $\alpha x - \log(\sigma(\beta x))$ & 
        $70.30 {\scriptscriptstyle \pm 0.58}$ \\
        $\max\{\textrm{Swish}(x), 0\}$ & 
        $69.43 {\scriptscriptstyle \pm 0.69}$ \\
        \midrule
        
        \textbf{Additional Per-channel PANGAEA Run 1} \\
        $\max \{ \min\{ \alpha \cdot x , \textrm{ELU}(x)\} , 0 \}$ & $70.53 {\scriptscriptstyle \pm 0.31}$\\ 
        $\alpha \cdot \max\left\{\beta \cdot \textrm{ReLU}\left(\frac{\textrm{Swish}(x)}{\gamma}\right), x\right\}$ & $70.52 {\scriptscriptstyle \pm 0.39}$\\
        $\max\{ \textrm{ReLU}(\textrm{Swish}(\alpha \cdot x)), \beta \cdot x\}$ & $70.44 {\scriptscriptstyle \pm 0.44}$\\
        \midrule
        
        \textbf{Additional Per-channel PANGAEA Run 2} \\
        $\max\{ \textrm{Swish}(\alpha \cdot x), x \}$ & $71.03 {\scriptscriptstyle \pm 0.40}$\\
        $\max\{ \textrm{Swish}(x), \textrm{arcsinh}(\alpha \cdot \beta \cdot x)\}$ & $70.52 {\scriptscriptstyle \pm 0.35}$\\ 
        $\max\{ \textrm{Swish}(x), \textrm{arcsinh}(\alpha \cdot x)\}$ & $70.41 {\scriptscriptstyle \pm 0.38}$\\
        \midrule
        
        \textbf{Per-neuron PANGAEA}\\
        $\alpha \cdot \max\{\beta \cdot \textrm{Swish}(\gamma \cdot x), \textrm{Swish}(x)\}$ & $71.25 {\scriptscriptstyle \pm 0.35}$\\
        $\alpha \cdot x - (\beta \cdot \textrm{Swish}(\gamma \cdot x))$ & $71.23 {\scriptscriptstyle \pm 0.18}$\\
        $\textrm{ELU}(\textrm{ELU}(\alpha \cdot x) + \log(\sigma(0)))$ & $71.20 {\scriptscriptstyle \pm 0.25}$\\
        \midrule
        
        \textbf{Random Sample $(n=500)$}\\
        Per-layer $\max\{\alpha \cdot \textrm{Swish}(\beta \cdot x), \textrm{Softplus}(x)\}$ & $70.32$ \\
        Per-channel $\alpha \cdot x^2$ & $70.91$\\
        Per-neuron $\textrm{bessel\_i0e}(|x|) + \alpha \cdot |x|$ & $\bm{71.66}$ \\
        \midrule
        
        ReLU & $69.64 {\scriptscriptstyle \pm 0.65}$\\
        \bottomrule
    \end{tabular}
    
    \label{tab:pangaea:resnetv1_lcn}
\end{table}

There are multiple ways to analyze the similarity of the two PANGAEA runs.  First, a simple and relevant metric is to look at the test accuracies of the functions discovered by each search.  Table~\ref{tab:pangaea:resnetv1_lcn} shows that both PANGAEA runs discovered multiple good activation functions.  The accuracies achieved by these functions are substantially higher than those achieved by ReLU. Most importantly, although the functions themselves are different, they resulted in similar accuracies as functions from the original PANGAEA run.  

A second way is to compare the time course of discovery. To this end, Figure \ref{fig:pangaea:moving_acc_lcn} shows how the populations of $P=64$ activation functions improved over time in the two PANGAEA runs.  In both cases, the initial functions are relatively poor.  As evolution progresses, both runs discover better functions at similar rates.  This result shows that in addition to the final results, the PANGAEA process as a whole is stable and reliable.

\begin{figure}
    \centering
    \includegraphics[width=0.75\linewidth]{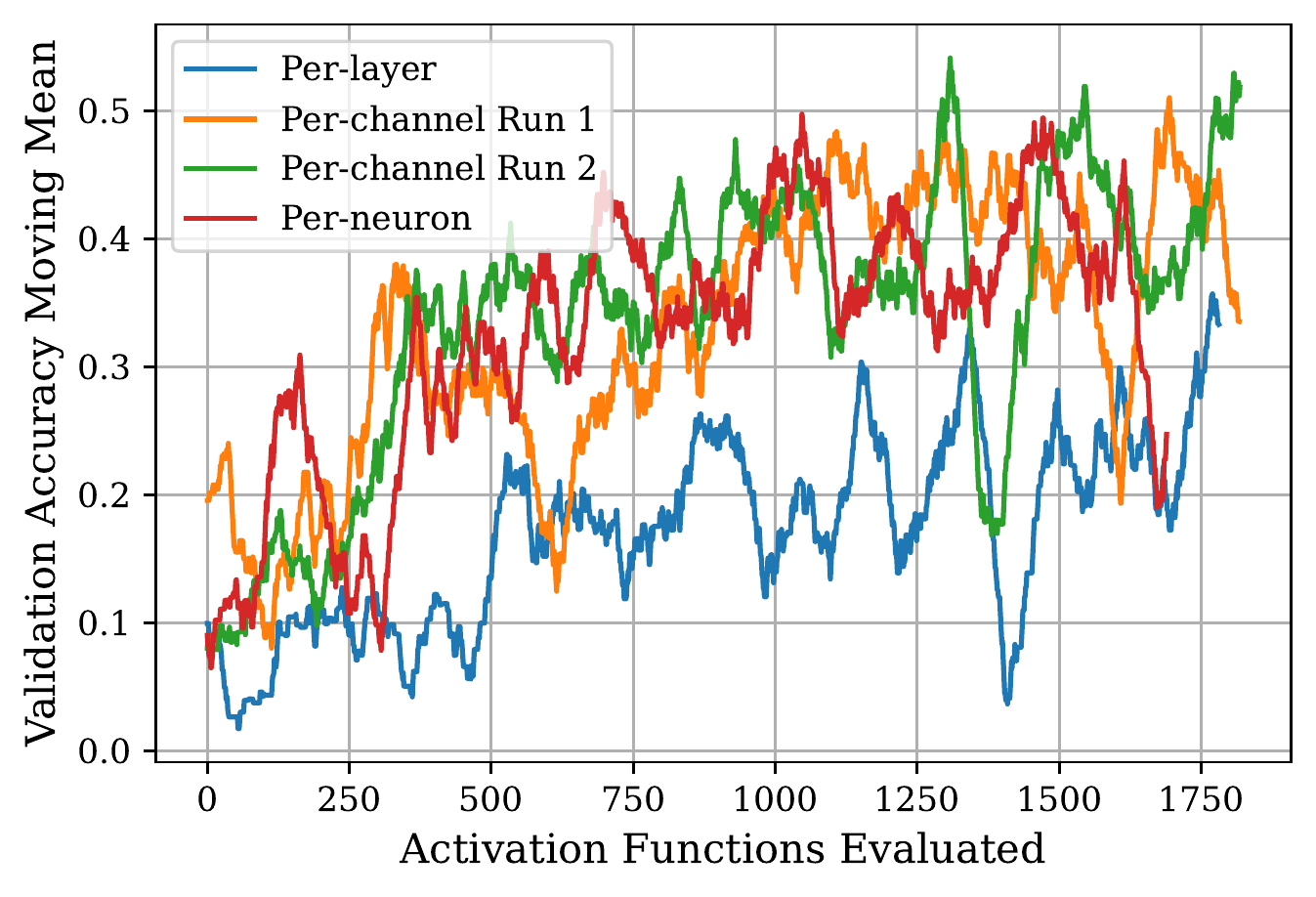}
    \caption{Average population fitness across time for four independent PANGAEA runs.  The plots show the average validation accuracy achieved with the 64 most recently evaluated activation functions at any given time in the search process.  All four PANGAEA runs gradually discover better activation functions as they explore the search space, with the per-layer run slightly below the others.  Importantly, the two per-channel PANGAEA runs progress at similar rates, demonstrating the reliability of PANGAEA.}
    \label{fig:pangaea:moving_acc_lcn}
\end{figure}

A third way is to compare the complexity, i.e.\ time it takes to train the network with the discovered functions. Figure~\ref{fig:pangaea:scatter_hist_lcn} shows the cost of all of the activation functions considered throughout each PANGAEA run. The runtimes of activation functions are comparable across different validation accuracies, suggesting that both runs discovered functions of similar complexity.

\begin{figure}
    \centering
    \includegraphics[width=0.9\linewidth]{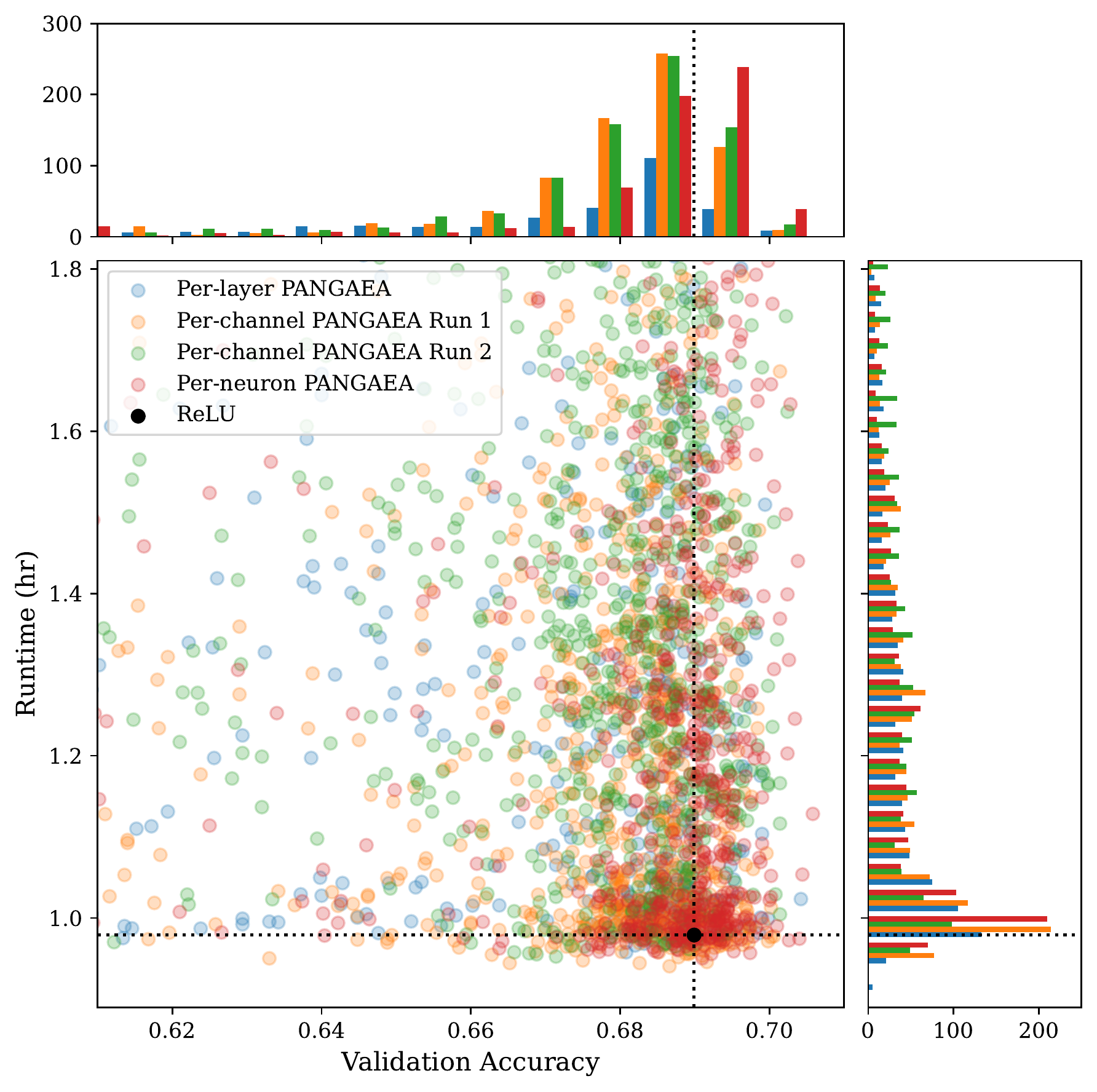}
    \caption{Fitness (validation accuracy) and compute cost (runtime in hours) among all activation functions considered in four independent PANGAEA processes on ResNet-v1-56.  Each point represents a different activation function.  The distribution of fitness and compute cost are shown in histograms on the top and right, respectively.  All PANGAEA variants explore activation functions of similar complexity and reliably discover many novel functions that outperform ReLU.}
    \label{fig:pangaea:scatter_hist_lcn}
\end{figure}

In sum, the two PANGAEA runs produced comparable final results, progressed at comparable rates, and searched through functions of similar complexity.  These results suggest that PANGAEA is a reliable process that can consistently outperform baseline activation functions.

\subsection{Parameters: per-layer, per-channel, or per-neuron?}

Learnable parameters in activation functions can be per-layer, per-channel, or per-neuron. It is not clear which setting is the best.  For example, per-neuron parameters are the default setting in the TensorFlow implementation of PReLU \cite{he2015delving}.  However, \citet{he2015delving} also experimented with per-channel and per-layer implementations.  Further preliminary experiments for this chapter (Table~\ref{tab:pangaea:prelu_lcn}) suggest that per-neuron PReLU is best for WRN-10-4 and ResNet-v2-56, but this setting is the worst for ResNet-v1-56, which benefits most from per-layer parameters.  

Similarly, no clear trends were observed in preliminary PANGAEA experiments.  For some activation functions and architectures per-neuron parameters were beneficial, presumably due to the added expressivity of each neuron learning its own optimal activation function.  In other cases per-layer was better, possibly due to an implicit regularization effect caused by all neurons within a layer using the same activation function.  As a compromise between the expressivity and regularization of these two strategies, per-channel parameters were utilized in the main experiments.  However, the preliminary results suggest that performance may be further optimized by specializing the parameter setting to the activation function and to the architecture.

\begin{table}
    \centering
    \caption{CIFAR-100 test accuracy with different architectures and PReLU variants reported as mean $\pm$ sample standard deviation across ten runs.  Per-neuron PReLU gets the best performance on WRN-10-4 and ResNet-v2-56, but per-layer PReLU is the best for ResNet-v1-56. \\}
    \begin{tabular}{llll}
        \toprule
        & \textbf{WRN-10-4} & \textbf{ResNet-v1-56} & \textbf{ResNet-v2-56} \\
        \midrule
        Per-layer PReLU & $71.92 {\scriptscriptstyle \pm 0.41}$ & $\bm{71.40} {\scriptscriptstyle \pm 0.59}$ & $73.54 {\scriptscriptstyle \pm 0.21}$ \\
        Per-channel PReLU & $71.15 {\scriptscriptstyle \pm 0.41}$ & $71.25 {\scriptscriptstyle \pm 0.54}$ & $74.52 {\scriptscriptstyle \pm 0.24}$ \\
        Per-neuron PReLU & $\bm{72.23} {\scriptscriptstyle \pm 0.37}$ & $69.77 {\scriptscriptstyle \pm 0.40}$ & $\bm{75.10} {\scriptscriptstyle \pm 0.53}$ \\
        \bottomrule
    \end{tabular}
    \label{tab:pangaea:prelu_lcn}
\end{table}

To explore this idea further, per-layer and per-neuron versions of PANGAEA were run from scratch on ResNet-v1-56. Both of these PANGAEA runs produced good activation functions that outperformed ReLU substantially (Table \ref{tab:pangaea:resnetv1_lcn}). Interestingly, although per-layer PReLU outperformed per-neuron PReLU with ResNet-v1-56 (Table~\ref{tab:pangaea:prelu_lcn}), PANGAEA performed the best in the per-neuron setting (Table \ref{tab:pangaea:resnetv1_lcn}).  Indeed, although the per-layer PANGAEA runs still discovered good activation functions, their average performance during search was often lower than that of the per-channel or per-neuron variants (Figure \ref{fig:pangaea:moving_acc_lcn}).  These findings suggest that the distribution of per-layer activation functions may be long-tailed: Powerful per-layer activation functions do exist, but they may be more difficult to discover compared to per-channel or per-neuron activation functions.

In order to separate the search space from the search algorithm, in a further experiment 500 per-layer, per-channel, and per-neuron activation functions were randomly created and trained once with ResNet-v1-56 (the functions were first initialized randomly as shown in Figure \ref{fig:pangaea:initialization}, and then mutated randomly three times as shown in Figure \ref{fig:pangaea:mutation}).  The best activation functions from these random samples are included in Table \ref{tab:pangaea:resnetv1_lcn}.  The best per-layer activation function outperformed ReLU by a large margin, but was not as powerful as those discovered with PANGAEA; the best per-neuron function outperformed all other variants.  These results thus suggest that the distribution of good per-neuron functions may be long-tailed as well, but at a higher level of performance than per-layer and per-channel functions.

In sum, although more research is needed to discover a principled way to select per-layer, per-channel, or per-neuron parameters in a given situation, PANGAEA is flexible enough to discover good functions for all three of these cases.

\subsection{Efficiency of PANGAEA}

PANGAEA's computational efficiency needs to be evaluated from two perspectives.  First, how much compute is necessary to find good activation functions?  Second, once a good activation function is found, how much more expensive is it to use it in a network compared to a baseline function like ReLU?  This section aggregates data from the per-layer, the two per-channel, and the per-neuron PANGAEA runs to demonstrate that PANGAEA is surprisingly efficient in both respects.

First, Figure \ref{fig:pangaea:compute_spent} shows how the four PANGAEA runs discovered better activation functions with increasingly more compute.  All four runs discovered an activation function that outperformed ReLU relatively early in the search.  Because some activation functions are unstable and cause training to fail, they require negligible compute.  Computational resources can instead be focused on functions that appear promising.  The implications of Figure \ref{fig:pangaea:compute_spent} are that in practice, PANGAEA can be used to improve over a baseline activation function relatively cheaply.  If better performance is needed, additional compute can be used to continue the search until the desired performance is achieved.

\begin{figure}[t]
    \centering
    \includegraphics[width=0.75\linewidth]{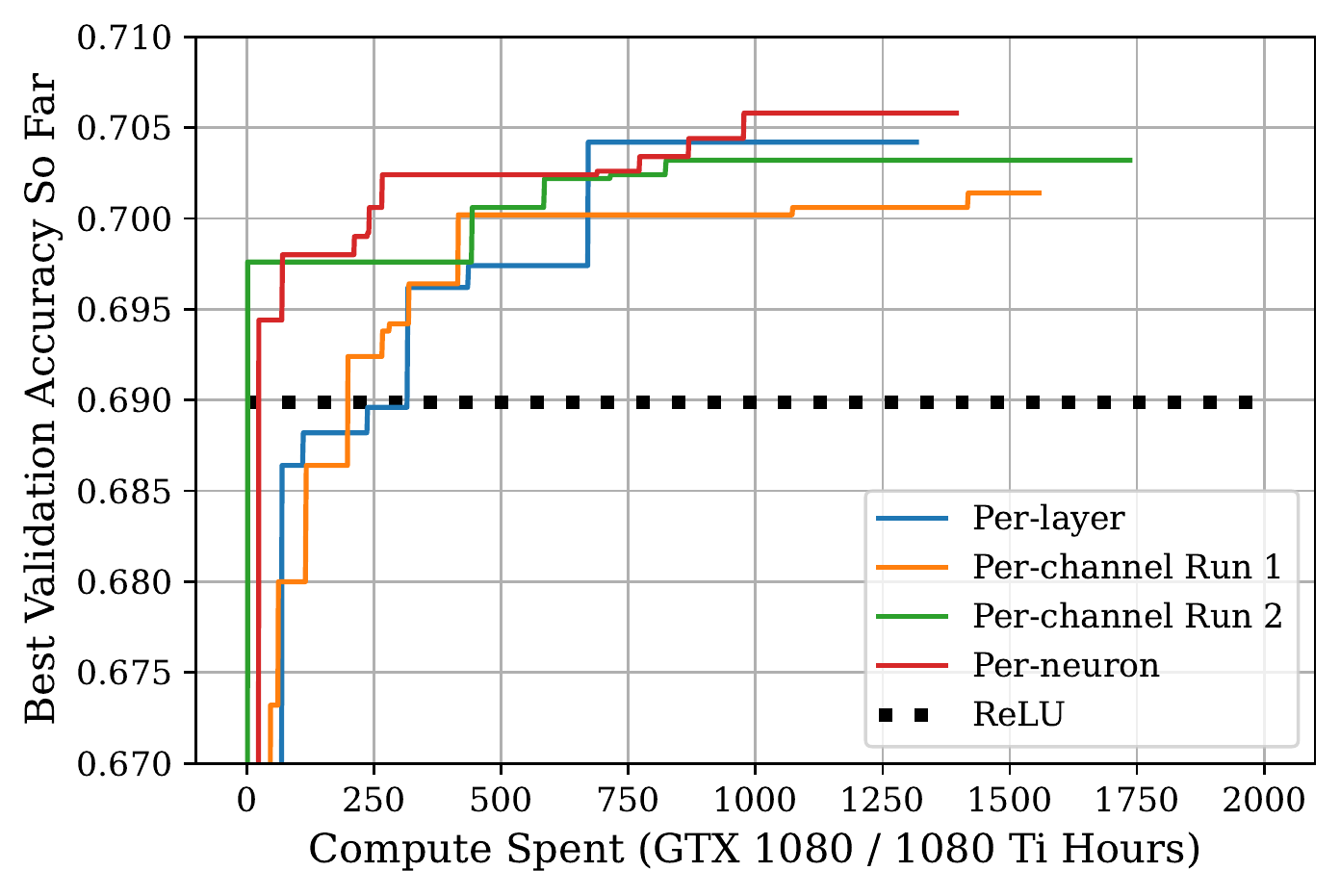}
    \caption{Computational efficiency of PANGAEA.  The plot shows the performance of the best activation function discovered so far ($y$-axis) after a given amount of compute was spent ($x$-axis).  All four PANGAEA runs discover activation functions that outperform ReLU with relatively little compute, demonstrating the efficiency of PANGAEA.  If even better performance is needed, additional compute can be spent.}
    \label{fig:pangaea:compute_spent}
\end{figure}

Second, Figure \ref{fig:pangaea:scatter_hist_lcn} shows the distribution of accuracy and compute cost of all activation functions evaluated in all four PANGAEA runs on ResNet-v1-56.  Each point in the scatter plot represents a unique activation function discovered in one of the searches, and the distribution of accuracies and compute costs are shown as histograms above and to the side of the main plot.  The results confirm earlier conclusions: All strategies find many activation functions that beat ReLU, and per-neuron PANGAEA discovers more high-performing functions than per-channel or per-layer PANGAEA. The total amount of time it took to train the architecture with the given activation function is shown as ``runtime'' in the vertical axis.  Interestingly, there is a wide range in this metric: some activation functions are significantly more expensive than ReLU, while others are essentially the same. Thus, the figure shows that there exist plenty of activation functions that significantly beat ReLU but do not incur a significant computational overhead.  This distribution also suggests that a multi-objective approach that optimizes for both accuracy and computational cost simultaneously could be effective.

\section{Additional Results with Learnable Activation Functions}
\label{sec:pangaea:pausplash}

PAU and SPLASH achieved worse-than-expected performance in Table \ref{tab:pangaea:results}, so additional experiments were run to investigate their behavior.

\paragraph{PAU}

\citet{pade-molina2019pad} utilized a specialized training setup to achieve their results with PAU.  In particular, they used a constant learning rate and no weight decay for the PAU layers, but used a learning rate decay of 0.985 per epoch and weight decay of $0.0005$ for the other weights.  They also used a smaller batch size of 64, and trained for 400 epochs instead of 200.  Even though the paper does not mention it, it is possible that such a specialized setup is necessary to achieve good performance with PAU.  The experiments in this section utilized this same training setup (but only trained for 200 epochs for fairness) to verify that the PAU implementation was correct.  

Unfortunately, some relevant hyperparameters were not included in the PAU paper \cite{pade-molina2019pad}.  These settings include the fixed learning rate used for the PAU layers, whether Nesterov momentum is utilized, and which approximation of Leaky ReLU is used to initialize the PAU weights.  This missing information makes it difficult to replicate the original performance exactly.  After significant trial-and-error, the following settings worked well: The learning rate for the PAU layers was 0.01, the initial learning rate for other weights was also 0.01, Nesterov momentum was not used, and the PAU weights were initialized to approximate Leaky ReLU with a slope of 0.01. 

Table \ref{tab:pangaea:pau_specialized_setup} shows the performance of WRN-10-4 and ResNet-v2-56 using these discovered hyperparameters and the specialized training setup from \citet{pade-molina2019pad}.  The performance is comparable to other baseline activation functions.  In some cases, the runs failed because of training instability (results were filtered out if the training accuracy was below 0.5).  For all hyperparameter combinations tested, PAU was unstable with ResNet-v1-56.  Thus, it is possible to get good performance with PAU, but the performance is highly sensitive to the training setup and choice of hyperparameters. 

\begin{table}
    \centering
    \caption{CIFAR-100 test accuracy with PAU using a specialized training setup.  Performance is comparable to other baseline activation functions, but some runs fail due to training instability.\\}
    \begin{tabular}{lll}
    \toprule
    & \textbf{WRN-10-4} & \textbf{ResNet-v2-56} \\
    \midrule
    Accuracy & $62.56 {\scriptscriptstyle \pm 4.84}$ & 
    $69.59 {\scriptscriptstyle \pm 1.66}$\\
    Failed Runs & 3 of 10 & 7 of 10 \\
    \bottomrule
    \end{tabular}
    \label{tab:pangaea:pau_specialized_setup}
\end{table}

Note that a standard, most commonly used setup was used throughout the main experiments in the chapter for all baseline comparisons. The reason is that there are dozens of such comparisons in this chapter, and it is possible that each one could benefit from a specialized setup---a setup that may not even be fully known at this time. Therefore, a standard setup was necessary to ensure that the comparisons are fair.

\paragraph{SPLASH}

In addition to ResNet-v1-56 in the main experiments, SPLASH was trained with ResNet-v1-20, ResNet-v1-32, and ResNet-v1-44 for a more thorough characterization of its performance.  Each architecture was trained ten times, resulting in training curves shown in Figure \ref{fig:pangaea:rnv1_splash}.

With ResNet-v1-20, final test accuracy of the independent runs is between 0.661 and 0.684, which agrees with the results by \citet{tavakoli2020splash}.  However, two of the ten runs failed in the middle of training because the loss became undefined (Figure \ref{fig:pangaea:rnv1_splash}), suggesting that SPLASH units can be unstable.  Progressing to the deeper ResNet-v1-32, the effect was more pronounced.  As shown in Figure \ref{fig:pangaea:rnv1_splash}, only two of the ten runs progressed past epoch 30, while no run trained to completion.  With ResNet-v1-44 and ResNet-v1-56, training failed within the first epoch, so the training curves are not shown.

These results thus confirm that the implementation is correct, reproducing the results of \citet{tavakoli2020splash}. However, they also lead to the interesting observation that SPLASH units are effective with shallow networks but struggle with deeper ones, like the ones evaluated in this chapter.

\begin{figure}[t]
    \centering
    \includegraphics[width=0.49\textwidth]{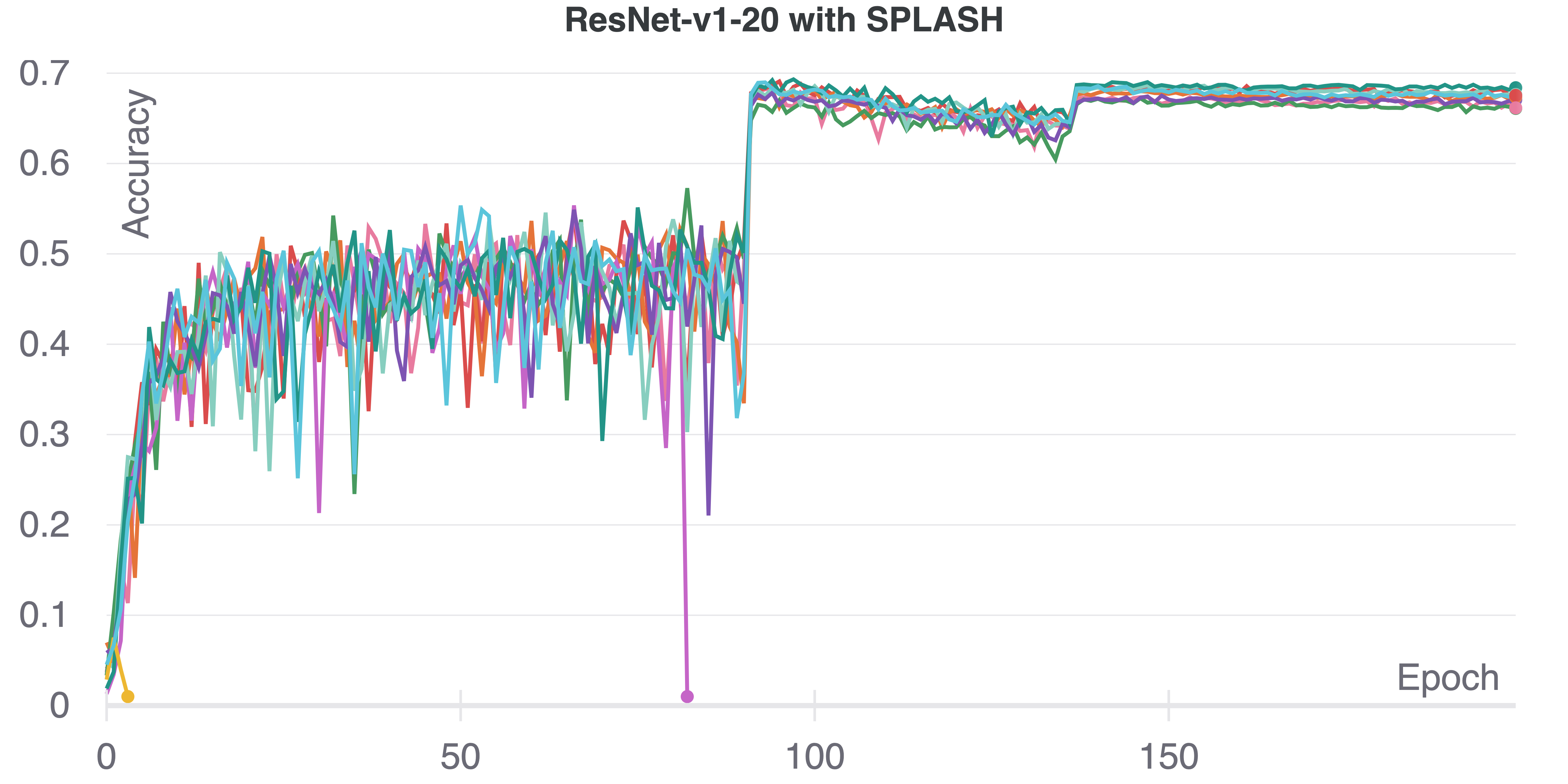}
    \includegraphics[width=0.49\textwidth]{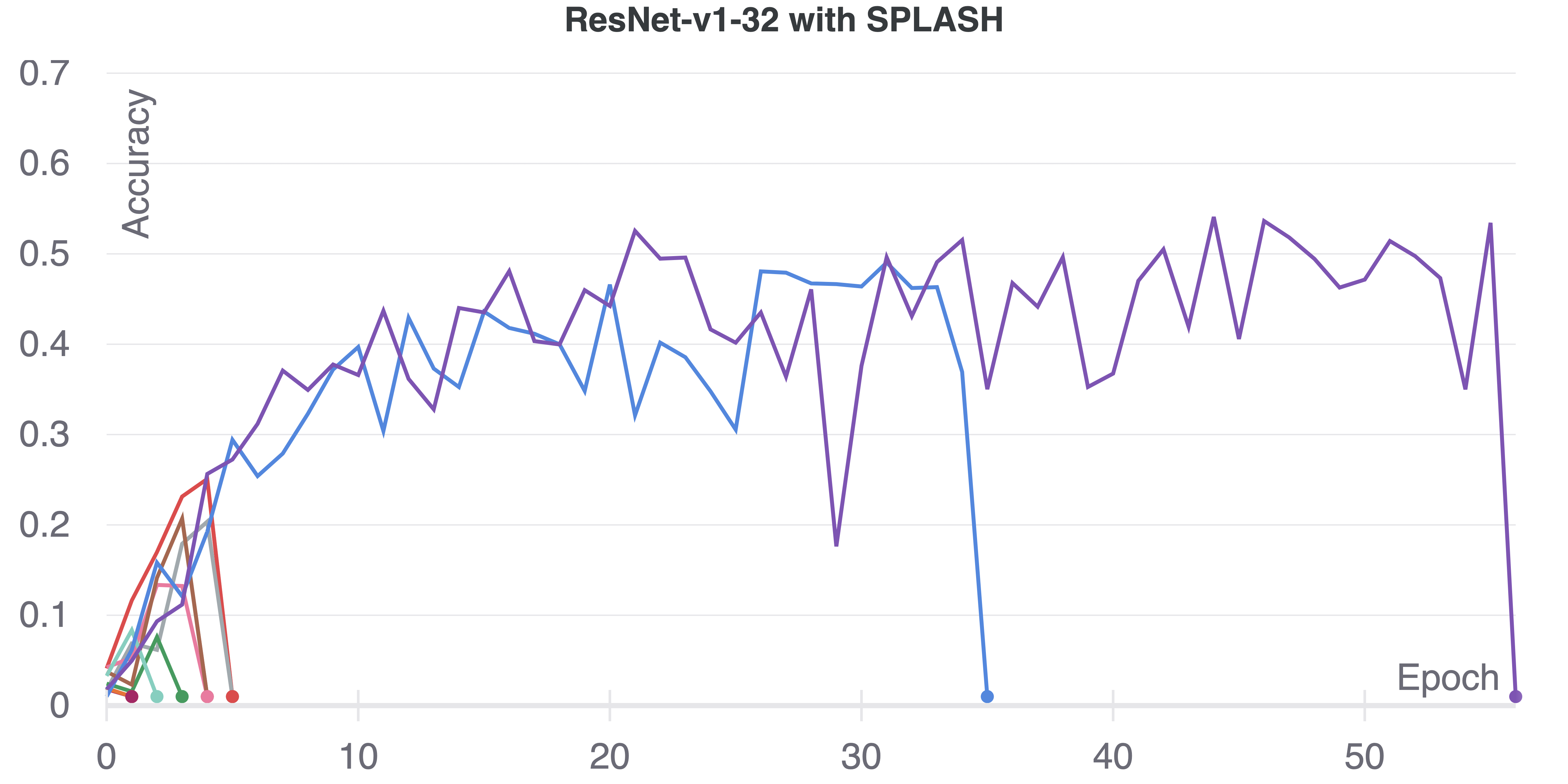}
    \caption{ResNet-v1-20 and ResNet-v1-32 test accuracy on CIFAR-100 with the SPLASH activation function.  SPLASH units work well in shallow networks, but become unstable with increased depth.  This result explains the success of SPLASH in the original work by \citet{tavakoli2020splash}, and also shows why SPLASH units fail with the architectures considered in this work.}
    \label{fig:pangaea:rnv1_splash}
\end{figure}

\section{Baseline Activation Function Details}
\label{sec:pangaea:baseline}

The following activation functions were used as baseline comparisons in Table \ref{tab:pangaea:results}.  Some functions were also utilized in the search space (Table \ref{tab:pangaea:searchspace}).
\begin{itemize}
    \item $\textrm{ReLU} = \max\{x, 0\}$ \citep{nair2010rectified}.
    
    \item $\textrm{ELiSH} = \frac{x}{1+e^{-x}} \texttt{ if } x \geq 0 \texttt{ else } \frac{e^x - 1}{1 + e^{-x}}$  \citep{basirat2018quest}. 
    
    \item $\textrm{ELU} = x \texttt{ if } x \geq 0 \texttt{ else } \alpha (e^x-1)$, with $\alpha = 1$ \citep{elu}.
    
    \item $\textrm{GELU} = x \Phi(x)$, with $\Phi(x) = P(X \leq x), X \sim \mathcal{N}(0, 1)$, approximated as $0.5x(1 + \tanh[\sqrt{2/\pi}(x + 0.044715x^3)])$ \citep{hendrycks2016gaussian}.
        
    \item $\textrm{HardSigmoid} = \max\{0, \min\{1, 0.2x + 0.5\}\}$.
    
    \item $\textrm{Leaky ReLU} = x \texttt{ if } x \geq 0 \texttt{ else } 0.01 x$ \citep{maas2013rectifier}.
        
    \item $\textrm{Mish} = x \cdot \tanh(\textrm{Softplus}(x))$ \citep{misra2019mish}.         
    
    \item $\textrm{SELU} = \lambda x \texttt{ if } x \geq 0 \texttt{ else } \lambda \alpha (e^x-1)$, with $\lambda = 1.05070098$, $\alpha = 1.67326324$ \citep{selu}.
    
    \item $\textrm{sigmoid} = (1 + e^{-x})^{-1}$.
    
    \item $\textrm{Softplus} = \log(e^x + 1)$.
    
    \item $\textrm{Softsign} = x / (|x| + 1)$. 
    
    \item $\textrm{Swish} = x \cdot \sigma(x)$, with $\sigma(x) = (1+e^{-x})^{-1}$ \citep{DBLP:conf/iclr/RamachandranZL18, elfwing2018sigmoid}.
    
    \item $\textrm{tanh} = \frac{e^x - e^{-x}}{e^x + e^{-x}}$.
    
    \item $\textrm{PReLU} = x \texttt{ if } x \geq 0 \texttt{ else } \alpha x$, where $\alpha$ is a per-neuron learnable parameter initialized to 0.25 \citep{he2015delving}.
    
    \item $\textrm{PSwish} = x \cdot \sigma (\beta x)$, where $\beta$ is a per-channel learnable parameter \citep{DBLP:conf/iclr/RamachandranZL18}.
    
    \item $\textrm{APL} = \max\{0, x\} + \sum_{s=1}^S a_s \max\{0, -x+b_s\}$, where $S=7$ and $a_s$ and $b_s$ are per-neuron learnable parameters \citep{apl-agostinelli2014learning}.
    
    \item $\textrm{PAU} = \frac{\sum_{j=0}^m a_j x^j}{1 + |\sum_{k=1}^n b_kx^k|}$, where $m=5, n=4$, and $a_j$ and $b_k$ are per-layer learnable parameters initialized so that the function approximates Leaky ReLU with a slope of 0.01 \citep{pade-molina2019pad}.
        
    \item $\textrm{SPLASH} = \sum_{s=1}^{(S+1)/2} a_s^+ \max\{0, x-b_s\} + a_s^-\max\{0, -x-b_s\}$, where $S=7, b = [0, 1, 2, 2.5]$, and $a_s^+$ and $a_s^-$ are per-layer learnable parameters initialized as $a_1^+=1$ and all other $a = 0$ \citep{tavakoli2020splash}.
    
\end{itemize}

\section{Scope of PANGAEA Search Space}
\label{sec:pangaea:proofs}
This section shows that any piecewise real analytic function can be represented as a PANGAEA computation graph containing operators from Table \ref{tab:pangaea:searchspace}.  In the main experiments, PANGAEA computation graphs were restricted to having at most seven nodes and three learnable parameters $\alpha$, $\beta$, and $\gamma$ for efficiency.  Throughout this section the node and parameter constraints are removed.  Parameters take on the role of any real-valued constant, and the set of functions in PANGAEA without node or parameter constraints is denoted as $\mathcal{G}_\infty$.  Before proving the main result, the following two lemmas are needed.

\begin{lemma}
\label{lemma:pangaea:real_analytic}
If $f \in C^\omega$ is a real analytic function, then $f \in \mathcal{G}_\infty$.
\end{lemma}
\begin{proof}
As $f$ is real analytic, it can be expressed in the form
\begin{equation}
    \label{eq:pangaea:real_analytic}
    f(x) = \sum_{n=0}^\infty a_n(x-x_0)^n,
\end{equation}
with parameters $x_0, a_0, a_1, \ldots \in \mathbb{R}$.  As PANGAEA contains the zero, one, addition, and negation operators, the set of integers $\mathbb{Z}$ is contained in $\mathcal{G}_\infty$.  This accounts for the exponent $n$ in the expression above.  All other operators (addition, subtraction, multiplication, exponentiation) in Equation \ref{eq:pangaea:real_analytic} are included in Table \ref{tab:pangaea:searchspace}, and so $f \in \mathcal{G}_\infty$.
\end{proof}

\begin{lemma}
\label{lemma:pangaea:indicator}
Given parameters $a, b \in \mathbb{R}$ where $a < b$, the indicator functions
\begin{align}
    \mathbf{1}_{(-\infty, b)}(x) &= 
    \begin{cases}
        1 & x < b \\
        0 & x \geq b
    \end{cases}\\
    \mathbf{1}_{(a, \infty)}(x) &= 
    \begin{cases}
        1 & x > a \\
        0 & x \leq a
    \end{cases}\\
    \mathbf{1}_{(a,b)}(x) &= 
    \begin{cases}
        1 & x \in (a, b)\\
        0 & x \notin (a, b)
    \end{cases}\\
    \mathbf{1}_a(x) &= 
    \begin{cases}
        1 & x = a \\
        0 & x \neq a
    \end{cases}
\end{align}
are in $\mathcal{G}_\infty$.
\end{lemma}
\begin{proof}
Recall that PANGAEA implements the binary division operator $x_1 / x_2$ as \linebreak \texttt{tf.math.divide\_no\_nan}, which returns $0$ if $x_2 = 0$.  The indicator function $\mathbf{1}_{(-\infty, b)}(x)$ can then be implemented as
\begin{equation}
    \label{eq:pangaea:first_indicator}
    \mathbf{1}_{(-\infty, b)}(x) = \frac{\max\{b-x,0\}}{b-x}.
\end{equation}
There are three cases: if $x < b$, the expression evaluates to one.  If $x = b$ or $x > b$, the expression evaluates to zero.  By the same reasoning,
\begin{equation}
    \mathbf{1}_{(a, \infty)}(x) = \frac{\min\{a-x,0\}}{a-x},
\end{equation}
which evaluates to one if $x > a$ and zero otherwise.  Finally, note that
\begin{equation}
    \mathbf{1}_{(a,b)} = \mathbf{1}_{(-\infty, b)}\mathbf{1}_{(a, \infty)}
\end{equation}
and \begin{equation}
    \label{eq:pangaea:last_indicator}
    \mathbf{1}_a = (1-\mathbf{1}_{(-\infty, a)})(1-\mathbf{1}_{(a, \infty)}).
\end{equation}
All operators in Equations \ref{eq:pangaea:first_indicator}-\ref{eq:pangaea:last_indicator} (maximum, minimum, subtraction, multiplication, division, zero, one) are in the PANGAEA search space in Table \ref{tab:pangaea:searchspace}, and so the indicator functions are in $\mathcal{G}_\infty$.
\end{proof}

\begin{theorem}
\label{thm:pangaea:piecewise_real_analytic}
If a function $f$ is piecewise real analytic, then $f \in \mathcal{G}_\infty$.
\end{theorem}
\begin{proof}
If a function $f$ is piecewise real analytic, then it is representable by the form
\begin{equation}
    f(x) = 
    \begin{cases}
    f_0(x) & x \in (-\infty, k_1) \\
    K_1 & x = k_1\\
    f_1(x) & x \in (k_1, k_2) \\
    K_2 & x = k_2 \\
    & \vdots \\
    f_{n-1}(x) & x \in (k_{n-1}, k_n) \\
    K_n & x = k_n \\
    f_n(x) & x \in (k_n, \infty)
    \end{cases},
\end{equation}
where parameters $K_1, K_2, \ldots, K_N, k_1, k_2, \ldots, k_n \in \mathbb{R}$ are real-valued, $k_1 < k_2 < \cdots < k_n$ are increasing, and $f_0, f_1, \ldots, f_n \in C^\omega$ are real analytic functions.  An equivalent representation of $f$ is the following:
\begin{equation}
    \label{eq:pangaea:alt_representation}
    \begin{aligned}
        f(x) = \mathbf{1}_{(-\infty, k_1)}(x)f_0(x) + \mathbf{1}_{k_1}(x)K_1 + \mathbf{1}_{(k_1,k_2)}(x)f_1(x) + \mathbf{1}_{k_2}(x)K_2 + \cdots \\ \phantom{} + \mathbf{1}_{(k_{n-1},k_n)}(x)f_{n-1}(x) + \mathbf{1}_{k_n}(x)K_n + \mathbf{1}_{(k_n,\infty)}(x)f_{n}(x).
    \end{aligned}
\end{equation}
By Lemmas \ref{lemma:pangaea:real_analytic} and \ref{lemma:pangaea:indicator}, the real analytic functions $f_i$ and the indicator functions $\mathbf{1}_{(\cdot, \cdot)}$ are in $\mathcal{G}_\infty$.  Beyond these functions, Equation \ref{eq:pangaea:alt_representation} utilizes only addition and multiplication, both of which are operators included in Table \ref{tab:pangaea:searchspace}.  Therefore, $f \in \mathcal{G}_\infty$.
\end{proof}

\section{Size of the PANGAEA Search Space}
\label{sec:pangaea:searchspace}

This section analyzes the size of the PANGAEA search space as implemented in the main experiments.  Let $g$ represent a general computation graph, and let $f$ represent a specific activation function representable by $g$.  For example, if we have $g(x) = \texttt{binary}(\texttt{unary1}(x), \texttt{unary2}(x))$, then one possible activation function is $f(x) = \tanh(x) + \textrm{erf}(x)$, and another could be $f(x) = \alpha| x| \cdot \sigma(\beta \cdot x)$.  

Let $U=27$ and $B=7$ be the number of unary and binary operators in the PANGAEA search space, respectively, and let $E=3$ be the maximum number of learnable parameters that can be used to augment a given activation function.  Given a computation graph $g$, let $u_g$ and $b_g$ be the number of unary and binary nodes and let $e_g$ be the number of edges in $g$.  For example, with the functional form $g(x) = \texttt{unary1}(\texttt{unary2}(x))$, we have $u_g=2$, $b_g=0$, and $e_g=3$.  With $g(x) = \texttt{binary}(\texttt{unary1}(x), \texttt{unary2}(x))$, we have $u_g=2$, $b_g=1$, and $e_g=5$.  The quantity $e_g$ includes the edges from the input nodes $x$ and edges to the output node $g(x)$ (see Figure \ref{fig:pangaea:initialization}).

\begin{table}[ht]
    \centering
    \caption{The number of activation functions representable by a computation graph with a given number of nodes.  The PANGAEA search space contains over ten trillion activation functions, and therefore provides a good foundation for finding powerful activation functions with different properties.\\}
    \adjustbox{max width=\linewidth}{%
    \begin{tabular}{cccccc}
    \toprule
    & Binary Nodes $b_g$ & Unary Nodes $u_g$ & Edges $e_g$ & Arrangements & Activation Functions \\ \midrule
    $\mathcal{G}_1$                & 0 & 1 & 2 & 1  & 108 \\ \midrule
    $\mathcal{G}_2$                & 0 & 2 & 3 & 1 & 5,832 \\ \midrule
    \multirow{2}*{$\mathcal{G}_3$} & 0 & 3 & 4 & 1 & \multirow{2}*{427,923} \\
                                   & 1 & 2 & 5 & 1 \\ \midrule
    \multirow{2}*{$\mathcal{G}_4$} & 0 & 4 & 5 & 1 & \multirow{2}*{31,177,872} \\
                                   & 1 & 3 & 6 & 3 \\ \midrule
    \multirow{3}*{$\mathcal{G}_5$} & 0 & 5 & 6 & 1 & \multirow{3}*{2,210,558,364} \\
                                   & 1 & 4 & 7 & 6 \\
                                   & 2 & 3 & 8 & 2 \\ \midrule
    \multirow{3}*{$\mathcal{G}_6$} & 0 & 6 & 7 & 1 & \multirow{3}*{152,059,087,566} \\
                                   & 1 & 5 & 8 & 10 \\
                                   & 2 & 4 & 9 & 10 \\ \midrule
    \multirow{4}*{$\mathcal{G}_7$} & 0 & 7 & 8 & 1 & \multirow{4}*{10,015,741,690,785} \\
                                   & 1 & 6 & 9 & 15 \\
                                   & 2 & 5 & 10 & 30 \\
                                   & 3 & 4 & 11 & 1 \\
    \bottomrule
    \end{tabular}
    }
    \label{tab:pangaea:pangaea_analysis}
\end{table}

Let $\mathcal{F}_g$ denote the set of all activation functions $f$ that can be represented within the computation graph $g$.  By enumerating the different choices of unary and binary operators, as well as the locations for up to three learnable parameters $\alpha$, $\beta$, and $\gamma$, we find the size of the set to be
\begin{equation}
    | \mathcal{F}_g | = U^{u_g} \cdot B^{b_g} \cdot \sum_{i=0}^E \binom{e_g}{i}.
\end{equation}
Let $\mathcal{G}_j$ denote the set of computation graphs $g$ containing $j$ nodes.  For example,
\begin{equation}
\mathcal{G}_3 = \{ g(x) = \texttt{unary1}(\texttt{unary2}(\texttt{unary3}(x))), g(x) = \texttt{binary}(\texttt{unary1}(x), \texttt{unary2}(x)) \}.    
\end{equation}
Table \ref{tab:pangaea:pangaea_analysis} shows the possible combinations of binary nodes $b_g$, unary nodes $u_g$, and edges $e_g$ for each set $\mathcal{G}_j$.  Additionally, the table shows the number of computation graph arrangements possible for a given $b_g$, $u_g$, and $e_g$.  For example, if $b_g = 2$, $u_g = 3$, and $e_g = 8$, the computation graph could take one of two forms: either 
\begin{equation}
g(x) = \texttt{binary1}(\texttt{binary2}(\texttt{unary1}(x), \texttt{unary2}(x)), \texttt{unary3}(x))    
\end{equation}
or
\begin{equation}
g(x) = \texttt{binary1}(\texttt{unary1}(x), \texttt{binary2}(\texttt{unary2}(x), \texttt{unary3}(x))).
\end{equation}
The number of activation functions in PANGAEA is therefore
\begin{equation}
    \sum_{j=1}^7 \sum_{g \in \mathcal{G}_j} |\mathcal{F}_g| = 10{,}170{,}042{,}948{,}450.
\end{equation}

Naturally there exist duplicates within this space.  The functions $f(x) = \textrm{ReLU}(x)$ and $f(x) = \max\{x, 0\}$ have different computation graphs but are functionally identical.  Nevertheless, this analysis still provides a useful characterization of the size and diversity of the PANGAEA search space.  It is orders of magnitude larger than spaces considered in prior work \citep{bingham2020gecco, DBLP:conf/iclr/RamachandranZL18, basirat2018quest}, and yet PANGAEA consistently discovers functions that outperform ReLU and other baseline functions.


\section{Discussion}
\label{sec:pangaea:discussion}

It is difficult	to select an appropriate activation function for a
given architecture because the activation function, network 
topology, and training setup interact in
complex ways.  It is especially promising that PANGAEA discovered activation functions that significantly outperformed
the baselines, since the architectures and training setups were standard and developed with ReLU. A compelling research direction is to jointly optimize the architecture, training setup, and activation function.

More specifically, there has been significant recent research in
automatically discovering the architecture of neural networks through
gradient-based, reinforcement learning, or neuroevolutionary methods
\citep{elsken2019neural, wistuba2019survey, real2019regularized}.  In
related work, evolution was used discover novel loss functions
automatically
\citep{gonzalez2020improved,gonzalez2020evolving,liang2020population},
outperforming the standard cross entropy loss.	In the future, it may
be possible to optimize many of these aspects of neural network design
jointly.  Just as new activation functions improve the accuracy of
existing network architectures, it is likely that different
architectures will be discovered when the activation function is not
ReLU.  One such example is EfficientNet \citep{tan2019efficientnet},
which achieved state-of-the-art accuracy for ImageNet \citep{deng2009imagenet} using the Swish activation function
\citep{DBLP:conf/iclr/RamachandranZL18, elfwing2018sigmoid}.
Coevolution \cite{gomez2008accelerated, moriarty1997forming, yang2008large, potter2000cooperative} of activation functions, topologies, loss
functions, and possibly other aspects of neural network design
could allow taking advantage of interactions between them, leading to
further improvements in	the future.  

Similarly, there can be multiple properties of an activation function
that make it useful in different scenarios.  PANGAEA optimized for
activation functions that lead to high accuracy.  Another promising area
of future work is to search for activation functions that are also efficient to compute (Figure \ref{fig:pangaea:scatter_hist_lcn}), improve adversarial robustness \cite{xie2020smooth}, stabilize training \cite{selu}, or meet other objectives like easing optimization or providing implicit regularization \cite{li2019towards}. These functions could be discovered with multi-objective optimization, or through other methods like a carefully designed search space or adding regularization to the parameters of the activation functions.

\section{Conclusion}

This chapter introduced PANGAEA, a technique for automatically designing
novel, high-performing, parametric activation functions.  PANGAEA
builds a synergy of two different optimization processes: evolutionary
population-based search	for the	general	form, and
gradient descent-based fine-tuning of the parameters of	the activation
function.  Compared to previous studies \cite{bingham2020gecco, DBLP:conf/iclr/RamachandranZL18}, the search space is extended
to include deeper and more complex functional forms, including ones
unlikely to be discovered by humans.  The parameters are adapted during 
training and are different in different locations of the
architecture, thus customizing the functions over both time and space and resulting in improved performance as a result.
PANGAEA is able to discover general activation functions that perform
well across architectures, and specialized functions taking advantage of
a particular architecture, significantly outperforming previously proposed
activation functions in both cases.
It is thus a promising step towards automatic configuration of neural networks.  In order to continue this progress, the next chapter develops an automatic weight initialization algorithm.  This algorithm provides a contribution to AutoML on its own, but also complements PANGAEA by making the evaluation of novel activation functions more robust.

\chapter{AutoInit: Analytic Signal-Preserving Weight Initialization for Neural Networks}
\label{chap:autoinit}

In order to ensure stable signal propagation in neural networks, it is crucial that the weight initialization strategy takes the shape of the activation function into account.  While the previous chapters discovered new activation functions that improved the performance of various neural networks, some activation functions evaluated during the searches turned out to be ineffective.  Many of these functions could have been effective if the network had been properly initialized.  Because current weight initialization algorithms to not automatically adapt to novel activation functions, this chapter presents an approach that does, called AutoInit.  AutoInit adapts to novel activation functions and layer types in a neural network.  It thus constitutes progress in AutoML in its own right, and compliments the previous chapters by making the evaluation of new activation functions more reliable.  The AutoInit package is available at \url{https://github.com/cognizant-ai-labs/autoinit}.

\section{Motivation}

Proper weight initialization is crucial to achieve high performance with deep networks.  A common motif in such networks is repeated layers or building blocks.  Thus, if a given layer amplifies or diminishes the forward or backward propagation of signals, repeated applications of that layer will result in exploding or vanishing signals, respectively \cite{hochreiter1991untersuchungen, hanin2018neural}.  This phenomenon makes optimization difficult, and can even exceed machine precision. The issue persists regardless of whether the weights are sampled to be uniform, normal, or orthogonal \cite{saxe2013exact, hu2020provable}.

While many initialization strategies have been proposed in the past, these strategies apply only to neural networks with specific activation functions, topologies, or layer types.  Thus, researchers designing new models or activation functions have two options.  The first option is to derive weight initialization strategies manually for every architecture considered, which is generally difficult and time consuming.  The second option is to use existing initialization strategies in new settings, where they may be incorrect and therefore misleading:  A candidate model may appear poor when it is the suboptimal initialization that makes training difficult.

To overcome this problem, this chapter proposes AutoInit, an algorithm that automatically calculates analytic mean- and variance-preserving weight initialization for neural networks.  Since AutoInit is algorithmic, it relieves the researcher from a difficult but consequential step in model design.  It is no longer necessary to use existing weight initialization strategies in incorrect settings: AutoInit provides an appropriate default initialization automatically, resulting in better and more reliable performance.

\section{Neural Network Signal Propagation}
\label{sec:autoinit:signal_propagation}

AutoInit aims to stabilize signal propagation throughout an entire neural network. More precisely, consider a layer that shifts its input by $\alpha$ and scales the input by a factor of $\beta$.  Given an input signal with mean $\mu_\inn$ and variance $\nu_\inn$, after applying the layer, the output signal will have mean $\mu_\out = \alpha + \beta \mu_\inn$ and variance $\nu_\out = \beta^2 \nu_\inn$.  In a deep network in which the layer is applied $L$ times the effect is compounded and the signal at the final layer has mean and variance
\begin{equation}
    \mu_\out = \beta^L\mu_\inn + \alpha(\beta^L +\beta^{L-1} + \cdots + \beta + 1), \;\; \nu_\out = \beta^{2L}\nu_\inn.
\end{equation}
If $|\beta| > 1$, the network will suffer from a mean shift and exploding signals as it increases in depth:
\begin{equation}
    \label{eq:autoinit:explode}
    \lim_{L \rightarrow \infty} \mu_\out = \infty, \quad
    \lim_{L \rightarrow \infty} \nu_\out = \infty.
\end{equation}
In the case that $|\beta| < 1$, the network will suffer from a mean shift and vanishing signals:
\begin{equation}
    \label{eq:autoinit:vanish}
    \lim_{L \rightarrow \infty} \mu_\out = \alpha / (1 - \beta), \quad
    \lim_{L \rightarrow \infty} \nu_\out = 0.
\end{equation}
AutoInit calculates analytic mean- and variance-preserving weight initialization so that $\alpha=0$ and $\beta=1$, thus avoiding the issues of mean shift and exploding/vanishing signals.

\section{The AutoInit Framework}
\label{sec:autoinit:algorithm}

AutoInit is a general framework that adapts to different layer types.  Its implementation is outlined in Algorithm \ref{alg:autoinit}. A given \texttt{layer} in a neural network receives as its input a tensor $x$ with mean $\mu_\inn$ and variance $\nu_\inn$.  After applying the \texttt{layer}, the output tensor has mean $\mu_\out = \E(\texttt{layer}(x))$ and variance $\nu_\out = \Var(\texttt{layer}(x))$.  The function $g_{\texttt{layer}}$ maps input mean and variance to output mean and variance when the \texttt{layer} is applied:
\begin{equation}\small
    \label{eq:autoinit:g}
    g_\texttt{layer} : (\mu_\inn, \nu_\inn) \mapsto (\mu_\out, \nu_\out).
\end{equation}
Note that $g$ in Equation \ref{eq:autoinit:g} depends on the type of $\texttt{layer}$; e.g.\ $g_{\texttt{Dropout}}$ and $g_{\texttt{ReLU}}$ are different functions.  For layers with trainable weights, the mean and variance mapping will depend on those weights.  For example, the function $g_{\texttt{Conv2D},\theta}$ maps input mean and variance to output mean and variance after the application of a \texttt{Conv2D} layer parameterized by weights $\theta$.  Deriving $g$ for all layers makes it possible to model signal propagation across an entire neural network.  Thus, if $\mu_\inn$ and $\nu_\inn$ are known, it is natural to calculate initial weights $\theta$ such that the layer output will have zero mean and unit variance.  
\SetKwProg{Fn}{def}{:}{\textbf{end}}
\SetKwFunction{initlayer}{initialize}%
\begin{algorithm}[t]
\caption{AutoInit}
\KwIn{Network with layers $L$, directed edges $E$}
$\texttt{output\_layers} = \{l \in L \mid (l, l') \notin E \:\forall\: l' \in L \}$\\
\For{$\textrm{\tt output\_layer in output\_layers}$}{
    \texttt{initialize}(\texttt{output\_layer})
}
~\\
\Fn{\initlayer{\rm \tt layer}}{
    $\texttt{layers\_in} = \{l \in L \mid (l, \texttt{layer}) \in E\}$ \newline
    $i = 1$\\
    \For{$\textrm{\tt layer\_in in layers\_in}$}{
        $\mu_{\inn_i}, \nu_{\inn_i} = \texttt{initialize}(\texttt{layer\_in})$\\
        $i = i + 1$
    }
    $\mu_\inn = (\mu_{\inn_1}, \mu_{\inn_2}, \ldots, \mu_{\inn_N})$\\
    $\nu_\inn = (\nu_{\inn_1}, \nu_{\inn_2}, \ldots, \nu_{\inn_N})$\\
    \eIf{\textrm{\tt layer} {\rm has weights} $\theta$}{
        initialize $\theta$ s.t. $g_{\texttt{layer}, \theta}(\mu_\inn, \nu_\inn) = (0,1)$\\
        $\mu_\out, \nu_\out = 0, 1$
    }{
        $\mu_\out, \nu_\out = g_{\texttt{layer}}(\mu_\inn, \nu_\inn)$
    }
    \Return $\mu_\out, \nu_\out$
}
\label{alg:autoinit}
\end{algorithm}

For example, for \texttt{Conv2D} layers, one possibility is
\begin{equation}
    \label{eq:autoinit:init_example}
    \theta \sim \mathcal{N}\left(0, 1/\sqrt{\texttt{fan\_in}(\nu_\inn + \mu_\inn^2)}\right) \! \implies \! g_{\texttt{Conv2D}, \theta} (\mu_\inn, \nu_\inn) = (0, 1).
\end{equation}

The AutoInit framework includes mean and variance mapping functions $g$ for the majority of layers used in modern architectures.  Section \ref{sec:autoinit:mean_variance_estimation} details how these functions and the corresponding initialization strategies (e.g.\ Equation \ref{eq:autoinit:init_example}) are derived.  New layers can be included by deriving $g$ manually, or by approximating it through Monte Carlo simulation.  This approach ensures that reliable estimates for $\mu_\inn$ and $\nu_\inn$ are available at all layers in a network, which in turn allows for weight initialization that stabilizes the signals to have zero mean and unit variance, avoiding the issues of mean shift and exploding/vanishing signals (Equations \ref{eq:autoinit:explode} and \ref{eq:autoinit:vanish}).

The main advantage of AutoInit is that it is a general method.  Unlike prior work, which imposes design constraints, AutoInit adapts to different settings automatically in order to improve performance.  Sections \ref{sec:autoinit:convolutional} through \ref{sec:autoinit:afn_meta_learning} demonstrate this adaptability experimentally from several perspectives: different classes of models (convolutional, residual, transformer), hyperparameter settings (activation function, dropout rate, weight decay, learning rate, optimizer), model depths (nine layer CNN to 812 layer ResNet), image sizes (ten-class $28 \times 28$ grayscale to 1{,}000-class $160 \times 160$ RGB), and data modalities (vision, language, tabular, multi-task, transfer learning).  AutoInit also outperforms data-dependent initialization methods and stabilizes convolutional, residual, and transformer networks without normalization layers.  This generality is shown to be particularly useful in neural architecture search and activation function discovery, where thousands of new designs need to be evaluated robustly.  AutoInit produces specialized weight initialization strategies for each candidate, which allows for measuring their performance more accurately. As a result, better solutions are discovered.  The experiments thus show that AutoInit is an effective initialization algorithm for existing networks as well as a good starting point for networks that may be developed in the future.

\section{Mean and Variance Estimation for Different Layer Types}
\label{sec:autoinit:mean_variance_estimation}

In the AutoInit framework of Algorithm~\ref{alg:autoinit}, the mean and variance mapping function $g$ needs to be defined for each type of layer in a given neural network.  This section presents derivations for a majority of the most commonly used layers available in the TensorFlow package \cite{abadi2016tensorflow} at time of writing.  Extending AutoInit to support new layers in the future will require deriving the function $g$ for those layers.  Monte Carlo sampling can also be used as an approximation for $g$ before it is manually derived.  

In the following paragraphs, $x$ denotes the input to a layer, and $y$ is the output.  The incoming and outgoing means and variances are denoted as $\mu_\inn \coloneqq \E(x)$, $\mu_\out \coloneqq \E(y)$, $\nu_\inn \coloneqq \Var(x)$, and $\nu_\out \coloneqq \Var(y)$.  The notation \texttt{Conv\{1D,2D,3D\}} is used to refer to \texttt{Conv1D}, \texttt{Conv2D}, and \texttt{Conv3D}, and analogously for other layer types.  Inputs to each layer are assumed to be independent and normally distributed.  Although these assumptions may not always hold exactly, experiments show that AutoInit models signal propagation across different types of networks well in practice.

\paragraph{Convolution and Dense Layers}
The analysis in the next paragraph applies to \texttt{Conv\{1D,\allowbreak2D,\}}, \texttt{DepthwiseConv\{1D,2D\}}, and \texttt{Dense} layers, since convolution layers are dense layers with sparse connectivity.  Notation and derivation are inspired by that of \citet{glorot2010understanding} and \citet{he2015delving}.  

A feedforward layer can be written as $y = W x + b$, where $x$ is the input, $W$ is a $\texttt{fan\_out} \times \texttt{fan\_in}$ weight matrix, $b$ is a vector of biases, and $y$ is the result.  Assume the elements of $W$ are mutually independent and from the same distribution, and likewise for the elements of $x$.  Further assume that $W$ and $x$ are independent of each other. The outgoing mean can then be written as $\mu_\out = \E(W)\mu_\inn$. For the outgoing variance, letting $W$ have zero mean and expanding the product of independent random variables yields $\nu_\out = \texttt{fan\_in}\Var(W)(\nu_\inn + \mu_\inn^2)$.  Sampling the weights $W$ according to
\begin{equation} \label{eq:autoinit:variance_scaling}
    W \sim \mathcal{N}\left(0, \frac{1}{\sqrt{\texttt{fan\_in}(\nu_\inn + \mu_\inn^2)}}\right)
\end{equation}
or
\begin{equation} \label{eq:autoinit:variance_scaling_1}
    {\small
    W \sim \mathcal{U}\left(-\frac{\sqrt{3}}{\sqrt{\texttt{fan\_in}(\nu_\inn + \mu_\inn^2)}}, \frac{\sqrt{3}}{\sqrt{\texttt{fan\_in}(\nu_\inn + \mu_\inn^2)}}\right)}
\end{equation}
is sufficient to ensure that 
\begin{equation}
    \mu_\out = 0 \textrm{ and } \nu_\out = 1.
\end{equation}


\paragraph{Activation Functions}
The analysis in the next paragraph accounts for all activation functions in TensorFlow, including \texttt{elu}, \texttt{exponential}, \texttt{gelu}, \texttt{hard\_sigmoid}, \texttt{LeakyReLU}, \texttt{linear}, \texttt{PReLU}, \texttt{ReLU}, \texttt{selu}, \texttt{sigmoid}, \texttt{softplus}, \texttt{softsign}, \texttt{swish}, \texttt{tanh}, and \texttt{ThresholdedReLU} \citep[by][respectively]{elu, hendrycks2016gaussian, maas2013rectifier, he2015delving, nair2010rectified, selu, DBLP:conf/iclr/RamachandranZL18, elfwing2018sigmoid, courbariaux2015binaryconnect}, and in fact extends to any integrable \texttt{Activation} function $f$.

Let $p_\mathcal{N}(x; \mu, \sigma)$ denote the probability density function of a Gaussian distribution with mean $\mu$ and standard deviation $\sigma$.  By the law of the unconscious statistician,
\begin{align}
    \mu_\out &= \int_{-\infty}^\infty f(x) p_\mathcal{N}(x; \mu_\inn, \sqrt{\nu_\inn}) \diff x, \\
    \nu_\out &= \int_{-\infty}^\infty f(x)^2 p_\mathcal{N}(x; \mu_\inn, \sqrt{\nu_\inn}) \diff x - \mu_\out^2.
\end{align}
These integrals are computed for an arbitrary activation function $f$ with adaptive quadrature, a well-established numerical integration approach that approximates integrals using adaptively refined subintervals \cite{piessens2012quadpack, 2020SciPy-NMeth, numpy-harris2020array}.

\paragraph{Dropout Layers}
Dropout layers randomly set \texttt{rate} percentage of their inputs to zero \cite{srivastava2014dropout}.  Therefore, 
\begin{equation}
    \mu_\out = \mu_\inn(1 - \texttt{rate}) \textrm{ and } \nu_\out = \nu_\inn(1 - \texttt{rate}).
\end{equation}
However, this analysis only applies to \texttt{SpatialDropout\{1D,2D,3D\}} layers.  For regular \texttt{Dropout} layers, TensorFlow automatically scales the values by $1 / (1 - \texttt{rate})$ to avoid a mean shift towards zero.\footnote{\url{https://github.com/tensorflow/tensorflow/blob/v2.5.0/tensorflow/python/keras/layers/core.py\#L149-L150}}  Adjusting for this change gives 
\begin{equation}
    \mu_\out = \mu_\inn \textrm{ and } \nu_\out = \nu_\inn/(1 - \texttt{rate}).
\end{equation}

\paragraph{Pooling Layers}
The same approach applies to all commonly used pooling layers, including \texttt{AveragePooling\{1D,2D,3D\}}, \texttt{MaxPooling\{1D,\allowbreak2D,3D\}}, \texttt{GlobalAveragePooling\{1D,2D,3D\}}, and \texttt{GlobalMaxPooling\{1D,2D,3D\}}.

Let $op(\cdot)$ be the average operation for an average pooling layer, and the maximum operation for a max pooling layer.  Define $K$ to be the pool size of the layer.  For standard 1D, 2D, and 3D pooling layers, $K$ would equal $k$, $k \times k$, and $k \times k \times k$, respectively.  The global pooling layers can be seen as special cases of the standard pooling layers where the pool size is the same size as the input tensor, except along the batch and channel dimensions.  Analytically, the outgoing mean and variance can be expressed as
\begin{align}
        \mu_\out &= \idotsint_{\mathbb{R}^{K}} op(x_1, x_2, \ldots, x_{K}) \cdot \prod_{i=1}^{K} p_\mathcal{N}(x_i; \mu_\inn, \sqrt{\nu_\inn}) \diff x_1 \diff x_2 \cdots \diff x_{K},\\
        \nu_\out &= \idotsint_{\mathbb{R}^{K}} op(x_1, x_2, \ldots, x_{K})^2 \cdot \prod_{i=1}^{K} p_\mathcal{N}(x_i; \mu_\inn, \sqrt{\nu_\inn}) \diff x_1 \diff x_2 \cdots \diff x_{K} - \mu_\out^2,
\end{align}
where the $x_i$ represent tensor entries within a pooling window.  Unfortunately, even a modest $3 \times 3$ pooling layer requires computing nine nested integrals, which is prohibitively expensive.  In this case, a Monte Carlo simulation is appropriate.
Sample $x_{1_j}, x_{2_j}, \ldots x_{K_j}$ from $\mathcal{N}(\mu_\inn, \sqrt{\nu_\inn})$ for $j = 1, \ldots, S$ and return
\begin{align}
    \mu_\out &= \frac{1}{S} \sum_{j=1}^S op(x_{1_j}, x_{2_j}, \ldots, x_{K_j}),\\
    \nu_\out &= \frac{1}{S} \sum_{j=1}^S op(x_{1_j}, x_{2_j}, \ldots, x_{K_j})^2 - \mu_\out.
\end{align}

\paragraph{Normalization Layers}
\texttt{BatchNormalization}, \texttt{LayerNormalization}, and \linebreak \texttt{GroupNormalization} normalize the input to have mean zero and variance one \cite{ioffe2015batch, ba2016layer, wu2018group}.  Thus, 
\begin{equation}
    \mu_\out = 0 \textrm{ and } \nu_\out = 1.
\end{equation}

\paragraph{Arithmetic Operators}
Assume the input tensors $x_1, x_2, \ldots, x_N$ with means \linebreak $\mu_{\inn_1}, \mu_{\inn_2}, \ldots, \mu_{\inn_N}$ and variances $\nu_{\inn_1}, \nu_{\inn_2}, \ldots, \nu_{\inn_N}$ are independent.  The following mean and variance mapping functions are derived.
    For an \texttt{Add} layer,
    \begin{equation}\mu_\out = \sum_{i=1}^N \mu_{\inn_i} \textrm{ and } \nu_\out = \sum_{i=1}^N \nu_{\inn_i}.\end{equation}
    For an \texttt{Average} layer,
    \begin{equation}\mu_\out = \frac{1}{N}\sum_{i=1}^N \mu_{\inn_i} \textrm{ and } \nu_\out = \frac{1}{N^2}\sum_{i=1}^N
    \nu_{\inn_i}.\end{equation}
    For a \texttt{Subtract} layer,
    \begin{equation}\mu_\out = \mu_{\inn_1} - \mu_{\inn_2} \textrm{ and } \nu_\out = \nu_{\inn_1} + \nu_{\inn_2}.\end{equation}
    Finally, for a \texttt{Multiply} layer,
    \begin{equation}\mu_\out = \prod_{i=1}^N \mu_{\inn_i} \textrm{ and } \nu_\out = \prod_{i=1}^N (\nu_{\inn_i} + \mu_{\inn_i}^2) - \prod_{i=1}^N \mu_{\inn_i}^2.\end{equation}


\paragraph{Concatenation Layers}
Assume the inputs $x_1, x_2, \ldots, x_N$ with means $\mu_{\inn_1}, \mu_{\inn_2}, \ldots, \mu_{\inn_N}$ and variances $\nu_{\inn_1}, \nu_{\inn_2}, \ldots, \nu_{\inn_N}$ are independent, and let input $x_i$ have $C_i$ elements.  Then for a \texttt{Concatenate} layer,
\begin{align}
    \mu_\out &= \frac{1}{\sum C_i} \sum_{i=1}^N C_i \mu_{\inn_i},\\
    \nu_\out &= \frac{1}{\sum C_i}\sum_{i=1}^N C_i(\nu_{\inn_i} + \mu_{\inn_i}^2) - \mu_\out^2.
\end{align}

\paragraph{Recurrent Layers}
A Monte Carlo simulation can be used to estimate the outgoing mean and variance for recurrent layers, including \texttt{GRU}, \texttt{LSTM}, and \texttt{SimpleRNN} \cite{hochreiter1997long, chung2014empirical}.  Recurrent layers often make use of activation functions like sigmoid and tanh that constrain the scale of the hidden states.  Because of this practice, recurrent layers should be initialized with a default scheme or according to recent research in recurrent initialization \cite{chen2018dynamical, gilboa2019dynamics}.  AutoInit will then estimate the outgoing mean and variance in order to inform appropriate weight scaling elsewhere in the network.

\paragraph{Padding Layers}
\texttt{ZeroPadding\{1D,2D,3D\}} layers augment the borders of the input tensor with zeros, increasing its size.  Let $z$ be the proportion of elements in the tensor that are padded zeros.  Then $z = (\texttt{padded\_size} - \texttt{original\_size}) / \texttt{padded\_size}$, and
\begin{equation}
    \mu_\out = \mu_\inn(1-z) \textrm{ and } \nu_\out = \nu_\inn(1-z).
\end{equation}

\paragraph{Shape Adjustment Layers}
Many layers alter the size or shape of the input tensor but do not change the distribution of the data.  These layers include \texttt{Flatten}, \texttt{Permute}, \texttt{Reshape}, \texttt{UpSampling\{1D,2D,3D\}}, and \texttt{Cropping\{1D,2D,3D\}} layers.  The same is true of TensorFlow API calls \texttt{tf.reshape}, \texttt{tf.split}, and \texttt{tf.transpose}.  For these layers, 
\begin{equation}
    \mu_\out = \mu_\inn \textrm{ and } \nu_\out = \nu_\inn.
\end{equation}

\paragraph{Input Layers}
An \texttt{InputLayer} simply exposes the model to the data, therefore 
\begin{equation}
    \mu_\out = \mu_{\mathrm{data}} \textrm{ and } \nu_\out = \nu_{\mathrm{data}}.
\end{equation} 
TensorFlow allows nesting models within other models.  In this use case where the \texttt{InputLayer} does not directly connect to the training data, 
\begin{equation}
    \mu_\out = \mu_\inn \textrm{ and } \nu_\out = \nu_\inn.
\end{equation}

\paragraph{Matrix Multiplication}
A call to \texttt{tf.matmul} takes input tensors $x_1$ and $x_2$ of shape $\cdots \times m \times n$ and $\cdots \times n \times p$ and produces the output tensor $x_{\out}$ of shape $\cdots \times m \times p$ with entries
\begin{equation}
    x_{\out_{\cdots,i,j}} = \sum_{k = 1}^n x_{1_{\cdots,i,k}} x_{2_{\cdots,k,j}}.
\end{equation}
Assuming independent matrix entries, the output statistics can then be calculated as
\begin{align}
    \mu_\out &= n \mu_{\inn_1} \mu_{\inn_2},\\
    \nu_\out &= n \left((\nu_{\inn_1} + \mu_{\inn_1}^2)(\nu_{\inn_2} + \mu_{\inn_2}^2) - \mu_{\inn_1}^2\mu_{\inn_2}^2\right).
\end{align}

\paragraph{Reduction Operators}
A call to $\texttt{tf.reduce\_mean}$ reduces the size of the input tensor by averaging values along one or more axes.  For example, averaging an input tensor of shape $128 \times 8 \times 8 \times 256$ along axes 1 and 2 would produce an output tensor of shape $128 \times 1 \times 1 \times 256$.  Let $D$ represent the product of the length of the axes being averaged over (in the example above, $8 \times 8 = 64$).  The output tensor has
\begin{equation}
    \mu_\out = \mu_\inn \textrm{ and } \nu_\out = \nu_\inn / D.
\end{equation}
The function $\texttt{tf.reduce\_sum}$ performs similarly, summing entries instead of averaging them.  In this case,
\begin{equation}
    \mu_\out = D\mu_\inn \textrm{ and } \nu_\out = D\nu_\inn.
\end{equation}

\paragraph{Maintaining Variance $\mathbf{\neq 1}$.}
In Algorithm \ref{alg:autoinit}, AutoInit initializes weights so that the output signal at each layer has mean zero and variance one.  Although signal variance $\nu = 1$ is sufficient for stable signal propagation, it is not a necessary condition.  Indeed, other values for the signal variance $\nu$ could be utilized, as long as $\nu$ remains consistent throughout the network.  If a different $\nu$ is desired, weights can be initialized according to Equation \ref{eq:autoinit:variance_scaling} or \ref{eq:autoinit:variance_scaling_1} and then multiplied by $\sqrt{\nu}$. For instance, such a modification was done for the CoAtNet model in Section~\ref{sec:autoinit:coatnet}, resulting in slightly improved final performance.

\section{Hyperparameter Variation in CNNs}
\label{sec:autoinit:convolutional}

\paragraph{Experiment Setup}
The first experiment tests AutoInit's performance across a range of hyperparameter values for CNNs.  The experiment focuses on the All-CNN-C architecture \cite{springenberg2015striving}, which consists of convolutional layers, ReLU activation functions, dropout layers, and a global average pooling layer at the end of the network.  This simple design helps identify performance gains that can be attributed to proper weight initialization.  The network is trained on the CIFAR-10 dataset \cite{krizhevsky2009learning} using the standard setup (Appendix \ref{ap:details:autoinit}).  In particular, the baseline comparison is the ``Glorot Uniform'' strategy \citep[also called Xavier initialization; ][]{glorot2010understanding}, where weights are sampled from $\mathcal{U}\left(-\frac{\sqrt{6}}{\sqrt{\texttt{fan\_in} + \texttt{fan\_out}}}, \frac{\sqrt{6}}{\sqrt{\texttt{fan\_in} + \texttt{fan\_out}}}\right)$.

\paragraph{Hyperparameter Variation}
In separate experiments, the activation function, dropout rate, weight decay, and learning rate multiplier were changed.  While one hyperparameter was varied, the others were fixed to the default values.

\paragraph{Results}
Figure \ref{fig:autoinit:allcnnc_hyperparams} shows the performance of the network with the default initialization and with AutoInit in these different settings.  
In sum, AutoInit improved performance in every hyperparameter variation evaluated.  As Figure~\ref{fig:autoinit:signal_propagation_allcnnc} shows, AutoInit is adaptive.  It alters the initialization to account for different activation functions and dropout rates automatically.

\begin{figure}
    \centering
    \includegraphics[width=\linewidth]{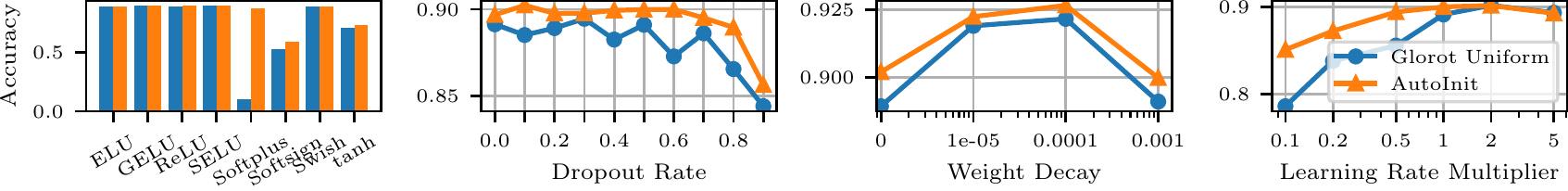}
    \caption{All-CNN-C test accuracy on CIFAR-10.  AutoInit results in comparable or better performance with different activation functions, better performance across all dropout rates and weight decay settings, and is less sensitive to the choice of learning rate than the default initialization.}
    \label{fig:autoinit:allcnnc_hyperparams}
\end{figure}

\begin{figure}
    \centering
    \begin{tikzpicture}
    \draw (0, 0) node[anchor=south west, inner sep=0, align=left] {\includegraphics[width=\linewidth]{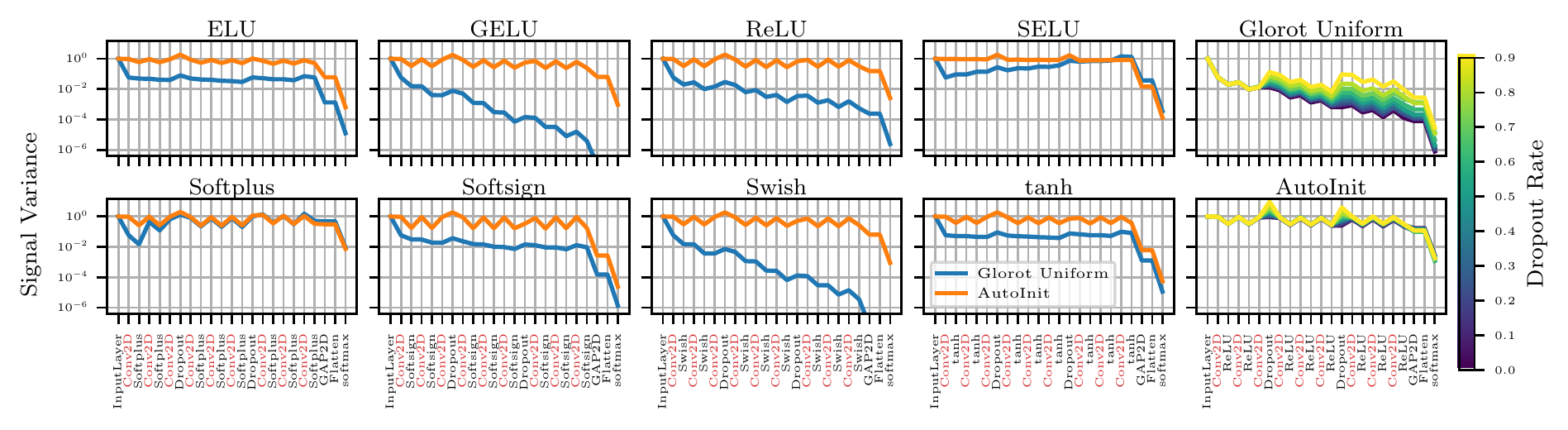}};
    \draw (0.6, 0) node[anchor=south west, inner sep=0, align=left] {(a)};
    \draw (13.2, 0) node[anchor=south west, inner sep=0, align=left] {(b)};
    \end{tikzpicture}
    \caption{Signal propagation in All-CNN-C networks with different (a) activation functions and (b) dropout rates.  With the default initialization, signals often vanish with depth, and their behavior is inconsistent across activation functions and dropout rates. With AutoInit, the variance fluctuates naturally as each layer modifies its input.  At layers with weights (marked in {\color{red}red}), AutoInit scales the weights appropriately to return the variance to approximately 1.0, stabilizing training in each case.}
    \label{fig:autoinit:signal_propagation_allcnnc}
\end{figure}

AutoInit is also robust.  Even as other hyperparameters like learning rate and weight decay change, AutoInit still results in a higher performing network than the default initialization.  The results thus suggest that AutoInit provides an improved default initialization for convolutional neural networks.

\section{Stability in Deep ResNets}
\label{sec:autoinit:stability_deep_resnets}

This section expands the experimental analysis of AutoInit to residual networks, focusing on preactivation residual networks of various depths \cite{he2016identity}.  The training setup is standard unless explicitly stated otherwise (Appendix \ref{ap:details:autoinit}).  In particular, the initialization is ``He Normal'' \cite{he2015delving}, where weights are sampled from $\mathcal{N}(0, \sqrt{2/\texttt{fan\_in}})$.  

\paragraph{Visualizing Signal Propagation} 
Figure \ref{fig:autoinit:signal_propagation_resnet} shows how the signal variance changes with depth.  With ResNet-56, the variance increases where the shortcut connection and residual branch meet, and the variance drops whenever ReLU is applied.  Although the variance increases exponentially with the default initialization and linearly with AutoInit (note the log scale on the $y$ axis), training is still stable because batch normalization layers return the signal to variance 1.0.  Without batch normalization, the signal variance never stabilizes under the default initialization.  In contrast, removing batch normalization is not an issue with AutoInit; the signal variance remains stable with depth.

\begin{figure}
    \centering
    \includegraphics[width=0.75\linewidth]{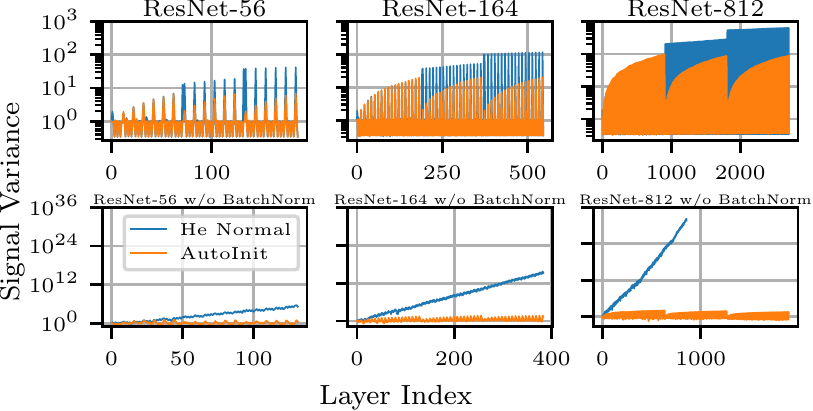}
    \caption{Signal propagation in residual networks.  Gaussian input was fed to the networks and empirical variance computed at each layer.  Since \texttt{ReLU}, \allowbreak \texttt{BatchNormalization}, and \texttt{Add} are counted as individual layers in this diagram, the total number of layers is different from that in the architecture name (i.e.\ ResNet-164 has 164 convolutional layers but over 500 total layers).  The default initialization causes exploding signals, while AutoInit ensures signal propagation is stable.}
    \label{fig:autoinit:signal_propagation_resnet}
\end{figure}

With the deeper ResNet-164 and ResNet-812 networks, the conclusions are similar but more pronounced.  In the case of ResNet-812 without batch normalization, the signals explode so severely that they exceed machine precision.  AutoInit avoids this issue entirely.

\paragraph{Stable Initial Learning Rates}
Exploding or vanishing signals make optimization difficult because they result in gradient updates that are too large to be accurate or too small to be meaningful.  This phenomenon can be observed when the network does not exceed chance accuracy. Therefore, a simple way to quantify whether a weight initialization is effective is to observe a network's performance after a few epochs.

Using this metric, AutoInit was compared against the default initialization by training unnormalized versions of ResNet-56, ResNet-164, and ResNet-812 for five epochs with a variety of learning rates.  With the default initialization, ResNet-56 requires a learning rate between $10^{-8}$ and $0.5 \times 10^{-3}$ to begin training, but training was not possible with ResNet-164 or ResNet-812 because of exploding signals (Figure \ref{fig:autoinit:resnet_vary_initial_lr}a).  AutoInit stabilizes training for all three networks, and its effect does not diminish with depth.  The networks remain stable with higher learning rates between $10^{-4}$ and $0.05$.  Such rates speed up learning, and also correlate with better generalization performance \cite{jastrzkebski2017three, smith2018don, smith2018bayesian}.

\begin{figure}
    \centering
    \begin{minipage}{0.49\linewidth}
    \begin{tikzpicture}
    \draw (0, 0) node[anchor=south west, inner sep=0, align=left] {\includegraphics[width=\linewidth]{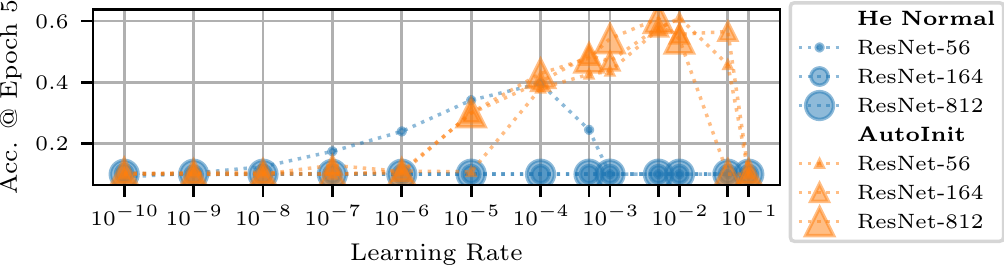}};
    \draw (0, 0) node[anchor=south west, inner sep=0, align=left] {(a)};
    \end{tikzpicture}
    \end{minipage}
    \hfill
    \begin{minipage}{0.49\linewidth}
    \begin{tikzpicture}
    \draw (0, 0) node[anchor=south west, inner sep=0, align=left] {\includegraphics[width=\linewidth]{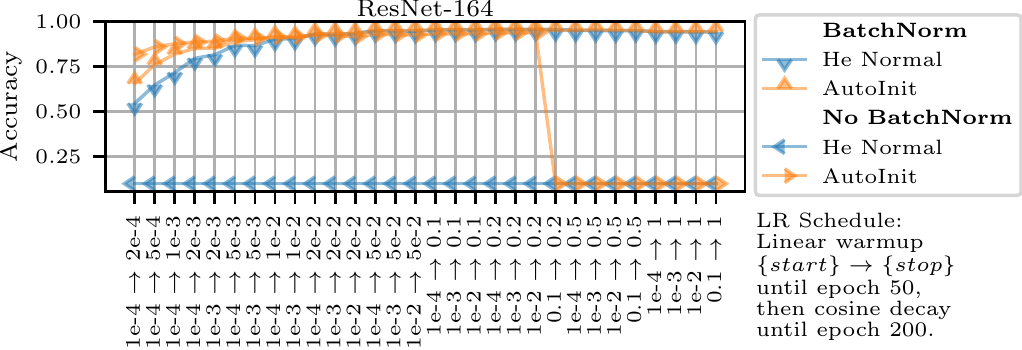}};
    \draw (0, 0) node[anchor=south west, inner sep=0, align=left] {(b)};
    \end{tikzpicture}
    \end{minipage}
    \caption{ResNet accuracy on CIFAR-10 with different settings. (a) Accuracy of unnormalized ResNet architectures after five epochs of training with different learning rates and weight initializations.  While default initialization makes training difficult in ResNet-56 and impossible at greater depths, AutoInit results in consistent training at all depths.  (b) Accuracy of ResNet-164 with a variety of learning rate schedules and initializations. AutoInit is comparable to or outperforms the default initialization in every case.}
\label{fig:autoinit:resnet_vary_initial_lr}
\end{figure}

\paragraph{Full ResNet Training}
In the third residual network experiment, ResNet-164 was trained to completion on CIFAR-10 with different learning rate schedules. All schedules included a linear warm-up phase followed by a decay to zero using cosine annealing \cite{loshchilov2016sgdr}.

Figure \ref{fig:autoinit:resnet_vary_initial_lr}b displays the performance with a variety of such schedules.  When the best learning rate schedules are used, ResNet-164 achieves comparable performance with the default initialization and with AutoInit.  However, when a suboptimal schedule is used, performance degrades more quickly with the default initialization than it does with AutoInit.  Without batch normalization, the network requires proper weight initialization for stability.  In this case, ResNet-164 with the default initialization fails to train regardless of the learning rate schedule, whereas AutoInit results in high accuracy for the majority of them.

Together, the experiments in this section show that AutoInit is effective with deep networks.  It prevents signals from exploding or vanishing, makes it possible to use larger learning rates, and achieves high accuracy, with and without batch normalization. 

\section{High-Resolution Images with Transformers}
\label{sec:autoinit:coatnet}
This section extends AutoInit to transformer architectures and applies them to high-resolution image classification.  Specifically, AutoInit is applied to CoAtNet, a model that combines convolutional and attention layers \cite{dai2021coatnet}.  The model is trained on Imagenette, a subset of 10 classes from the ImageNet dataset \cite{imagenette, deng2009imagenet}.  Imagenette allows evaluating AutoInit in a high-resolution image classification task with a $132\times$ smaller carbon footprint than the full ImageNet dataset would (Appendix \ref{ap:infrastructure}).  As Table \ref{tab:autoinit:coatnet} shows, AutoInit outperforms six commonly used initialization schemes as well as the default initialization, which initializes convolutional layers from $\mathcal{N}(0,\sqrt{2/\texttt{fan\_out}})$ and fully-connected layers from $\mathcal{U}\left(-\frac{\sqrt{6}}{\sqrt{\texttt{fan\_in} + \texttt{fan\_out}}}, \frac{\sqrt{6}}{\sqrt{\texttt{fan\_in} + \texttt{fan\_out}}}\right)$.  Furthermore, AutoInit stabilizes the network even when normalization layers are removed, suggesting that it is a promising candidate towards developing normalizer-free transformer architectures.  Full details are in Appendix \ref{ap:details:autoinit}.

\begin{table}
    \centering
    \caption{CoAtNet top-1 accuracy on Imagenette, shown as median of three runs.  The first four experiments vary the activation function, while the fifth removes all normalization layers from the architecture. A ``-'' indicates that training diverged.  AutoInit produces the best model in three of the five settings, and remains stable even without normalization layers.\\}
    \begin{tabular}{llllll}
        \toprule 
        CoAtNet & w/ GELU & w/ ReLU & w/ SELU & w/ Swish & w/o Norm \\ \midrule 
        Default Init.  & 89.38 & 89.22 & 86.09 & 88.69 & - \\
        Glorot Normal  & 91.44 & 91.54 & 87.59 & 90.42 & \textbf{85.89} \\
        Glorot Uniform & 91.16 & 91.18 & \textbf{88.25} & 90.06 & 85.73 \\
        He Normal      & 88.48 & 88.05 & 86.11 & 88.36 & - \\
        He Uniform     & 88.66 & 87.87 & 86.37 & 88.41 & - \\
        LeCun Normal   & 91.11 & 90.57 & 87.80 & 90.83 & - \\
        LeCun Uniform  & 90.55 & 90.65 & 87.67 & 90.57 & - \\
        AutoInit & \textbf{92.48} & \textbf{92.15} & 86.80 & \textbf{92.28} & 85.73 \\ 
        \bottomrule
    \end{tabular}
    \label{tab:autoinit:coatnet}
\end{table}

\section{Scaling up to ImageNet}
\label{sec:autoinit:imagenet}

In order to compliment the results from Section \ref{sec:autoinit:coatnet} and demonstrate that AutoInit can scale to more difficult tasks, ResNet-50 was trained from scratch on ImageNet with the default initialization and with AutoInit.  As Table \ref{tab:autoinit:imagenet} shows, AutoInit improves top-1 and top-5 accuracy in this task as well.  Full training details are in Appendix \ref{ap:details:autoinit}.

\begin{table}
    \centering
    \caption{ResNet-50 top-1 and top-5 validation accuracy on ImageNet.  AutoInit improves performance, even with large and challenging datasets.\\}
    \begin{adjustbox}{max width=\linewidth}
    \begin{tabular}{lll}
        \toprule 
        & top-1 & top-5 \\ \midrule 
        Default Init. & 74.33 & 91.60 \\
        AutoInit & \textbf{75.35} & \textbf{92.03} \\
        \bottomrule
    \end{tabular}
    \end{adjustbox}
    \label{tab:autoinit:imagenet}
\end{table}

\section{Contrast with Data-Dependent Initialization}
\begin{figure}
    \centering
    \includegraphics[width=0.75\linewidth]{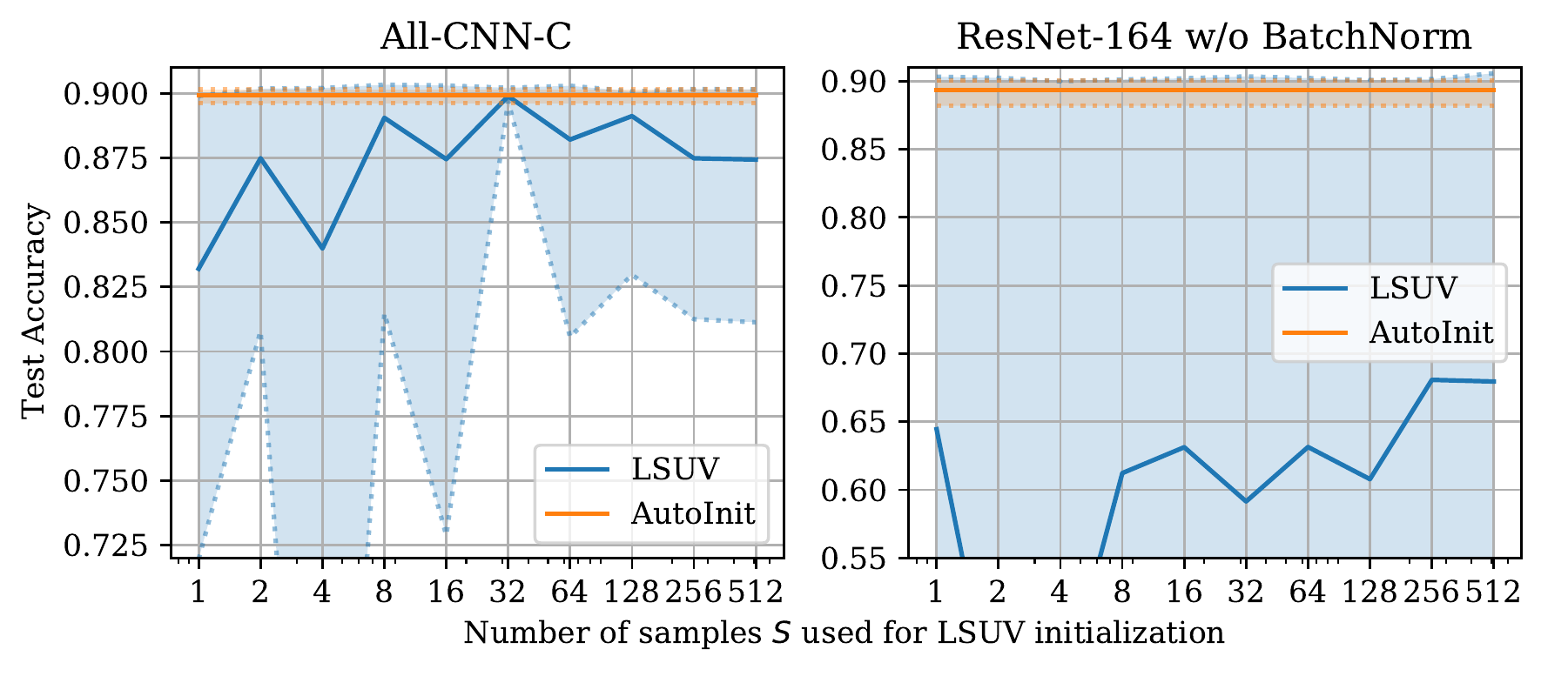}
    \caption{Mean CIFAR-10 test accuracy for AutoInit vs.\ LSUV with different numbers of samples $S$.  Each evaluation is repeated 10 times; the shaded area shows the maximum and minimum accuracy among all trials.  AutoInit is consistent, but LSUV struggles when $S$ is small or the network is deep.}
    \label{fig:autoinit:performance}
\end{figure}

\label{sec:autoinit:autoinit_vs_lsuv}
The layer-sequential unit-variance (LSUV) algorithm is the most natural \linebreak data-dependent initialization comparison to AutoInit because both approaches aim to scale the weights appropriately in an architecture-agnostic way. LSUV pre-initializes the weights with an existing approach, feeds $S$ training samples through the network, and adjusts the scale of the weights so that each layer's output variance is approximately one \cite{mishkin2015all}.

Data-dependent initialization is time-consuming for large $S$ (indeed, even $S=1$ is used in practice \cite{kingma2018glow}).  However, if $S$ is too small, the samples may not reflect the statistics of the dataset accurately, leading to poor initialization.  Figure \ref{fig:autoinit:performance} demonstrates this phenomenon.  In some training runs LSUV matches the performance of AutoInit, but in many instances the randomly selected samples do not accurately reflect the overall dataset and performance suffers.  Since AutoInit is not data-dependent, it does not have this issue.  Details of this experiment are in Appendix \ref{ap:details:autoinit}.

\section{Enabling Neural Architecture Search}
\label{sec:autoinit:nas}

Sections \ref{sec:autoinit:convolutional} through \ref{sec:autoinit:autoinit_vs_lsuv} demonstrated that AutoInit works well for convolutional, residual, and transformer networks with a variety of hyperparameter values and depths. In this section, the results are extended to a broader variety of network topologies and types of tasks, for two reasons. First, whereas custom weight initialization may be developed by hand for the most popular machine learning benchmarks, it is unlikely to happen for a variety of architectures and tasks beyond them. Second, as new types of neural network designs are developed in the future, it will be important to initialize them properly to reduce uncertainty in their performance.  This section evaluates the generality of AutoInit by applying it to the variety of networks generated in a neural architecture search process with five types of tasks.

\usetikzlibrary{shapes.geometric, arrows, positioning, calc}
\tikzstyle{input} = [rectangle, rounded corners, text centered, draw=black, fill=red!30]
\tikzstyle{unary} = [rectangle, rounded corners, text centered, draw=black, fill=yellow!30]
\tikzstyle{binary} = [rectangle, rounded corners, text centered, draw=black, fill=blue!30]
\tikzstyle{arrow} = [thick, ->,>=stealth]
\tikzstyle{output} = [rectangle, rounded corners, text centered, draw=black, fill=green!30]
\tikzstyle{invisible} = [rectangle, text centered]
\tikzstyle{maybeparam} = [draw, circle, dashed, fill=cyan!30]

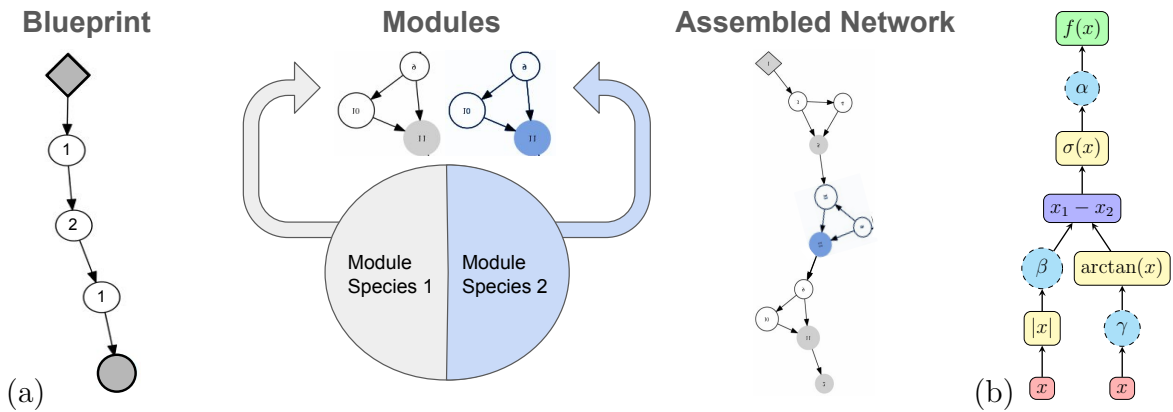
\begin{figure}[!t]
    \centering
    \begin{tikzpicture}
    \draw (0, 0) node[anchor=south west, inner sep=0, align=left] {\includegraphics[width=0.77\linewidth]{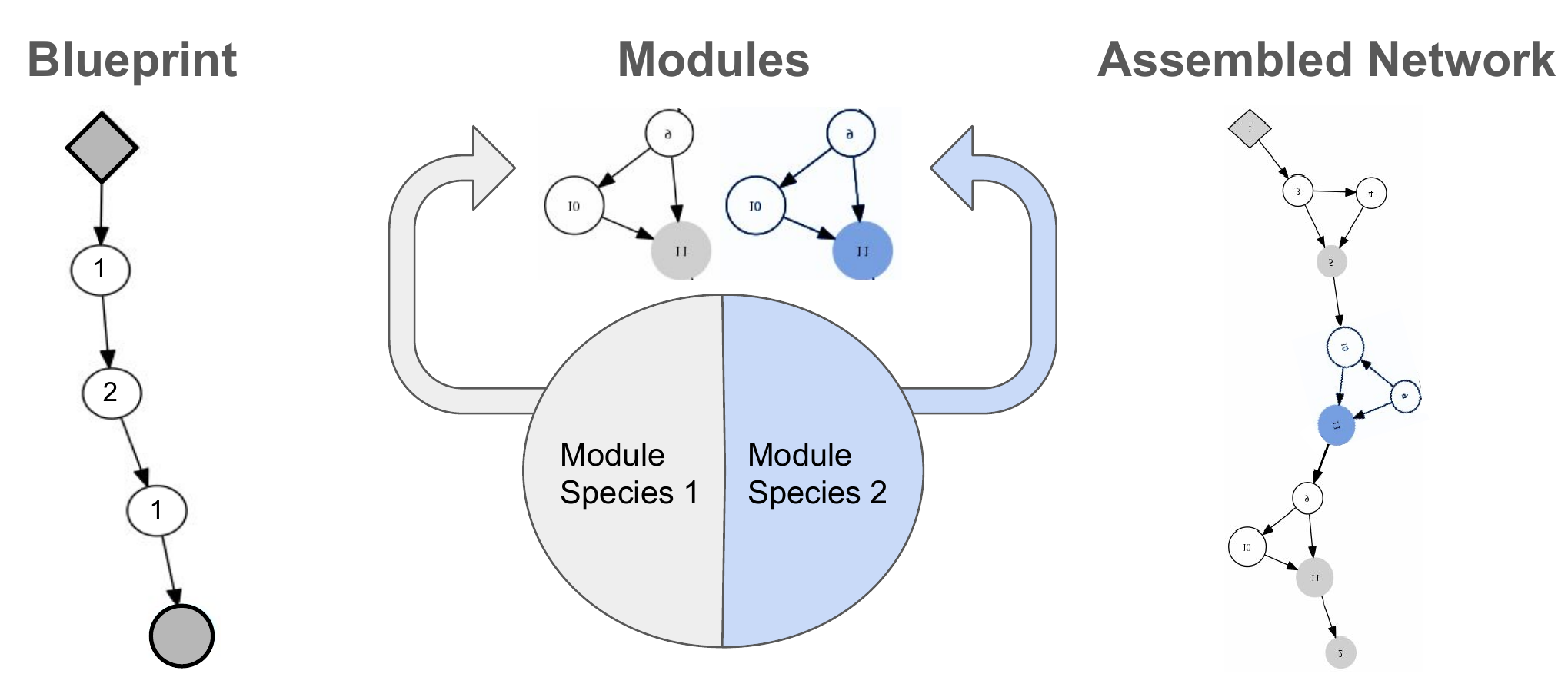}};
    \draw (0, 0) node[anchor=south west, inner sep=0, align=left] {(a)};
    \end{tikzpicture}
    \begin{tikzpicture}
    \draw (0, 0) node[anchor=south west, align=left] {
    
        \begin{adjustbox}{max width=0.12\linewidth}
            \begin{tikzpicture}[node distance=3em]
        
                    \node (output) [output] {$f(x)$};
                    \node (p1) [maybeparam, below of=output] {$\alpha$};
                    \node (unary) [unary, below of=p1] {$\sigma(x)$};
                    \node (binary) [binary, below of=unary] {$x_1 - x_2$};
                    \node (p5) [maybeparam, below of=binary, xshift=-2em] {$\beta$};
                    \node (unary1) [unary, below of=p5] {$|x|$};
                    \node (unary2) [unary, below of=binary, xshift=2em] {$\textrm{arctan}(x)$};
                    \node (p4) [maybeparam, below of=unary2] {$\gamma$};
                    \node (input1) [input, below of=unary1] {$x$};
                    \node (input2) [input, below of=p4] {$x$};

                    \draw [arrow] (input1) -- (unary1);
                    \draw [arrow] (input2) -- (p4);
                    \draw [arrow] (p4) -- (unary2); 
                    \draw [arrow] (unary1) -- (p5);
                    \draw [arrow] (p5) -- (binary);
                    \draw [arrow] (unary2) -- (binary);
                    \draw [arrow] (binary) -- (unary);
                    \draw [arrow] (unary) -- (p1);
                    \draw [arrow] (p1) -- (output);
                    
                \end{tikzpicture}
        \end{adjustbox}
        
    };
    \draw (-0.5, 0) node[anchor=south west, inner sep=0, align=left] {(b)};
    \end{tikzpicture}

    \caption{(a) The CoDeepNEAT method.  Modules replace nodes in the blueprint to create a candidate neural network. (b) An example activation function created with the PANGAEA method.  The computation graph represents the parametric function $f(x) = \alpha \cdot \sigma (\beta \cdot | x | - \arctan ( \gamma \cdot x ) )$.  CoDeepNEAT and PANGAEA generate a variety of architectures and activation functions that can be used to evaluate AutoInit's generality and flexibility.}
    \label{fig:autoinit:codeepneat_pangaea}
\end{figure}

\begin{figure*}
    \centering

    \begin{tikzpicture}
    \draw (0, 0) node[anchor=south west, inner sep=0, align=left] {\includegraphics[width=\linewidth]{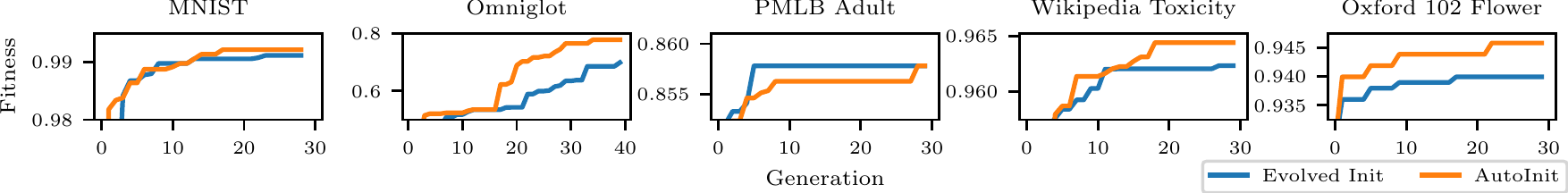}};
    \draw (0, 0) node[anchor=south west, inner sep=0, align=left] {(a)};
    \end{tikzpicture}\\
    
        \begin{tikzpicture}
            \draw (0, 0) node[anchor=south west, inner sep=0, align=left] {\includegraphics[width=168pt]{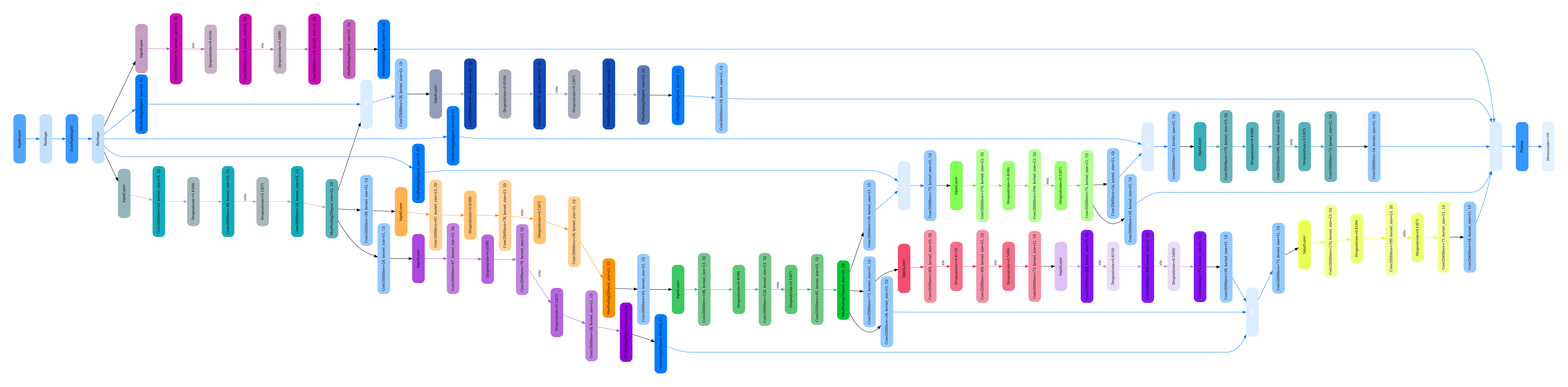}
            \includegraphics[width=97pt]{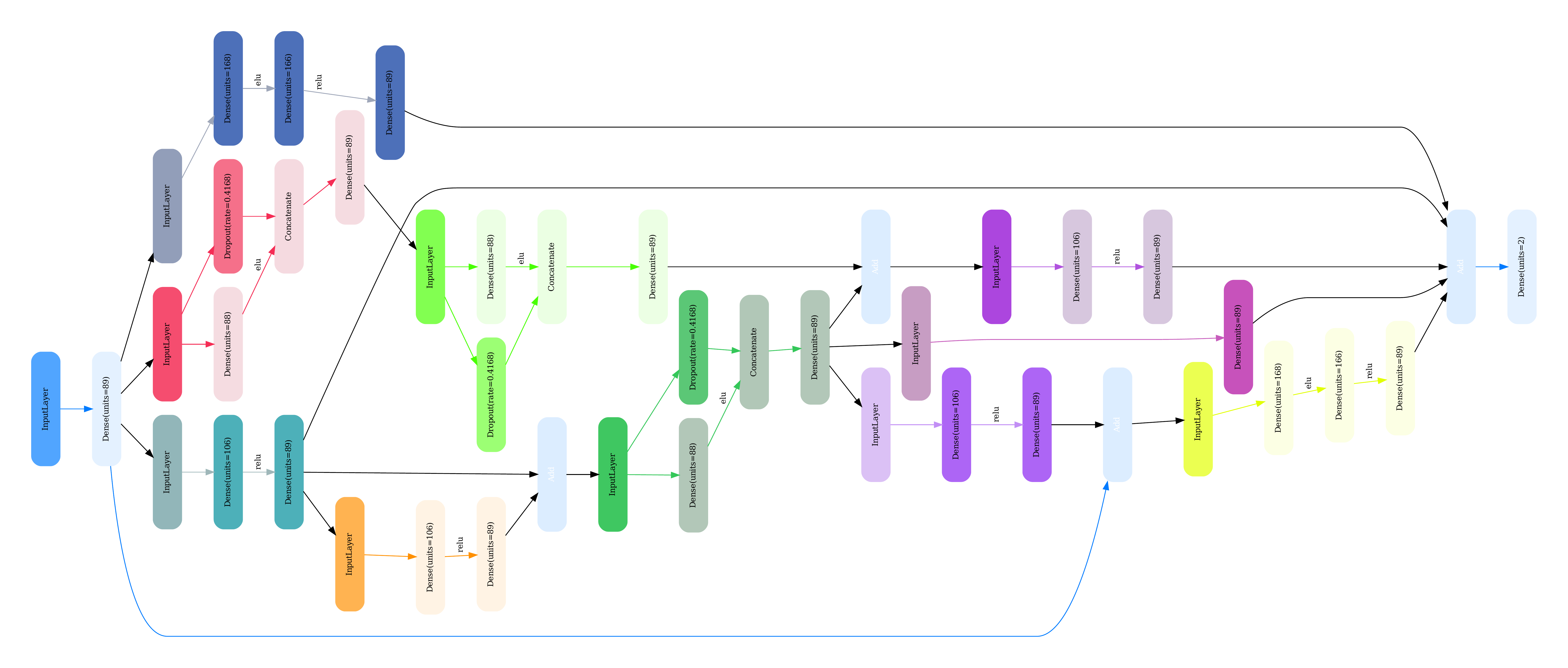}
            \includegraphics[width=196pt]{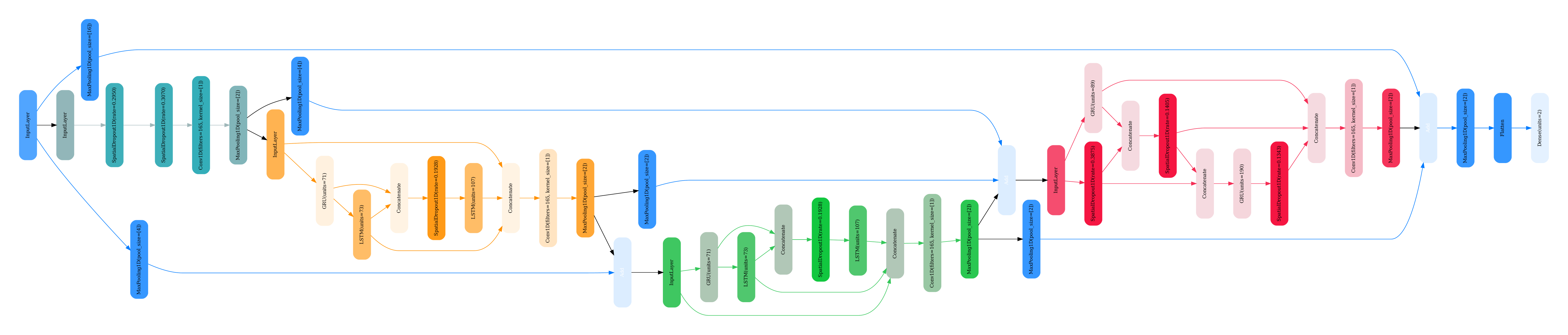}\\
            \adjustbox{max width=\linewidth}{
                \begin{tabular}{p{168pt}<{\centering}p{97pt}<{\centering}p{196pt}<{\centering}}
                \textbf{MNIST} & \textbf{PMLB Adult} & \textbf{Wikipedia Toxicity}
                \end{tabular}}
            };
        \draw (0, 0) node[anchor=south west, inner sep=0, align=left] {(b)};
        \end{tikzpicture}

    
    \caption{Evaluation of AutoInit with neural architecture search. (a) Performance improvement over generations in the five tasks.  AutoInit outperforms the evolved initialization on four tasks and matches it on one.  (b) Representative networks evolved with AutoInit.  Although the networks are distinct, AutoInit initializes them properly, leading to good performance in each case.
    }
    \label{fig:autoinit:enn_results}
\end{figure*}

\paragraph{The CoDeepNEAT Architecture Search Method}
Neural networks are evolved using CoDeepNEAT \cite{miikkulainen2019evolving, liang2019evolutionary}.  
CoDeepNEAT extends previous work on evolving network topologies and weights \cite{moriarty1997forming, stanley2002evolving} to the level of evolving deep learning architectures.  Utilizing a cooperative coevolution framework \cite{potter2000cooperative}, CoDeepNEAT evolves populations of modules and blueprints simultaneously (Figure \ref{fig:autoinit:codeepneat_pangaea}a).  Modules are small neural networks, complete with layers, connections, and hyperparameters.  Blueprints are computation graphs containing only nodes and directed edges.  To create a candidate neural network, CoDeepNEAT chooses a blueprint and replaces its nodes with selected modules.  This mechanism makes it possible to evolve deep, complex, and recurrent structures, while taking advantage of the modularity often found in state-of-the-art models.  In addition to the network structure, CoDeepNEAT evolves hyperparameters like dropout rate, kernel regularization, and learning rate.  The network weights are not evolved, but instead trained with gradient descent.  The generality of CoDeepNEAT helps minimize human design biases and makes it well-suited to analyzing AutoInit's performance in a variety of open-ended machine learning settings.

\paragraph{Tasks}
Using CoDeepNEAT, networks are evolved for their performance in vision (MNIST), language (Wikipedia Toxicity), tabular (PMLB Adult), multi-task (Omniglot), and transfer learning (Oxford 102 Flower) tasks (Appendix~\ref{ap:details:autoinit_enn_details}).

\paragraph{Results}
Figure \ref{fig:autoinit:enn_results}a shows how CoDeepNEAT discovers progressively better networks over time on the five tasks.  Evolution often selects different weight initialization strategies for the different layers in these networks, so this scheme is already a flexible and powerful baseline.  However, by accounting for each model's unique topology and hyperparameters, AutoInit outperforms the baseline in four of the five tasks, and matches it in the fifth.

Beyond performance, three interesting phenomena can be observed. First, the mean population fitness varies greatly with the default initialization in each task, sometimes dropping significantly from one generation to the next (Figure \ref{fig:autoinit:enn_results_detail}).  Though some variation is natural in a stochastic evolutionary process like CoDeepNEAT, AutoInit makes the discovery process more reliable by stabilizing the performance of the entire population.

\begin{figure}
    \centering
    \includegraphics[width=0.49\linewidth]{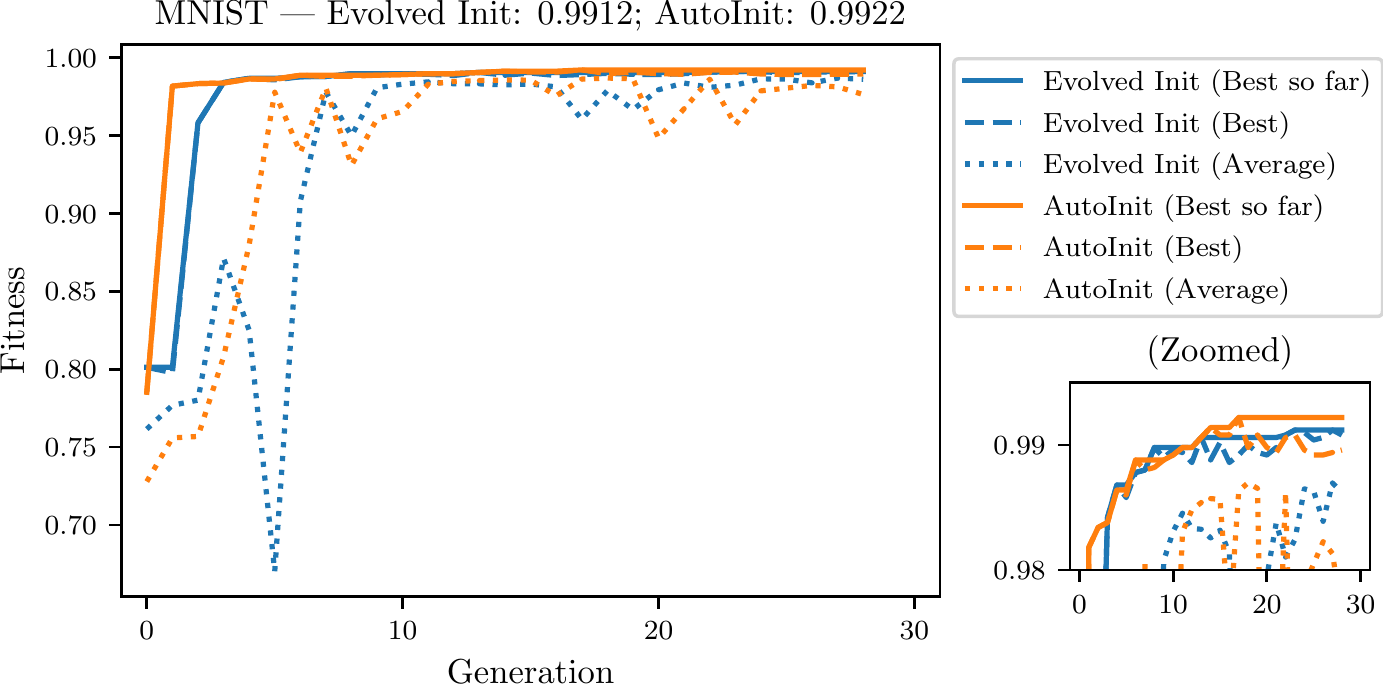}
    \includegraphics[width=0.49\linewidth]{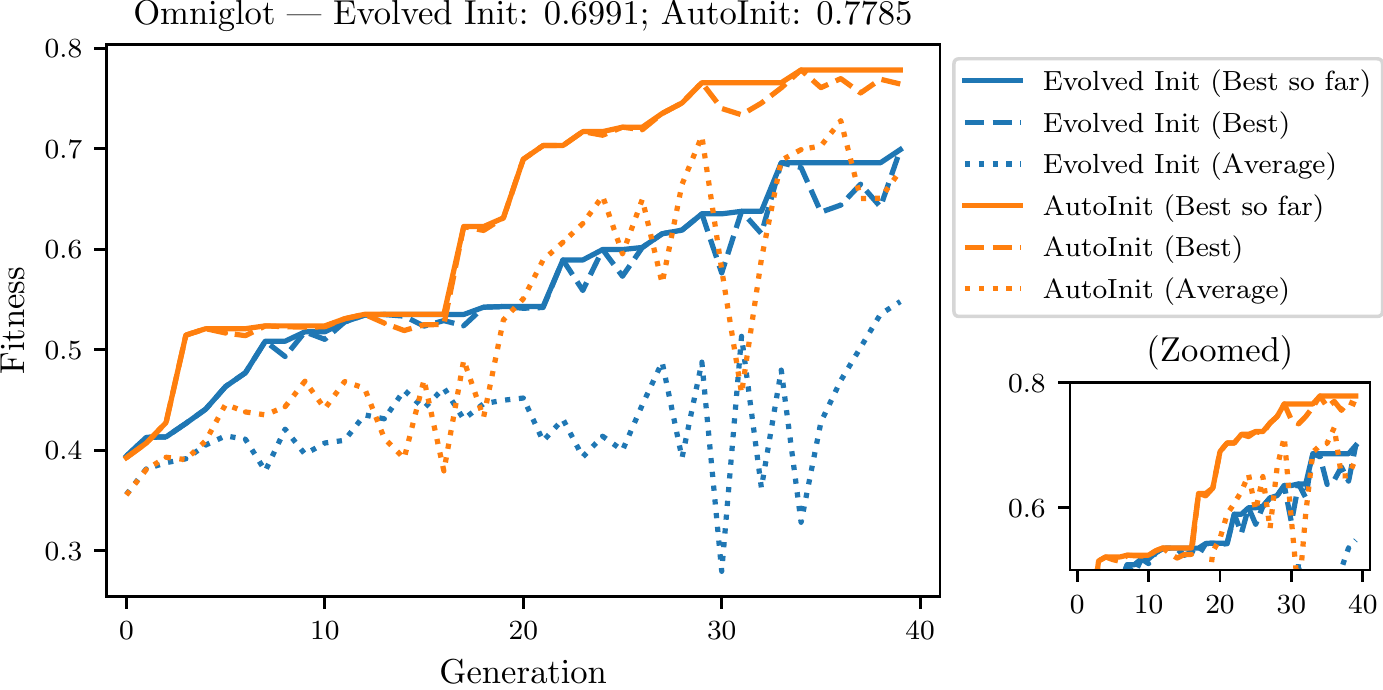}\\[3ex]
    \includegraphics[width=0.49\linewidth]{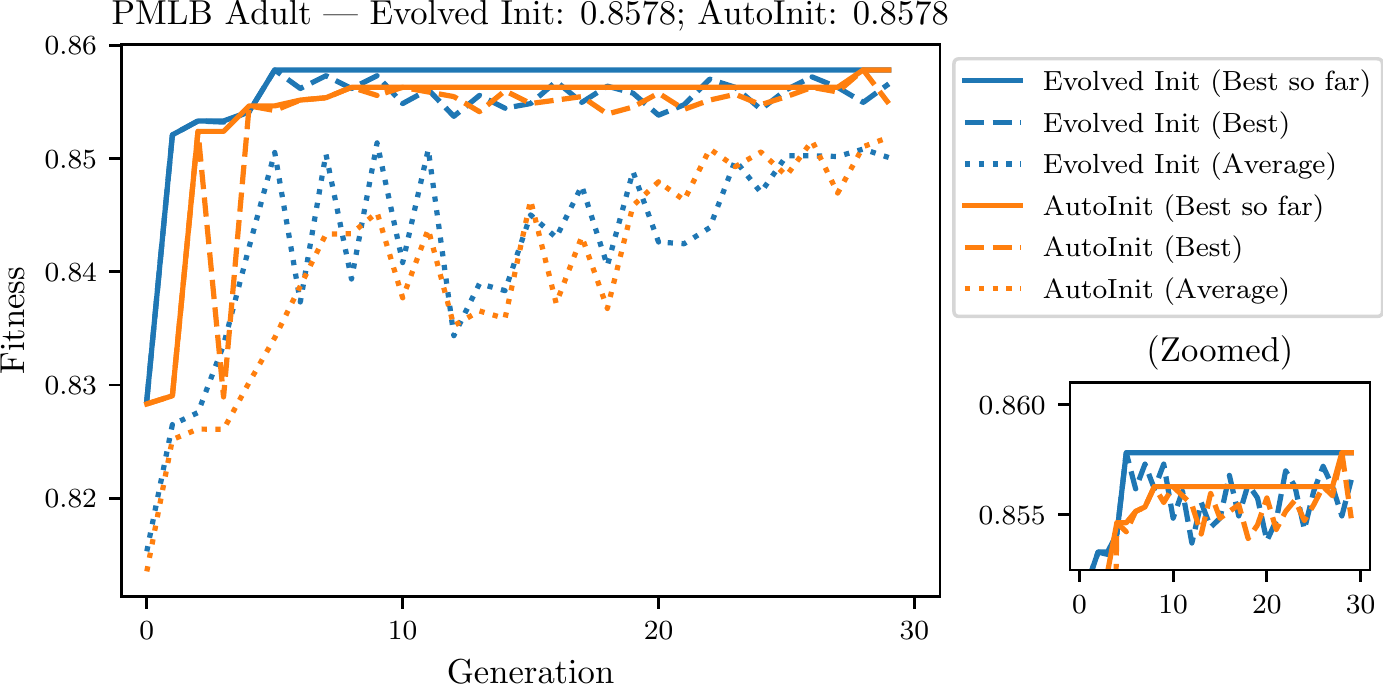}
    \includegraphics[width=0.49\linewidth]{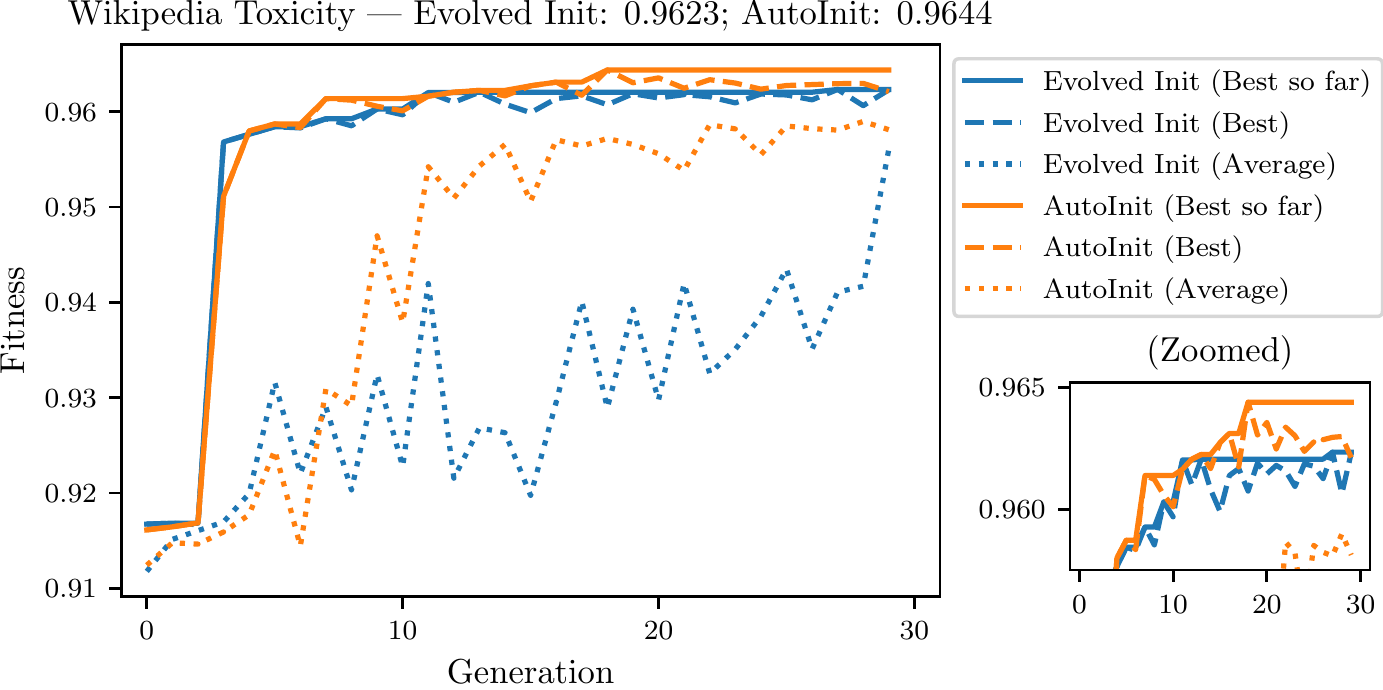}\\[3ex]
    \includegraphics[width=0.49\linewidth]{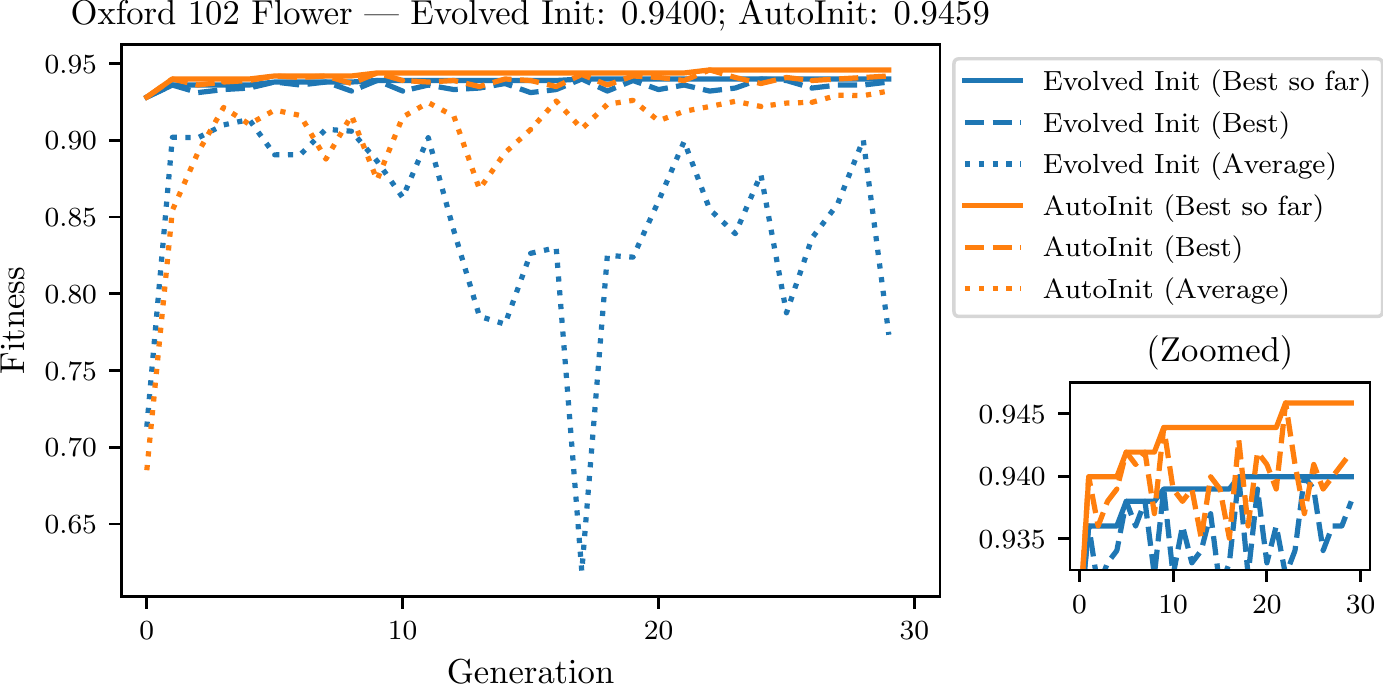}
    \caption{Progress of neural architecture search in the five tasks.  The data is the same as that in Figure~\ref{fig:autoinit:enn_results}, but this plot also shows how AutoInit can stabilize mean population performance, leading to more reliable discovery of powerful models.}
    \label{fig:autoinit:enn_results_detail}
\end{figure}

Second, hyperparameters play a large role in the final performance of the dense networks, in particular in the ``Oxford 102 Flower'' task.  While CoDeepNEAT discovers good models with both initialization strategies, performance is consistently higher with AutoInit. This finding agrees with Section \ref{sec:autoinit:convolutional}, where AutoInit was shown to be robust to different hyperparameter values.  

Third, while many networks exhibit motifs popular in existing architectures, such as alternating convolution and dropout layers and utilizing residual connections, other phenomena are less common (Figure \ref{fig:autoinit:enn_results}b).  For example, the networks make use of different activation functions and contain several unique information processing paths from the input to the output.  Because AutoInit provides effective initialization in each of these cases, it allows for taking full advantage of unusual design choices that might otherwise hurt performance under default initialization schemes.

The results in this section suggest that AutoInit is an effective, general-purpose algorithm that provides a more effective initialization than existing approaches when searching for new models.

\section{Enabling~Activation~Function~Discovery}
\label{sec:autoinit:afn_meta_learning}

As new activation functions are developed in the future, it will be important to adjust weight initialization to maintain stable signal propagation.  Since AutoInit makes this adjustment automatically, it is well-suited to the task.  Indeed, Figure \ref{fig:autoinit:allcnnc_hyperparams} confirmed that AutoInit improves performance with several existing activation functions.  This section presents a more challenging task.  To simulate future research in activation function design, hundreds of novel activation functions were generated as arbitrary computation graphs and trained with a CNN.  AutoInit's ability to initialize each of these networks was then evaluated.  The method for creating such activation functions is described first, followed by experimental details, and results on stability, performance, and generality.

\paragraph{Creating Novel Activation Functions}
An important area of automated machine learning (AutoML) is to discover new, better activation functions \cite{basirat2018quest, DBLP:conf/iclr/RamachandranZL18, bingham2020gecco, liu2020evolving}.
Among existing approaches, PANGAEA \cite{bingham2022discovering} has the most flexible search space and is therefore used to generate new functions in this section.

PANGAEA represents activation functions as computation graphs containing unary and binary operators (Figure \ref{fig:autoinit:codeepneat_pangaea}b).  Creating a novel activation function involves three steps.  First, a minimal computation graph is initialized with randomly selected unary and binary operators.  Second, the functional form of the activation function is modified by applying three random mutations to increase diversity.  Third, the function is augmented with up to three learnable parameters.  These parameters are analogous to those in other parametric activation functions, such as PReLU \cite{he2015delving}; they are initialized to one and learned during training by gradient descent.  Through this process, it is possible to understand to what extent AutoInit can improve performance with activation functions that have yet to be discovered.

\paragraph{Experimental Setup}
An important insight in this domain is that in addition to modifying the variance of the signals in a network, activation functions can induce mean shifts.  Prior work encouraged stability by reparameterizing the weights to have zero empirical mean \cite{huang2017centered, qiao2019micro, brock2021characterizing}.  An alternative and more direct approach is to modify the activation function itself so that it does not cause a mean shift in the first place.  Given an activation function $f$ with Gaussian mean $\mu_f = \frac{1}{\sqrt{2\pi}}\int_{-\infty}^\infty f(x) e^{-x^2/2} \diff x$, this goal can be accomplished with $\tilde{f} \coloneqq f - \mu_f$, which has zero Gaussian mean. To take advantage of this idea, a version of AutoInit called AutoInit++ was created for this domain, thus extending AutoInit slightly beyond weight initialization.

Thus, three initialization strategies were compared.  With the default initialization, weights were sampled from $\mathcal{U}\left(-\frac{\sqrt{6}}{\sqrt{\texttt{fan\_in} + \texttt{fan\_out}}}, \frac{\sqrt{6}}{\sqrt{\texttt{fan\_in} + \texttt{fan\_out}}}\right)$ \cite{glorot2010understanding}.  With AutoInit, the weights were sampled from $\mathcal{N}\left(0, 1/\sqrt{\texttt{fan\_in}\mu_f}\right)$ to account for an arbitrary activation function $f$; the dropout adjustment (Section~\ref{sec:autoinit:convolutional}) was not used. Finally, AutoInit++ takes advantage of $\tilde{f}$ as described above, but is otherwise identical to AutoInit.

For each initialization strategy, 200 activation functions were created using the PANGAEA process.  Each activation function was used with the All-CNN-C architecture on the CIFAR-10 dataset following the standard training setup.  To avoid overfitting to the test set when evaluating such a large number of activation functions, the accuracy with a balanced validation set of 5000 images is reported instead.

\paragraph{Stability}
Achieving better-than-chance accuracy is a useful metric of training stability (Section \ref{sec:autoinit:stability_deep_resnets}).  As shown in Figure \ref{fig:autoinit:random_afns_hist}a, many activation functions result in chance accuracy regardless of how the network is initialized.  This phenomenon is not surprising; since the activation functions are arbitrary computation graphs, many of them will turn out to be poor.  With the default initialization strategy, 149 activation functions caused training to fail in this way.  With AutoInit, the number of failed activation functions dropped to 130, and with AutoInit++, it further decreased to 117.  AutoInit and AutoInit++ thus make training more stable, allowing it to succeed for a greater number of activation functions.

\begin{figure}
    \centering
    \begin{tikzpicture}
    \draw (0, 0) node[anchor=south west, inner sep=0, align=left] {\includegraphics[width=0.7\linewidth]{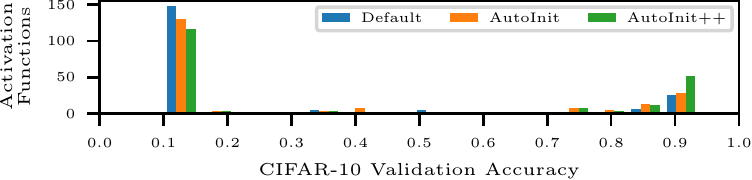}};
    \draw (0, 0) node[anchor=south west, inner sep=0, align=left] {(a)};
    \end{tikzpicture}\\[1em]
    \begin{tikzpicture}
    \draw (0, 0) node[anchor=south west, inner sep=0, align=left] {\includegraphics[width=0.7\linewidth]{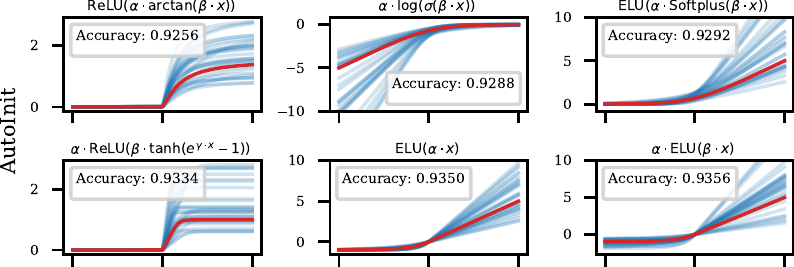}\\[1em]
    \includegraphics[width=0.7\linewidth]{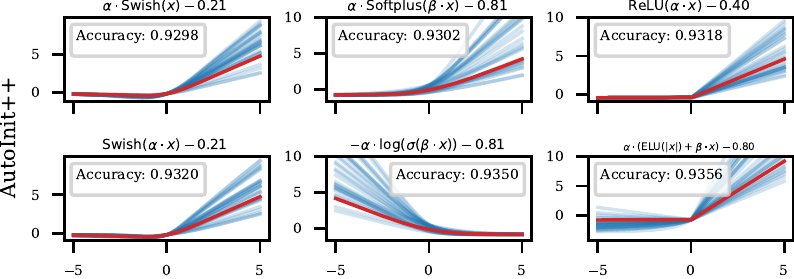}};
    \draw (0, 0) node[anchor=south west, inner sep=0, align=left] {(b)};
    \end{tikzpicture}
    
    \caption{Evaluation of AutoInit with activation function discovery. (a) Distribution of accuracies achieved with 200 activation functions and different weight initialization strategies. AutoInit and AutoInit++ make training more stable and allow more high-performing activation functions to be discovered than the default initialization does.  (b) High-performing activation functions.  The red line shows the function at initialization, with $\alpha = \beta = \gamma = 1$.  The blue lines show the shapes the activation function takes during training, created by sampling $\alpha, \beta, \gamma$ from $\mathcal{U}(0.5, 2.0)$.  AutoInit's flexibility should turn out useful for developing new activation functions in the future.}
    \label{fig:autoinit:random_afns_hist}
\end{figure}

\paragraph{Performance}
Beyond training stability, a good weight initialization should also improve performance.  As a baseline, when trained with ReLU and the default initialization, All-CNN-C achieved 89.10\% test accuracy.  Twenty-two of 200 activation functions from the PANGAEA search space outperformed this accuracy with the default initialization.  With AutoInit, this number increased to 26, and with AutoInit++, to 50---a notable improvement.  Thus, with the default initialization, one can naively create a randomly generated computation graph activation function and have roughly a one in nine chance of outperforming ReLU, but with AutoInit++, this probability increases to one in four.

Indeed, the Mann-Whitney U test \cite{mann1947test} concludes that the distribution of accuracies induced by AutoInit++ is \textit{stochastically larger} than that from AutoInit $(p < 0.05)$ or the default initialization $(p < 0.01)$.  This result means that for any level of performance, it is always more probable to discover an activation function that achieves that level of performance when initializing with AutoInit++ versus AutoInit or the default initialization.  The result implies that activation function researchers who properly initialize their networks are more likely to discover state-of-the-art activation functions, while staying with the default initialization may hinder that research effort.  More detailed statistical significance analyses are included in Section \ref{sec:autoinit:stat_sig}.

\paragraph{Generality}
Figure \ref{fig:autoinit:random_afns_hist}b plots several activation functions from the PANGAEA search space.  Many discovered functions have similar shapes to existing functions.  However, others are nonmonotonic, have discontinuous derivatives, or saturate to nonzero values.  These properties are less common in existing activation functions.  This observation suggests that AutoInit is a general approach that does not depend on a specific type of activation function; it may therefore serve as a useful tool in developing new such functions in the future.

\section{Statistical Significance of Results in Activation Function Meta-Learning}

\label{sec:autoinit:stat_sig}

Sampling activation functions from the PANGAEA search space results in a distribution of possible models for each weight initialization strategy.  Comparing the empirical distribution functions (EDFs) induced by each initialization strategy makes it possible to quantify the importance of the initialization \cite{radosavovic2019network}.

Given $n$ activation functions with errors $\{e_i\}$, the EDF $F(e) = \frac{1}{n}\sum_{i=1}^n\mathbf{1}[e_i < e]$ gives the fraction of activation functions that result in error less than $e$.  Let $F_\textrm{default}$, $F_\textrm{AutoInit}$, and $F_\textrm{AutoInit++}$ be the EDFs for the three initialization strategies.  Figure \ref{fig:autoinit:afn_cdf} plots these EDFs along with the Kolmogorov-Smirnov test statistic $D = \sup_e |F_1(e) - F_2(e)|$, which measures the maximum vertical discrepancy between two EDFs \cite{massey1951kolmogorov}.  This statistic shows that (1) AutoInit outperforms the default initialization $(D = 0.105)$; (2) AutoInit++ delivers an even greater boost in performance over the default initialization $(D = 0.191)$; and (3) AutoInit++ is measurably better than AutoInit $(D = 0.122)$, confirming that having zero Gaussian mean is a useful property for activation functions to have.

Other ways of measuring statistical significance lead to similar conclusions. For instance, consider the null hypothesis that $F_\textrm{default} = F_\textrm{AutoInit}$.  In other words, this null hypothesis states that AutoInit provides no benefit and that the accuracies obtained come from the same underlying distribution.  With the Epps-Singleton test \cite{epps1986omnibus} this null hypothesis is rejected with $p < 0.05$.  Similarly, the test rejects the null hypothesis that $F_\textrm{default} = F_\textrm{AutoInit++}$ with $p < 0.001$.  Even stronger statements can be made in the case of AutoInit++.  With the Mann-Whitney U test \cite{mann1947test}, the null hypothesis that $F_\textrm{default}(e) \geq F_\textrm{AutoInit++}(e)$ for some $e$ is rejected $(p < 0.01)$ in favor of the alternative that $F_\textrm{default}(e) < F_\textrm{AutoInit++}(e)$ for all $e$.  Similarly, the null hypothesis that $F_\textrm{AutoInit}(e) \geq F_\textrm{AutoInit++}(e)$ for some $e$ is rejected $(p < 0.05)$ in favor of the alternative that $F_\textrm{AutoInit}(e) < F_\textrm{AutoInit++}(e)$ for all $e$.  As discussed above, this result states that the distribution of accuracies induced by AutoInit++ is stochastically larger than that from AutoInit or the default initialization.

\begin{figure}
    \centering
    \includegraphics[width=0.75\linewidth]{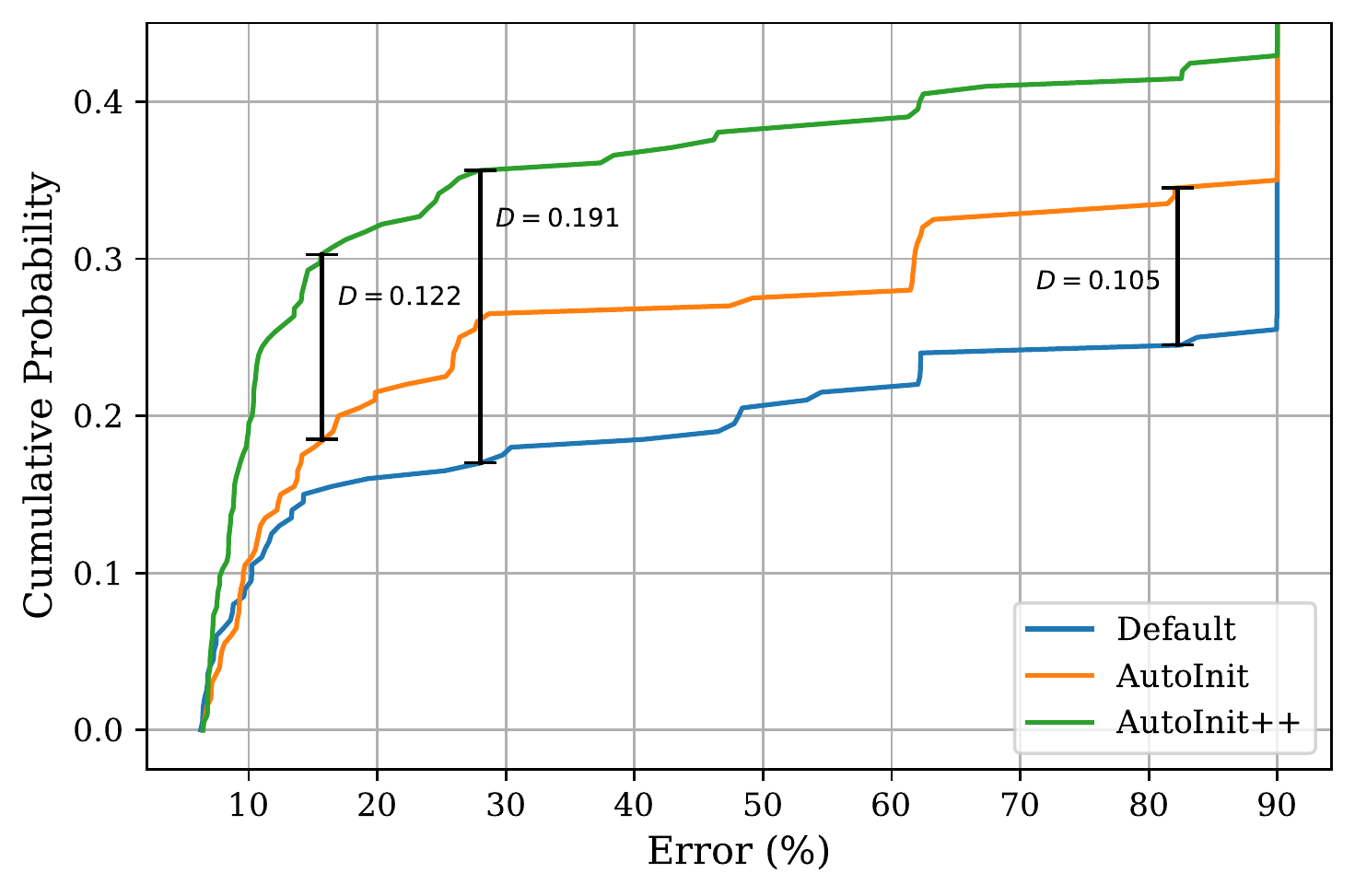}
    \caption{Error EDFs for PANGAEA activation functions when using different weight initialization strategies.  The Kolmogorov-Smirnov statistic $D$ quantifies the maximum vertical distance between the EDFs, and shows that proper initialization provides a measurable increase in expected performance.  Notice that the $x$-axis shows percent error, and not accuracy as in Figure~\ref{fig:autoinit:random_afns_hist}.}
    \label{fig:autoinit:afn_cdf}
\end{figure}

\section{Discussion}
\label{sec:autoinit:future_work}

AutoInit is based on understanding and utilizing the training dynamics of neural networks, leading to higher and more robust performance, and facilitating further advances in meta-learning. It can be improved and its scope broadened in several ways in the future, as outlined below.

\paragraph{Experiments in Other Domains}
The experiments in this chapter demonstrate that AutoInit can improve performance in a variety of settings, suggesting that it can be applied to other domains as well.  For instance in reinforcement learning, good estimates of activation statistics are usually not available due to the online nature of the algorithm.  It is not possible to stabilize training using e.g.\ batch normalization, but it may be possible to do it with AutoInit.  Similarly, training of generative adversarial networks \cite{goodfellow2014generative} is often unstable, and proper initialization may help.  
Applying AutoInit to such different domains should not only make them more reliable, but also lead to a better understanding of their training dynamics.

\paragraph{Accelerating Model Search}
In Sections~\ref{sec:autoinit:nas} and~\ref{sec:autoinit:afn_meta_learning},  AutoInit was shown to facilitate the discovery of better neural network designs and activation functions. This ability is possible because AutoInit is a general method, i.e.\ not restricted to a single class of models, and it could similarly augment other meta-learning algorithms \citep[e.g.\ those reviewed by][]{elsken2019neural, wistuba2019survey}.

However, this finding points to an even more promising idea.  As model search techniques become more prevalent in real-world applications, it will be most worthwhile to derive general principles rather than specific instantiations of those principles.  For example, past weight initialization strategies improved performance with specific activation functions through manual derivation of appropriate weight scaling (Section~\ref{sec:autoinit:weight_init_for_afns}). In contrast, AutoInit is a general method, leveraging Gaussian quadrature for any activation function.  Similarly, AutoInit resulted in better initialization than strategies discovered by CoDeepNEAT through evolution (Section~\ref{sec:autoinit:nas}).  Further, AutoInit++ (Section~\ref{sec:autoinit:afn_meta_learning}), rather than producing a few high-performing activation functions, introduces the general property that activation functions with zero Gaussian mean ($\tilde{f} \coloneqq f - \mu_f$) tend to perform well.  This property discovered a highly diverse set of powerful activation functions in the PANGAEA search space (Figure \ref{fig:autoinit:random_afns_hist}).  

Thus, AutoInit is successful because it is not a single initialization strategy, but rather a mapping from architectures to initialization strategies.  Such mappings, whether focused on initialization or some other aspect of model design, deserve increased attention in the future.  They can lead to performance gains in a variety of scenarios.  They also accelerate model search by focusing the search space to more promising regions.  If one does not have to worry about discovering a good initialization, compute power can instead be used in other areas, like designing architectures and activation functions.  Thus, general tools like AutoInit save time and resources, and lead to better models as a result.

\label{sec:autoinit:technical_extensions}


\paragraph{Initial Weight Distributions}
AutoInit calculates appropriate weight scaling, but it does not impose a distribution from which weights are drawn (Equation \ref{eq:autoinit:variance_scaling}).  All experiments in this Chapter used a truncated normal distribution.  In preliminary experiments, AutoInit also used untruncated normal, uniform, and orthogonal distributions, but no clear trends were observed.  Indeed, assuming weights are scaled appropriately, whether training is stable depends only on the architecture and not the distribution from which weights are sampled \cite{hanin2018neural}.  However, this conclusion applies only in limited theoretical settings; in other settings, orthogonal initialization was found to be beneficial \cite{saxe2013exact, hu2020provable}.  Whether there is a single distribution that is optimal in every case, or whether certain distributions are better-suited to different models, tasks, or layers, remains an open question, and a compelling direction for future research.

\paragraph{Variations of AutoInit}
Several variations of the core AutoInit algorithm can be devised that may improve its performance.  For example, AutoInit stabilizes signals by analyzing the forward pass of activations from the input to the output of the network.  It is possible to similarly model the backward pass of gradients from the output to the input.  Indeed, past weight initialization strategies have sometimes utilized signals in both directions  \cite{glorot2010understanding, arpit2019initialize}.  It would be interesting to find out whether AutoInit could similarly benefit from analyzing backward-propagating signals.

Alternative objectives beyond mean and variance stabilization could also be considered.  Two promising objectives are tuning the conditioning of the Fisher information matrix \cite{pennington2018spectrum} and achieving dynamical isometry \cite{xiao2018dynamical}.  Mean field theory and nonlinear random matrix theory \cite{pennington2019nonlinear} could potentially be used to implement these objectives into AutoInit.

\paragraph{Support for New Layer Types}
AutoInit calculates outgoing mean and variance estimates for the majority of layer types available in current deep learning frameworks (Section~\ref{sec:autoinit:mean_variance_estimation}).  If AutoInit encounters an unknown layer, the default behavior is to assume that the mean and variance are not changed by that layer: $g_\texttt{layer}(\mu_\inn, \nu_\inn) = \mu_\inn, \nu_\inn$.  This fallback mechanism tends to work well; if there are only a few unknown layers, then the variance estimation will be incorrect only by a constant factor and training can proceed.  However, mean and variance estimation functions $g$ can be derived for new types of layers as they are developed, either analytically or empirically with Monte Carlo sampling, thus taking full advantage of AutoInit's ability to stabilize training in the future as well.

\paragraph{Tighter Integration with Deep Learning Frameworks}
Using AutoInit is simple in practice.  The AutoInit package provides a wrapper around TensorFlow models.  The wrapper automatically traverses the TensorFlow computation graph, calculates mean and variance estimations for each layer, and reinstantiates the model with the correct weight scaling.  However, this implementation can be streamlined.  The most effective approach would be to integrate AutoInit natively with deep learning frameworks like TensorFlow \cite{abadi2016tensorflow} and PyTorch \cite{paszke2019pytorch}.  Native integration would not just make AutoInit easier to use, it would also make it more accessible to general machine learning practitioners.  For example, TensorFlow provides a few initialization strategies that can be leveraged by changing the \texttt{kernel\_initializer} keyword in certain layers.  However, implementing other weight initialization strategies requires subclassing from the \texttt{Initializer} base class, which is both time-consuming and complicated, especially for non-experts.  Native integration would ensure that the benefits of smarter initialization are available immediately to the wider machine learning community.

\section{Conclusion}
\label{sec:autoinit:conclusion}

This chapter introduced AutoInit, an algorithm that calculates analytic mean- and variance-preserving weight initialization for neural networks automatically.  In convolutional networks, the initialization improved performance with different activation functions, dropout rates, learning rates, and weight decay settings.  In residual networks, AutoInit prevented exploding signals, allowed training with higher learning rates, and improved performance with or without batch normalization.  In transformers, AutoInit was scaled up to high-resolution image classification, and improved performance with several activation functions with and without normalization.  AutoInit also improved accuracy on the ImageNet dataset.  The initialization is independent of data and is therefore efficient and reliable. AutoInit's generality proved instrumental in two types of AutoML. In neural architecture search, new architectures were evaluated more accurately, resulting in better networks in vision, language, tabular, multi-task, and transfer learning settings.  In activation function discovery, AutoInit stabilized training and improved accuracy with a large diversity of novel activation functions.  Thus, AutoInit serves to make machine learning experiments more robust and reliable, resulting in higher performance, and facilitating future research in AutoML.  Although AutoInit accelerated activation function discovery, the process is still computationally expensive.  To reduce this cost, the next chapter learns better activation functions in a data-driven way, improving efficiency by orders of magnitude.

\chapter{AQuaSurF: Efficient Activation Function Optimization through Surrogate Modeling}
\label{chap:aquasurf}

Activation functions are an important choice in neural network design \cite{apicella2021survey, nwankpa2018activation}.  In order to realize the benefits of good activation functions, researchers often design new functions based on characteristics like smoothness, groundedness, monotonicity, and limit behavior.  While these properties have proven useful, humans are ultimately limited by design biases and by the relatively small number of functions they can consider.  On the other hand, automated search methods can evaluate thousands of unique functions, and as a result, often discover better activation functions than those designed by humans.  However, such approaches do not usually have a theoretical justification, and instead focus only on performance.  This limitation results in computationally inefficient ad hoc algorithms that may miss good solutions and may not scale to large models and datasets.

This chapter addresses these drawbacks in a data-driven way through three steps. First, in order to provide a foundation for theory and algorithm development, convolutional, residual, and vision transformer based architectures were trained from scratch with 2{,}913 different activation functions, resulting in three activation function benchmark datasets: \texttt{Act-Bench-CNN}, \texttt{Act-Bench-ResNet}, and \texttt{Act-Bench-ViT}.  These datasets make it possible to analyze activation function properties at a large scale in order to determine which are most predictive of performance.

The second step was to characterize the activation functions in these benchmark datasets analytically, leading to a surrogate performance measure. Exploratory data analysis revealed two activation function properties that are highly indicative of performance: (1) the spectrum of the Fisher information matrix associated with the model’s predictive distribution at initialization, and (2) the activation function’s output distribution.  Both sets of features contribute unique information.  They are both predictive of performance on their own, but they are most powerful when used in tandem.  These features were combined to create a metric space where a low-dimensional representation of the activation functions could be learned. This space can then be used as a surrogate in the search for good activation functions.

In the third step, this surrogate was evaluated experimentally, first by verifying that it can discover known good functions in the benchmark datasets efficiently and reliably, and second by demonstrating that it can discover improved activation functions in new tasks on CIFAR-100 and ImageNet. The representation turned out to be so powerful that an out-of-the-box regression algorithm was able to search it effectively.  The approach, called \technique (\techniqueexpanded), is orders of magnitude more efficient than past work.  Indeed, whereas previous approaches evaluated hundreds or thousands of activation functions, \technique requires only tens of evaluations in order to discover functions that outperform a wide range of baseline activation functions in each context.  Code implementing the \technique algorithm is available at \codeurl.

Prior research on activation function optimization and Fisher information matrices is reviewed in Chapter \ref{chap:background}. This work extends it in three ways. First, the benchmark collections are made available at \dataurl, providing a foundation for further research on activation function optimization.  Second, the low-dimensional representation of the Fisher information matrix makes it a practical surrogate measure, making it possible to apply it to not only activation function design, but potentially also to other applications in the future. Third, the already-discovered functions can be used immediately to improve performance in image processing tasks, and potentially in other tasks in the future.

\section{Activation Function Benchmark Datasets}
\label{sec:aquasurf:activation_function_datasets}

As the first step, this section introduces three activation function benchmark datasets: \texttt{Act-Bench-CNN}, \texttt{Act-Bench-ResNet}, and \texttt{Act-Bench-ViT}.  Each dataset contains training results for 2{,}913 unique activation functions when paired with different architectures and tasks: All-CNN-C on CIFAR-10, ResNet-56 on CIFAR-10, and MobileViTv2-0.5 on Imagenette \cite{springenberg2015striving, he2016identity, mehta2022separable, krizhevsky2009learning, imagenette}.  These functions were created using the main three-node computation graph from PANGAEA \citep{bingham2020gecco}. Details are in Appendix \ref{ap:details:aquasurf_search_space}.

\begin{figure}
    \centering
    \includegraphics[width=0.75\linewidth]{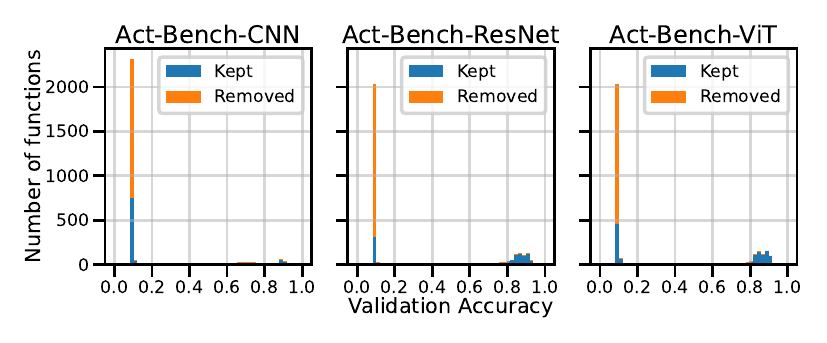}
    \caption{Distribution of validation accuracies with 2{,}913 unique activation functions from the three benchmark datasets.  Many activation functions result in failed training (indicated by the chance accuracy of 0.1), suggesting that searching for activation functions is a challenging problem.  However, most of these functions do not have valid FIM eigenvalues, and can thus be filtered out effectively.}
    \label{fig:aquasurf:accuracy_distribution}
\end{figure}

Figure \ref{fig:aquasurf:accuracy_distribution} shows the distribution of validation accuracies in these datasets.  In all three datasets, the distribution is highly skewed towards functions that result in failed training.  The plots suggest that it is difficult to design good activation functions, and explain why existing methods are often so computationally expensive.  Notwithstanding this difficulty, the histograms show that many unique functions do achieve good performance.  Thus, searching for new activation functions is a worthwhile task that requires a smart approach.

\begin{figure}
    \centering
    \includegraphics[width=0.75\linewidth]{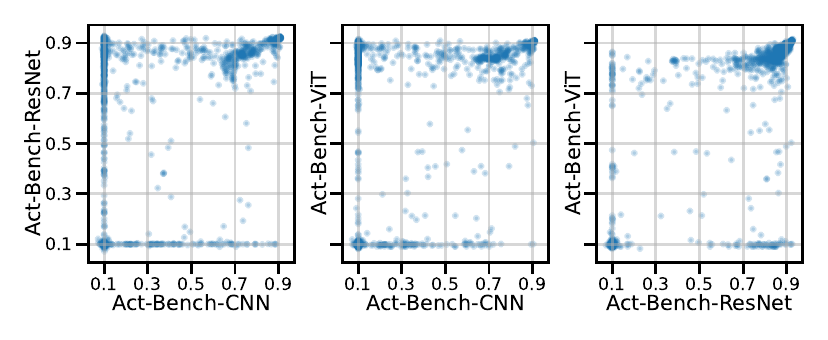}
    \caption{Distribution of validation accuracies across the benchmark datasets.  Each point represents one activation function, and its $x$ and $y$ coordinates represent its performance in two of the three datasets. Some activation functions perform well on all tasks, and their improvements correlate.  Others are more specialized and succeed only with one or two tasks.}
    \label{fig:aquasurf:accuracy_scatter}
\end{figure}

The scatter plots in Figure \ref{fig:aquasurf:accuracy_scatter} show the same distribution of accuracies as the histograms in Figure \ref{fig:aquasurf:accuracy_distribution}, but through scatter plots that show how activation functions perform across different tasks.  Two interesting observations can be made.  First, all three plots contain linearly correlated clusters of points in the upper right corner.  This finding suggests that there are modifications to activation functions that make them more powerful, regardless of the task.  Second, the clusters of points in the upper left and lower right corners represent activation functions that succeed in one task but fail in another, demonstrating that the best results come from discovering functions specialized to individual tasks.

The three benchmark datasets thus form a foundation for developing and evaluating methods for automated activation function design. In the next two sections, they are used to develop a measure that serves as surrogate performance metric, making it possible to scale up activation function optimization to large networks and datasets.
\section{Fisher Information Matrix Details}
\label{sec:aquasurf:fim_details}

In order to calculate the FIM, this chapter uses the K-FAC approach \cite{martens2015optimizing, grosse2016kronecker, martens2018kronecker}.  This technique is summarized in this section, with notation similar to that of \citet{grosse2016kronecker}.

\subsection{Preliminaries}

A feedforward neural network maps an input $\mathbf{a}_0 = \mathbf{x}$ to an output $\mathbf{a}_L = f(\mathbf{x}; \bm{\theta})$ through a series of $L$ layers.  Each layer $l \in \{1, \ldots, L\}$ is comprised of a weight matrix $\mathbf{W}_l$, a bias vector $\mathbf{b}_l$, and an element-wise activation function $\phi_l$.  With $\bar{\mathbf{W}}_l = \begin{pmatrix} \mathbf{b}_l & \mathbf{W}_l \end{pmatrix}$ and $\bar{\mathbf{a}}_l = \begin{pmatrix} 1 & \mathbf{a}_l^\top \end{pmatrix}^\top$, each layer implements the transformation
\begin{align}
    \mathbf{s}_l &= \bar{\mathbf{W}}_l\bar{\mathbf{a}}_{l-1}, \\
    \mathbf{a}_l &= \phi_l(\mathbf{s}_l).
\end{align}
Let $\bm{\theta} = \begin{pmatrix} \vect(\bar{\mathbf{W}}_1)^\top\ \cdots\ \vect(\bar{\mathbf{W}}_L)^\top \end{pmatrix}^\top$ represent the vector of all network parameters.  Parameterized by $\bm{\theta}$ and given inputs $\mathbf{x}$ drawn from a training distribution $Q_{\mathbf{x}}$, the neural network defines the conditional distribution $R_{\mathbf{y} | f(\mathbf{x} ; \bm{\theta})}$.  The Fisher information matrix associated with this model is
\begin{equation}
    \mathbf{F} = \mathop{\E}_{\mathclap{\substack{
        \mathbf{x} \sim Q_{\mathbf{x}} \\
        \mathbf{y} \sim R_{\mathbf{y} | f(\mathbf{x} ; \bm{\theta})}
    }}}
    \left[
        \nabla_{\bm{\theta}} \mathcal{L}(\mathbf{y}, f(\mathbf{x}; \bm{\theta})) 
        \nabla_{\bm{\theta}} \mathcal{L}(\mathbf{y}, f(\mathbf{x}; \bm{\theta}))^\top
    \right].
\end{equation}
As usual in deep learning, the loss function $\mathcal{L}(\mathbf{y}, \mathbf{z})$ represents the negative log-likelihood associated with $R_{\mathbf{y} | f(\mathbf{x} ; \bm{\theta})}$ and quantifies the discrepancy between the model's prediction $\mathbf{z} = f(\mathbf{x}; \bm{\theta})$ and the true label $\mathbf{y}$.  The network is trained to minimize the loss by updating its parameters according to the gradient $\nabla_{\bm{\theta}} \mathcal{L}(\mathbf{y}, f(\mathbf{x}; \bm{\theta}))$.

\subsection{Approximations}

For ease of notation, write $\mathcal{D}\mathbf{v} = \nabla_{\mathbf{v}} \mathcal{L}(\mathbf{y}, f(\mathbf{x}; \bm{\theta}))$.  Recalling that \linebreak $\bm{\theta} = \begin{pmatrix} \vect(\bar{\mathbf{W}}_1)^\top\ \cdots\ \vect(\bar{\mathbf{W}}_L)^\top \end{pmatrix}^\top$, the FIM can be expressed as an $L \times L$ block matrix:
\begin{equation}
    \mathbf{F} = \begin{pmatrix}
        \E \left[ \vect(\mathcal{D}\bar{\mathbf{W}}_1) \vect(\mathcal{D}\bar{\mathbf{W}}_1)^\top \right] & 
        \cdots &
        \E \left[ \vect(\mathcal{D}\bar{\mathbf{W}}_1) \vect(\mathcal{D}\bar{\mathbf{W}}_L)^\top \right] \\
        \vdots & \ddots & \vdots \\
        \E \left[ \vect(\mathcal{D}\bar{\mathbf{W}}_L) \vect(\mathcal{D}\bar{\mathbf{W}}_1)^\top \right] & 
        \cdots &
        \E \left[ \vect(\mathcal{D}\bar{\mathbf{W}}_L) \vect(\mathcal{D}\bar{\mathbf{W}}_L)^\top \right]
    \end{pmatrix}.
\end{equation}
Note that $\mathcal{D}\bar{\mathbf{W}}_l = \mathcal{D}\mathbf{s}_l\bar{\mathbf{a}}_{l-1}^\top$, and recall that $\vect(\mathbf{u}\mathbf{v}^\top) = \mathbf{v} \otimes \mathbf{u}$.  Each block of the FIM can be written as
\begin{align}
    \mathbf{F}_{i,j} &= \E \left[ \vect(\mathcal{D}\bar{\mathbf{W}}_i) \vect(\mathcal{D}\bar{\mathbf{W}}_j)^\top \right] \\
    &= \E \left[ \vect(\mathcal{D}\mathbf{s}_i\bar{\mathbf{a}}_{i-1}^\top) \vect(\mathcal{D}\mathbf{s}_j\bar{\mathbf{a}}_{j-1}^\top)^\top \right] \\
    &= \E \left[( \bar{\mathbf{a}}_{i-1} \otimes \mathcal{D}\mathbf{s}_i)( \bar{\mathbf{a}}_{j-1} \otimes \mathcal{D}\mathbf{s}_j)^\top \right] \\
    &= \E \left[( \bar{\mathbf{a}}_{i-1} \otimes \mathcal{D}\mathbf{s}_i)( \bar{\mathbf{a}}_{j-1}^\top \otimes \mathcal{D}\mathbf{s}_j^\top) \right] \\
    &= \E \left[ \bar{\mathbf{a}}_{i-1} \bar{\mathbf{a}}_{j-1}^\top \otimes \mathcal{D}\mathbf{s}_i \mathcal{D}\mathbf{s}_j^\top \right].
\end{align}
Two approximations are necessary in order to make representation of the FIM practical.  First, assume that different layers have uncorrelated weight derivatives.  The FIM can then be approximated as a block diagonal matrix, with $\mathbf{F}_{i,j} = \mathbf{0}$ if $i \neq j$.  Second, if one approximates the pre-activation derivatives $\mathcal{D}\mathbf{s}_l$ and activations $\bar{\mathbf{a}}_{l-1}^\top$ as independent, then the diagonal blocks of the FIM can be further decomposed into the Kronecker product of two smaller matrices:
\begin{equation}
    \mathbf{F}_{l, l} = \E \left[ \bar{\mathbf{a}}_{l-1} \bar{\mathbf{a}}_{l-1}^\top \otimes \mathcal{D}\mathbf{s}_l \mathcal{D}\mathbf{s}_l^\top \right] \approx \E \left[ \bar{\mathbf{a}}_{l-1}\bar{\mathbf{a}}_{l-1}^\top \right] \otimes \E \left[ \mathcal{D}\mathbf{s}_l\mathcal{D}\mathbf{s}_l^\top \right].
\end{equation}
Let $\mathbf{\Omega}_{l} = \E \left[ \bar{\mathbf{a}}_{l} \bar{\mathbf{a}}_{l}^\top \right]$ and $\mathbf{\Gamma}_{l} = \E \left[ \mathcal{D}\mathbf{s}_l \mathcal{D}\mathbf{s}_l^\top \right]$.  The approximate empirical FIM is then written as
\begin{equation}
    \hat{\mathbf{F}} = 
    \begin{pmatrix}
        \mathbf{\Omega}_0 \otimes \mathbf{\Gamma}_1 
        & & \mathbf{0} \\
        & \ddots & \\
        \mathbf{0} & &
        \mathbf{\Omega}_{L-1} \otimes \mathbf{\Gamma}_L
    \end{pmatrix}.
\end{equation}

\subsection{Layer-Specific Implementation}

The above example illustrates FIM approximation for a simple feedforward network.  However, most modern architectures contain several different kinds of layers.  Some layers like pooling, normalization, or dropout layers do not have trainable weights, and therefore these layers are not included in the FIM \cite{ioffe2015batch, srivastava2014dropout}.  

Each diagonal entry $\mathbf{\Omega}_{l-1} \otimes \mathbf{\Gamma}_l$ corresponds to one layer with weights.  The calculation differs slightly depending on the layer type, but otherwise the example above can be straightforwardly extended to more complicated networks.  Calculations for three common layer types are presented below.

\paragraph{Dense Layers}
For dense layers, the matrices $\mathbf{\Omega}_{l-1}$ and $\mathbf{\Gamma}_l$ can be readily computed with one forward and backward pass through the network using a mini-batch of data.  The eigenvalues are then computed using standard techniques.

\paragraph{Convolutional Layers}
Convolutional layers require special consideration to calculate $\mathbf{\Omega}_{l-1}$ and $\mathbf{\Gamma}_l$.  For a given layer, let $M$ represent the batch size, $\mathcal{T}$ the set of spatial locations (typically a two-dimensional grid), $\Delta$ the set of spatial offsets from the center of the filter, and $I$ and $J$ the number of output and input maps, respectively.  The activations are represented by the $M \times |\mathcal{T}| \times J$ array $\mathbf{A}_{l-1}$.  The weights are represented by the $I \times |\Delta| \times J$ array $\mathbf{W}_l$ which is interpreted as an $I \times |\Delta|J$ matrix.  The expansion operator $\llbracket \cdot \rrbracket$ extracts patches around each spatial location and flattens them into vectors that become the rows of a matrix: $\llbracket \mathbf{A}_{l-1} \rrbracket$ is a $M|\mathcal{T}| \times J|\Delta|$ matrix.  

Similar to the feedforward networks, the bias (if used) can be prepended to the weights matrix as $\bar{\mathbf{W}}_l = \begin{pmatrix} \mathbf{b}_l & \mathbf{W}_l \end{pmatrix}$ and a homogeneous column of ones to the expanded activations as $\llbracket \mathbf{A}_{l-1} \rrbracket_H = \begin{pmatrix} \mathbf{1} & \llbracket \mathbf{A}_{l-1} \rrbracket \end{pmatrix}$.  This constructions allows the forward pass to be written as 
\begin{align}
    \mathbf{S}_l &= \llbracket \mathbf{A}_{l-1} \rrbracket_H \bar{\mathbf{W}}_l^\top, \\
    \mathbf{A}_l &= \phi\left( \mathbf{S}_l \right),
\end{align}
from which the factors are computed as
\begin{align}
    \mathbf{\Omega}_l &= \E \left[ \llbracket \mathbf{A}_{l} \rrbracket_H^\top \llbracket \mathbf{A}_{l} \rrbracket_H \right], \\
    \mathbf{\Gamma}_l &= \frac{1}{|\mathcal{T}|} \E \left[ \mathcal{D}\mathbf{S}_l^\top \mathcal{D}\mathbf{S}_l \right].
\end{align}

\paragraph{Depthwise Convolutional Layers}
Depthwise convolutional layers utilize separate kernels for each channel.  In this case, $\llbracket \mathbf{A}_{l-1} \rrbracket$ is a $M|\mathcal{T}|J \times |\Delta|$ matrix.  Otherwise, the factors $\mathbf{\Omega}_{l-1}$ and $\mathbf{\Gamma}_l$ are calculated in the same way as they are for standard convolutional layers.

\subsection{Eigenvalue Calculation}
Because $\hat{\mathbf{F}}$ is a block-diagonal matrix, its eigenvalues are simply the combined eigenvalues of each block: $\lambda({\hat{\mathbf{F}}}) = \{\lambda({\hat{\mathbf{F}}_l})\}_{l=1}^L$.  The eigenvalue calculation for one block $\hat{\mathbf{F}}_l = \mathbf{\Omega}_{l-1} \otimes \mathbf{\Gamma}_l$ is further simplified by first computing the eigenvalues $\lambda(\mathbf{\Omega}_{l-1})$ and $\lambda(\mathbf{\Gamma}_l)$ for each Kronecker factor separately and then returning all pairwise products from the two sets of eigenvalues.  For numerical stability, the eigenvalues can first be log-scaled and then all pairwise sums from the two sets are returned.  Calculating the eigenvalues requires one forward and backward pass through the network with a mini-batch of data.  The computational cost is therefore relatively cheap, especially compared with the cost of fully training a network from scratch.

It is possible for the FIM eigenvalues to be invalid.  For example, if the forward propagated activations or backward propagated gradients explode or vanish, then the diagonal entries $\mathbf{\Omega}_{l-1} \otimes \mathbf{\Gamma}_l$ may be undefined.  Such invalid values result from activation functions that are unstable. Therefore, invalid FIM eigenvalues provide a good way to filter out bad activation functions.

\section{Features and Distance Metrics}
\label{sec:aquasurf:features_dist_metrics}

In order to make efficient search for good activation functions possible, the surrogate space needs to be low-dimensional, represent informative features, and have an appropriate distance metric. In the second step, an approach is developed based on two kinds of features: the eigenvalues of the Fisher information matrix, and the outputs of the activation function.  In this section, each feature type is motivated first, and a metric is then developed for computing distances between activation functions. They are then put together into a surrogate in the next section.

\subsection{FIM Eigenvalues}
The Fisher information matrix (FIM) is an important concept in characterizing neural network models.  Viewed from various perspectives, the FIM determines a neural network's capacity for learning, ability to generalize, the robustness of the network to small perturbations of its parameters, and the geometry of the loss function near the global minimum \cite{karakida2019universal, liang2019fisher, liao2018approximate, hayase2021spectrum, karakida2021pathological, jastrzebski2021catastrophic, furusho2020theoretical}.

The FIM has $|{\bm \theta}|$ eigenvalues.  The distribution of eigenvalues can be represented by binning the eigenvalues into an $m$-bucket histogram, and this $m$-dimensional vector serves as a computational characterization of the network.  Importantly, different activation functions induce different FIM eigenvalues for a given neural network. They can be calculated at initialization and do not require training; they can thus serve as a low-dimensional feature-vector representation of the activation function. The FIM eigenvalues are immediately useful for filtering out poor activation functions; if they are invalid, the activation function is likely to fail in training (Figure \ref{fig:aquasurf:accuracy_distribution}).  However, in order to use them as a surrogate, a distance metric needs to be defined.

Given a neural network architecture $f$, let $f_\phi$ and $f_\psi$ be two instantiations with different activation functions $\phi$ and $\psi$. Let $\mu_l$ and $\nu_l$ represent the distributions of eigenvalues corresponding to the weights in layer $l$ of neural networks $f_\phi$ and $f_\psi$, respectively, and let $w_l$ be the number of weights in layer $l$ of the networks.  The distance between $f_\phi$ and $f_\psi$ is then computed as a weighted layer-wise sum of 1-Wasserstein distances
\begin{equation}
    \label{eq:aquasurf:dist_fim_eigs}
    d(f_\phi, f_\psi) = \sum_{l=1}^L \frac{W_1(\mu_l, \nu_l)}{w_l}.
\end{equation}

With this distance metric, the FIM eigenvalue vector representations encode a low-dimensional embedding space for activation functions, making efficient search for good functions possible. Note, however, that the FIM eigenvalues incorporate multiple sources of information, including the activation function, neural network structure, data distribution, and loss function.  This fact makes the eigenvalues powerful features, but also introduces noise in the prediction process.  Fortunately, it is possible to combine them with another feature, activation function outputs, to address this shortcoming.

\subsection{Activation Function Outputs}
The shape of an activation function $\psi$ can be described by a vector of $n$ sample values $\psi(x)$. If the network's weights are appropriately initialized, the input activations to its neurons are initially distributed as $\mathcal{N}(0,1)$ \cite{bingham2021autoinit}. Therefore, the sampling $x \sim \mathcal{N}(0,1)$ provides an $n$-dimensional feature vector that represents the expected use of the activation function at initialization.  A distance metric in this feature vector space can be defined naturally as the Euclidean distance
\begin{equation}
    \small
    \label{eq:aquasurf:dist_fn_outputs}
    d(f_\phi, f_\psi) = \sqrt{\frac{\sum_{i=1}^n(\phi(x_i) - \psi(x_i))^2}{n}}, \quad x \sim \mathcal{N}(0,1).
\end{equation}
Functions with similar shapes will have a small distance between them, while those with different shapes will have a large distance.  Because these output feature vectors depend only on the activation function, they are reliable and inexpensive to compute. Most importantly, together with the FIM eigenvalues, they constitute a powerful surrogate search space, as will be demonstrated in the next section.

\section{Using the Features as a Surrogate}
\label{sec:aquasurf:visualizing_umap}

In this section, the UMAP dimensionality reduction technique is used to visualize the FIM and output features across the benchmark datasets. This visualization leads to a combined surrogate space that can be used to accelerate the search for good activation functions. 

\subsection{Visualization with UMAP} 
The features developed above can be visualized using the UMAP algorithm \cite{mcinnes2018umap}.  UMAP is a general dimension reduction approach similar to t-SNE, but is better at scaling to large sample sizes and preserving global structure \cite{van2008visualizing}.  As a first demonstration, Figure \ref{fig:aquasurf:interpolation} shows a 2D representation of the 2{,}913 activation functions in the benchmark datasets. Each function was represented as an 80-dimensional vector of output values. Interpolating between embedded points confirms that UMAP learns a good underlying representation.

\begin{figure}
    \centering
    \includegraphics[width=0.6\linewidth]{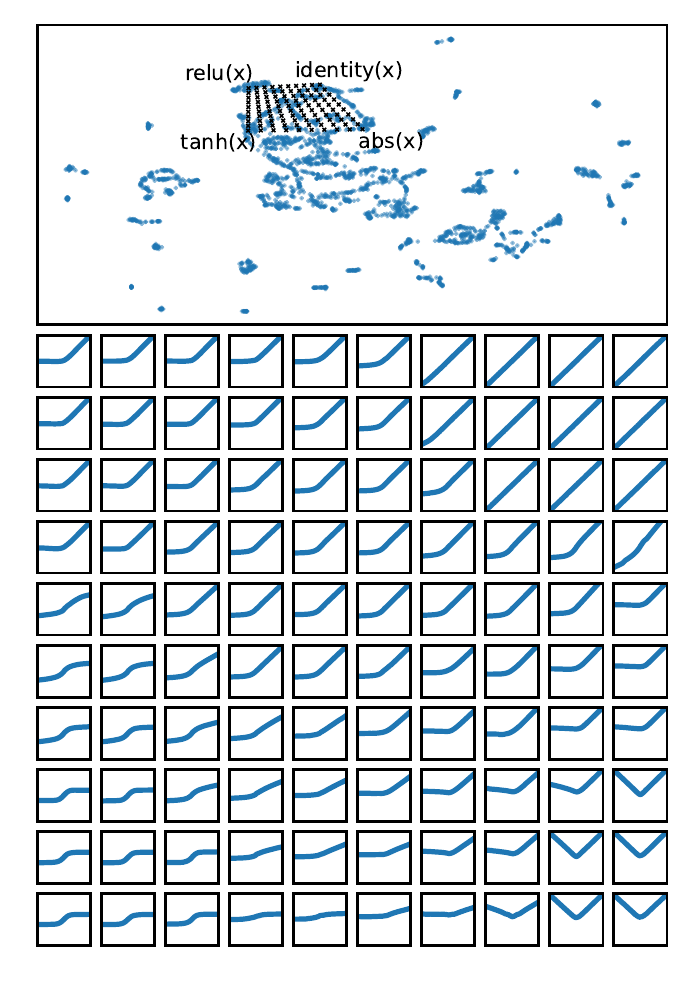}
    \caption{UMAP embedding of the 2{,}913 activation functions in the benchmark datasets. Each point stands for a unique activation function, represented by an 80-dimensional output feature vector. The embedding locations of four common activation functions are labeled.  The black x's mark coordinates interpolating between these four functions, and the grid of plots on the bottom shows reconstructed activation functions at each of these points.  UMAP interpolates smoothly between different kinds of functions, suggesting that it is a good approach for learning low-dimensional representations of activation functions.}
    \label{fig:aquasurf:interpolation}
\end{figure}

UMAP was also used to project the activation functions to nine two-dimensional spaces according to the distance metrics in Equations \ref{eq:aquasurf:dist_fim_eigs} and \ref{eq:aquasurf:dist_fn_outputs}.  In Figure \ref{fig:aquasurf:umap_embeddings}, each column represents a different benchmark dataset (\texttt{Act-Bench-CNN}, \texttt{Act-Bench-ResNet}, or \texttt{Act-Bench-ViT}) and each row a different distance metric (FIM eigenvalues with $m=\lfloor |{\bm\theta|} / 100 \rfloor$, activation function outputs with $n=1{,}000$, or both). The plots include only activation functions that were not filtered out.  Each point represents a unique activation function, and the points are colored according to their validation accuracy in the benchmark task. Note that although the performance of each activation function is already known, this information was not given to UMAP; the embeddings are entirely unsupervised. 

\begin{figure}
    \centering
    \includegraphics[width=0.75\linewidth]{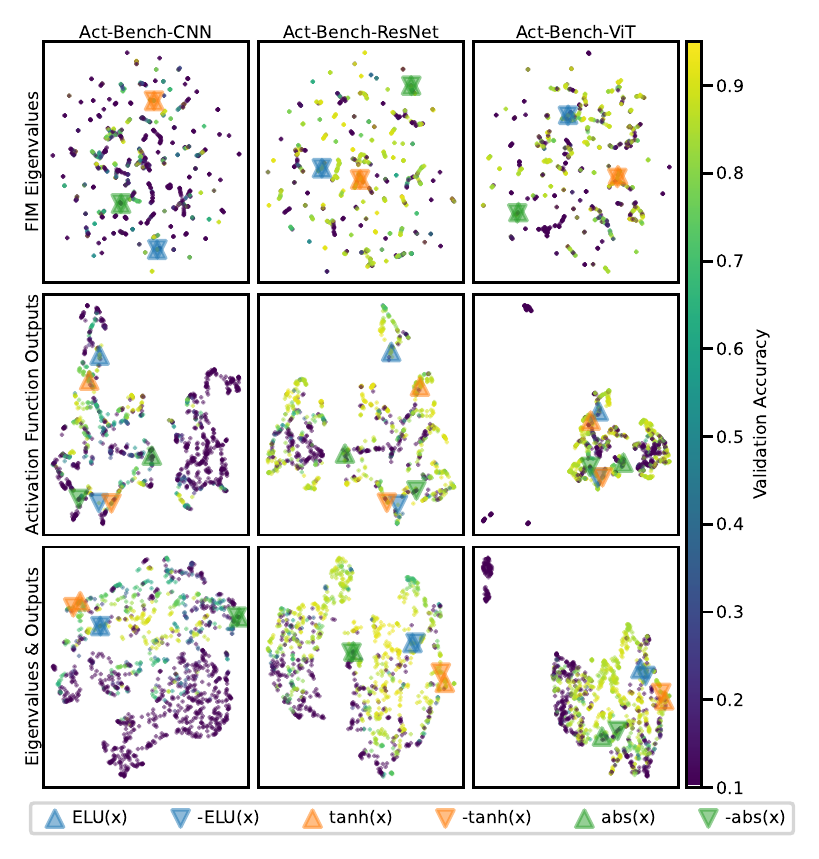}
    \caption{UMAP embeddings of activation functions for each dataset (column) and feature type (row).  Each point represents a unique activation function, and the points are colored according to their validation accuracy on the given dataset.  The colored triangles identify the locations of six well-known activation functions. The areas of similar performance are more continuous in the bottom row; that is, using both FIM eigenvalues and activation function outputs provides a better low-dimensional representation than using either feature alone.}
    \label{fig:aquasurf:umap_embeddings}
\end{figure}

Thus, the visualizations in Figure~\ref{fig:aquasurf:umap_embeddings} illustrate how predictive each feature type is of activation function performance in each dataset.  The next subsections evaluate each feature type in this role in detail, and show that utilizing both features provides better results than either feature alone.   Details are presented in Section~\ref{sec:aquasurf:features_details}.

\subsection{FIM Eigenvalues}
The first row of Figure~\ref{fig:aquasurf:umap_embeddings} shows the 2D UMAP embeddings of the FIM eigenvalue vectors associated with each activation function.  There are several clusters in these plots where the points share similar colors.  These regions indicate distinct activation functions with similar FIM eigenvalues.  Such functions induce similar training dynamics in the neural network and eventually lead to similar performance.  On the other hand, some clusters contain activation functions with a wide range of performances, and some points do not belong to any cluster at all.  Overall, the plots suggest that FIM eigenvalues are a useful predictor of performance, but also that incorporating additional information could lead to better results.

\subsection{Activation Function Outputs}
The middle row of Figure \ref{fig:aquasurf:umap_embeddings} shows the 2D UMAP embeddings of the output vectors associated with each activation function.  Points are close to each other in this space if the corresponding activation functions have similar shapes.  These plots are demonstrably more informative than the plots based on the FIM eigenvalues in three ways.  First, the purple points are better separated from the others.  This separation means that activation functions that fail (those achieving 0.1 chance accuracy) are better separated from those that do well.  Second, most points' immediate neighbors have similar colors.  This similarity means that activation functions with similar shapes lead to similar accuracy, and analyzing activation function outputs on their own is more informative than analyzing the FIM eigenvalues.  Third, the plots include multiple regions where there are one-dimensional manifolds that exhibit smooth transitions in accuracy, from purple to blue to green to yellow.  Thus, not only does UMAP successfully embed similar activation functions near each other, but it also is able to organize the activation functions in a meaningful way.  

There is one drawback to this approach: the performant activation functions (those represented by yellow dots) are often in distinct clusters.  This dispersion means that a search algorithm would have to explore multiple areas of the search space in order to find all of the best functions.  As the next subsection suggests, this issue can be alleviated by utilizing both FIM eigenvalues and activation function outputs. 

\subsection{Combining Features: Eigenvalues \& Outputs} 
The UMAP algorithm uses an intermediate fuzzy topological representation to represent relationships between data points, similar to a neighborhood graph.  This property makes it possible to combine multiple sources of data by taking intersections or unions of the representations in order to yield new representations \cite{mcinnes2018umap}.  The bottom row of Figure \ref{fig:aquasurf:umap_embeddings} utilizes both FIM eigenvalues and activation function outputs by taking the union of the two representations.  Thus, activation functions are embedded close to each other in this space if they have similar shapes, if they induce similar FIM eigenvalues, or both.

The bottom row of Figure \ref{fig:aquasurf:umap_embeddings} shows the benefits of combining the two features.  Unlike the activation function output plots, which contain multiple clusters of high-performing activation functions in different locations in the embedding space, the combined UMAP model embeds all of the best activation functions in similar regions.  The combined UMAP model also places poor activation functions (purple points) in the edge of the embedding space, and brings good functions (yellow points) to the center.  Thus, the embedding space is more convex, and therefore easier to optimize.

In general, activation functions with similar shapes lead to similar performances, and those with different shapes often produce different results.  This property is why the middle row of Figure \ref{fig:aquasurf:umap_embeddings} appears locally smooth.  However, in some cases the shape of the activation function does not tell the whole story, and additional information is needed to ascertain its performance.  

For example, the colored triangles in Figure \ref{fig:aquasurf:umap_embeddings} identify the location of six activation functions in the low-dimensional space: $\textrm{ELU}(x)$, $-\textrm{ELU}(x)$, $\textrm{tanh}(x)$, $-\textrm{tanh}(x)$, $|x|$, and $-|x|$.  In the activation function output space (middle row), all of these functions are mapped to different regions of the space.  The points are spread apart because an activation function and its negative have very different shapes, i.e.\ their output will be different for every nonzero input (Figure \ref{fig:aquasurf:compare_eigenvalues}).  In contrast, in the FIM eigenvalue space (top row of Figure \ref{fig:aquasurf:umap_embeddings}), the points for these pairs of functions overlap because the FIM eigenvalues are comparable (Figure \ref{fig:aquasurf:compare_eigenvalues}).  Indeed, assuming the weights are initialized from a distribution symmetric about zero, negating an activation function does not change the training dynamics of a neural network, and they are functionally equivalent.

\begin{figure}
    \centering
    \includegraphics[width=0.7\linewidth]{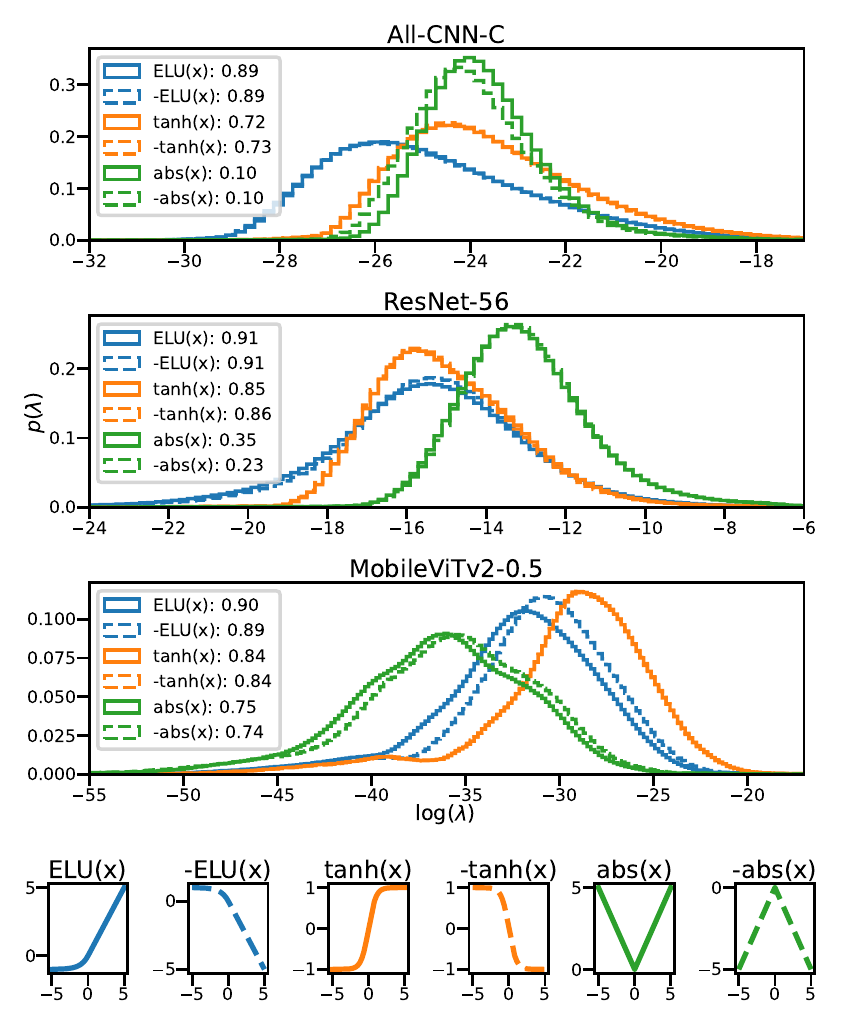}
    \caption{FIM eigenvalue distributions for different architectures and activation functions.  The legends show the activation function and the corresponding validation accuracy in different tasks.  Although negating an activation function changes its shape, it does not substantially change its behavior nor its performance.  FIM eigenvalues capture this relationship between activation functions.  The eigenvalues are thus useful for finding activation functions that appear different but in fact behave similarly, and these discoveries in turn improve the efficiency of activation function search.}
    \label{fig:aquasurf:compare_eigenvalues}
\end{figure}

This issue complicates the search process in two ways.  The first is that good activation functions are mapped to different regions of the embedding space, and so a search algorithm must explore multiple areas in order to find the best function.  The second challenge is that distinct regions of the space may contain redundant information: if $\textrm{ELU}(x)$ is known to be a good activation function, it is not helpful to spend compute resources evaluating $-\textrm{ELU}(x)$ only to discover that it achieves the same performance.  

Negating an activation function is a clear example of a modification that changes the shape of the activation function, but does not affect the training of a neural network.  More broadly, it is likely that there exist activation functions that differ in other ways (besides just negation), but that still induce similar training dynamics in neural networks.  Fortunately, utilizing FIM eigenvalues and activation function outputs together provides enough information to tease out these relationships.  FIM eigenvalues take into account the activation function, the neural network architecture, the loss function, and the data distribution.  The eigenvalues are more meaningful features than activation function outputs, which only depend on the shape of the function.  However, as Figure \ref{fig:aquasurf:umap_embeddings} shows, the FIM eigenvalues are noisier features, while the activation function outputs are quite reliable.  Thus, utilizing both features is a natural way to combine their strengths and address their weaknesses.  

\subsection{Constructing a Surrogate}
These observations suggest an opportunity for an effective surrogate measure: The UMAP coordinates in the bottom row of Figure~\ref{fig:aquasurf:umap_embeddings} have the information needed to predict how well an activation function will perform. They capture the essence of the $m$- and $n$-dimensional feature vectors, and distill it into a 2D representation that can be computed efficiently and used to guide the search for good functions. As the third step in this research, the next two sections evaluate this process experimentally, demonstrating that it is efficient and reliable, and that it scales to new and challenging datasets and search spaces.

\section{Features and Surrogate Details}
\label{sec:aquasurf:features_details}

This section describes how the activation function features were implemented and how the surrogate was constructed.

\paragraph{Calculating FIM Eigenvalues}
The FIM eigenvalues were calculated for each activation function as discussed in Section \ref{sec:aquasurf:features_dist_metrics}.  The eigenvalues were log-scaled for numerical stability.  By definition, the number of eigenvalues is the same as the number of weights in the neural network.  To save space, the eigenvalues were binned to histograms.  For a layer $l$ with $|{\bm \theta}_l|$ weights, $\lfloor |{\bm \theta}_l| / 100 \rfloor$ equally sized bins from $-100$ to $100$ were used.  One histogram was computed for each layer in a network, and all of the histograms were concatenated together into a single feature vector for a given activation function. In this manner, the total dimensionality was 13{,}692 for All-CNN-C, 16{,}500 for ResNet-56, and 11{,}013 for MobileViTv2-0.5.

\paragraph{Calculating Activation Function Outputs}
The activation function outputs $y = f(x)$ were calculated for each activation function $f$ by sampling $n=$1{,}000 values $x \sim \mathcal{N}(0,1)$ and truncating to the range $[-5,5]$.  The same random inputs were used for all activation functions.

\paragraph{Per-Layer FIM Eigenvalues}
In Figure \ref{fig:aquasurf:compare_eigenvalues}, the eigenvalues for the entire network are shown for completeness.  However, the UMAP representations shown in Figure \ref{fig:aquasurf:umap_embeddings} were produced by keeping the eigenvalues at each layer separate and computing a weighted distance between them (according to Equation \ref{eq:aquasurf:dist_fim_eigs}).  As pointed out above, FIM eigenvalues are informative but noisy features.  In preliminary experiments, keeping the eigenvalues separate at each layer reduced some of this noise, resulting in a more informative Figure \ref{fig:aquasurf:umap_embeddings} and consequently improving the performance of the search algorithms.

\paragraph{UMAP Settings}
UMAP exposes a number of parameters that can be used to customize its behavior \cite{mcinnes2018umap}.  The \texttt{metric} parameter determines how distances are computed between points, the \texttt{n\_neighbors} parameter adjusts the tradeoff between the local and global structure of the data, and the \texttt{min\_dist} parameter controls the minimum distance between points in the embedding space.  

The plots in Figure \ref{fig:aquasurf:umap_embeddings} were produced by computing the distances between FIM eigenvalues and activation function outputs.  For the FIM eigenvalues \texttt{UMAP(metric=`manhattan', n\_neighbors=3, min\_dist=0.1)} was used, and for the activation function outputs \linebreak \texttt{UMAP(metric=`euclidean', n\_neighbors=15, min\_dist=0.1)} was used.  The distance metrics were chosen to implement Equations \ref{eq:aquasurf:dist_fim_eigs} and \ref{eq:aquasurf:dist_fn_outputs}.  

In preliminary experiments, decreasing \texttt{n\_neighbors} from the default of 15 down to 3 for the FIM eigenvalues qualitatively improved the embedding for the combined features.  The combined features were visualized with a union model, i.e.\ \texttt{umap\_combined = umap\_fim\_eigs + umap\_fn\_outputs} \cite{mcinnes2018umap}.

\section{Searching on the Benchmark Tasks}
\label{sec:aquasurf:benchmark_search}

Searching for activation functions typically requires training a neural network from scratch in order to evaluate each candidate function fully, which is often computationally expensive.  With the benchmark datasets, all of the results are already precomputed.  This information makes it possible to experiment with different search algorithms and conduct repeated trials to understand the statistical significance of the results.  These results serve to inform both algorithm design and feature selection, as demonstrated in this section.

\subsection{Setup}
Three algorithms were evaluated: weighted $k$-nearest regression with $k=3$ (KNR), random forest regression (RFR), and support vector regression (SVR).  Gaussian Process Regression (GPR) was also evaluated but found to be inconsistent in preliminary experiments (Appendix~\ref{ap:details:aquasurf}). Random search (RS) was included as a baseline comparison; it did not utilize the FIM eigenvalue filtering mechanism.  The algorithms were used out of the box with default hyperparameters from the scikit-learn package \cite{pedregosa2011scikit}.  They were provided different activation function features in order to understand their potential to predict performance.  The features included FIM eigenvalues, activation function outputs, or both.  The features were preprocessed and embedded in a two-dimensional space by UMAP.  These representations are visualized in Figure \ref{fig:aquasurf:umap_embeddings}; the coordinates of each point correspond exactly to the information given to the regression algorithms.

The ReLU activation function is ubiquitous in machine learning. For many neural network architectures, the performance with ReLU is already known \cite{nair2010rectified, nwankpa2018activation, apicella2021survey}, which makes it a good starting point for search.  For this reason, the search algorithms began by evaluating ReLU and seven other randomly chosen activation functions.  In general, such evaluation requires training from scratch, but with the benchmark datasets, it requires only looking up the precomputed results.  The algorithms then used the validation accuracy of these eight functions to predict the performance of all unevaluated functions in the dataset. The activation function with the highest predicted accuracy was then evaluated.  The performance of this new function was then added to the list of known results, and this process continued until 100 activation functions had been evaluated.  Each experiment comprising a different search algorithm, activation function feature set, and benchmark dataset was repeated 100 times.  The full experimental details are in Appendix \ref{ap:details:aquasurf}.

\begin{figure}
    \centering
    \includegraphics[width=0.6\linewidth]{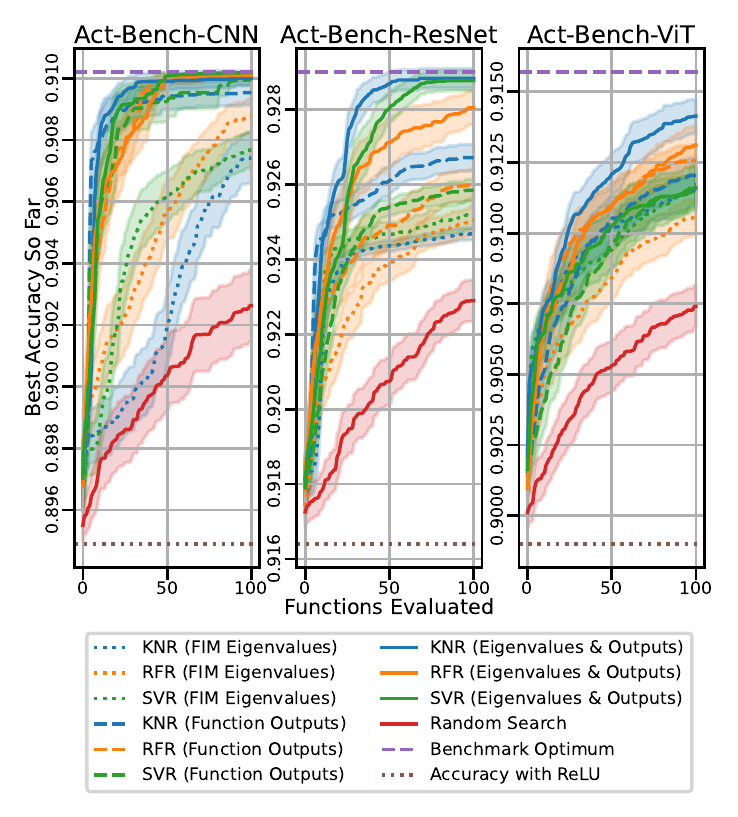}
    \caption{Search results on the three benchmark datasets.  Each curve represents a different search algorithm (KNR, RFR, or SVR) utilizing a different UMAP feature (FIM eigenvalues, function outputs, or both; these features are visualized in Figure \ref{fig:aquasurf:umap_embeddings}).  The curves represent the validation accuracy of the best activation function discovered so far, averaged across 100 independent trials, and the shaded areas show the 95\% confidence interval around the mean.  In all cases, regression with UMAP features outperforms random search, and searching with both eigenvalues and outputs outperforms searching with either feature alone.  Of the three regression algorithms, KNR performs the best, rapidly surpassing ReLU and quickly discovering near-optimal activation functions in all benchmark tasks. Thus, the features make it possible to find good activation functions efficiently and reliably even with off-the-shelf search methods; the benchmark datasets make it possible to demonstrate these conclusions with statistical reliability.}
    \label{fig:aquasurf:benchmark_search}
\end{figure}
\subsection{Results}
Figure \ref{fig:aquasurf:benchmark_search} shows the results of the searches.  Importantly, the curves do not depict just one search trial.  Instead, they represent the average performance aggregated from 100 independent runs, which is made possible by the benchmark datasets.  As indicated by the shaded confidence intervals, the results are reliable and are not simply due to chance.

A number of conclusions can be drawn from Figure \ref{fig:aquasurf:benchmark_search}. First, all search algorithms, even random search, reliably discover activation functions that outperform ReLU.  This finding is supported by previous work (reviewed in Chapter \ref{chap:background}): Although ReLU is a good activation function that performs well in many different tasks, better performance can be achieved with novel activation functions.  Therefore, continuing to use ReLU in the future is unlikely to lead to best results; The choice of the activation function should be an important part of the design, similar to the choice of the network architecture or the selection of its hyperparameters.

Second, all regression algorithms outperform random search.  This finding holds across the three types of activation function features and across the three benchmark datasets.  The FIM eigenvalues and activation function outputs are thus important in predicting performance of activation functions.

Third, regression algorithms trained on both FIM eigenvalues and activation function outputs outperform algorithms trained on just eigenvalues or outputs alone.  This result is consistent across the regression algorithms and benchmark datasets. It suggests that the FIM eigenvalues and activation function outputs contribute complimentary pieces of information.  The finding quantitatively reinforces the qualitative visualization in Figure \ref{fig:aquasurf:umap_embeddings}: FIM eigenvalues are useful for matching activation functions that induce similar training dynamics in neural networks, activation function outputs enable a low-dimensional representation where search is more practical, and combining the two features results in a problem that is more convex and easier to optimize.

Fourth, the searches are efficient.  Previous approaches require hundreds or thousands of evaluations to discover good activation functions \cite{DBLP:conf/iclr/RamachandranZL18, bingham2020gecco, bingham2022discovering}.  In contrast, this chapter leverages FIM eigenvalues and activation function outputs to reduce the problem to simple two-dimensional regression; the features are powerful enough that out-of-the-box regression algorithms can discover good functions with only tens of evaluations.  This efficiency makes it possible to search for better functions directly on large datasets such as ImageNet \cite{deng2009imagenet}, demonstrated next.

\section{Searching on New Tasks}
\label{sec:aquasurf:search_new_tasks}

The experiments in Section \ref{sec:aquasurf:benchmark_search} utilized precomputed datasets and search spaces to demonstrate that UMAP embeddings are useful in predicting activation function performance, and that KNR can find good functions based on them.  In this section, these insights are extended to new datasets and search spaces, demonstrating that \technique scales up to new and more challenging tasks.

\subsection{Setup}
The experiments were scaled up in two ways. First, while the network architectures were the same, the datasets were much larger and more challenging: All-CNN-C on CIFAR-100, ResNet-56 on CIFAR-100, and MobileViTv2-0.5 on ImageNet. Second, a much larger space with 425{,}896 unique activation functions was searched, based on four-node computation graphs (Appendix \ref{ap:details:aquasurf_search_space}). This space is large, diverse, and not precomputed, putting the conclusions from the benchmark experiments to test in a production setting.

Based on the benchmark results, KNR with $k=3$ was used as the search algorithm.  The searches all begin by evaluating the same eight existing activation functions: ELU, ReLU, SELU, sigmoid, Softplus, Softsign, Swish, and tanh.  From this starting point, eight workers operated in parallel evaluating the activation functions with the highest predicted performance. Full experimental details are in Appendix \ref{ap:details:aquasurf}.

\subsection{Results}
Figure \ref{fig:aquasurf:search_progress} shows that all three searches find improved activation functions over time, and Figure \ref{fig:aquasurf:large_search_space} shows how the searches navigate the search space.  In every experiment, new activation functions were discovered that outperform all baseline functions.  Although the search space is large, the searches are efficient, requiring only tens of evaluations to improve performance.  Impressively, the search with ResNet-56 on CIFAR-100 produced an activation function that outperformed all baselines on just the second evaluation.

\begin{figure}
    \centering
    \includegraphics[width=0.75\linewidth]{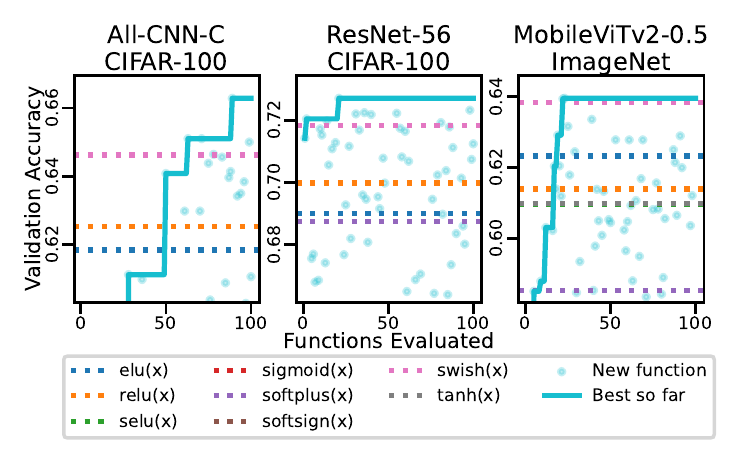}
    \caption{Progress of activation function searches.  Each point represents the validation accuracy with a unique activation function, and the solid line indicates the performance of the best activation function found so far. \technique discovers new activation functions that outperform all baseline functions in every case.}
    \label{fig:aquasurf:search_progress}
\end{figure}
\begin{figure}
    \centering
    \includegraphics[width=0.85\linewidth]{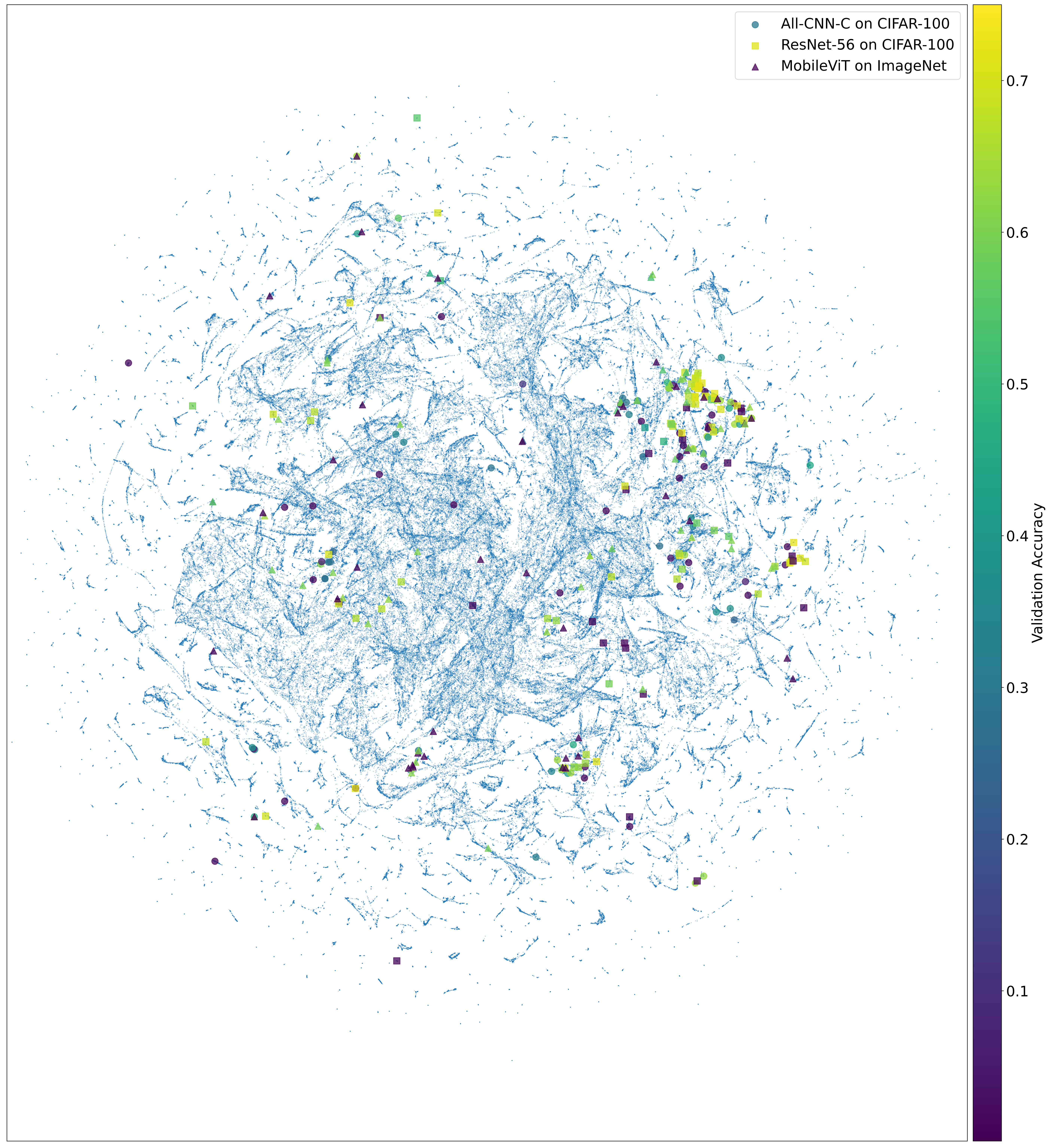}
    \caption{Low-dimensional UMAP representation of the 425{,}896 function search space.  The activation functions are embedded according to their outputs; each point represents a unique function.  The larger points represent activation functions that were evaluated during the searches; they are colored according to their validation accuracy.  Although the space is vast, the searches require only tens of evaluations to discover good activation functions.}
    \label{fig:aquasurf:large_search_space}
\end{figure}

Table \ref{tab:aquasurf:search_results} shows the final results from \technique.  The results reinforce the fact that substantial gains can be obtained when using better activation functions than the default ReLU, and especially those optimized specifically for the task.  The activation functions can also be transferred to new tasks to improve performance, as shown in Table \ref{tab:aquasurf:transfer}.

\newlength{\tabheight}
\setlength{\tabheight}{13ex}

\begin{table}
    \caption{Accuracy with different activation functions.  The CIFAR-100 results show the median test accuracy from three runs, and the ImageNet results show the validation accuracy from a single run.  \technique discovers novel activation functions that outperform all baselines in every case. This result demonstrates both that good functions matter, and the power of optimizing them to the task.\\}
    \centering
    \adjustbox{max height=\tabheight}{%
    \begin{tabular}{ll}
        \toprule
        \multicolumn{2}{c}{All-CNN-C on CIFAR-100} \\
        \midrule
        $\textrm{HardSigmoid}(\textrm{HardSigmoid}(x)) \cdot \textrm{ELU}(x)$ & \textbf{0.6990} \\
        $\sigma(\textrm{Softsign}(x)) \cdot \textrm{ELU}(x)$ & 0.6950 \\
        $\textrm{Swish}(x) / \textrm{SELU}(1)$ & 0.6931 \\ 
        \midrule 
        ELU & 0.6312 \\
        ReLU & 0.6897 \\
        SELU & 0.0100 \\
        sigmoid & 0.0100 \\
        Softplus & 0.6563 \\
        Softsign & 0.2570 \\
        Swish & 0.6913 \\
        tanh & 0.3757 \\
        \bottomrule
    \end{tabular}
    }
    \adjustbox{max height=\tabheight}{%
    \begin{tabular}{ll}
        \toprule
        \multicolumn{2}{c}{ResNet-56 on CIFAR-100} \\
        \midrule
        $\textrm{Swish}(-2x)$ & \textbf{0.7469} \\
        $\textrm{SELU}(\sinh(e^{\arctan(x)}-1))$ & 0.7458 \\
        $x \cdot \textrm{erfc}(\textrm{ELU}(x))$ & 0.7419 \\
        \midrule 
        ELU & 0.7411 \\
        ReLU & 0.7348 \\
        SELU & 0.6967 \\
        sigmoid & 0.5766 \\
        Softplus & 0.7397 \\
        Softsign & 0.6624 \\
        Swish & 0.7401 \\
        tanh & 0.6754 \\
        \bottomrule
    \end{tabular}
    }
    \adjustbox{max height=\tabheight}{%
    \begin{tabular}{ll}
        \toprule
        \multicolumn{2}{c}{MobileViTv2-0.5 on ImageNet} \\
        \midrule
        $-x \cdot \sigma(x) \cdot \textrm{HardSigmoid}(x)$ & \textbf{0.6396} \\
        $\textrm{ELU}(\textrm{Swish}(-x))$ & 0.6394 \\
        $\textrm{Swish}(x) \cdot \textrm{erfc}(\textrm{bessel\_i0e}(x))$ & 0.6336 \\
        \midrule 
        ELU & 0.6233 \\
        ReLU & 0.6139 \\
        SELU & 0.6096 \\
        sigmoid & 0.5032 \\
        Softplus & 0.5853 \\
        Softsign & 0.5710 \\
        Swish & 0.6383 \\
        tanh & 0.6098 \\
        \bottomrule
    \end{tabular}
    }
    \label{tab:aquasurf:search_results}
\end{table}

\begin{table}
    \centering
    \caption{ResNet-50 top-1 accuracy on ImageNet.  Results are the median of three runs.  The best activation functions discovered in the searches (Table \ref{tab:aquasurf:search_results}) successfully transfer to this new task, with eight of the nine functions outperforming ReLU.\\}
    \begin{tabular}{ll}
        \toprule
        $-x \cdot \sigma(x) \cdot \textrm{HardSigmoid}(x)$ & $\bf 0.7776$ \\
        $\textrm{Swish}(x) / \textrm{SELU}(1)$ & $0.7771$ \\
        $\textrm{Swish}(x) \cdot \textrm{erfc}(\textrm{bessel\_i0e}(x))$ & $0.7755$ \\
        $\sigma(\textrm{Softsign}(x)) \cdot \textrm{ELU}(x)$ & $0.7734$ \\
        $\textrm{SELU}(\sinh(e^{\arctan(x)}-1))$ & $0.7719$ \\ 
        $\textrm{HardSigmoid}(\textrm{HardSigmoid}(x)) \cdot \textrm{ELU}(x)$ & $0.7718$ \\
        $\textrm{ELU}(\textrm{Swish}(-x))$ & $0.7679$ \\
        $\textrm{Swish}(-2x)$ & $0.7664$ \\
        $x \cdot \textrm{erfc}(\textrm{ELU}(x))$ & $0.7635$ \\
        \midrule
        $\textrm{ReLU}(x)$ & $0.7660$ \\
        \bottomrule
    \end{tabular}
    \label{tab:aquasurf:transfer}
\end{table}

Figure \ref{fig:aquasurf:function_plots} illustrates the different activation functions discovered during the searches.  Visually, the best functions (shown in \ref{fig:aquasurf:function_plots}a) are similar to existing functions like ELU and Swish, with subtle changes in their saturation value, the slope of the positive segment, and the width and depth of the negative bump. This result is not surprising since these functions formed the starting point for the search process.  Indeed, after a few good functions were discovered, much of the search process focused on refining their design (Figure \ref{fig:aquasurf:large_search_space}). Although these refinements appear small, they were not known ahead of time and they are significant, as evidenced by the final results (Table \ref{tab:aquasurf:search_results}).  

\begin{figure}[ht]
    \centering
    \begin{tikzpicture}
    \draw (0, 0) node[anchor=south west, inner sep=0, align=left] {\includegraphics[width=0.49\linewidth]{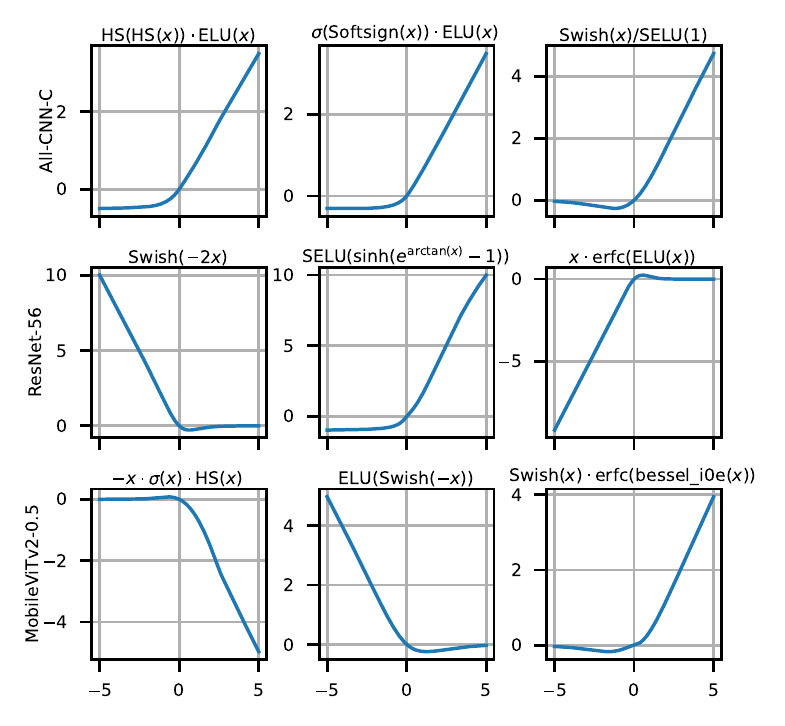}
    \includegraphics[width=0.49\linewidth]{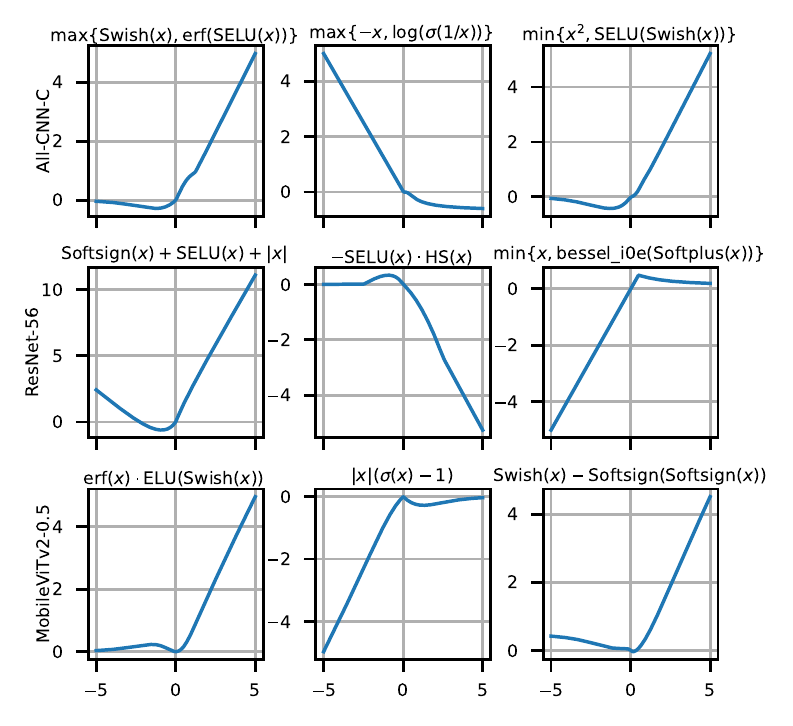}};
    \draw (0.25\linewidth, 0) node[anchor=north, inner sep=0, align=center] {\scriptsize (a)};
    \draw (0.75\linewidth, 0) node[anchor=north, inner sep=0, align=center] {\scriptsize (b)};
    \end{tikzpicture}
    \caption{Sample activation functions discovered with \technique.  ``HS'' stands for HardSigmoid.  (a) The top three functions discovered in each search.  Many of these functions are fine-tuned versions of existing activation functions like ELU and Swish.  (b) Selected novel activation functions.  All of these functions outperformed ReLU and are distinct from existing activation functions.  Such designs may serve as a foundation for further improvement and specialization in new tasks.}
    \label{fig:aquasurf:function_plots}
\end{figure}

With \technique, several interesting activation functions were discovered (Figure~\ref{fig:aquasurf:function_plots}b). While they were not the best, all of them outperformed ReLU.  These functions have properties uncommon among the usual deep learning activation functions:  Many of them have discontinuous derivatives at $x=0$; some do not saturate, but diverge as $x \rightarrow \pm \infty$; in contrast to Swish, which features a negative bump, many of these functions contain positive bumps.  In the future, these designs may be refined further, and perhaps produce better activation functions for specific new tasks. 

Together, the plots show that \technique is capable of both exploitation (Figure \ref{fig:aquasurf:function_plots}a) and exploration (Figure \ref{fig:aquasurf:function_plots}b).  In the future, it will be interesting to consider tradeoffs between these concepts. A more comprehensive discussion of this and other future research directions is in Section~\ref{sec:aquasurf:future_work}.

\section{Discussion}
\label{sec:aquasurf:future_work}

This chapter demonstrated that FIM eigenvalues and activation function outputs are efficient and reliable features that can predict performance of activation functions accurately.  This finding enabled discovering better activation functions for various tasks, improving the state of the art in machine learning.  Because the technique is efficient, it was possible to scale it up to large datasets such as ImageNet.  These discoveries inspire several avenues for future research, discussed below.

\paragraph{New Search Spaces} The PANGAEA search space was used in this chapter because it is known to work well for deep architectures \cite{bingham2022discovering}.  In the future it will be interesting to explore search spaces with different unary, binary, and $n$-ary operators.  Beyond computation graphs, it may also be possible to apply techniques in this chapter to optimize continuous vector representations of activation functions \cite{apl-agostinelli2014learning, pade-molina2019pad}.

\paragraph{Exploration vs. Exploitation}
The KNR approach was utilized to search for new activation functions because it performed well on the benchmark datasets (Section \ref{sec:aquasurf:benchmark_search}).  In the future, it will be interesting to consider other algorithms and analyze their tradeoffs between exploration and exploitation.  For example, in a resource-constrained environment where improvement is needed quickly, a more exploitative approach could be used to find an improved activation function in a short time.  On the other hand, if substantial compute is available, an approach that focuses on exploration could be used to discover activation functions that perform well but are maximally different from functions used in modern architectures (Figure \ref{fig:aquasurf:function_plots}b).  Novelty search \cite{lehman2011abandoning} could serve as a suitable approach, and such discoveries could further understanding of how neural networks utilize different kinds of activation functions to learn.

\paragraph{Optimizing Multiple Activation Functions}
In a typical neural network design, the same activation function is used throughout the network.  However, recent work has shown that it may be beneficial to have different activation functions at different locations, and further, that it may be useful to have different activation functions in the early and late stages of training \cite{bingham2022discovering}.  Indeed, many hybrid architectures use Swish in convolutional layers and ReLU in attention layers \cite{mehta2022separable}.  Unfortunately, it is difficult to design these strategies manually, and so practitioners often use a single activation function for simplicity.

The techniques proposed in this chapter may provide an avenue toward optimizing multiple activation functions in tandem.  For example, the features for multiple candidate activation functions could be concatenated into a single feature vector, and this vector could be projected with UMAP to a low-dimensional space where performance prediction is more straightforward. 

\paragraph{Optimizing Other Aspects of Neural Network Design} 
By fixing the neural network architecture and varying the activation function, this chapter showed that it is possible to use FIM eigenvalues to infer future performance.  As the FIM is a fundamental quantity in machine learning, it may be possible to apply a similar strategy to optimize other aspects of neural network design, such as normalization layers, loss functions, or data augmentation strategies \cite{liu2020evolving, gonzalez2020improved, gonzalez2020evolving, cubuk2018autoaugment}.  If a meaningful distance metric between such objects can be defined, then UMAP could be used to map them to a low-dimensional space where performance prediction is much simpler.

\paragraph{Reverse Engineering Activation Functions}
UMAP was used to project activation functions to a low-dimensional space, and regression algorithms to predict the performance of activation functions in this space, i.e.\ to serve as a fitness function for the search. However, it is possible that there is no activation function that maps to the optimum of this fitness landscape.  Indeed, because such search spaces are finite, the activation functions do not completely fill them.  For example, there are empty regions in Figure \ref{fig:aquasurf:umap_embeddings}, corresponding to activation functions outside of the predefined search space.

What should be done if an empty region of the embedding space has a higher predicted fitness than any of the candidate activation functions?  In the chapter, these regions were simply ignored, and the activation function with the highest predicted fitness was used.  However, in the future, it may be possible to create activation functions that map to these empty spaces, an in so doing improve performance.  One approach could be based on inverse transforms: Given a coordinate in the low-dimensional embedding space, UMAP can apply an inverse transform and return an object that would have mapped to those coordinates.  This technique was already used for visualization in Figure \ref{fig:aquasurf:interpolation}.  Using this approach, UMAP could generate a hypothetical desired FIM eigenvalue distribution, or a list of activation function outputs.  

There are two challenges to this approach.  First, because UMAP is a dimensionality-reduction algorithm, different activation functions can map to the same location in the embedding space.  Thus, the mapping from embedding space back to activation functions is not well defined.  Second, even if UMAP prescribes a FIM eigenvalue distribution that is predicted to result in good performance, it may be difficult to manually design an activation function to satisfy that distribution.

However, a generated list of prescribed activation function outputs is already a good start.  From this list, it is possible to construct an activation function that interpolates through these points, either in a piecewise linear fashion, with splines, or using some other standard technique.  Even without the corresponding FIM eigenvalues, such an approach could potentially improve the efficiency of novel activation function discovery, and lead to better designs for activation functions in the future.

\section{Conclusion}
\label{sec:aquasurf:conclusion}

This chapter introduced three benchmark datasets, \texttt{Act-Bench-CNN}, \texttt{Act-Bench-\allowbreak ResNet}, and \texttt{Act-Bench-ViT}, to support research on activation function optimization.  Experiments with these datasets showed that FIM eigenvalues and activation function outputs, and their low-dimensional UMAP embeddings, predict activation function performance accurately, and can thus be used as a surrogate for finding better functions, even with out-of-the-box regression algorithms.  These conclusions extended from the benchmark datasets to new and challenging real-world tasks, where better activation functions were discovered on CIFAR-100 and ImageNet.  The study reinforces the idea that activation function design is an important part of deep learning, and shows \technique is an efficient and flexible mechanism for doing it.

\chapter{Discussion and Future Work}
\label{chap:discussion}

This chapter reviews the main ideas that were introduced in this dissertation, and includes directions for possible future work.

\section{Optimizing Other Aspects of Neural Network Design}

While much of the literature in AutoML has focused on hyperparameter optimization and neural architecture search (Chapter \ref{chap:background}), and this dissertation focused on activation functions (Chapters \ref{chap:gecco}, \ref{chap:pangaea}, \ref{chap:aquasurf}) and weight initialization (Chapter \ref{chap:autoinit}), in the future it will be important to similarly optimize other aspects of neural network design.  This dissertation showed that automated methods can be more effective and more creative than humans, and it is likely that these benefits extend to other areas as well.  Indeed, existing work has shown that automating the design of loss functions \cite{gonzalez2020improved, gonzalez2020evolving}, learning rate schedules \cite{carvalho2020autolr}, data augmentation strategies \cite{cubuk2018autoaugment, cubuk2020randaugment}, and optimizers \cite{chen2023symbolic} is promising and worthy of future research.

\section{General and Specialized Solutions}

CAFE and PANGAEA showed that general activation functions that perform well across architectures and tasks exist, but that specialized activation functions designed for specific architectures and tasks give the best performance.  To date, this idea of specialized solutions has been overlooked in AutoML, and it will be worth pursuing in the future.  It is likely that similar benefits can be found by discovering specialized versions of other objects like loss functions or data augmentation strategies.  Discovering specialized solutions is often more computationally expensive than finding general solutions, so it will also be important to focus on efficient search algorithms.  For example, AQuaSurF was much more efficient than CAFE and PANGAEA, and made specialized activation function search practical for large tasks like ImageNet.

\section{Joint Optimization of Multiple Components}

This dissertation automated the design of activation functions and weight initialization strategies, and thus made it possible to understand the effects of optimizing multiple components of neural network design simultaneously.  It turns out that doing so is beneficial: Optimizing both the activation function and weight initialization was synergistic, resulting in better performance than optimizing either component alone.  It is likely joint optimization of other aspects of neural networks will lead to similar benefits.  It is also possible that this practice will lead to new and surprising design choices.  For example, simply changing the activation function can inspire designs for a different weight initialization \cite{he2015delving}, dropout strategy \cite{selu}, normalization implementation \cite{liu2020evolving}, and topology \cite{tan2019efficientnet}.  All of the components in a machine learning system interact in complex ways, and the future challenge will be to exploit these interactions in order to find the combinations that lead to the best performance.

\section{Computational Cost Considerations}

AutoML approaches tend to be computationally expensive because evaluating new designs often requires training neural networks from scratch.  These costs will need to be addressed in the future, especially if multiple aspects of neural networks are to be optimized at the same time.  In this dissertation, AQuaSurF utilized a surrogate model to improve efficiency over CAFE and PANGAEA by orders of magnitude.  AutoInit, on the other hand, identified the principle of stable signal propagation, and used this principle to define a mapping from architectures to initialization strategies. AutoInit was thus efficient by design.  By evaluating different neural network designs fairly, AutoInit also accelerated neural architecture search and activation function discovery: processes that are often expensive.  In the future, it will be important to focus on efficient search algorithms, low-cost evaluation proxies, surrogate models, and general principles in order to ensure that AutoML is practical to implement.

\section{Optimizing Alternative Objectives}

The techniques in this dissertation used accuracy as the primary evaluation metric, but in some scenarios other objectives may be more appropriate.  Possibilities include maximizing adversarial robustness, minimizing inference time, customizing designs to run well on specific hardware, or even a multi-objective combination of these.  Future approaches could optimize any of these objectives, and could conceivably present the user with a Pareto front of options that comprise tradeoffs between multiple such objectives.

\section{Designing Better Representations}

Instead of using hand-coded activation function representations like CAFE and PANGAEA, AQuaSurF learned activation function representations in a data-driven way.  These representations turned out to be much more informative, and allowed AQuaSurF to discover good activation functions efficiently.  This result shows that representations matter in AutoML, and also highlights that a human-interpretable representation may not be the easiest to optimize.  It is likely that better encodings can similarly accelerate other AutoML algorithms.  Data-driven methods like AQuaSurF are promising ways to find these representations.  Although constructing benchmark datasets is initially expensive, it is a cost that is only spent once. It may then result in dramatically more efficient search algorithms that practitioners can implement more easily.

\section{Towards Artificial General Intelligence}

As it stands, AutoML approaches are not capable of recursive self-improvement, but rather execute a single step of improvement.  Part of the reason for this limitation is that AutoML algorithms are typically separate from the machine learning system they are optimizing.  Thus, they can only optimize the machine learning system to a limited extent before performance plateaus.  In order to enable recursive self-improvement in the future, it will be important to take the lessons learned from AutoML and implement them in a system that is self-aware.  Large language models (LLMs) may potentially provide a context where such an implementation is possible.  There already exists a vast amount of freely available information about machine learning techniques on websites like GitHub, Stack Overflow, and arXiv.  An LLM trained on these sources of data and aware of its own code could conceivably synthesize the current state of the art in machine learning and use this knowledge to improve upon itself repeatedly, thereby achieving recursive self-improvement.  Such an ability would be a major step in automated machine learning, and thus also towards artificial general intelligence.

\chapter{Conclusion}
\label{chap:conclusion}

This chapter summarizes the main contributions of this dissertation, and concludes by reviewing the progress that this dissertation has made towards fully automated machine learning.

\section{Contributions}

Chapter \ref{chap:gecco} introduced CAFE, which demonstrated that neural network performance can be improved by evolving the design of the activation function.  The chapter utilized exhaustive search in a small search space ($S_1$) and random search and evolution in a larger search space ($S_2$). It introduced novel activation functions that achieve high accuracy, outperforming both standard functions such as ReLU and novel functions such as Swish. The best activation functions successfully transferred from CIFAR-10 to CIFAR-100 and from WRN-28-10 to WRN-40-4. However, the best results were obtained by specializing functions for each architecture and dataset.

Chapter \ref{chap:pangaea} presented PANGAEA, a system that extended the abilities of CAFE by being more flexible in multiple ways.  The search space was extended to include deeper and more complex functional forms, including ones unlikely to be discovered by humans.  Instead of fixed activation functions, parametric activation functions were evolved.  This construction allowed the activation functions to take on different shapes at different locations in the network and in different stages of training.  PANGAEA used a synergy of two different optimization processes: evolutionary population-based search for the general form and gradient descent-based fine-tuning of the parameters of the activation function.  PANGAEA discovered general activation functions that perform well across architectures as well as specialized functions that take advantage of a particular architecture, significantly outperforming previously proposed activation functions in both cases.

Chapter \ref{chap:autoinit} described AutoInit, an algorithm that calculates analytic mean- and variance-preserving weight initialization for neural networks automatically.  In convolutional networks, such an initialization improved performance with different activation functions, dropout rates, learning rates, and weight decay settings.  In residual networks, AutoInit prevented exploding signals, allowed training with higher learning rates, and improved performance with or without batch normalization.  In transformers, AutoInit was scaled up to high-resolution image classification, where it improved performance with several activation functions with and without normalization.  AutoInit also improved accuracy on the ImageNet dataset.  In neural architecture search, new architectures were evaluated more accurately, resulting in better networks in vision, language, tabular, multi-task, and transfer learning settings.  In activation function discovery, AutoInit stabilized training and improved accuracy with a large diversity of novel activation functions.  Thus, AutoInit constitutes an important contribution on its own, but it also demonstrates the power of optimizing multiple aspects of neural network design simultaneously. 

Chapter \ref{chap:aquasurf} introduced \technique, as well as three benchmark datasets: \texttt{Act-Bench-CNN}, \texttt{Act-Bench-ResNet}, and \texttt{Act-Bench-ViT}.  Experiments with these datasets showed that FIM eigenvalues and activation function outputs, and their low-dimensional UMAP embeddings, predict activation function performance accurately, and can thus be used as a surrogate for finding better functions, even with out-of-the-box regression algorithms.  These conclusions extended from the benchmark datasets to new and challenging real-world tasks, where better activation functions were discovered on CIFAR-100 and ImageNet.  \technique demonstrated that learning better representations for activation functions allowed them to be optimized more effectively.  The technique reinforces the importance of activation function design, and makes automating such designs practical for machine learning practitioners to implement.

\section{Big Picture} 

The goal of AutoML is to replace biased and ad hoc human designs with ones that are creative, principled, more effective, and automatically generated.  This dissertation advanced the state of the art in AutoML with automated approaches for activation function design and weight initialization.  This dissertation also showed that specialized solutions give better performance than general ones, and demonstrated the power of optimizing multiple aspects of neural network design jointly.  These findings will be needed in order to realize the potential of fully automated machine learning, and to achieve AGI.

%
%
\appendices
\index{Appendices@\emph{Appendices}}%

\chapter{Training Details}
\label{ap:training_details}

This appendix contains training details for the experiments in the main text.

\section{CAFE}
\label{ap:details:gecco}

This section describes implementation details for the experiments in Chapter \ref{chap:gecco}.

\subsection{Training Setup}

A wide residual network \cite{zagoruyko2016wide} with depth 28 and widening factor 10 (WRN-28-10), implemented in TensorFlow \cite{abadi2016tensorflow}, was trained on the CIFAR-10 and CIFAR-100 image datasets \cite{krizhevsky2009learning}.  The architecture was comprised of repeated residual blocks that apply batch normalization and ReLU prior to each convolution.  In the experiments, all ReLU activations were replaced with a candidate activation function.  No other changes to the architecture were made. Hyperparameters were chosen to mirror those of \citet{zagoruyko2016wide} as closely as possible.  Featurewise center, horizontal flip, and ZCA whitening preprocessing were applied to the datasets.  Dropout probability was 0.3, and the architecture was optimized using stochastic gradient descent with Nesterov momentum 0.9.  WRN-40-4 (a deeper and thinner wide residual network architecture) was also used in some experiments for comparison.

The CIFAR-10 and CIFAR-100 datasets both have 50K training images, 10K testing images, and no standard validation set.  To prevent overfitting, balanced validation splits were created for both datasets by randomly selecting 500 images per class from the CIFAR-10 training set and 50 images per class from the CIFAR-100 training set.  The test set was not modified so that the results can be compared with other work.

\subsection{Search Implementation}

To discover activation functions, a number of search strategies were used.  Regardless of the strategy, the training set always consisted of 45K images while the validation set contained 5K images; the test set was never used during the search.  All ReLU activations in WRN-28-10 were replaced with a candidate activation function and the architecture was trained for 50 epochs.  The initial learning rate was set to 0.1, and decreased by a factor of 0.2 after epochs 25, 40, and 45.  Training for only 50 epochs made it possible to evaluate many activation functions without excessive computational cost.

The top three activation functions by validation accuracy from the entire search were returned as a result.  For each of these functions, a WRN-28-10 was trained from scratch for 200 epochs.  The initial learning rate was set to 0.1, and decreased by a factor of 0.2 after epochs 60, 120, and 160, mirroring the work by Zagoruyko and Komodakis \cite{zagoruyko2016wide}.  After training was complete, the test set accuracy was measured.  The median test accuracy of five runs was reported as the final result, as is commonly done in similar work in the literature \cite{DBLP:conf/iclr/RamachandranZL18}.

\section{PANGAEA}
\label{ap:details:pangaea}

This section describes the implementation details for the experiments in Chapter \ref{chap:pangaea}.

\subsection{Training Details}

\paragraph{Wide Residual Network (WRN-10-4)}

When measuring final performance after evolution, the standard WRN setup was used; all ReLU activations in WRN-10-4 were replaced with the evolved activation function, but no other changes to the architecture were made.  The network was optimized using stochastic gradient descent with Nesterov momentum 0.9.  The network was trained for 200 epochs; the initial learning rate was 0.1, and it was decreased by a factor of 0.2 after epochs 60, 120, and 160.  Dropout probability was set to 0.3, and L2 regularization of 0.0005 was applied to the weights.  Data augmentation included featurewise center, featurewise standard deviation normalization, horizontal flip, and random $32 \times 32$ crops of images padded with four pixels on all sides.  This setup was chosen to mirror the original WRN setup \citep{zagoruyko2016wide} as closely as possible.

During evolution of activation functions, the training was compressed to save time. The network was trained for only 100 epochs; the learning rate began at 0.1 and was decreased by a factor of 0.2 after epochs 30, 60, and 80.  Empirically, the accuracy achieved by this shorter schedule was sufficient to guide evolution; the computational cost saved by halving the time required to evaluate an activation function could then be used to search for additional activation functions.

\paragraph{Residual Network (ResNet-v1-56)}

As with WRN-10-4, when measuring final performance with ResNet-v1-56, the only change to the architecture was replacing the ReLU activations with an evolved activation function.  The network was optimized with stochastic gradient descent and momentum 0.9.  Dropout was not used, and L2 regularization of 0.0001 was applied to the weights.  In the original ResNet experiments \citep{he2016deep}, an initial learning rate of 0.01 was used for 400 iterations before increasing it to 0.1, and further decreasing it by a factor of 0.1 after 32K and 48K iterations.  An iteration represents a single forward and backward pass over one training batch, while an epoch consists of training over the entire training dataset.  In these experiments, the learning rate schedule was implemented by beginning with a learning rate of 0.01 for one epoch, increasing it to 0.1, and then decreasing it by a factor of 0.1 after epochs 91 and 137.  (For example, (48K iterations / 45K training images) * batch size of 128 $\approx$ 137.)  The network was trained for 200 epochs in total.  Data augmentation included a random horizontal flip and random $32 \times 32$ crops of images padded with four pixels on all sides, as in the original setup \citep{he2016deep}.

When evolving activation functions for ResNet-v1-56, the learning rate schedule was again compressed.  The network was trained for 100 epochs; the initial warmup learning rate of 0.01 still lasted one epoch, the learning rate increased to 0.1, and then decreased by a factor of 0.1 after epochs 46 and 68.  When evolving activation functions, their relative performance is more important than the absolute accuracies they achieve.  The shorter training schedule was therefore a cost-efficient way of discovering high-performing activation functions.

\paragraph{Preactivation Residual Network (ResNet-v2-56)}

The full training setup, data augmentation, and compressed learning rate schedule used during evolution for ResNet-v2-56 were all identical to those for ResNet-v1-56 with one exception: with ResNet-v2-56, it is not necessary to warm up training with an initial learning rate of 0.01 \citep{he2016identity}, so this step was skipped.

\paragraph{All-CNN-C} When measuring final performance with All-CNN-C, the ReLU activation function was replaced with an evolved one, but the setup otherwise mirrored that of \citet{springenberg2015striving} as closely as possible.  The network was optimized with stochastic gradient descent and momentum 0.9.  Dropout probability was 0.5, and L2 regularization of 0.001 was applied to the weights.  The data augmentation involved featurewise centering and normalizing, random horizontal flips, and random $32 \times 32$ crops of images padded with five pixels on all sides.  The initial learning rate was set to 0.01, and it was decreased by a factor of 0.1 after epochs 200, 250, and 300.  The network was trained for 350 epochs in total.

During evolution of activation functions, the same training setup was used.  It is not necessary to compress the learning rate schedule as was done with the residual networks because All-CNN-C trains more quickly.

\paragraph{CIFAR-10} As with CIFAR-100, a balanced validation set was created for CIFAR-10 by randomly selecting 500 images from each class, resulting in a training/validation/test split of 45K/5K/10K images.

\subsection{Implementing Custom Activation Functions}
\label{ap:details:pangaea_custom}

This section demonstrates how to implement different activation functions in a TensorFlow neural network.  For example, the code to create the All-CNN-C architecture with a custom activation function is:

{\scriptsize
\begin{verbatim}
def all_cnn_c(args):
    inputs = Input((32, 32, 3))
    x = Conv2D(96, kernel_size=3, strides=(1, 1), padding='same', kernel_regularizer=l2(0.001))(inputs)
    x = CustomActivation(args)(x)
    x = Conv2D(96, kernel_size=3, strides=(1, 1), padding='same', kernel_regularizer=l2(0.001))(x)
    x = CustomActivation(args)(x)
    x = Conv2D(96, kernel_size=3, strides=(2, 2), padding='same', kernel_regularizer=l2(0.001))(x)
    x = Dropout(0.5)(x)

    x = Conv2D(192, kernel_size=3, strides=(1, 1), padding='same', kernel_regularizer=l2(0.001))(x)
    x = CustomActivation(args)(x)
    x = Conv2D(192, kernel_size=3, strides=(1, 1), padding='same', kernel_regularizer=l2(0.001))(x)
    x = CustomActivation(args)(x)
    x = Conv2D(192, kernel_size=3, strides=(2, 2), padding='same', kernel_regularizer=l2(0.001))(x)
    x = CustomActivation(args)(x)
    x = Dropout(0.5)(x)

    x = Conv2D(192, kernel_size=3, strides=(1, 1), padding='same', kernel_regularizer=l2(0.001))(x)
    x = CustomActivation(args)(x)
    x = Conv2D(192, kernel_size=1, strides=(1, 1), padding='valid', kernel_regularizer=l2(0.001))(x)
    x = CustomActivation(args)(x)
    x = Conv2D(10, kernel_size=1, strides=(1, 1), padding='valid', kernel_regularizer=l2(0.001))(x)
    x = CustomActivation(args)(x)

    x = GlobalAveragePooling2D()(x)
    x = Flatten()(x)
    outputs = Activation('softmax')(x)

    return Model(inputs=inputs, outputs=outputs)
\end{verbatim}}

The \texttt{CustomActivation} is a wrapper that resolves to different activation functions depending on the \texttt{args} parameter.  After importing \texttt{from tensorflow.keras.layers import Activation}, built-in activation functions can be implemented as \texttt{Activation('relu')} or \texttt{Activation('tanh')}, for example.  Activation functions that are not built in and do not contain learnable parameters can be implemented with lambda functions.  For example, the Mish activation function can be implemented with \texttt{Activation(lambda x : x * tf.math.tanh(tf.keras.activations.softplus(x)))}.  Finally, activation functions with learnable parameters simply require subclassing a \texttt{Layer} object.  The code used to implement the PAU activation function is below; APL and SPLASH are implemented in a similar manner.

{\tiny
\begin{verbatim}
"""
Padé Activation Units: End-to-end Learning of Flexible Activation Functions in Deep Networks
https://arxiv.org/abs/1907.06732

PAU of degree (5, 4) initialized to approximate Leaky ReLU (0.01)
"""
import tensorflow as tf

from tensorflow.keras.layers import Layer
from tensorflow.keras.initializers import Constant

class PAU(Layer):
    def __init__(self, num_init=None, denom_init=None, param_shape='per-layer', **kwargs):
        super(PAU, self).__init__(**kwargs)
        self.num_init = num_init if num_init else [0.02979246, 0.61837738, 2.32335207, 3.05202660, 1.48548002, 0.25103717]
        self.denom_init = denom_init if denom_init else [1.14201226, 4.39322834, 0.87154450, 0.34720652]
        self.num_weights = []
        self.denom_weights = []
        self.param_shape = param_shape

    def build(self, input_shape):
        if self.param_shape == 'per-layer':
            param_shape = (1,)
        elif self.param_shape == 'per-channel':
            param_shape = list(input_shape[-1:])
        else:
            assert self.param_shape == 'per-neuron'
            param_shape = list(input_shape[1:])

        for i in range(6):
            self.num_weights.append(
                self.add_weight(
                    name=f'a{i}',
                    shape=param_shape,
                    initializer=Constant(self.num_init[i]),
                    trainable=True))
        for i in range(4):
            self.denom_weights.append(
                self.add_weight(
                    name=f'b{i+1}',
                    shape=param_shape,
                    initializer=Constant(self.denom_init[i]),
                    trainable=True))

    def call(self, inputs):
        num = tf.add_n([self.num_weights[i] * tf.math.pow(inputs, i) for i in range(6)])
        denom = 1 + tf.math.abs(tf.add_n([self.denom_weights[i] * tf.math.pow(inputs, i+1) for i in range(4)]))
        return num / denom

    def get_config(self):
        config = super(PAU, self).get_config()
        config.update({'num_init'    : self.num_init,
                       'denom_init'  : self.denom_init,
                       'param_shape' : self.param_shape})
        return config

\end{verbatim}
}

\section{AutoInit}
\label{ap:details:autoinit}

This section describes the details of experiments from Chapter \ref{chap:autoinit}.

\subsection{Convolutional, Residual, and Transformer Network Experiment Details}

This section contains implementation details for the experiments in Sections \ref{sec:autoinit:convolutional}-\ref{sec:autoinit:autoinit_vs_lsuv}.

\paragraph{All-CNN-C} The training setup follows that of \citet{springenberg2015striving} as closely as possible.  The network was trained with SGD and momentum 0.9.  The dropout rate was 0.5 and weight decay as L2 regularization was 0.001.  The data augmentation involved featurewise centering and normalizing, random horizontal flips, and random $32 \times 32$ crops of images padded with five pixels on all sides.  The initial learning rate was 0.01 and was decreased by a factor of 0.1 after epochs 200, 250, and 300 until training ends at epoch 350.

Because \citet{springenberg2015striving} did not specify how they initialized their weights, the networks were initialized with the ``Glorot Uniform'' strategy \citep[also called Xavier initialization; ][]{glorot2010understanding}, where weights were sampled from $\mathcal{U}\left(-\frac{\sqrt{6}}{\sqrt{\texttt{fan\_in} + \texttt{fan\_out}}}, \frac{\sqrt{6}}{\sqrt{\texttt{fan\_in} + \texttt{fan\_out}}}\right)$. This initialization is the default setting in TensorFlow,\footnote{\url{https://github.com/tensorflow/tensorflow/blob/v2.5.0/tensorflow/python/keras/layers/convolutional.py\#L608-L609}} and was sufficient to replicate the results reported by \citet{springenberg2015striving}. 

\paragraph{Residual Networks}
The networks were optimized with SGD and momentum 0.9.  Dropout was not used, and weight decay was 0.0001.  Data augmentation included a random horizontal flip and random $32 \times 32$ crops of images padded with four pixels on all sides.

\paragraph{CoAtNet}

A smaller variant of the CoAtNet architecture\footnote{\url{https://github.com/leondgarse/keras_cv_attention_models/blob/v1.3.0/keras_cv_attention_models/coatnet/coatnet.py\#L199}} was used in order to fit the model and data on the available GPU memory.  The architecture has three convolutional blocks with 64 channels, four convolutional blocks with 128 channels, six transformer blocks with 256 channels, and three transformer blocks with 512 channels.  This architecture is slightly deeper but thinner than the original CoAtNet-0 architecture, which has two convolutional blocks with 96 channels, three convolutional blocks with 192 channels, five transformer blocks with 384 channels, and two transformer blocks with 768 channels \cite{dai2021coatnet}.  The models are otherwise identical.

The training hyperparameters were inspired by \citet{wightman2021resnet} and are common in the literature.  Specifically, images were resized to $160 \times 160$.  The learning rate schedule was increased linearly from $1e^{-4}$ to $4e^{-4}$ for six epochs and it then followed a cosine decay until epoch 105.  Weight decay was set to 0.02 times the current learning rate at each epoch.  The model was trained with batch size 256 and optimized with AdamW \cite{loshchilov2017decoupled}.  Data augmentation included RandAugment applied twice with a magnitude of six \cite{cubuk2020randaugment}.  Mixup and Cutmix were also used with alpha 0.1 and 1.0, respectively \cite{zhang2017mixup, yun2019cutmix}.  The training images were augmented with random resized crops \cite{szegedy2015going} that were at minimum 8\% of the original image; after training the model was evaluated on 95\% center crops.

As discussed in Section \ref{sec:autoinit:mean_variance_estimation}, AutoInit maintains signal variance $\nu=1$, but it is also possible to adjust $\nu$ if desired.  In the CoAtNet experiments, $\nu = 0.01$ was found to give the best performance among $\nu = \{1, 0.1, 0.01, 0.001\}$.  The experiment removing normalization layers used the default of $\nu = 1$.

\paragraph{ImageNet} The experiment in Section \ref{sec:autoinit:imagenet} used the same training setup as the experiments with CoAtNet on Imagenette in the previous paragraph except for two changes.  The batch size was 2{,}048 (512 per GPU across four GPUs), and the maximum learning rate was $3.2e^{-2}$.

\paragraph{Data-Dependent Initialization Comparison}
In the experiments in Section \ref{sec:autoinit:autoinit_vs_lsuv}, a learning rate schedule inspired by superconvergence \cite{smith2019super} was used to save time.  The learning rate increased linearly to 0.1 during the first five epochs, and then decreased linearly for 20 epochs, after which training ended.  The weight decay for All-CNN-C was also decreased by a factor of 10.  This modification is common when networks are trained with superconvergence \cite{smith2019super}.

\paragraph{Initialization Time}
It is important to note that AutoInit does not incur a significant overhead.  Each layer must be visited once to be initialized, so the complexity is $O(L)$ where $L$ is the number of layers.  For example: All-CNN-C, ResNet-56, and ResNet-164 took 1, 33, and 106 seconds to initialize.  The costs are hardware-dependent, but only spent once, and are small compared to the cost of training.

\subsection{Neural Architecture Search Experiment Details}
\label{ap:details:autoinit_enn_details}

This appendix contains implementation details for the experiments in Section \ref{sec:autoinit:nas}.  Table \ref{tab:autoinit:hyperparameters} contains the training hyperparameters, neural network layers, evolutionary hyperparameters, and mutation probabilities used in each of the five tasks.  The five tasks used with CoDeepNEAT are:

\nocite{kingma2014adam}

\begin{table*}
    \centering
    \caption{Configuration of neural architecture search experiments in the five tasks.  Entries with a (B) or (M) suffix apply to CoDeepNEAT blueprints or modules, respectively.  These values were found to work well in preliminary experiments.  When AutoInit is applied to an evolved network, it replaces the weight initialization method selected by evolution, but the setup otherwise remains unchanged.  The neural architecture search experiments were designed to show that AutoInit improves performance in a wide variety of settings, including those with different data modalities, network topologies, computational complexities, and hyperparameter configurations.\\}
    \begin{adjustbox}{max width=0.9\linewidth}
    \begin{tabular}{llllll}
        \toprule
         & \textbf{MNIST} & \textbf{Omniglot} & \textbf{PMLB Adult} & \textbf{Wikipedia Toxicity} & \textbf{Oxford 102 Flower} \\ \midrule 
        \multicolumn{6}{l}{\textbf{Training Hyperparameters}} \\
        Activation & \multicolumn{5}{c}{All domains: \{ReLU, Linear, ELU, SELU\}} \\
        Batch Size & 128 & 1000 iterations & 32 & 128 & 8 \\
        Dropout Rate & $[0.0, 0.7]$ & $[0.0, 0.7]$ & $[0.0, 0.9]$ & $[0.0, 0.5]$ & $[0.0, 0.9]$ \\
        Epochs & 5 & 3 & 25 & 3 & 30 \\
        Filters/Units & [16, 96] & [16, 96] & [64, 192] & [64, 192] & [64, 192] \\
        Kernel Reg. & L2: $[10^{-9}, 10^{-3}]$ & L2: $[10^{-9}, 10^{-3}]$ & \{L1, L2\}: $[10^{-9}, 10^{-3}]$ & \{L1, L2\}: $[10^{-9}, 10^{-3}]$ & \{L1, L2\}: $[10^{-7}, 10^{-3}]$ \\
        Kernel Size & \{1, 3\} & \{1, 3\} & N/A & \{1, 3, 5, 7\} & N/A \\
        Learning Rate & $[10^{-4}, 10^{-2}]$ & $[10^{-4}, 10^{-3}]$ & $[10^{-4}, 10^{-2}]$ & $[10^{-4}, 10^{-2}]$ & $[10^{-4}, 10^{-1}]$ \\
        Optimizer & Adam & Adam & Adam & Adam & SGD (Nesterov=0.9) \\
        Weight Init. & \multicolumn{5}{c}{All domains: \{Glorot Normal, Glorot Uniform, He Normal, He Uniform\} or \textbf{AutoInit}} \\ \midrule
        
        \multicolumn{6}{l}{\textbf{Neural Network Layers}} \\
        Add & \checkmark & & \checkmark & \checkmark \\
        Concatenate & & & \checkmark & \checkmark \\
        Conv1D & & & & \checkmark & \\
        Conv2D & \checkmark & \checkmark & & & \\
        Dense & & & \checkmark & & \checkmark \\
        Dropout & \checkmark & \checkmark & \checkmark & & \checkmark \\
        GRU & & & & \checkmark & \\
        LSTM & & & & \checkmark & \\
        MaxPooling1D & & & & \checkmark & \\
        MaxPooling2D & \checkmark & \checkmark \\
        SpatialDropout1D & & & & \checkmark & \\
        WeightedSum & & \checkmark \\ \midrule
        
        \multicolumn{6}{l}{\textbf{Evolutionary Hyperparameters}} \\
        Elitism (B) & 0.4 & 0.4 & 0.4 & 0.2 & 0.1 \\
        Elitism (M) & 0.4 & 0.4 & 0.4 & 0.2 & 0.4 \\ 
        Evaluations (B) & 4 & 4 & 4 & 4 & 1 \\
        Generations & 30 & 40 & 30 & 30 & 30 \\
        Population Size (B) & 22 & 22 & 22 & 22 & 20 \\
        Population Size (M) & 56 & 56 & 56 & 56 & 20 \\
        Preserved Networks & 12 & 12 & 12 & 12 & 1 \\ 
        Species (B) & 1 & 1 & 1 & 1 & 1 \\
        Species (M) & 4 & 4 & 4 & 4 & 2 \\ \midrule
        
        \multicolumn{6}{l}{\textbf{Mutation Probabilities}} \\
        Change Hyperparam. & 0.25 & 0.25 & 0.25 & 0.5 & 0.5 \\
        New Connection (B) & 0.12 & 0.12 & 0.12 & 0.2 & 0.12 \\
        New Connection (M) & 0.08 & 0.08 & 0.08 & 0.2 & 0.08 \\
        New Layer (M) & 0.08 & 0.08 & 0.08 & 0.2 & 0.08 \\
        New Node (B) & 0.16 & 0.16 & 0.16 & 0.2 & 0.16 \\

        \bottomrule
    \end{tabular}
    \end{adjustbox}

    \label{tab:autoinit:hyperparameters}
\end{table*}

\paragraph{Vision} The MNIST dataset contains 28x28 grayscale images of handwritten digits 0-9.  There are 60,000 training images (5,000 of which were used for validation) and 10,000 test images \cite{lecun2010mnist}.  MNIST was used under the Creative Commons Attribution-Share Alike 3.0 license.

\paragraph{Language} In the Wikipedia Toxicity dataset, the task is to classify English Wikipedia comments as toxic or healthy contributions \cite{wulczyn2017ex}.  The dataset contains 92,835, 31,227, and 30,953 comments in the training, validation, and test sets, respectively.  

\paragraph{Tabular} In the Adult dataset \cite{kohavi1996scaling} from the Penn Machine Learning Benchmarks repository \citep[PMLB;][]{Olson2017PMLB}) the task is to predict whether an individual makes over \$50K per year based on 14 features.  Out of 48,842 total instances, 20\% were randomly separated to create a test set.  The dataset was used under the MIT License.

\paragraph{Multi-Task} The Omniglot dataset contains handwritten characters in 50 different alphabets \cite{lake2015human}; classifying characters in each alphabet is a natural multi-task problem.  The characters are $105 \times 105$ grayscale images, and there are 20 instances of each character.  To save compute resources, 20 of the 50 alphabets were randomly selected for experiments.  A fixed training, validation, and testing split of 50\%, 20\%, and 30\% was used with each task.  The learning rate decayed as $\texttt{learning\_rate} = 0.1^{\texttt{epoch}/10}*\texttt{initial\_learning\_rate}$ during training.  The dataset was used under the MIT License.

\paragraph{Transfer Learning} A DenseNet-121 network was first pretrained on the ImageNet dataset \cite{deng2009imagenet, huang2017densely}. Models were then evolved to utilize its embeddings to classify images in the Oxford 102 Flower dataset, consisting of 102 types of flowers found in the United Kingdom \cite{Nilsback08}.  Each class has between 40 and 258 images; the training and validation sets have 10 images per class, and the test set contains the remaining images from the dataset.  During training, the weight decay (L2 loss) was scaled by the current learning rate.  Images were also augmented to improve generalization performance.  Images were randomly flipped horizontally, rotated up to 40 degrees, shifted up/down and left/right up to 20\%, and shear intensity and zoom range varied up to 20\%.

\section{AQuaSurF}
\label{ap:details:aquasurf}

This section specifies the details for the experiments in Chapter \ref{chap:aquasurf}.

\subsection{Implementation Details}

\paragraph{Training Details}
\newcommand\tabwidthmul{0.745}

\begin{table}
    \centering
    \caption{Training details and hyperparameter values used in the experiments in Chapter \ref{chap:aquasurf}.\\}
    \adjustbox{max width=\tabwidthmul\linewidth}{%
    \begin{tabular}{ll}
        \toprule
        \multicolumn{2}{c}{All-CNN-C on CIFAR-10 and CIFAR-100} \\
        \midrule
        Batch Size & 128 \\
        Dropout & 0.5 \\
        Epochs & 25 for \texttt{Act-Bench-CNN} and search (Figure \ref{fig:aquasurf:search_progress}), 50 for full evaluation (Table \ref{tab:aquasurf:search_results}) \\
        Image Size & $32 \times 32$ \\
        Learning Rate & Linear warmup to 0.1 for five epochs, then linear decay \\
        Mean/Std. Normalization & Yes \\
        Momentum & 0.9 \\
        Optimizer & SGD \\
        Random Crops & $ 32 \times 32$ crops of images padded with four pixels on all sides \\
        Random Flips & Yes \\
        Weight Decay & $1e^{-4}$ \\
        Weight Initialization & AutoInit \cite{bingham2021autoinit} \\
        \bottomrule
    \end{tabular}
    }\\
    \vspace{0.5em}

    \adjustbox{max width=\tabwidthmul\linewidth}{%
    \begin{tabular}{ll}
        \toprule
        \multicolumn{2}{c}{ResNet-56 on CIFAR-10 and CIFAR-100} \\
        \midrule
        Batch Size & 128 \\
        Dropout & 0.0 \\
        Epochs & 25 for \texttt{Act-Bench-ResNet} and search (Figure \ref{fig:aquasurf:search_progress}), 50 for full evaluation (Table \ref{tab:aquasurf:search_results}) \\
        Image Size & $32 \times 32$ \\
        Learning Rate & Linear warmup to 0.1 for five epochs, then linear decay \\
        Mean/Std. Normalization & No \\
        Momentum & 0.9 \\
        Optimizer & SGD \\
        Random Crops & $ 32 \times 32$ crops of images padded with five pixels on all sides \\
        Random Flips & Yes \\
        Weight Decay & $1e^{-4}$ \\
        Weight Initialization & AutoInit \cite{bingham2021autoinit} \\
        \bottomrule
    \end{tabular}
    }\\
    \vspace{0.5em}

    \adjustbox{max width=\tabwidthmul\linewidth}{%
    \begin{tabular}{ll}
        \toprule
        \multicolumn{2}{c}{MobileViTv2-0.5 on Imagenette and ImageNet} \\
        \midrule
        Batch Size & 256 \\
        CutMix Alpha \cite{yun2019cutmix} & 1.0 \\
        Epochs & 105 \\
        Evaluation Center Crop & 95\% \\
        Image Size & $160 \times 160$ \\
        Learning Rate & Linear warmup from $1e^{-4}$ to $4e^{-3}$ for five epochs, then cosine decay to $1e^{-6}$\\
        Mixup Alpha \cite{zhang2017mixup} & 0.1 \\
        Optimizer & AdamW \cite{loshchilov2017decoupled} \\
        RandAugment \cite{cubuk2020randaugment} & Magnitude six, applied twice \\
        Random Resized Crop \cite{szegedy2015going} & Minimum 8\% of the original image \\
        Weight Decay & $0.02\times$ current learning rate \\
        \bottomrule
    \end{tabular}
    }\\
    \vspace{0.5em}

    \adjustbox{max width=\tabwidthmul\linewidth}{%
    \begin{tabular}{ll}
        \toprule
        \multicolumn{2}{c}{ResNet-50 on ImageNet} \\
        \midrule
        Batch Size & 256 \\
        CutMix Alpha \cite{yun2019cutmix} & 1.0 \\
        Epochs & 105 \\
        Evaluation Center Crop & 95\% \\
        Image Size & $160 \times 160$ \\
        Learning Rate & Linear warmup from $1e^{-4}$ to $2e^{-3}$ for five epochs, then cosine decay to $1e^{-6}$\\
        Mixup Alpha \cite{zhang2017mixup} & 0.1 \\
        Optimizer & AdamW \cite{loshchilov2017decoupled} \\
        RandAugment \cite{cubuk2020randaugment} & Magnitude six, applied twice \\
        Random Resized Crop \cite{szegedy2015going} & Minimum 8\% of the original image \\
        Weight Decay & $0.02\times$ current learning rate \\
        Weight Initialization & AutoInit \cite{bingham2021autoinit} \\
        \bottomrule
    \end{tabular}
    }

    \label{tab:aquasurf:training_details}
\end{table}

For CIFAR-10 and CIFAR-100, balanced validation sets were created by sampling 5{,}000 images from the training set.  Full training details and hyperparameters are listed in Table \ref{tab:aquasurf:training_details}.

\paragraph{Search Implementation}
In order to predict performance for an unevaluated activation function, the function outputs and FIM eigenvalues must first be computed.  Thus, the searches in Section \ref{sec:aquasurf:search_new_tasks} were implemented in three steps.  First, activation function outputs for all 425,896 activation functions in the search space were calculated.  This computation is inexpensive and easily parallelizable.  Second, eight workers operated in parallel to sample activation functions uniformly at random from the search space and calculate their FIM eigenvalues.  Third, once the number of activation functions with FIM eigenvalues calculated reached 5{,}000, seven of the workers began the search by evaluating the functions with the highest predicted performance.  The eighth worker continued calculating FIM eigenvalues for new functions so that their performance could be predicted during the search. This setup allowed taking best advantage of the available compute for the regression-type search methods.

The experiments on ImageNet required substantially more compute than the experiments on CIFAR-100.  For this reason, all eight workers evaluated activation functions once the number of functions with FIM eigenvalues reached 7{,}000.

Computing FIM eigenvalues took approximately 26 seconds, 84 seconds, and 37 seconds per activation function for All-CNN-C, ResNet-56, and MobileViTv2-0.5, respectively.  This cost is not trivial, but it is well worth it, as the experiments in Chapter \ref{chap:aquasurf} show.

\paragraph{Unique Activation Functions}
Different computation graphs can result in the same activation function (e.g.\ $\max\{x,0\}$ and $\max\{0,x\}$).  In the benchmark dataset and in the larger search space of Section \ref{sec:aquasurf:search_new_tasks}, repeated activation functions were filtered out.  A total of 1{,}000 inputs were sampled $\mathcal{N}(0,1)$ and truncated to $[-5,5]$.  Two activation functions were considered the same if their outputs were identical.

\subsection{Preliminary Experiments}

Several variations to the approach presented in the main text were also evaluated in preliminary experiments. The approach turned out to be robust to most of them, but the results also justify the choices used for the main experiments.

\begin{figure}
    \centering
    \includegraphics[width=\linewidth]{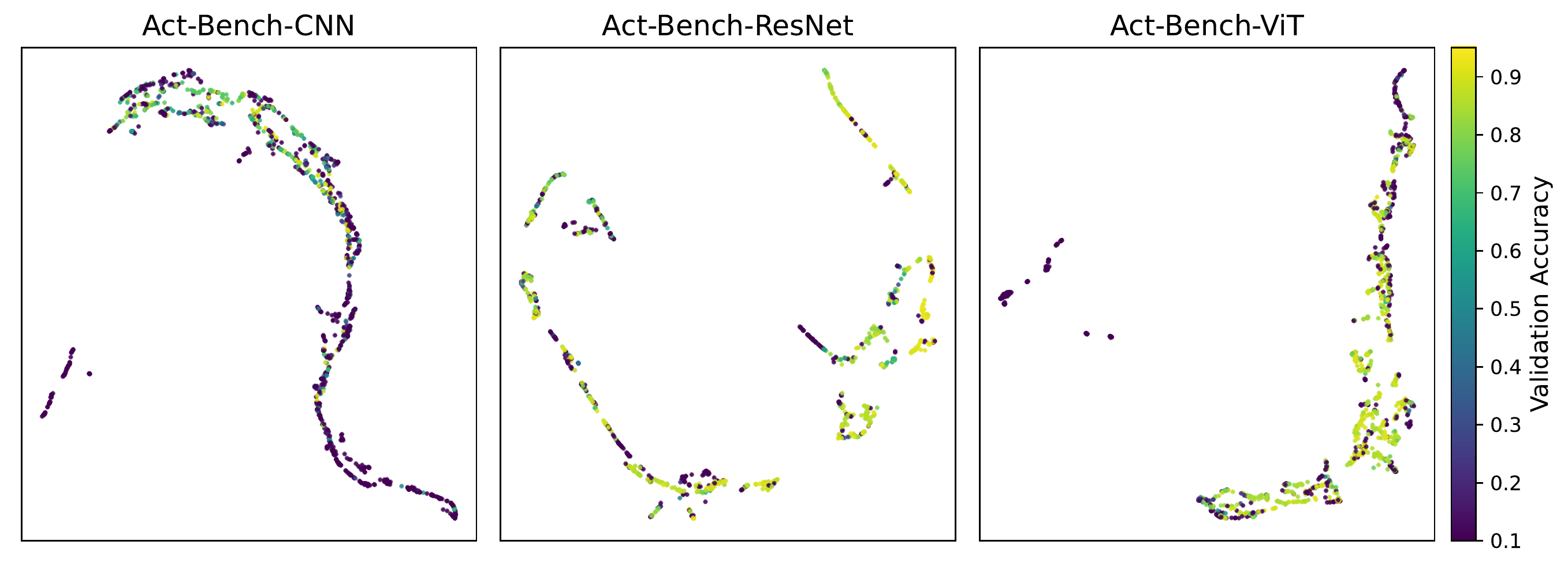}
    \caption{UMAP projections of FIM eigenvalues using the default hyperparameter of \texttt{n\_neighbors=15}.  The embedding is informative but also noisy. 
 Using \texttt{n\_neighbors=3}, as shown in the main text, improved performance.}
    \label{fig:aquasurf:fim_default_n_neighbors}
\end{figure}

\paragraph{Improving the Combined UMAP Projection}
Figure \ref{fig:aquasurf:fim_default_n_neighbors} displays a projection of FIM eigenvalues using default UMAP hyperparameters.  The plots show the eigenvalues organized in multiple distinct one-dimensional manifolds.  Again, FIM eigenvalues are noisy features; there are some clusters of activation functions achieving similar performance, but there are also regions where performance varies widely.  As mentioned in the main text, this issue was addressed by reducing the UMAP parameter \texttt{n\_neighbors} to three.  This change reduced the connectivity of the low-dimensional FIM eigenvalue representation, resulting in a space with many distinct clusters (as seen in Figure \ref{fig:aquasurf:umap_embeddings}).  

On its own, this setting did not improve the search on the benchmark datasets.  However, it did improve performance when the FIM eigenvalues were combined with activation function outputs (as was discussed in Section \ref{sec:aquasurf:visualizing_umap}).  The reason is that the UMAP model for the activation function outputs did not decrease \texttt{n\_neighbors}, and so the combined UMAP model relied more on the activation function outputs than it did on the FIM eigenvalues.  As Figure \ref{fig:aquasurf:umap_embeddings} shows, the activation function outputs are reliable but sometimes project good activation functions to distinct regions in the search space.  Introducing extra connectivity into the fuzzy topological representation via the FIM eigenvalues was sufficient to address this issue, bringing good activation functions to common regions of the space.

\paragraph{Increasing the Dimension of the UMAP Projections}
The UMAP plots show two-dimensional projections of FIM eigenvalues and activation function outputs.  Regression algorithms were also trained on five and 10-dimensional projections.  These runs resulted in comparable or worse performance. Therefore, the two-dimensional projections were selected for simplicity and for consistency between the algorithm implementation and figure visualizations.

\paragraph{Gaussian Process Regression}
As an alternative search method, Gaussian process regression (GPR) was evaluated in activation function search. Several different acquisition mechanisms were used, including expected improvement, probability of improvement, maximum predicted value, and upper confidence bound.  The approach worked well, but the results were inconsistent across the different acquisition mechanisms.  GPR was also more expensive to run compared to the algorithms in the main text (KNR, RFR, SVR), and so those algorithms were used instead for simplicity and efficiency.

\paragraph{Adjusting $k$ in KNR}
The initial experiments with the KNR algorithm used $k=3$.  Experimenting with $k=\{1,5,8\}$ did not reliably improve performance, so $k=3$ was kept.

\paragraph{Uniformly Spaced Inputs for Activation Function Outputs}
In an alternative implementation, equally spaced inputs from $-5$ to $5$ were given to the activation functions instead of normally distributed inputs. This variation did not noticeably change the quality of the embeddings nor the performance of the search algorithms. Therefore, normal inputs were used for consistency with Equation \ref{eq:aquasurf:dist_fn_outputs}.  Figure \ref{fig:aquasurf:interpolation} is the only exception; it used 80 inputs equally spaced from $-5$ to $5$ and increased the UMAP parameter \texttt{min\_dist} to 0.5.  These settings improved the quality of the reconstructed activation functions in the plot.

\subsection{Activation Function Search Spaces}
\label{ap:details:aquasurf_search_space}

The activation functions in this chapter were implemented as computation graphs from the PANGAEA search space \cite{bingham2022discovering}. The space includes unary and binary operators, in addition to existing activation functions \cite{nair2010rectified, elu, selu, DBLP:conf/iclr/RamachandranZL18, elfwing2018sigmoid}. This approach allows specifying families of functions in a compact manner. It is thus possible to focus the search on a space where good functions are likely to be located, and also to search it comprehensively.

\paragraph{Benchmark Datasets}
The benchmark datasets introduced in Section \ref{sec:aquasurf:activation_function_datasets} contain every activation function of the three-node form \texttt{binary(unary($x$),unary($x$))} using the operators in Table \ref{tab:aquasurf:search_space}.  The result is 5{,}103 activation functions, of which 2{,}913 are unique. This space is visualized in Figure~\ref{fig:aquasurf:umap_embeddings}.  

\begin{table}
    \centering
    \caption{Activation function search spaces were defined through computation graphs consisting of basic unary and binary operators as well as existing activation functions \cite{bingham2022discovering}.\\}
    \begin{adjustbox}{max width=\linewidth}
    \begin{tabular}{llll} \toprule 
        \multicolumn{3}{c}{\textbf{Unary}} & \multicolumn{1}{c}{\textbf{Binary}} \\ \midrule
        $0$         & $\textrm{erf}(x)$     & $\textrm{ReLU}(x)$        & $x_1 + x_2$        \\
        $1$         & $\textrm{erfc}(x)$    & $\textrm{ELU}(x)$         & $x_1 - x_2$        \\
        $x$         & $\textrm{sinh}(x)$    & $\textrm{SELU}(x)$        & $x_1 \cdot x_2$    \\
        $-x$        & $\textrm{cosh}(x)$    & $\textrm{Swish}(x)$       & $x_1 / x_2$        \\
        $|x|$       & $\textrm{tanh}(x)$    & $\textrm{Softplus}(x)$    & $x_1^{x_2}$        \\
        $x^{-1}$    & $\textrm{arcsinh}(x)$ & $\textrm{Softsign}(x)$    & $\max\{x_1, x_2\}$ \\
        $x^2$       & $\textrm{arctan}(x)$  & $\textrm{HardSigmoid}(x)$ & $\min\{x_1, x_2\}$ \\
        $e^x$       & $e^x-1$               & $\textrm{bessel\_i0e}(x)$ &                    \\
        $\sigma(x)$ & $\log(\sigma(x))$     & $\textrm{bessel\_i1e}(x)$ &                    \\
        \bottomrule
    \end{tabular}
    \end{adjustbox}
    \label{tab:aquasurf:search_space}
\end{table}

For \texttt{Act-Bench-CNN} and \texttt{Act-Bench-ResNet}, the accuracies are the median from three runs.  For \texttt{Act-Bench-ViT}, the results are from single runs due to computational costs.

\paragraph{New Tasks}
The experiments in Section \ref{sec:aquasurf:search_new_tasks} utilized a larger search space.  Specifically, it was based on the following four-node computation graphs: 
\texttt{binary(unary(unary($x$)),unary($x$))},
\texttt{binary(unary($x$),unary(unary($x$)))},
\texttt{n-ary(unary($x$),unary($x$),unary($x$))},\\
\texttt{unary(binary(unary($x$),unary($x$)))}, and
\texttt{unary(unary(unary(unary($x$))))}.
The unary and binary nodes used the operators in Table \ref{tab:aquasurf:search_space}, and the $n$-ary node used the sum, product, maximum, and minimum operators.  Together, these computation graphs create a search space with 1{,}023{,}516 functions, of which 425{,}896 are unique. This space is visualized in Figure~\ref{fig:aquasurf:large_search_space}.

\chapter{Compute Infrastructure and Cost}
\label{ap:infrastructure}

This appendix describes the infrastructure used to run the experiments in the main text, and the computational cost incurred from running them.

\section{Compute Infrastructure}

Four main systems were used to run experiments for this dissertation.  

\paragraph{HTCondor}
In the first system, HTCondor \cite{thain2005distributed} was used for scheduling jobs.  Jobs were placed on NVIDIA GeForce GTX 1080 and 1080 Ti GPUs.

\paragraph{Slurm}
The second system used the Slurm workload manager, and also used 1080 and 1080 Ti GPUs.

\paragraph{StudioML}
In the third system, StudioML software \cite{gorunner, StudioML} was used to place jobs on machines with NVIDIA GeForce GTX 1080 Ti and RTX 2080 Ti GPUs.

\paragraph{AWS}
The fourth system was an AWS \texttt{g5.48xlarge} instance with eight NVIDIA A10G GPUs.  The instance ran in Oregon (\texttt{us-west-2}) and was powered by renewable energy, so the experiments with this system contributed no carbon emissions.

\section{Cost}

The overall cost of each system depends on the hardware used.  With this detail in mind, the costs for the four systems in this dissertation in terms of GPU hours are estimated below.

\subsection{CAFE}
Ten generations of evolution took approximately 2,000 GPU hours using the HTCondor system.

\subsection{PANGAEA}
PANGAEA experiments used both the HTCondor and Slurm systems.  When a job began executing, a parent activation function was selected by sampling $S=16$ functions from the $P=64$ most recently evaluated activation functions.  This is a minor difference from the original regularized evolution \citep{real2019regularized}, which is based on a strict sliding window of size $P$.  This approach may haven given extra influence to some activation functions, depending on how quickly or slowly jobs were executed in each of the clusters.  In practice the method was highly effective; it allowed evolution to progress quickly by taking advantage of extra compute when demand on the clusters was low. 

It is difficult to know ahead of time how computationally expensive the evolutionary search will be.  Some activation functions immediately resulted in an undefined loss, causing training to end.  In that case only a few seconds had been spent and another activation function could immediately be evaluated.  Other activation functions trained successfully, but their complicated expressions resulted in longer-than-usual training times.  In these experiments, evolution for WRN-10-4 took 2,314 GPU hours, evolution for ResNet-v1-56 took 1,594 GPU hours, and evolution for ResNet-v2-56 took 2,175 GPU hours.  These numbers do not include costs for reranking and repeated runs in the final experiments.  

\subsection{AutoInit}
AutoInit experiments were run using the StudioML system, while the CoAtNet experiments (Section \ref{sec:autoinit:coatnet}) used the AWS system. Training CoAtNet on Imagenette required an average of 0.91 GPU hours per run.  Training CoAtNet on ImageNet once took 119.89 GPU hours.  Training on Imagenette instead of ImageNet therefore required $119.89 / 0.91 \approx 132$ times less compute.

\subsection{AQuaSurF}
AQuaSurF was implemented using the AWS system.  The total compute cost for the search experiments in Section \ref{sec:aquasurf:search_new_tasks} was 14.49 GPU-hours for All-CNN-C on CIFAR-100, 21.67 GPU-hours for ResNet-56 on CIFAR-100, and 196.25 GPU-days for MobileViTv2-0.5 on ImageNet.  This cost includes the time to train the eight baseline activation functions and then to evaluate 100 additional functions.


\phantomsection
\addcontentsline{toc}{chapter}{Bibliography}
\bibliographystyle{abbrvnat}  
\bibliography{references}        
\index{Bibliography@\emph{Bibliography}}



\begin{vita}
Garrett Joseph Bingham was born in Provo, Utah in 1996.  He received the Bachelor of Science degree in Computer Science \& Mathematics from Yale University in May 2019 and enrolled at the University of Texas at Austin in August 2019.

\end{vita}

\end{document}